\documentclass[lettersize,onecolumn]{IEEEtran}
\usepackage{amsmath,amsfonts}
\usepackage{algorithmic}
\usepackage{algorithm}
\usepackage{array}
\usepackage[caption=false,font=normalsize,labelfont=sf,textfont=sf]{subfig}
\usepackage{textcomp}
\usepackage{stfloats}
\usepackage{url}
\usepackage{verbatim}
\usepackage{graphicx}
\usepackage{cite}
\usepackage{xcolor}
\usepackage{fancyhdr}
\usepackage{float} 
\usepackage{comment}
\usepackage{color}
\usepackage{amsthm}
\usepackage{bm}
\usepackage{amssymb,bbm}
\usepackage[colorlinks=True,
linkcolor=blue,
citecolor={navyBlue},
urlcolor=blue,
pagebackref=true,backref=page]{hyperref}
\usepackage{enumitem}
\usepackage[font=scriptsize]{caption}
\usepackage{mathtools}
\usepackage{mdframed} 
\usepackage{thmtools}
\usepackage{dsfont}
\newcommand{\sdfrac}[2]{\text{\scriptsize$\tfrac{#1}{#2}$}}
\newcommand{\assign}{:=}
\newcommand{\mathd}{\mathrm{d}}
\newcommand{\mathe}{\mathrm{e}}

\newcommand{\tmop}[1]{\ensuremath{\operatorname{#1}}}

\newcommand{\tmmathbf}[1]{\ensuremath{\boldsymbol{#1}}}

\newcommand{\tmtextbf}[1]{\text{{\bfseries{#1}}}}
\newenvironment{itemizedot}{\begin{itemize} }{\end{itemize}}
\newcommand{\ind}{\mathrel{\perp\!\!\!\perp}} 
\newcommand{\infixor}{\text{ or }}
\newcommand{\E}{\mathbb{E}}
\renewcommand{\P}{\mathbb{P}}

\newcommand{\sgn}{\textnormal{sgn}}

\definecolor{Gred}{RGB}{219, 50, 54}
\definecolor{Ggreen}{RGB}{60, 186, 84}
\definecolor{Gblue}{RGB}{72, 133, 237}
\definecolor{Gyellow}{RGB}{247, 178, 16}
\definecolor{ToCgreen}{RGB}{0, 128, 0}
\definecolor{myGold}{RGB}{231,141,20}
\definecolor{myBlue}{rgb}{0.19,0.41,.65}
\definecolor{myPurple}{RGB}{175,0,124}
\definecolor{orange}{HTML}{ff7f0e}
\definecolor{blue}{HTML}{1f77b4}

\definecolor{blueGrotto}{HTML}{059DC0}
\definecolor{royalBlue}{HTML}{057DCD}
\definecolor{navyBlue}{HTML}{0B579C}
\definecolor{framebg}{RGB}{241,244,247}     
\definecolor{niceRed}{RGB}{190,38,38}

\definecolor{shadecolor}{gray}{0.90}

\declaretheoremstyle[
  bodyfont=\normalfont  
]{nonitalic}
\newmdenv[
    backgroundcolor=framebg,
    roundcorner=3pt,
    skipabove=2pt,
    linewidth=0pt,
    innertopmargin=2pt
]{myframe}

\declaretheorem[within=section]{definition}
\declaretheorem{nosection} 

\declaretheorem[sibling=nosection, name=Theorem]{theoremtxt}
\declaretheorem[sibling=definition, name=Theorem]{theoremapx}
\newenvironment{theorem}
  {\begin{myframe}\begin{theoremapx}}
  {\end{theoremapx}\end{myframe}}

\declaretheorem[sibling=nosection, name=Proposition]{propositiontxt}
\declaretheorem[sibling=definition, name=Proposition]{propositionapx}

\declaretheorem[sibling=nosection, name=Lemma]{lemmatxt}
\declaretheorem[sibling=definition, name=Lemma]{lemmaapx}
\newenvironment{lemma}
    {\begin{myframe}\begin{lemmaapx}}
    {\end{lemmaapx}\end{myframe}}

\declaretheorem[sibling=nosection, name=Corollary]{corollarytxt}
\declaretheorem[sibling=definition, name=Corollary]{corollaryapx}

\declaretheorem[style=nonitalic,sibling=style, numbered=no]{remark}
\usepackage{orcidlink}

\begin{document}

\title{Structural Properties, Cycloid Trajectories and Non-Asymptotic Guarantees of EM Algorithm for Mixed Linear Regression}

\author{Zhankun Luo\hspace{0.5mm}\raisebox{0.8ex}{\orcidlink{0000-0002-3626-1988}},~\IEEEmembership{Graduate Student Member, IEEE},
 and Abolfazl Hashemi\hspace{0.5mm}\raisebox{0.8ex}{\orcidlink{0000-0002-8421-4270}},~\IEEEmembership{Member, IEEE}
\thanks{Manuscript submitted to {\it IEEE Transactions on Information Theory} on Nov 6, 2025. The authors are with the School of Electrical and Computer Engineering, Purdue University, IN 47907, USA. (E-mail: \{luo333, abolfazl\}@purdue.edu.) 
An early version of this work have been presented in part at the 2024 International Conference on Machine Learning (ICML), July 2024, Vienna, Austria~\cite{luo24cycloid}. \textit{(Corresponding author: Zhankun Luo.)}
}
}

\markboth{IEEE Transactions on Information Theory}%
{Shell \MakeLowercase{\textit{et al.}}: Structural Properties, Cycloid Trajectories and Non-Asymptotic Guarantees of EM Algorithm for Mixed Linear Regression}

\IEEEpubid{0000--0000~\copyright~2025 IEEE}

\maketitle

\begin{abstract}
This work investigates the structural properties, cycloid trajectories, and non-asymptotic convergence guarantees of the Expectation-Maximization (EM) algorithm for two-component Mixed Linear Regression (2MLR) with unknown mixing weights and regression parameters.
Recent studies have established global convergence for 2MLR with known balanced weights and super-linear convergence in noiseless and high signal-to-noise ratio (SNR) regimes.
However, the theoretical behavior of EM in the fully unknown setting remains unclear, with its trajectory and convergence order not yet fully characterized.
We derive explicit EM update expressions for 2MLR with unknown mixing weights and regression parameters across all SNR regimes and analyze their structural properties and cycloid trajectories.
In the noiseless case, we prove that the trajectory of the regression parameters in EM iterations traces a cycloid by establishing a recurrence relation for the sub-optimality angle, while in high SNR regimes we quantify its discrepancy from the cycloid trajectory.
The trajectory-based analysis reveals the order of convergence: linear when the EM estimate is nearly orthogonal to the ground truth, and quadratic when the angle between the estimate and ground truth is small at the population level.
Our analysis establishes non-asymptotic guarantees by sharpening bounds on statistical errors between finite-sample and population EM updates, relating EM's statistical accuracy to the sub-optimality angle, and proving convergence with arbitrary initialization at the finite-sample level.
This work provides a novel trajectory-based framework for analyzing EM in Mixed Linear Regression.
\end{abstract}


\section{Introduction}\label{sec:intro}
Mixture models of parameterized distributions, including Mixture of Linear Regression (MLR) and Gaussian Mixture Models (GMM), provide a flexible framework for modeling complex dependencies in practice.
Its flexibility makes it well-suited for modeling datasets affected by corrupted observations, incomplete data, and latent variable structures~\cite{beale1975missing}.
In this work, we consider the symmetric two-component mixed linear regression (2MLR), defined as
\begin{equation}\label{eq:model}
    y = (-1)^{z+1} \langle x, \theta^\ast\rangle + \varepsilon,
\end{equation}
where \(s := (x, y) \in \mathbb{R}^d \times \mathbb{R}\) denotes the observation, consisting of the covariate \(x \in \mathbb{R}^d\) and the response \(y \in \mathbb{R}\).  
The latent variable \(z \in \{1, 2\}\), indicating the component label, follows the categorical distribution \(\mathcal{C}\mathcal{A}\mathcal{T}(\pi^\ast)\) with probabilities \(\mathbb{P}[z=1] = \pi^\ast(1)\) and \(\mathbb{P}[z=2] = \pi^\ast(2)\).  
The additive noise \(\varepsilon\) is Gaussian with zero mean and variance \(\sigma^2\).  
The regression parameters are \(\theta^\ast \in \mathbb{R}^d\), and the mixing weights are \(\pi^\ast = (\pi^\ast(1), \pi^\ast(2))\), satisfying \(\pi^\ast(1), \pi^\ast(2) \geq 0\) and \(\pi^\ast(1) + \pi^\ast(2) = 1\).
The ground-truth mixing weights \(\pi^\ast = (\pi^\ast(1), \pi^\ast(2))\) and regression parameters \(\theta^\ast\) of the data-generating process are unknown.  
The signal-to-noise ratio (SNR) is defined as \(\eta := \|\theta^\ast\|/\sigma\).  
We observe i.i.d. samples \(\mathcal{S} = \{s_i\}_{i=1}^n = \{(x_i, y_i)\}_{i=1}^n\) drawn from this data-generating process.

Maximum Likelihood Estimation (MLE) provides a principled procedure for parameter inference of the above model. 
It seeks estimates \((\theta, \pi)\) of the ground-truth parameters \((\theta^\ast, \pi^\ast)\) by maximizing the likelihood.
Given estimated parameters \((\theta, \pi)\), where \(\theta \in \mathbb{R}^d\) and \(\pi = (\pi(1), \pi(2))\) with \(\pi(1), \pi(2) \geq 0\) and \(\pi(1) + \pi(2) = 1\), the population log-likelihood is defined as
\(
    \mathbb{E}_{s \sim p(s \mid \theta^\ast, \pi^\ast)} [\log p(s \mid \theta, \pi)],
\)
while in practice we maximize the empirical log-likelihood 
\(
    \frac{1}{n} \sum_{i=1}^n \log p(s_i \mid \theta, \pi)
\)
based on the observed samples \(\mathcal{S} = \{s_i\}_{i=1}^n\).

However, computing the Maximum Likelihood Estimate (MLE) for high-dimensional data is intractable due to its non-convexity and numerous spurious local maxima.
A PCA-based technique~\cite{tipping1999mixtures} links the underlying geometric structure and probabilistic perspective with Gaussian covariates and noise.
A meta-learning strategy~\cite{kong20nips, kong20icml} was applied to estimate the parameters of MLR from small batches.
The method in~\cite{shen2019nips} introduced an iterative adaptation of least trimmed squares to cope with corrupted MLR data.
Another alternative is the moment-based approach~\cite{li2018learning}, which is combined with gradient descent.
Moreover, the Expectation Maximization (EM) method is particularly advantageous due to its computational efficiency and practical implementation simplicity. 
For the model in \eqref{eq:model}, EM estimates the regression parameters and mixing weights from observations through a two-step iterative process: 
The E-step evaluates the expected log-likelihood under the current parameter estimates, whereas the M-step updates the parameters by maximizing this expectation.
This iterative process maximizes the lower bound on the MLE objective until convergence.

A foundational work~\cite{dempster1977maximum} presented the modern EM algorithm and demonstrated its likelihood to be monotonically increasing with EM updates.
The theoretical analysis in~\cite{wu1983convergence} established the global convergence of a unimodal likelihood under some regularity conditions.
Empirical evidence~\cite{de1989mixtures,jordan1994hierarchical,jordan1995convergence} demonstrated EM success in the MLR problem. 
A comprehensive framework of EM for MLR involving unknown latent variables~\cite{wedel1995mixture} was also introduced.

\subsection{Motivation}
As motivating applications, we emphasize the tasks of haplotype assembly in bioinformatics and genomics~\cite{cai2016structured}
and phase retrieval, which emerges in numerous fields such as optics, acoustics, and quantum information~\cite{candes2015phase}, demonstrating the practical implications of our study.

\noindent\textbf{Haplotype Assembly.}
Haplotypes comprise sequences of chromosomal variations in genome that are pivotal for determining the human disease susceptibility. 
Haplotype assembly requires piecing together these sequences from a mixture of sequenced chromosome fragments. 
Particularly, humans carry two haplotypes, i.e., they are diploid organisms (see \cite{cai2016structured} for a clear mathematical formulation of the problem). 
For diploids, the main challenge is to reassemble two different haplotypes (binary sequences of single nucleotide polymorphisms (SNPs)) from corruptedshort sequencing reads. 
Each read relates to a nearby window of the genome but comes from one of the pair of chromosomes. 
The uncertainty in the haplotype origin of each read can be characterized as a symmetric two-component mixed linear regression (2MLR), which is the model discussed in our work. 
As explained in \cite{cai2016structured}, let $\theta^\ast\in \{-1, +1\}^d$ denote one haplotype, and the other haplotype is its opposite $-\theta^\ast$. 
The sequencing data consists of $n$ corrupted reads $\{(x_i, y_i)\}_{i=1}^n$, where $x_i \in \{0,1\}^d$ is a sparse binary vector representing these SNPs covered by the read, and $y_i \in \mathbb{R}$ is the measured read signal (e.g., aggregate allele value), represented as
\begin{equation}
y_i = (-1)^{z_i+1} \langle x_i, \theta^\ast\rangle + \varepsilon_i,
\end{equation}
where $z_i \in \{1, 2\}$ represents a latent variable showing whether the read comes from haplotype $\theta^\ast$ or $-\theta^\ast$, and $\varepsilon_i$ stands for the additive noise. 
The research in~\cite{sankararaman2020comhapdet} shows the effectiveness of EM-type methods for this problem,
as the authors note that ``this iterative update rule is reminiscent of the class of Expectation Maximization algorithms.''

\noindent\textbf{Phase Retrieval.}
In the context of phase retrieval, as highlighted in Section 3 of~\cite{dana2019estimate2mix} and elaborated in Section 3.5 of~\cite{chen2013convex}, 
a well-recognized connection exists between the phase retrieval problem and the symmetric 2MLR. 
Notably, squaring the response variable $y$ and subtracting the variance $\sigma^2$ yields:
\begin{equation}
\begin{aligned}
y' := y^2 - \sigma^2 = \vert \langle x, \theta^\ast\rangle\vert^2 + \xi.
\end{aligned}
\end{equation}
Essentially, this reduces to the phase retrieval model with a heteroskedastic, zero-mean error $\xi:=2(-1)^{z+1}\langle x, \theta^\ast\rangle + (\varepsilon^2-\sigma^2)$. 
Thus, convergence guarantees for the phase retrieval problem follow directly from the analysis of symmetric 2MLR.

\subsection{Problem Setup and EM Updates}
To develop a theoretical understanding of EM behavior for the 2MLR model,
we adopt the standard assumption on the covariates \(x\sim\mathcal{N}(0, I_d)\) 
(see also page 1 of~\cite{dana2019estimate2mix}, page 6 of~\cite{balakrishnan2017statistical}, page 4 of~\cite{reisizadeh2024mixture}; and page 3 of~\cite{luo24cycloid}).
The special case of 2MLR with known mixing weights has been intensively studied in a line of prior works~\cite{balakrishnan2017statistical, dana2019estimate2mix, kwon2019global, kwon2021minimax}.
The population EM update rule for 2MLR with known balanced mixing weights \(\pi = \pi^\ast =  \left( \frac{1}{2}, \frac{1}{2} \right)\) is given by~\cite{balakrishnan2017statistical} as follows:
\begin{equation}
  M (\theta) =\mathbb{E}_{s\sim p(s\mid\theta^\ast, \pi^\ast)} 
   \tanh \left(
  \frac{y\langle x, \theta \rangle}{\sigma^2}\right)
  y  x.
\end{equation}

We extend our analysis from the special case of known balanced mixing weights to the general case of unknown mixing weights
by introducing an auxiliary variable \(\nu\) to characterize the imbalance of the estimated mixing weights, i.e., \(\tanh \nu = \pi(1)-\pi(2)\).
The corresponding ground truth value \(\nu^\ast\) satisfies the analogous relation \(\tanh \nu^\ast = \pi^\ast(1)-\pi^\ast(2)\).
Formally, these \(\nu\) and \(\nu^\ast\) are given by
\begin{equation}
    \nu  := \frac{\ln \pi(1) - \ln \pi(2)}{2}, \quad \nu^\ast := \frac{\ln \pi^\ast(1) - \ln \pi^\ast(2)}{2}.
\end{equation}
With the help of \(\nu\), the population EM update rule for regression parameters \(\theta\) with unknown mixing weights \(\pi\) is extended as follows:
\begin{equation}\label{eq:theta}
  M(\theta, \nu) \assign \mathbb{E}_{s\sim p(s\mid\theta^\ast, \pi^\ast)} 
  \tanh \left(
  \frac{y \langle x, \theta \rangle}{\sigma^2} +\nu\right)
  y  x.
\end{equation}
In the absence of the assumption that the covariates satisfy \(x\sim\mathcal{N}(0, I_d)\), so that \(\E[xx^\top] = I_d\), 
the population EM update rule becomes \(M(\theta, \nu) = \E[xx^\top]^{-1}\mathbb{E}_{s\sim p(s\mid\theta^\ast, \pi^\ast)} \tanh \left(\frac{y \langle x, \theta \rangle}{\sigma^2} +\nu\right) yx\).
Furthermore, the corresponding population EM update rule for the imbalance of mixing weights \(\tanh\nu\) is given by:
\begin{equation}\label{eq:nu}
  N(\theta, \nu) \assign \mathbb{E}_{s\sim p(s\mid\theta^\ast, \pi^\ast)} 
  \tanh \left(
  \frac{y \langle x, \theta \rangle}{\sigma^2} +\nu\right).
\end{equation}
The finite-sample EM update rules with \(n\) samples \(\mathcal{S} = \{(x_i, y_i)\}_{i=1}^n\) are:
\begin{equation}\label{eq:finite}
\begin{aligned}
M_n  (\theta, \nu) &:= \left( \frac{1}{n}  \sum_{i = 1}^n x_i x_i^{\top}
\right)^{-1}
\left( \frac{1}{n}  \sum_{i = 1}^n \tanh \left(
\frac{y_i\langle x_i, \theta\rangle}{\sigma^2} +\nu \right) y_i x_i \right),\\
N_n  (\theta, \nu) &:= \frac{1}{n}  \sum_{i = 1}^n \tanh \left(
\frac{y_i\langle x_i, \theta\rangle}{\sigma^2} +\nu\right).
\end{aligned}
\end{equation}
To facilitate theoretical analysis, we adopt the Easy EM method with the update rule of regression parameters as follows:
\begin{equation}\label{eq:easyEM}
    M^{\tmop{easy}}_n  (\theta, \nu) :=  \frac{1}{n}  \sum_{i = 1}^n \tanh \left(
\frac{y_i\langle x_i, \theta\rangle}{\sigma^2} +\nu \right) y_i x_i. 
\end{equation}
Derivations of both the population and finite-sample EM update rules are provided in Appendix~\ref{sup:derive_em},
while similar derivations of EM updates for 2GMM can be found on pages 5-6 of~\cite{weinberger2022algorithm}.

\subsection{Related Works}
Both GMM and MLR can be interpreted as instances of subspace clustering, which naturally leads to similarities in their EM analysis.
For GMM, well-separated spherical Gaussians with parameterization can be learned to nearly optimal accuracy using a variant of EM~\cite{dasgupta2007}.
Moreover, the linear global convergence of EM under well-separated spherical Gaussians with initialization inside a neighborhood of the truth has been demonstrated~\cite{zhao2020statistical}.
For GMM with $k \geq 3$ components, EM with random initialization is often trapped in local minima with high probability, and the likelihood at local maxima can be arbitrarily worse than that of any global maximum~\cite{jin2016local}.
The possible types of local minima for EM and $k$-means (i.e., EM with hard labels) in GMM under a separation condition have been identified~\cite{chen2020likelihood, qian2022spurious}.
Building on these characterized structures of local minima, a general framework was introduced to escape them~\cite{zhang2020symmetry, hong2022geometric}, thereby unifying variants of $k$-means from a geometric standpoint.
Finally, the relationship between EM and the method of moments for GMM, established through an asymptotic expansion of the log-likelihood in the low signal-to-noise ratio (SNR) regime, has been uncovered~\cite{katsevich2023likelihood}.
The particular case of GMM with $k=2$ components (2GMM) has been extensively investigated.
Global convergence of EM with random initialization for spherical 2GMM was demonstrated in~\cite{klusowski2016statistical, xu2016global, daskalakis17b}.
This convergence property was analyzed across all SNR regimes in~\cite{wu2021randomly}, while restricting the initialization to a very small neighborhood.
The convergence study was additionally generalized from spherical Gaussians to rotation-invariant log-concave distributions in~\cite{qian2019global,qian2020local}.
The critical SNR threshold for exact recovery of 2GMM was explored in~\cite{ndaoud2018sharp}.

Similarly, EM for MLR with two components (2MLR) and random initialization achieves global convergence.
The global convergence of EM for 2MLR with proper initialization within a ball around the true parameters was first established in~\cite{balakrishnan2017statistical}.
This convergence guarantee was extended to the high SNR setting for the case where the cosine angle between the initial parameters and the ground truth is sufficiently large~\cite{dana2019estimate2mix}.
It was additionally verified in~\cite{kwon2022dissertation, kwon2019global} that EM for 2MLR converges from random initialization with high probability.
Bounds on the statistical error of EM for 2MLR across varying SNR regimes were derived in~\cite{yudong2018trans}.
Generalization error bounds for the log-likelihood of the first-order EM for 2MLR were presented in~\cite{xu2020towards}.
The statistical error and convergence rate of EM for 2MLR under all SNR regimes were further investigated in~\cite{kwon2021minimax}.
Alternating Minimization (AM), a hard-label variant of EM, for 2MLR in the noiseless case was studied in~\cite{yi2014alternating, yi2016solving}.
Accordingly, a super-linear convergence rate of AM for 2MLR in the noiseless setting within a specific convergence region was shown in~\cite{ghosh20a}.
The noiseless scenario was generalized to the high SNR regime while maintaining super-linear convergence in~\cite{kwon2021minimax}.
Finally, a convergence analysis of EM for MLR with multiple components, covering the most general settings, was given in~\cite{kwon2020converges}.
Prior works on the convergence analysis of EM for 2MLR have neglected unbalanced mixing weights and focused on the balanced scenario.
The works in~\cite{dwivedi2020sharp,dwivedi2020unbalanced,luo2025characterizing} incorporated unbalanced mixing weights of 2GMM/2MLR and demonstrated a pronounced difference in statistical error and convergence rate between the unbalanced and balanced settings, for the specific case without parameter separation, 
namely when the location/regression parameters are degenerate, \(\theta^\ast=\vec{0}\), in 2GMM/2MLR.

Therefore, it remains an open question to investigate the convergence of EM for 2MLR with unknown mixing weights and regression parameters.
Our work addresses this gap by providing a comprehensive analysis of the properties of EM updates, convergence guarantees, and statistical errors of EM iterations for 2MLR with unknown mixing weights and regression parameters.


\subsection{Main Contributions}
In this work, we provide explicit expressions for the log-likelihood function (Proposition~\ref{prop:nll}) of the symmetric two-component Mixed Linear Regression (2MLR) model with unknown mixing weights, 
and subsequently derive the EM update rules (Equations~\eqref{eq:theta},~\eqref{eq:nu},~\eqref{eq:finite}). We establish the connection between these EM updates 
and the gradient-descent dynamics of the log-likelihood (Proposition~\ref{prop:em_update_nll_gradients}), 
and obtain closed-form expressions of the EM update rules with unknown mixing weights and regression parameters (Theorem~\ref{theorem:explicit_em_update}).

We characterize the structural properties of the EM updates 
by bounding the length of EM iterations across all signal-to-noise ratio (SNR) regimes (Proposition~\ref{prop:boundedness}), 
identifying their fixed points and proving contraction around them (Propositions~\ref{prop:distinct_fixed_points},~\ref{prop:contraction_property}). 
In the noiseless setting, we show that the EM iterations follow a cycloid trajectory (Proposition~\ref{prop:parametric_cycloid}), 
derive this from the explicit noiseless updates (Corollary~\ref{cor:em_updates_noiseless}), 
and quantify the deviation from this trajectory when the SNR is large but finite (Proposition~\ref{prop:deviation_cycloid_limit}).

In the noiseless setting, we establish a recurrence relation for the sub-optimality angle (Proposition~\ref{prop:recurrence_angle}), 
prove linear growth when the EM estimate is nearly orthogonal to the ground truth (Proposition~\ref{prop:linear_growth_angle}), 
and show quadratic convergence when the angle is small (Proposition~\ref{prop:quadratic_convergence_angle}). 
We relate the accuracy of EM estimate (for regression parameters and mixing weights) to the sub-optimality angle (Proposition~\ref{prop:errors_em_updates_angle}), 
yielding population-level convergence guarantees (Theorem~\ref{theorem:population_level_convergence}).

At the finite-sample level, we couple the population-level analysis with statistical error bounds to establish convergence guarantees for the EM algorithm under arbitrary initialization. 
Specifically, we derive projected and statistical error bounds for the regression parameters (Propositions~\ref{prop:projected_error_regression},~\ref{prop:statistical_error_regression}), which are independent of the mixing weights, 
and a statistical error bound for the mixing weights (Proposition~\ref{prop:statistical_error_mixing_weights}) that depends on the sub-optimality angle and the ground-truth mixing weights. 
We then characterize the connection between the statistical accuracy of the EM updates and the sub-optimality angle (Proposition~\ref{prop:statistical_accuracy_em_updates}) 
and establish convergence guarantees for the sub-optimality angle at the finite-sample level (Proposition~\ref{prop:convergence_angle}) by combining the population-level convergence analysis with the bounds on the statistical errors and statistical accuracy of the EM updates. 
Consequently, by leveraging the connection between the statistical accuracy of the EM updates and the sub-optimality angle, we obtain finite-sample convergence guarantees for the EM algorithm with arbitrary initialization (Proposition~\ref{prop:initialization_easy_em}) of unknown mixing weights and regression parameters (Theorem~\ref{theorem:finite_sample_convergence}).

Overall, our results correspond directly to the three main components of this work:
\begin{itemize}
    \item \textbf{Structural Properties:} We provide a unified characterization of the structural properties of the EM updates for the symmetric 2MLR model with unknown mixing weights and regression parameters across all SNR regimes (Section~\ref{sec:updates});
    \item \textbf{Cycloid Trajectories:} We introduce a trajectory-based framework that reveals the cycloid trajectory of EM iterations and connects their geometric behavior to the accuracy of the EM updates (Section~\ref{sec:population}); and
    \item \textbf{Non-Asymptotic Guarantees:} We establish convergence guarantees for finite-sample EM updates under arbitrary initialization of the  mixing weights and regression parameters in the noiseless setting (Section~\ref{sec:finite}).
\end{itemize}

\subsection{Notations}
\noindent\textbf{Imbalance of Mixing Weights.} 
The absolute value of the imbalance \(\tanh |\nu|\) corresponds exactly to the \(\ell_1\) distance between the mixing weights \(\pi = (\pi(1), \pi(2))\) and the balanced weights \(\frac{\mathds{1}}{2} = \left(\frac{1}{2}, \frac{1}{2}\right)\), where \(\mathds{1} = (1, 1)\) is the all-ones vector.
These quantities are related to the minimum of \(\pi(1)\) and \(\pi(2)\) by:
\begin{equation}
\tanh |\nu| = \left\|\pi - \frac{\mathds{1}}{2}\right\|_1, \quad \frac{1 - \tanh |\nu|}{2} = \min(\pi(1), \pi(2)).
\end{equation}
The ground truth counterparts follow the same relations, namely
\(\tanh |\nu^\ast| = \left\|\pi^\ast - \frac{\mathds{1}}{2}\right\|_1\) and \(\frac{1 - \tanh |\nu^\ast|}{2} = \min(\pi^\ast(1), \pi^\ast(2))\).

\noindent\textbf{Sub-optimality Angles.} 
The \(\ell_2\) norm of a vector is denoted by \(\|\cdot\|\), and \(\langle \cdot, \cdot \rangle\) represents the inner product. 
We define the unit direction vectors for the estimated and ground truth regression parameters as \(\vec{e}_1 = \theta / \|\theta\|\) and \(\hat{e}_1 = \theta^\ast / \|\theta^\ast\|\), respectively, 
The corresponding orthogonal direction vectors are \(\vec{e}_2 = (\theta^\ast - \vec{e}_1 \vec{e}_1^\top \theta) / \|\theta^\ast - \vec{e}_1 \vec{e}_1^\top \theta\|\) and \(\hat{e}_2 = (\theta - \hat{e}_1 \hat{e}_1^\top \theta^\ast) / \|\theta - \hat{e}_1 \hat{e}_1^\top \theta^\ast\|\).
The cosine of the angle between the estimated regression parameters \(\theta\) and the ground truth \(\theta^\ast\) is \(\rho := \langle \theta, \theta^\ast \rangle / \|\theta\| \|\theta^\ast\|\), with its sign denoted by \(\sgn(\rho)\) and \(\sgn(0) = 1\).
From this, we define the sub-optimality angles \(\varphi\) and \(\phi\):
\begin{equation}
\varphi := \frac{\pi}{2} - \arccos |\rho|, \quad \phi := 2 \arccos |\rho|.
\end{equation}
The ratio of the norms of the regression parameters is denoted by \(k := \|\theta\| / \|\theta^\ast\|\).

\noindent\textbf{Auxiliary Definitions.} 
We define \(A_\eta\) and \(\varphi_\eta\) in terms of the SNR \(\eta\), the ratio \(k\), and the sub-optimality angles \(\varphi, \phi\):
\begin{equation}
A_\eta := k \eta^2 \sqrt{1 + \eta^{-2}}, \quad \varphi_\eta := \arcsin \frac{\sin \varphi}{\sqrt{1 + \eta^{-2}}}
\end{equation}
with the convention \(\sqrt{1 + \eta^{-2}} = \sin \varphi / \sin \varphi_\eta\) when \(\varphi = 0\).
To express results compactly, we use Gaussian random variables \(g, g_\eta, g' \sim \mathcal{N}(0,1)\) with the following correlation structure:
\begin{equation}
\mathbb{E}[g g'] = \sin \varphi, \quad \mathbb{E}[g_\eta g'] = \sin \varphi_\eta, \quad \mathbb{E}[g g_\eta] = \frac{1}{\sqrt{1 + \eta^{-2}}} = \frac{\sin \varphi_\eta}{\sin \varphi}.
\end{equation}

The modified Bessel function of the second kind of order 0, denoted \(K_0(\cdot)\), is introduced to aid our analysis.
The probability density function of the product of two independent Gaussian random variables with zero mean and unit variance is exactly \(K_0(|\cdot|)/\pi\) (see page 50, Section 4.4 Bessel Function Distributions, Chapter 12 Continuous Distributions (General) of~\cite{johnson1970continuous}).

We use the asymptotic notations \(\Omega(\cdot), \mathcal{O}(\cdot)\), and \(\Theta(\cdot)\) (see page 528 of~\cite{lehman2017mathematics}) to help our analysis.
The notation \(f = \Omega(g)\), which can also be written as \(f \gtrsim g\), indicates \(g = \mathcal{O}(f)\), equivalently \(g \lesssim f\); that is, there exists a universal constant \(C>0\) such that \(|f(x)| \geq C |g(x)|\) for all sufficiently large \(x\).
The notation \(f = \Theta(g)\), or \(f \asymp g\), means that both \(f = \mathcal{O}(g)\) and \(g = \mathcal{O}(f)\) hold.
The symbols \(a \vee b\) and \(a \wedge b\) denote the maximum and minimum of \(a\) and \(b\), respectively.

\subsection{Organization}

We organize the paper as follows:

In Section~\ref{sec:updates}, we provide the explicit expressions of the negative log-likelihood and EM update rules for the symmetric 2MLR model with unknown mixing weights and regression parameters,
and characterize the structural properties of the EM updates for the symmetric 2MLR model across all SNR regimes. 
The detailed derivations of the negative log-likelihood and EM update rules are given in Appendix~\ref{sup:derive_em},
while the proofs of the structural properties of the EM updates are given in Appendix~\ref{sup:em_update_rules}.

In Section~\ref{sec:population}, We analyze the population EM updates in the noiseless setting, 
showing that the population EM iterations evolve along a cycloid trajectory 
governed by a recurrence relation of the sub-optimality angle, exhibiting linear growth and quadratic convergence of the sub-optimality angle,
and relating the accuracy of mixing weights and regression parameters to the sub-optimality angle.
The proofs of the results in this section are given in Appendix~\ref{sup:population_analysis}.


In Section~\ref{sec:finite}, we establish the non-asymptotic convergence guarantees of the finite-sample EM updates in the noiseless setting by
deriving statistical error bounds for both regression parameters and mixing weights,
characterizing the connection between the statistical accuracy of EM updates and the sub-optimality angle,
and showing convergence guarantees of the sub-optimality angle under arbitrary initialization.
The proofs of the results in this section are given in Appendix~\ref{sup:finite_sample_analysis}.


In Section~\ref{sec:experiments}, we provide numerical experiments to validate the theoretical results.
We first show the cycloid trajectory of the EM iterations empirically in the noiseless setting,
and then validate the theoretical findings on the convergence guarantees for the EM algorithm with arbitrary initialization of unknown mixing weights and regression parameters at finite-sample level.
The code for empirical experiments is available at \url{https://github.com/dassein/cycloid_em_tit}.


\section{Explicit EM Update Expressions and Properties of EM Update Rules}\label{sec:updates}
In this section, we focus on the general structural properties of EM updates for 2MLR across all SNR regimes, 
which consists of three parts:
In the first part, we give the expressions of the negative log-likelihood (Proposition~\ref{prop:nll}) at the population level and the finite-sample level 
by using expectations with respect to the ground truth distribution of the observed data \(s=(x, y)\).
Then we rewrite the EM update rules (Lemma~\ref{lemma:em_update_gradients}) by using expectations with respect to the ground truth distribution of the observation data \(s=(x, y)\).
Finally, we establish the connection between the EM update rules and the gradient descent of the negative log-likelihood function (Proposition~\ref{prop:em_update_nll_gradients}).

In the second part, we characterize the properties of the EM update rules 
by starting from giving the explicit expectation expressions of the EM update rules with respect to only two Gaussian random variables (Theorem~\ref{theorem:explicit_em_update}).
Based on the explicit expressions, we establish bounds for the length of the EM update rules for regression parameters (Proposition~\ref{prop:boundedness}).
Then we give fixed points for the EM update rules (Propositions~\ref{prop:distinct_fixed_points}) and provide an analysis of the contraction of the EM update rules around the fixed points (Proposition~\ref{prop:contraction_property}).

In the third part, we give the EM update rules with only the sub-optimality angle \(\varphi\) in the noiseless setting (Corollary~\ref{cor:em_updates_noiseless}), i.e., when the SNR \(\eta \to \infty\),
which will be shown to follow a cycloid trajectory in the next section (Proposition~\ref{prop:parametric_cycloid} in Section~\ref{sec:population}).
Furthermore, we bound the difference between the EM update rule and its limit of the cycloid trajectory in the finite high SNR regime (Proposition~\ref{prop:deviation_cycloid_limit}).
The proofs of the results in this section are provided in Appendix~\ref{sup:em_update_rules}, and the detailed derivations of EM update rules are provided in Appendix~\ref{sup:derive_em}.

\begin{propositiontxt}[Population and Finite-Sample Negative Log-Likelihood]\label{prop:nll}
    Let \(f(\theta, \pi):=-\E_{s\sim p(s\mid \theta^\ast, \pi^\ast)}[\ln p(s\mid \theta, \pi)]\) be the negative log-likelihood function at the population level,
    and \(f_n(\theta, \pi):=-\frac{1}{n}\sum_{i=1}^n \ln p(s_i\mid \theta, \pi)\) be the negative log-likelihood function at the finite-sample level for the dataset \(\mathcal{S}=\{s_i\}_{i=1}^n=\{(x_i, y_i)\}_{i=1}^n\) of \(n\) i.i.d. samples.
    Then, \(f(\theta, \pi)\) and \(f_n(\theta, \pi)\) can be expressed as:
    \begin{equation}
    \begin{aligned}
        f(\theta, \pi) & =  \frac{1}{2\sigma^2}\langle \theta, \E[x x^\top] \cdot \theta \rangle -\E\left[\ln \frac{\cosh\left(y\langle x, \theta \rangle/\sigma^2+\nu\right)}{\cosh \nu}\right]+\mathtt{C}\\
        f_n(\theta, \pi) & =  \frac{1}{2\sigma^2}\left\langle \theta, \frac{1}{n}\sum_{i=1}^n x_i x_i^\top \cdot \theta \right\rangle -\frac{1}{n}\sum_{i=1}^n \ln \frac{\cosh\left(y_i\langle x_i, \theta \rangle/\sigma^2+\nu\right)}{\cosh \nu}+\mathtt{C}_n
    \end{aligned}
    \end{equation}
    where \(\E[\cdot]=\E_{s\sim p(s\mid \theta^\ast, \pi^\ast)}[\cdot]\) is the expectation over the ground truth distribution \(p(s\mid \theta^\ast, \pi^\ast)\), and \(\mathtt{C}\) and \(\mathtt{C}_n\) are constants that are independent of \(\theta\) and \(\pi\).
\end{propositiontxt}

\begin{remark}
For the 2MLR model with Gaussian noise \(\varepsilon \sim \mathcal{N}(0, \sigma^2)\), 
we have shown that the negative log-likelihood function consists of three terms.
The first term is quadratic in the regression parameters \(\theta\), involving
the covariance matrix of the covariates \(x\) at the population level, 
or the sample covariance matrix at the finite-sample level. 
Under the assumption that the covariates \(x\sim \mathcal{N}(0, I_d)\), we have \(\E[x x^\top] = I_d\) at the population level,
and the first term simplifies to \(\|\theta\|^2/2\sigma^2\).
The third term at the population level, \(\mathtt{C}=\frac{1}{2\sigma^2}\E_{s\sim p(s\mid \theta^\ast, \pi^\ast)}[y^2] + \E_{s\sim p(s\mid \theta^\ast, \pi^\ast)}[\ln p(x)] - \frac{1}{2} \ln(2\pi \sigma^2)\), is a constant term that depends on the ground truth distribution of the observed data \(s=(x, y)\sim p(s\mid \theta^\ast, \pi^\ast)\).
Similarly, the third term at the finite-sample level \(\mathtt{C}_n=\frac{1}{2\sigma^2}\frac{1}{n}\sum_{i=1}^n y_i^2 + \frac{1}{n}\sum_{i=1}^n \ln p(x_i) - \frac{1}{2} \ln(2\pi \sigma^2)\) is a constant that depends on the samples \(\mathcal{S}=\{s_i\}_{i=1}^n=\{(x_i, y_i)\}_{i=1}^n \stackrel{\text{i.i.d.}}{\sim} p(s_i\mid \theta^\ast, \pi^\ast)\) drawn from the ground truth distribution.
Also, \(f\) is the KL divergence between the ground truth distribution \(p(s\mid \theta^\ast, \pi^\ast)\) and the model distribution \(p(s\mid \theta, \pi)\), plus a constant term (see Appendix~\ref{sup:derive_em}).
\end{remark}

\begin{lemmatxt}[EM Update Rules and Gradients]\label{lemma:em_update_gradients}
    Let \(U(\theta, \nu):= \E_{s\sim p(s\mid \theta^\ast, \pi^\ast)}\ln \cosh \left(y\langle x, \theta \rangle/\sigma^2+\nu\right)\) at population level and \(U_n(\theta, \nu):= \frac{1}{n}\sum_{i=1}^n \ln \cosh \left(y_i\langle x_i, \theta \rangle/\sigma^2+\nu\right)\) at finite-sample level,
    then the population EM update rules \(M(\theta, \nu), N(\theta, \nu)\) and finite-sample EM update rules \(M_n(\theta, \nu), N_n(\theta, \nu)\) for regression parameters \(\theta\) and imbalance of mixing weights \(\tanh \nu\) are:
    \begin{equation}
        \begin{aligned}
            M(\theta, \nu) &= \E[x x^\top]^{-1}\sigma^2\nabla_\theta U(\theta, \nu), \quad 
            &N(\theta, \nu) &= \nabla_\nu U(\theta, \nu),\\
            M_n(\theta, \nu) &= \left(\frac{1}{n}\sum_{i=1}^n x_i x_i^\top\right)^{-1}\sigma^2\nabla_\theta U_n(\theta, \nu), \quad 
            &N_n(\theta, \nu) &= \nabla_\nu U_n(\theta, \nu).
        \end{aligned}
    \end{equation}
\end{lemmatxt}

\begin{remark}
The EM algorithm maximizes the evidence lower bound (ELBO) on the observed log-likelihood, 
namely by minimizing the surrogate functions of the negative log-likelihood functions \(f\) and \(f_n\) at the population level and the finite-sample level, respectively 
(see Appendix~\ref{sup:derive_em} for the expressions of the surrogate functions).
The minimizers of the surrogate functions yield the EM update rules for regression parameters \(\theta\) and the imbalance of mixing weights \(\tanh \nu = \pi(1) - \pi(2)\):
\(M(\theta, \nu), N(\theta, \nu)\) and \(M_n(\theta, \nu), N_n(\theta, \nu)\) at the population level and the finite-sample level, respectively (see equations~\eqref{eq:theta},~\eqref{eq:nu},~\eqref{eq:finite} for EM update rules).
By Leibniz's rule, we exchange the order of taking expectation and gradient of \(\ln(\cosh(y\langle x, \theta \rangle/\sigma^2+\nu))\), 
and noting that \(\frac{\mathd}{\mathd t}\ln \cosh(t) = \tanh(t)\), we have established the above relation between EM update rules and the functions \(U(\theta, \nu), U_n(\theta, \nu)\).
\end{remark}

\begin{propositiontxt}[Connection between EM Update Rules and Gradient Descent]\label{prop:em_update_nll_gradients}
    The population/finite-sample EM update rules give:
    \begin{equation}
        \begin{aligned}
            M(\theta, \nu)   &= \theta - \E[x x^\top]^{-1}\sigma^2 \nabla_\theta f(\theta, \pi), 
            \quad & N(\theta, \nu)   &= \tanh \nu - \nabla_\nu f(\theta, \pi),\\
            M_n(\theta, \nu) &= \theta - \left(\frac{1}{n}\sum_{i=1}^n x_i x_i^\top\right)^{-1}\sigma^2 \nabla_\theta f_n(\theta, \pi), 
            \quad & N_n(\theta, \nu) &= \tanh \nu - \nabla_\nu f_n(\theta, \pi).
        \end{aligned}
    \end{equation}
\end{propositiontxt}

\begin{remark}
Note that the negative log-likelihood functions \(f, f_n\) can be expressed in terms of the functions \(U, U_n\) 
as \(f(\theta, \pi) = \frac{1}{2\sigma^2} \langle \theta, \E[x x^\top] \cdot \theta \rangle + \ln \cosh \nu - U(\theta, \nu)\)
and \(f_n(\theta, \pi) = \frac{1}{2\sigma^2} \left\langle \theta, \frac{1}{n}\sum_{i=1}^n x_i x_i^\top \cdot \theta \right\rangle + \ln \cosh \nu - U_n(\theta, \nu)\).
The first term depends only on the regression parameters \(\theta\), the second term depends only on the imbalance of mixing weights \(\tanh \nu\),
and the third term, given by \(-U\) or \(-U_n\), is a function of both \(\theta\) and \(\nu\).
By taking the gradients of the negative log-likelihood functions \(f\) and \(f_n\) with respect to \(\theta\) and \(\nu\), and using the above relation between the EM update rules and the functions \(U, U_n\), 
we establish the connection between the EM update rules and the gradient descent dynamics for the negative log-likelihood functions \(f\) and \(f_n\) (Proposition~\ref{prop:em_update_nll_gradients}).
This relation implies that \((\theta, \tanh \nu)\) is a fixed point of the population EM update rules \(M(\theta, \nu), N(\theta, \nu)\) 
if and only if \((\theta, \tanh \nu)\) is a stationary point of the negative log-likelihood function \(f\), namely \(\nabla_\theta f =\vec{0}, \nabla_\nu f =0\) at this point.
See also the discussion of the fixed points of EM update rules in Proposition~\ref{prop:distinct_fixed_points} and Proposition~\ref{prop:contraction_property}.
Our results extend the connection between EM update rules only for regression parameters and gradient descent of the negative log-likelihood under known balanced mixing weights \(\pi = \pi^\ast = (\frac{1}{2}, \frac{1}{2})\) (see Lemma 7 on page 19 of \cite{kwon2019global}, Lemma 2 on page 6523 of \cite{kwon2024global})
to the more general case with unknown mixing weights for EM updates of regression parameters and mixing weights.
\end{remark}

\begin{theoremtxt}[Explict EM Update Expressions]\label{theorem:explicit_em_update}
    Let \(A_\eta = k \eta^2 \sqrt{1+\eta^{-2}}, k=\frac{\| \theta \|}{\| \theta^{\ast} \|}, \eta=\frac{\| \theta^{\ast} \|}{\sigma}\), 
    and \(\varphi = \frac{\pi}{2} - \arccos | \rho |\in (0, \pi/2]\), \(\rho = \frac{\langle \theta, \theta^\ast \rangle}{\|\theta\|\|\theta^\ast\|}\), 
    and \(\varphi_\eta = \arcsin(\sin \varphi/\sqrt{1+\eta^{-2}})\), namely \(\sqrt{1+\eta^{-2}} = \sin \varphi/\sin \varphi_\eta\),
    and Gaussian random variables \(g_\eta, g'\sim \mathcal{N}(0, 1)\) with \(\E[g_\eta g'] = \sin \varphi_\eta\), then population EM update rules for regression parameters \(\theta\) and imbalance of mixing weights \(\tanh \nu\) are:
    \begin{equation}
        \begin{aligned}
        M (\theta, \nu) /\| \theta^{\ast} \| & =  \vec{e}_1 \frac{\sin \varphi}{\sin \varphi_\eta} \mathbb{E}[ \tanh \left(
        A_\eta g_\eta g' + \tmop{sgn} (\rho) (- 1)^{z + 1} \nu
        \right) g_\eta g']\\
        & +  \vec{e}_2  \frac{\cos \varphi}{\cos^2 \varphi_\eta} 
        \tmop{sgn} (\rho) \mathbb{E} [\tanh \left( A_\eta g_\eta g' + \tmop{sgn} (\rho) (- 1)^{z + 1} \nu \right) g_\eta (g_\eta - \sin\varphi_\eta \cdot g')]\\
        N (\theta, \nu) & =  \tmop{sgn} (\rho) \mathbb{E} \left[ \tanh \left( (-
        1)^{z + 1} A_\eta g_\eta g' + \tmop{sgn} (\rho) \nu \right) \right].
        \end{aligned}
    \end{equation}
\end{theoremtxt}

\begin{remark}
The above explicit expressions of EM update rules still hold when \(\rho=0\), namely when \(\langle \theta, \theta^\ast\rangle =0\), 
under the convention that \(\sgn(\rho) = \sgn(0) =1, \sin \varphi/ \sin \varphi_\eta = \sqrt{1+\eta^{-2}}\).
Moreover, the above explicit expressions of EM update rules show that the EM update for regression parameters \(M(\theta, \nu)\in \text{span}\{\theta, \theta^\ast \}\)
must lie in the space spanned by \(\theta, \theta^\ast\).
The population EM update rules for the normalized regression parameters and the imbalance of mixing weights,\(M(\theta, \nu)/\|\theta^\ast\|\) and \(N(\theta, \nu)\),
are determined by \(\rho=\frac{\langle \theta, \theta^\ast \rangle}{\|\theta\|\|\theta^\ast\|}, k=\frac{\|\theta\|}{\|\theta^\ast\|}, \tanh \nu = \pi(1) - \pi(2)\), and the SNR \(\eta = \| \theta^{\ast} \|/\sigma\),
while the other parameters \(A_\eta = k \eta^2 \sqrt{1+\eta^{-2}}\) and \(\sin \varphi_\eta = \sin \varphi/\sqrt{1+\eta^{-2}}\) are determined by \(k, \eta\), and \(\rho=\sgn(\rho)\sin \varphi\).
The application of Theorem~\ref{theorem:explicit_em_update} enables us to establish the following structural properties of EM update rules on boundedness (Proposition~\ref{prop:boundedness}), fixed points (Proposition~\ref{prop:distinct_fixed_points}), the contraction property (Proposition~\ref{prop:contraction_property}), and the cycloid trajectory (Corollary~\ref{cor:em_updates_noiseless}, Proposition~\ref{prop:deviation_cycloid_limit}).
\end{remark}

\begin{propositiontxt}[Boundedness of EM Update Rule]\label{prop:boundedness}
    The EM update rule for \(\theta\) is bounded by the following bound, which depends on SNR \(\eta = \| \theta^{\ast} \|/\sigma\):
    \begin{equation}
        \left\|M(\theta, \nu)\right\| \leq  \frac{\arctan \eta}{\pi/2} \|\theta^\ast\|+ \frac{2}{\pi} \sigma.
    \end{equation}
\end{propositiontxt}

\begin{remark}
The above bound for the EM update rule characterizes the range of the EM update rule for regression parameters \(\theta\).
When the SNR \(\eta\) is small, the noise term dominates the magnitude of the EM update. In particular, 
when \(\eta \to 0\) and the magnitude of the previous iteration \(\|\theta\|/\sigma \to \infty\), we have \(A_\eta \to \infty\) and \(\varphi_\eta \to 0\),
and therefore \(\|M(\theta, \nu)\| \to \frac{2}{\pi} \sigma\).
When the SNR \(\eta\) is large, the \(\|\theta\|\) term dominates the noise term. Specifically, 
as \(\eta \to \infty, |\rho| \to 1\) and \(\|\theta\| \neq0\), we have \(A_\eta \to \infty\) and \(\varphi_\eta \to \varphi = \frac{\pi}{2}\),
and hence \(\|M(\theta, \nu)\| \to \|\theta^\ast\|\). 
Therefore, our bound on the magnitude of the EM update rule \(\|M(\theta, \nu)\|\) is tight in the limiting SNR regimes, both as \(\eta \to 0\) and \(\eta \to \infty\).
Our result on the boundedness of the EM update rule for regression parameters is consistent with the previous bound \(3\sqrt{\|\theta^\ast\|^2 + \sigma^2}\) obtained under known balanced mixing weights \(\pi = \pi^\ast = (\frac{1}{2}, \frac{1}{2})\) (see Lemma 22 on page 52 of \cite{kwon2019global}),
since \(\frac{\arctan \eta}{\pi/2} \|\theta^\ast\|+ \frac{2}{\pi} \sigma \leq \|\theta^\ast\|+ \frac{2}{\pi} \sigma \leq \sqrt{(1+4/\pi^2)(\|\theta^\ast\|^2 + \sigma^2)} \leq 1.2 \sqrt{\|\theta^\ast\|^2 + \sigma^2} \leq 3\sqrt{\|\theta^\ast\|^2 + \sigma^2}\).
\end{remark}

\begin{propositiontxt}[Distinct Fixed Points of EM Update Rules]\label{prop:distinct_fixed_points}
    The EM update rules \(M(\theta, \nu), N(\theta, \nu)\) for regression parameters \(\theta\) and the imbalance of mixing weights \(\tanh \nu\) have the following three distinct fixed points \((\theta^\ast, \nu^\ast), (-\theta^\ast, -\nu^\ast), (\vec{0}, 0)\):
    \begin{equation}
    \begin{pmatrix}
        M(\theta^\ast, \nu^\ast) \\ N(\theta^\ast, \nu^\ast)
    \end{pmatrix} = \begin{pmatrix}
        \theta^\ast \\ \tanh \nu^\ast
    \end{pmatrix},\quad
    \begin{pmatrix}
        M(-\theta^\ast, -\nu^\ast) \\ N(-\theta^\ast, -\nu^\ast)
    \end{pmatrix} = \begin{pmatrix}
        -\theta^\ast \\ \tanh(-\nu^\ast)
    \end{pmatrix},\quad
    \begin{pmatrix}M(\vec{0}, 0) \\ N(\vec{0}, 0) \end{pmatrix} = \begin{pmatrix} \vec{0} \\ \tanh 0 \end{pmatrix}.
    \end{equation}
\end{propositiontxt}

\begin{remark}
Among these three distinct fixed points, the first two are the ground truth parameters \((\theta^\ast, \nu^\ast)\) and \((-\theta^\ast, -\nu^\ast)\),
and they are the only two global minimizers of the negative log-likelihood function \(f(\theta, \pi)\) at the population level.
Since \(f\) is equal to the KL divergence between the ground truth distribution \(p(s\mid \theta^\ast, \pi^\ast)\) and the model distribution \(p(s\mid \theta, \pi)\), plus a constant term (see Appendix~\ref{sup:derive_em}),
these two fixed points uniquely minimize \(f\), achieving a KL divergence value of zero.
Apart from these three distinct fixed points, there exist two additional fixed points of the EM update rules whose regression parameters lie in the plane \(\text{span}\{\theta, \theta^\ast\}\), as shown in Proposition~\ref{prop:contraction_property}.
These points exhibit a contraction property along the direction of \(\theta\) when \(\theta\) is not in the same/opposite direction of gound truth \(\theta^\ast\).
This result is also consistent with the previous findings on the fixed points of the EM update rule for regression parameters only, under known balanced mixing weights \(\pi = \pi^\ast=(\frac{1}{2}, \frac{1}{2})\) (see Theorem 1 on page 6 of \cite{kwon2019global} and Theorem 2 on page 6523 of \cite{kwon2024global}).
\end{remark}

\begin{propositiontxt}[Contraction Property of EM Update Rule around the Fixed Point]\label{prop:contraction_property}
    Suppose that the unit direction vector \(\vec{e}^\perp\) 
    is orthogonal to the ground truth \(\theta^\ast\) of regression parameters, namely \(\langle \theta^\ast, \vec{e}^\perp \rangle = 0, \|\vec{e}^\perp\|=1\), 
    there exists \(k^\ast(\eta)>0\) which is determined by \(\eta = \| \theta^{\ast} \|/\sigma\) such that the \((k^\ast(\eta)\|\theta^\ast\|\vec{e}^\perp, 0)\) is a fixed point of the EM update rule \(M(\theta, \nu), N(\theta, \nu)\) for regression parameters \(\theta\) and the imbalance of mixing weights \(\tanh \nu\) given any SNR \(\eta = \| \theta^{\ast} \|/\sigma > 0\):
    \begin{equation}
    \begin{pmatrix}
        M(k^\ast(\eta)\|\theta^\ast\|\vec{e}^\perp, 0) \\ N(k^\ast(\eta)\|\theta^\ast\|\vec{e}^\perp, 0)
    \end{pmatrix} = \begin{pmatrix}
        k^\ast(\eta)\|\theta^\ast\|\vec{e}^\perp \\ \tanh(0)
    \end{pmatrix}.
    \end{equation}
    For any \(k>0\),\(N(k\|\theta^\ast\|\vec{e}^\perp, 0) = \tanh(0)\) and  \(M(k\|\theta^\ast\|\vec{e}^\perp, 0), \vec{e}^\perp\) have the same direction, 
    \begin{equation}
    \begin{cases}
        0< \|M(k\|\theta^\ast\|\vec{e}^\perp, 0)\|/\|\theta^\ast\|-k < k^\ast(\eta)-k\quad \text{if } k<k^\ast(\eta),\\
        0= \|M(k\|\theta^\ast\|\vec{e}^\perp, 0)\|/\|\theta^\ast\|-k = k^\ast(\eta)-k\quad \text{if } k=k^\ast(\eta),\\
        0>\|M(k\|\theta^\ast\|\vec{e}^\perp, 0)\|/\|\theta^\ast\|-k > k^\ast(\eta)-k\quad \text{if } k>k^\ast(\eta).
    \end{cases}
    \end{equation}
    Moreover, the bounds of \(k^\ast(\eta)\) are:
    \begin{equation}
        \frac{1}{\sqrt{3}} < k^\ast(\eta) < \min\left(\frac{2}{\pi}\sqrt{1+\eta^{-2}}, 1\right), \forall\, \eta>0,
    \end{equation}
    with \(\lim_{\eta\to 0_+} k^\ast(\eta) = \frac{1}{\sqrt{3}}\) and \(\lim_{\eta\to \infty} k^\ast(\eta) = \frac{2}{\pi}\).
\end{propositiontxt}

\begin{remark}
    This indicates that there is a unique fixed point in the direction of \(\vec{e}^\perp\) (distinct from the fixed point in the opposite direction), which is orthogonal to the ground truth \(\theta^\ast\) of regression parameters.
    Furthermore, if the previous iteration lies in the direction of \(\vec{e}^\perp\), the population EM update rule will also remain in this direction of \(\vec{e}^\perp\), and the EM iterations get progressively closer to the fixed point,
    showing the contraction property of the EM update rule along \(\vec{e}^\perp\).
    Moreover, we demonstrate that the bounds for the normalized length of the EM update rule \(k^\ast(\eta)\) are tight across all SNR regimes.
    In particular, we have shown that the fixed point \(k^\ast(\eta)\to \frac{2}{\pi}\) as the SNR \(\eta \to \infty\), which aligns with the result 
    that the diameter of the rolling circle of the cycloid trajectory for EM iterations is \(\frac{2}{\pi}\|\theta^\ast\|\) in the noiseless setting (Proposition~\ref{prop:parametric_cycloid} in Section~\ref{sec:population}).
\end{remark}

\begin{figure}[!htbp]
    \centering
    \subfloat[]{\includegraphics[width=0.45\textwidth]{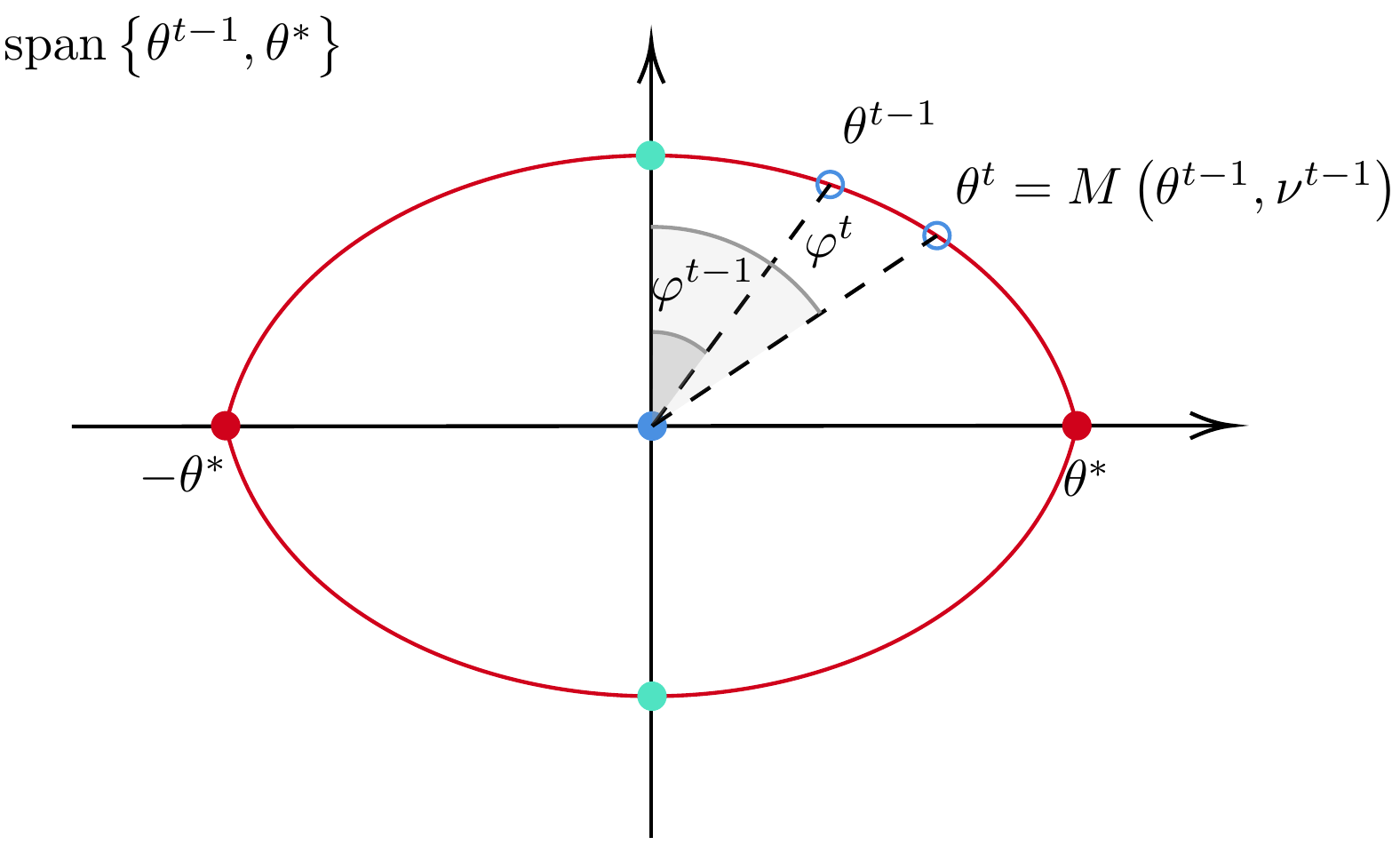}\label{fig:cycloid_trajectory_1}}
    \hfill
    \subfloat[]{\includegraphics[width=0.45\textwidth]{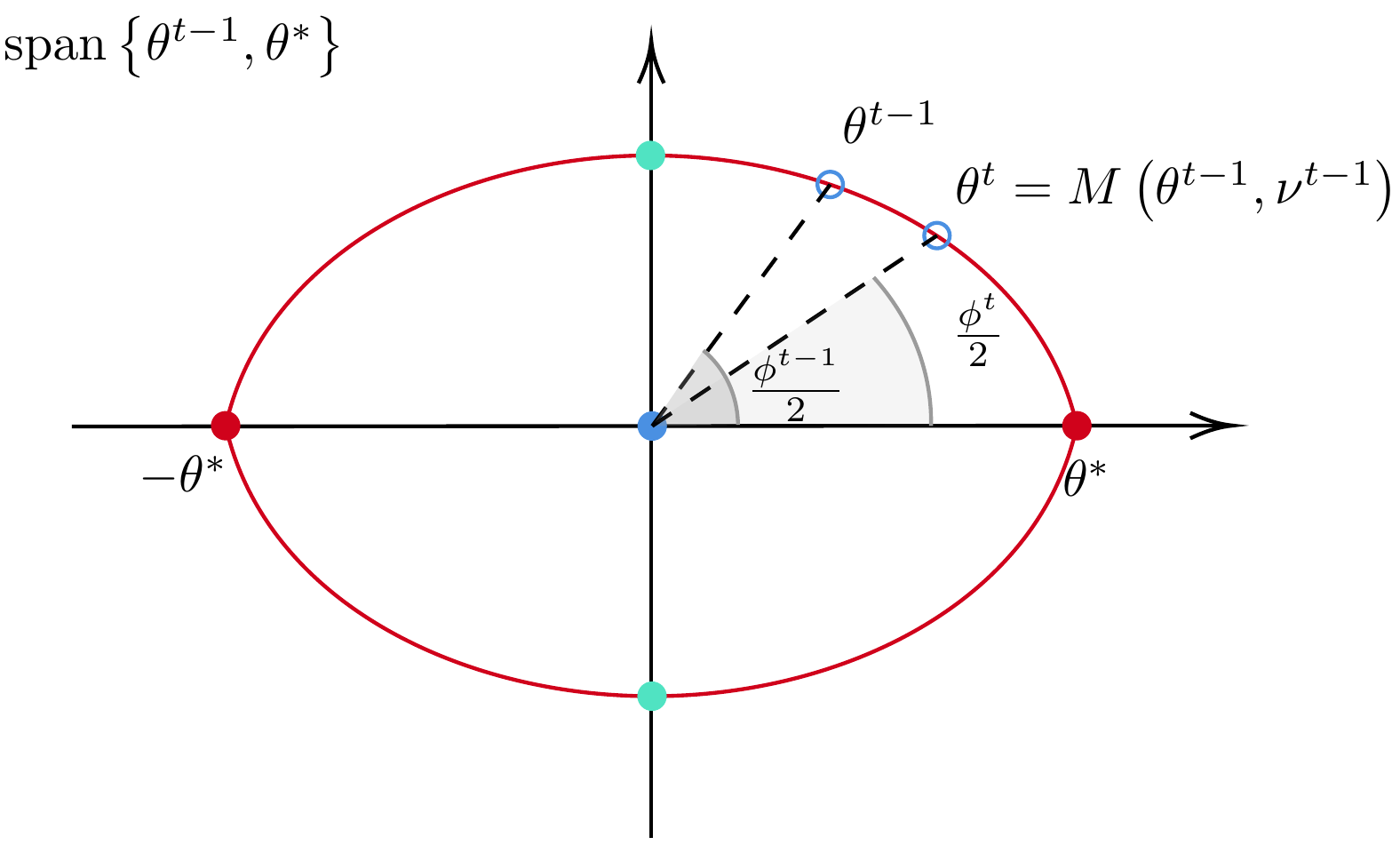}\label{fig:cycloid_trajectory_2}}
    \caption{The cycloid trajectory for the EM update rule \(M(\theta, \nu)\) of regression parameters in the noiseless setting (SNR \(\eta \to \infty\)), 
    and the fixed points of the population EM update rules are shown in the figure:
    these two red points stand for the ground truth parameters \(\theta^\ast\) and \(-\theta^\ast\),
    the blue point stands for the unstable fixed point \(\vec{0}\) as distinct fixed points of the EM update rules in Proposition~\ref{prop:distinct_fixed_points},
    and the green points stand for the two saddle points \(\pm \lim_{\eta\to\infty}k^\ast(\eta) \| \theta^\ast\| \hat{e}_2 = \pm \frac{2}{\pi} \| \theta^\ast\| \hat{e}_2\) as fixed points of the EM update rules on the plane \(\text{span}\{\theta, \theta^\ast\}\) in Proposition~\ref{prop:contraction_property} with contraction property 
    along the direction of \(\pm\hat{e}_2\) orthogonal to the ground truth \(\theta^\ast\).\\
    (a) Sub-optimality angle \(\varphi\): the angle between the unit direction vector \(\hat{e}_2\) and the regression parameters \(\theta\); 
    \(\varphi^{t}, \varphi^{t-1}\) correspond to the sub-optimality angles at the \(t\)-th and \((t-1)\)-th EM iterations, 
    where the regression parameters take the values \(\theta^t\) and \(\theta^{t-1}\), respectively.\\
    (b) Sub-optimality angle \(\phi\): twice the minimum angle between \(\theta\) and \(\pm\theta^\ast\), i.e., \(\phi = 2\arccos |\langle \theta, \theta^\ast \rangle|/(\|\theta\|\|\theta^\ast\|)\); 
    in the noiseless setting, \(\theta^{t}\) follows a cycloid trajectory with rolling radius \(\|\theta^\ast\|/\pi\), where the rolling angle is determined by the previous sub-optimality angle \(\phi^{t-1}\) (Proposition~\ref{prop:parametric_cycloid}).}
    \label{fig:cycloid_trajectory}
\end{figure}

\begin{corollarytxt}[EM Updates in Noiseless Setting, Corollary 3.3 in~\cite{luo24cycloid}]\label{cor:em_updates_noiseless}
    In the noiseless setting, namely SNR \(\eta := \| \theta^{\ast} \|/\sigma \to \infty\), the population EM update rules \(M(\theta, \nu), N(\theta, \nu)\) for regression parameters \(\theta\) and imbalance of mixing weights \(\tanh \nu\) are:
    \begin{equation}
        \begin{aligned}
            \frac{M(\theta, \nu)}{\| \theta^{\ast} \|} &= \frac{2}{\pi} \left[ \sgn(\rho)\varphi \frac{\theta^{\ast}}{\| \theta^{\ast} \|}  + \cos \varphi \frac{\theta}{\| \theta \|}  \right]\\
            N(\theta, \nu) &= \sgn(\rho) \frac{2}{\pi} \varphi \cdot \tanh \nu^{\ast}
        \end{aligned}
    \end{equation}
    where \(\varphi := \frac{\pi}{2}-\arccos |\rho|\), \(\rho := \frac{\langle \theta, \theta^\ast \rangle}{\|\theta\|\|\theta^\ast\|}\).
\end{corollarytxt}

\begin{remark}
In the noiseless setting, the first equation shows that the EM update rule is clearly a linear combination of the ground truth \(\theta^\ast\) and the current estimate \(\theta\).
It also indicates that the length of the normalized EM update rule reaches its minimum value \(\frac{2}{\pi}\) when \(\varphi = 0\), namely \(\langle \theta, \theta^\ast \rangle =0\);
and the length of the normalized EM update rule reaches its maximum value \(1\) when \(\varphi = \frac{\pi}{2}\), namely \(\theta^\ast, \theta\) are in the same/opposite direction.
Using the first equation, we show that the EM iterations of regression parameters \(\theta\) follow a cycloid trajectory,
whose rolling diameter is \(\frac{2}{\pi}\|\theta^\ast\|\) (Proposition~\ref{prop:parametric_cycloid} in Section~\ref{sec:population}).
The second equation shows that the EM update rule for the imbalance of mixing weights \(\tanh \nu\) is independent of the previous iteration \(\tanh \nu\),
and depends only on the ground truth imbalance of mixing weights \(\tanh \nu^\ast\) and the sub-optimality angle \(\varphi\), which is determined by the ground truth \(\theta^\ast\) and the current estimate \(\theta\).
This suggests that once the evolution of the sub-optimality angle \(\varphi\) is characterized in the noiseless setting, then the evolution of the imbalance of mixing weights \(\tanh \nu\) can also be determined.
\end{remark}

\begin{propositiontxt}[Deviation from Cycloid Limit of EM Updates in High SNR Regime]\label{prop:deviation_cycloid_limit}
    In the finite high SNR regime given by \(\eta \gtrsim \frac{1}{\min(1, \sqrt{k})\cos \varphi} \vee 1\), if the mixing weights are known balanced \(\pi = \pi^\ast = (\frac{1}{2}, \frac{1}{2})\), 
    then the difference between the EM update rule and its limit is bounded by:
    \begin{equation}
        \left\|\frac{M(\theta)}{\| \theta^{\ast} \|} - \lim_{\eta \rightarrow \infty} \frac{M(\theta)}{\| \theta^{\ast} \|}\right\|
        = \mathcal{O}\left(\eta^{-2} \vee \cos^2 \varphi \frac{\log \Lambda}{\Lambda^4}\right)
    \end{equation}
    where \(\Lambda \assign \eta \sqrt{k} \cos \varphi\), \(k:= \frac{\|\theta\|}{\|\theta^\ast\|}\) and \(\varphi := \frac{\pi}{2}-\arccos |\rho|\), \(\rho := \frac{\langle \theta, \theta^\ast \rangle}{\|\theta\|\|\theta^\ast\|}\).
\end{propositiontxt}

\begin{remark}
This result provides an explicit bound on the difference between the EM update rule and its limit (cycloid trajectory) in the finite high-SNR regime,
which generalizes the result obtained in the noiseless setting.
When the ratio \(k = \frac{\|\theta\|}{\|\theta^\ast\|}\) is large enough and the angle \(\frac{\pi}{2} - \varphi\) is large enough, meaning that the direction of the current iteration \(\theta\) is away from that of the ground truth \(\theta^\ast\),
the bound is dominated by \(\eta^{-2}\).
It also implies that the difference becomes smaller as the SNR \(\eta\) increases.
\end{remark}

\section{Population Level Analysis in Noiseless Setting}\label{sec:population}
In this section, we provide a comprehensive analysis of population EM updates in the noiseless setting, 
which consists of three parts:
In the first part, we derive a recurrence relation of the sub-optimality angle \(\varphi^t\) for population EM updates (Proposition~\ref{prop:recurrence_angle}),
based on the EM update rule with sub-optimality angle \(\varphi^t\) in the noiseless setting (Corollary~\ref{cor:em_updates_noiseless}) from the previous section (Section~\ref{sec:updates}).
We then demonstrate that the EM update iterations follow a cycloid trajectory by formulating the parametric equation of the sub-optimality angle \(\varphi^t\) (Proposition~\ref{prop:parametric_cycloid})
based on the recurrence relation.

In the second part, we show the linear growth of the sub-optimality angle \(\varphi^t\) (Proposition~\ref{prop:linear_growth_angle}) when the angle between the EM iteration \(\theta^t\) and the ground truth parameters \(\theta^\ast\) is large, 
i.e., when the sub-optimality angle \(\varphi^t = \frac{\pi}{2} - \arccos \left|\langle \theta^t, \theta^\ast \rangle/(\|\theta^t\|\|\theta^\ast\|)\right|\) is small.
We then establish the quadratic convergence of the sub-optimality angle \(\phi^t = 2\arccos \left|\langle \theta^t, \theta^\ast \rangle/(\|\theta^t\|\|\theta^\ast\|)\right|\) (Proposition~\ref{prop:quadratic_convergence_angle}) to zero when \(\varphi^t \leq 1.4\), that is, when the angle \(\arccos \left|\langle \theta^t, \theta^\ast \rangle/(\|\theta^t\|\|\theta^\ast\|)\right| \leq 0.7 \approx 40^\circ\) between the EM iteration \(\theta^t\) and the ground truth parameters \(\theta^\ast\) is sufficiently small.

In the third part, we provide bounds on the accuracy of regression parameters and mixing weights in the noiseless setting, in terms of the sub-optimality angle \(\phi^t\) (Proposition~\ref{prop:errors_em_updates_angle}), based on the parametric equation of the cycloid trajectory for regression parameters (Proposition~\ref{prop:parametric_cycloid}), 
and the EM update rule for mixing weights in terms of the sub-optimality angle \(\varphi^t\) (Corollary~\ref{cor:em_updates_noiseless}) from the previous section (Section~\ref{sec:updates}), respectively.
Finally, we establish the convergence guarantees of the EM updates in the noiseless setting (Theorem~\ref{theorem:population_level_convergence}) by combining the linear growth and quadratic convergence of sub-optimality angles \(\varphi^t, \phi^t\),
and the bounds on the accuracy of regression parameters and mixing weights in terms of the sub-optimality angle \(\phi^t\). 
The proofs of the results in this section are provided in Appendix~\ref{sup:population_analysis}.

\begin{propositiontxt}[Recurrence Relation of Sub-optimality Angle, Proposition 4.3 in~\cite{luo24cycloid}]\label{prop:recurrence_angle}
    In the noiseless setting, namely SNR \(\eta := \| \theta^{\ast} \|/\sigma \to \infty\), 
    if the sub-optimality angle \(\varphi^t \neq \frac{\pi}{2}\), 
    then the recurrence relation of the sub-optimality angle \(\varphi^t\) for population EM updates is:
    \begin{equation}
        \tan \varphi^{t+1} = \tan \varphi^t + \varphi^t (\tan^2 \varphi^t + 1)
    \end{equation}
    where \(\varphi^t := \frac{\pi}{2} - \arccos |\rho^t|\), \(\rho^t := \frac{\langle \theta^t, \theta^\ast \rangle}{\|\theta^t\|\|\theta^\ast\|}\).
\end{propositiontxt}

\begin{remark}
Following Corollary~\ref{cor:em_updates_noiseless} of EM update rules in the noiseless setting shown in the previous section (Section~\ref{sec:updates}), 
we can derive a recurrence relation for the sub-optimality angle \(\varphi^t\) in the population EM updates.
In particular, when \(\varphi^t\) is small, the recurrence relation can be viewed as a discretized version of 
the differential equation \(\mathd \tan \varphi = \varphi (\tan^2 \varphi + 1) \mathd t\), 
which gives a linear growth of \(\tan\varphi(t) \approx \varphi(t) = C \exp(t)\) for small \(\varphi(t)\) with \(C\geq0\).
When \(\varphi^t\) becomes large, approaching \(\frac{\pi}{2}\), i.e., \(\varphi \to \frac{\pi}{2}\), the second term \(\varphi^t (\tan^2 \varphi^t + 1)\) dominates the first term \(\tan \varphi^t\),
causing the growth of \(\tan \varphi^t\) to exceed linear growth. In this regime, we have \(\tan \varphi^{t+1} \approx \frac{\pi}{2} \tan^2 \varphi^t\), which suggests quadratic growth of \(\tan \varphi^t\).
\end{remark}

\begin{propositiontxt}[Parametric Equation for Cycloid Trajectory of EM Updates, Proposition 4.4 in~\cite{luo24cycloid}]\label{prop:parametric_cycloid}
    In the noiseless setting, namely SNR \(\eta := \| \theta^{\ast} \|/\sigma \to \infty\),
    the coordinates \(\mathtt{x}^t, \mathtt{y}^t\) of normalized vector \(\frac{\theta^t}{\|\theta^\ast\|}=\mathtt{x}^t \hat{e}_1 + \mathtt{y}^t \hat{e}_2^t = \mathtt{x}^t \hat{e}_1 + \mathtt{y}^t \hat{e}_2^0, \forall t\in\mathbb{Z}_+\)
    for population EM updates can be parameterized by the sub-optimality angle \(\phi^{t-1}\) as follows:
    \begin{equation}
        \begin{aligned}
            1-\sgn(\rho^0)\mathtt{x}^t & = \frac{1}{\pi}[\phi^{t-1} - \sin \phi^{t-1}]\\
            \mathtt{y}^t & = \frac{1}{\pi}[1- \cos \phi^{t-1}]
        \end{aligned}
    \end{equation}
    where \(\varphi^{t-1} := \frac{\pi}{2} - \arccos |\rho^{t-1}|\), \(\rho^{t-1} := \frac{\langle \theta^{t-1}, \theta^\ast \rangle}{\|\theta^{t-1}\|\|\theta^\ast\|}\).
    Hence, the trajecotry of EM iterations \(\theta^t\) is on the cycloid with a parameter \(\frac{\|\theta^\ast\|}{\pi}\), on the plane \(\text{span}\{\theta^0, \theta^\ast\}\).
\end{propositiontxt}

\begin{remark}
By applying Corollary~\ref{cor:em_updates_noiseless} for the EM update rule of regression parameters \(\theta\) in the noiseless setting, as shown in the previous section (Section~\ref{sec:updates}),
together with the recurrence relation of \(\varphi^t\) (Proposition~\ref{prop:recurrence_angle}), we show that EM iterations of regression parameters follow a cycloid trajectory,
with a rolling radius \(\frac{1}{\pi}\|\theta^\ast\|\).
The rolling angle at the \(t\)-th EM iteration is determined by the sub-optimality angle \(\phi^{t-1}\) from the previous \((t-1)\)-th EM iteration.
Therefore, the cycloid trajectory of population EM iterations can fully characterized once 
the sub-optimality angles \(\phi^t\) or \(\varphi^t=\frac{\pi - \phi^t}{2}\) are fully characterized for the population EM updates.
\end{remark}

\begin{propositiontxt}[Linear Growth of Sub-optimality Angle]\label{prop:linear_growth_angle}
    In the noiseless setting, namely SNR \(\eta := \| \theta^{\ast} \|/\sigma \to \infty\),
    \(\tan \varphi^t\) of the sub-optimality angle \(\varphi^t\) grows at least linearly:
    \begin{equation}
        \tan \varphi^{t+1} \geq 2\cdot \tan \varphi^t
    \end{equation}
    where \(\varphi^t := \frac{\pi}{2} - \arccos |\rho^t|\), \(\rho^t := \frac{\langle \theta^t, \theta^\ast \rangle}{\|\theta^t\|\|\theta^\ast\|}\).
\end{propositiontxt}

\begin{remark}
By analyzing the recurrence relation of \(\varphi^t\) above, we show that the linear growth guarantee of \(\tan \varphi^t\) occurs in the worst case,
which is consistent with our discussion in the previous remark on the recurrence relation of \(\varphi^t\).
In fact, by applying the trigonometric identity \(\sin(2\varphi^t) = 2 \tan \varphi^t / (1+\tan^2 \varphi^t)\),
we obtain \(\tan \varphi^{t+1} = (1+ \frac{2 \varphi^t}{\sin (2\varphi^t)}) \tan \varphi^t \geq 2\cdot \tan \varphi^t\).
When \(\varphi^t\) is small enough and away from \(\frac{\pi}{2}\), we have \(\sin (2\varphi^t) \approx 2 \varphi^t\), and thus the growth of \(\tan \varphi^t\) is linear.
However, as \(\varphi^t\) increases and approaches \(\frac{\pi}{2}\), the growth of \(\tan \varphi^t\) exceeds linear behavior, 
and we instead obtain quadratic convergence of \(\phi^t\) to zero when \(\phi^t = \pi - 2\varphi^t\) is sufficiently small.
This result strengthens our previous finding on the linear growth of \(\tan \varphi^t\), namely\(\tan \varphi^t \geq (\frac{1+\sqrt{5}}{2})\cdot \tan \varphi^t\) in the conference version of our paper (see Proposition 4.5 of \cite{luo24cycloid}), 
by showing the optimal factor \(2\geq \frac{1+\sqrt{5}}{2}\) in the linear growth rate \(\tan \varphi^{t+1} \geq 2\cdot \tan \varphi^t \geq (\frac{1+\sqrt{5}}{2})\cdot \tan \varphi^t\).
\end{remark}

\begin{propositiontxt}[Quadratic Convergence of Sub-optimality Angle]\label{prop:quadratic_convergence_angle}
    In the noiseless setting, namely SNR \(\eta := \| \theta^{\ast} \|/\sigma \to \infty\),
    the sub-optimality angle \(\phi^t\) converges quadratically to zero when \(\phi^t \leq 1.4\) is small enough:
    \begin{equation}
        \frac{\phi^{t+1}}{\pi} \leq \left[ \frac{\phi^t}{\pi} \right]^2
    \end{equation}
    where \(\phi^t : = 2\arccos |\rho^t|\), \(\rho^t := \frac{\langle \theta^t, \theta^\ast \rangle}{\|\theta^t\|\|\theta^\ast\|}\).
\end{propositiontxt}

\begin{remark}
    When \(\phi^t\) is small enough, namely \(\varphi^t = \frac{\pi -\phi^t}{2}\) is close to \(\pi/2\), we have \(\tan \varphi^t = \cot \frac{\phi^t}{2} \approx \frac{2}{\phi^t}\),
    and the term \(\tan^2 \varphi^t\) dominates the term \(\varphi^t\) in the recurrence relation.
    Therefore, we have \(2/\phi^{t+1} \approx \tan \varphi^{t+1} \approx \frac{\pi}{2} \tan^2 \varphi^t \approx \frac{\pi}{2} \left(2/\phi^t\right)^2 = 2\pi/[\phi^t]^2\).
    Through a rigorous derivation, we can show that the above inequality \(\phi^{t+1}/\pi \leq \left[ \phi^t/\pi \right]^2\) holds for sufficiently small \(\phi^t \leq 1.4\), 
    which guarantees the quadratic convergence of \(\phi^t\) to zero.
    While previous work (see page 5, Lemma 1 in~\cite{kwon2021minimax}) established the quadratic convergence rate 
    when \(\|\theta - \sgn(\rho) \theta^\ast\|/\|\theta^\ast\| \leq 1/10\) under high SNR,
    our analysis extends the region of quadratic convergence from 1/10 to \(\frac{1}{\pi}\sqrt{(\phi-\sin \phi)^2 + (1-\cos \phi)^2}_{\phi=1.4} \approx 0.30\).
\end{remark}

\begin{propositiontxt}[Accuracy of Population EM Updates and Sub-optimality Angle]\label{prop:errors_em_updates_angle}
    In the noiseless setting, namely SNR \(\eta := \| \theta^{\ast} \|/\sigma \to \infty\),
    \begin{equation}
        \begin{aligned}
            \frac{\| \theta^t - \sgn(\rho^0) \theta^\ast\|}{\| \theta^\ast\|} &\leq \frac{\left[\phi^{t-1}\right]^2}{2\pi} \\
            \|\pi^t - \bar{\pi}^\ast \|_1 &= \frac{\phi^{t-1}}{\pi}\cdot \left\| \pi^\ast - \frac{\mathds{1}}{2} \right\|_1
        \end{aligned}
    \end{equation}
    where \(\phi^{t-1} := 2(\frac{\pi}{2} - \arccos |\rho^{t-1}|)\), \(\rho^{t-1} := \frac{\langle \theta^{t-1}, \theta^\ast \rangle}{\|\theta^{t-1}\|\|\theta^\ast\|}\),
    and \(\bar{\pi}^\ast := \frac{\mathds{1}}{2}+\sgn(\rho^0)(\pi^\ast-\frac{\mathds{1}}{2}), \mathds{1} := (1, 1)\).
\end{propositiontxt}
\begin{remark}
    The first inequality provides a bound on the accuracy of regression parameters \(\theta^t\) with respect to the ground truth \(\theta^\ast\),
    which depends only on the sub-optimality angle \(\phi^{t-1}\) from the previous \((t-1)\)-th EM iteration.
    The second identity quantifies the accuracy of mixing weights \(\pi^t\) with respect to the ground truth \(\pi^\ast\),
    which depends only on the sub-optimality angle \(\phi^{t-1}\) and the imbalance of ground truth mixing weights from the previous \((t-1)\)-th EM iteration.
    These two relations suggest that we can establish the convergence guarantee of the population EM updates once we establish the convergence guarantee of the sub-optimality angle \(\phi^t\).
\end{remark}

\begin{theoremtxt}[Population Level Convergence, Theorem 4.1 in~\cite{luo24cycloid}]\label{theorem:population_level_convergence}
    In the noiseless setting, namely SNR \(\eta := \| \theta^{\ast} \|/\sigma \to \infty\), if the initial sup-optimality angle cosine \(\rho^0:=\frac{\langle \theta^0, \theta^\ast \rangle}{\|\theta^0\|\|\theta^\ast\|}\neq 0\),
    then with the number of total iterations at most \(T=\mathcal{O}(\log \frac{1}{|\rho^0|}\vee \log\log\frac{1}{\varepsilon})\), the error of EM update at the population level is bounded by:
    \(\frac{\| \theta^{T+1} - \sgn(\rho^0) \theta^\ast\|}{\| \theta^\ast\|} < \varepsilon\) and \(\|\pi^{T+1} - \bar{\pi}^\ast \|_1 = \left\|\pi^\ast - \frac{\mathds{1}}{2} \right\|_1 \mathcal{O}(\sqrt{\varepsilon})\).
    where \(\bar{\pi}^\ast := \frac{\mathds{1}}{2}+\sgn(\rho^0)(\pi^\ast-\frac{\mathds{1}}{2}), \mathds{1} := (1, 1)\).
\end{theoremtxt}

\begin{remark}
    If the initial sub-optimality angle \(\varphi^0 \geq 1\), then \(\phi^t \leq \pi - 2\varphi^0 < 1.4\) converges quadratically to zero (thus \(T_1=0\)).
    Otherwise, when the initial sub-optimality angle \(\varphi^0 < 1\), it takes at most \(T_1 = \mathcal{O}(\log \frac{1}{\tan \varphi^0})= \mathcal{O}(\log \frac{1}{|\rho^0|})\) EM iterations to ensure quadratic convergence of \(\phi^t \leq 1.4\) for all \(t\geq T_1\).
    Subsequently, the population EM updates require at most another \(T_2 = \mathcal{O}(\log\log\frac{1}{\varepsilon})\) EM iterations to ensure that the sub-optimality angle \(\phi^t < \sqrt{2\pi \varepsilon}\) is sufficiently small.
    Therefore, using the above expressions for the errors of regression parameters and mixing weights, we establish the convergence guarantee of the population EM updates with a total number of iterations at most \(T=T_1 + T_2 = \mathcal{O}(\log \frac{1}{|\rho^0|}\vee \log\log\frac{1}{\varepsilon})\).
\end{remark}


\section{Finite-Sample Level Analysis in Noiseless Setting}\label{sec:finite}
In this section, we provide a comprehensive analysis of finite-sample EM updates in the noiseless setting, 
which consists of two parts:
In the first part, we bound the statistical errors of regression parameters and mixing weights in the finite-sample setting (Propositions~\ref{prop:projected_error_regression},~\ref{prop:statistical_error_regression} and~\ref{prop:statistical_error_mixing_weights}), 
showing that the statistical error of regression parameters is on the order of \((d/n)^{1/2}\), 
while the statistical error of mixing weights depends on the sub-optimality angle \(\phi\) and the ground truth mixing weights \(\pi^\ast\).
Furthermore, we bound the statistical accuracy of EM updates for regression parameters and mixing weights (Proposition~\ref{prop:statistical_accuracy_em_updates}), 
showing that the statistical accuracy of EM updates for regression parameters is on the order of \(\phi^2 \vee \phi^{\frac{3}{2}} \sqrt{d/n} \vee \phi \sqrt{\log \sdfrac{1}{\delta}/n}\) 
and the statistical accuracy of EM updates for mixing weights depends on the sub-optimality angle \(\phi\) and the ground truth mixing weights \(\pi^\ast\).

In the second part, we establish the convergence guarantees of the sub-optimality angle \(\phi^t\) (Proposition~\ref{prop:convergence_angle})
by leveraging the linear growth and quadratic convergence of the sub-optimality angles (Propositions~\ref{prop:linear_growth_angle} and~\ref{prop:quadratic_convergence_angle}) at the population level from the previous section (Section~\ref{sec:population})
and combining them with the bounds on the statistical errors and the statistical accuracy of regression parameters and mixing weights in the finite-sample setting.
Finally, we establish the convergence guarantees of the EM updates in the finite-sample setting (Theorem~\ref{theorem:finite_sample_convergence}), 
by combining the convergence guarantees of the sub-optimality angle \(\phi^t\) (Proposition~\ref{prop:convergence_angle})  
and the connection between the sub-optimality angle \(\phi^t\) and the statistical accuracy of regression parameters and mixing weights (Proposition~\ref{prop:statistical_accuracy_em_updates}) obtained in the first part.
Moreover, the initialization of unknown mixing weights and regression parameters can be arbitrary, enabled by Easy EM followed by the standard EM update rules (Proposition~\ref{prop:initialization_easy_em}).
Therefore, the convergence guarantees for the finite-sample EM updates of the regression parameters and mixing weights are established and characterized.
The proofs of the results in this section are provided in Appendix~\ref{sup:finite_sample_analysis}.

\begin{propositiontxt}[Projected Error of Easy EM Update for Regression Parameters, Proposition 5.2 in~\cite{luo24cycloid}]\label{prop:projected_error_regression}
    In the noiseless setting, the projection on $\text{span}\{\theta,\theta^\ast\}$ for the statistical error of $\theta$ satisfies
    \begin{equation}
      \frac{\|P_{\theta,\theta^\ast}
      [M_n^{\tmop{easy}} (\theta, \nu) - M (\theta, \nu)]\|}{\|\theta^\ast\|}
      = \mathcal{O}\left(\sqrt{\frac{\log \frac{1}{\delta}}{n}} \vee \frac{\log \frac{1}{\delta}}{n}\right),
    \end{equation}
    with probability at least $1 - \delta$, where $M_n (\theta, \nu), M (\theta,
    \nu)$ are the EM update rules for $\theta$ at the Finite-sample level and the
    population level respectively, and the orthogonal projection matrix $P_{\theta,\theta^\ast}$ satisfies
    $\text{span}(P_{\theta,\theta^\ast})=\text{span}\{\theta,\theta^\ast\}$.
\end{propositiontxt}

\begin{remark}
  We can show that the projected error of the EM update rule for regression parameters 
  is the average of \(n\) i.i.d. sub-exponential two-dimensional random vectors.
  Therefore, by Bernstein's inequality for sub-exponential random variables (see Corollary 2.8.3 on page 38 of~\cite{vershynin2018prob}), the bound for the projection error of the EM update rule is obtained as shown above.
\end{remark}

\begin{propositiontxt}[Statistical Error of EM Update for Regression Parameters, Proposition 5.3 in~\cite{luo24cycloid}]\label{prop:statistical_error_regression}
    In the noiseless setting, the statistical error of $\theta$ for finite-sample EM
    updates with \(n\gtrsim d\vee \log \frac{1}{\delta}\) samples satisfies
    \small
    \begin{equation}
      \frac{\| M_n^{\tmop{easy}} (\theta, \nu) - M (\theta, \nu) \|}{\| \theta^{\ast} \|}
      = \mathcal{O} \left( \sqrt{\frac{d\vee \log \frac{1}{\delta}}{n}}\right)
      ,\quad
      \frac{\| M_n (\theta, \nu) - M (\theta, \nu) \|}{\| \theta^{\ast} \|}
      = \mathcal{O} \left( \sqrt{\frac{d\vee \log \frac{1}{\delta}}{n}} 
      \right),
    \end{equation}
    with probability at least $1 - \delta$, $M_n (\theta, \nu), M (\theta,
    \nu)$ denote the EM update rules for $\theta$ at the Finite-sample level and the
    Population level.
\end{propositiontxt}

\begin{remark}
  For the statistical error of the Easy EM update rule for regression parameters, we obtain a tight bound by decomposing 
  the error into two parts: the projected error and the error orthogonal to the space \(\text{span}\{\theta, \theta^\ast\}\),
  the projected error is already bounded in Proposition~\ref{prop:projected_error_regression},
  while the orthogonal error is bounded by applying the rotational invariance of
  Gaussians and rewriting the \(\ell_2\) norm of the error as the geometric mean of two Chi-square random variables,
  and leveraging the concentration inequality for Chi-square distribution (see Lemma 1, page 1325 in~\cite{laurent2000adaptive}).
  For the statistical error of the standard EM update rule for regression parameters, we obtain similar bounds by separately 
  bounding the decomposed errors (see the detailed derivation in Appendix~\ref{sup:finite_sample_analysis}).
  Our analysis yields a tighter bound \(\mathcal{O} ( \sqrt{(d\vee \log \sdfrac{1}{\delta})/n} )\) for the statistical error 
  compared to the previously established bound \(\mathcal{O}( \sqrt{d/n}\log \frac{n}{\delta} )\), which includes amultiplicative logarithmic factor \(\log \frac{n}{\delta}\) 
  arising from the standard symmetrization technique and the Ledoux-Talagrand contraction argument (Appendix E, Lemma 11, page 17 in~\cite{kwon2021minimax} and \cite{balakrishnan2017statistical})
\end{remark}

\begin{propositiontxt}[Statistical Error of EM Update for Mixing Weights]\label{prop:statistical_error_mixing_weights}
    In the noiseless setting, the statistical error of mixing weights for finite-sample EM updates
    satisfies
    \begin{equation}
    \left| N_n(\theta, \nu) - N(\theta, \nu) \right| 
    = \mathcal{O}\left(\frac{\log \frac{1}{\delta}/n}{\log\left(1+ \frac{\log\frac{1}{\delta}/n}{p}\right)}\wedge \sqrt{\frac{\log\frac{1}{\delta}}{n}} \right) 
    \end{equation}
    with probability at least \(1-\delta\), where \(N_n(\theta, \nu), N(\theta, \nu)\) denote the EM update rules for imbalance \(\tanh \nu\) of the mixing weights at the Finite-sample level and the Population level, 
    and \(p :=\| \pi^\ast - \frac{\mathds{1}}{2} \|_1 \frac{\phi}{2\pi} + \min(\pi^\ast(1), \pi^\ast(2))\), \(\phi = 2 \arccos |\rho|, \rho = \frac{\langle \theta, \theta^{\ast} \rangle}{\| \theta \| \cdot \| \theta^{\ast} \|}\).
\end{propositiontxt}

\begin{remark}
  We bound the statistical error of mixing weights by rewriting it 
  as an average of \(n\) i.i.d. Bernoulli random variables with the success probability \(p:=\| \pi^\ast - \frac{\mathds{1}}{2} \|_1 \frac{\phi}{2\pi} + \min(\pi^\ast(1), \pi^\ast(2))\).
  By applying the concentration inequality for Bernoulli r.v.'s (see Appendix~\ref{sup:lemma}), we obtain the bound for the statistical error of mixing weights as shown above.
  In particular, when \(\pi^\ast \to (1, 0)\) or \(\pi^\ast \to (0, 1)\) and the sub-optimality angle \(\phi\to 0\), 
  then the probability \(p=\|\pi^\ast - \frac{\mathds{1}}{2}\|_1 \frac{\phi}{2\pi} + \min(\pi^\ast(1), \pi^\ast(2)) \to 0\), 
  therefore the statistical error of mixing weights \(\left| N_n(\theta, \nu) - N(\theta, \nu) \right|\) also approaches zero.  
  In general, for fixed ground truth mixing weights \(\pi^\ast\), the statistical error of mixing weights \(\left| N_n(\theta, \nu) - N(\theta, \nu) \right|\) decreases as the sub-optimality angle \(\phi\) decreases.
\end{remark}

\begin{propositiontxt}[Statistical Accuracy of EM Updates for Regression Parameters and Mixing Weights]\label{prop:statistical_accuracy_em_updates}
  In the noiseless setting, the finite-sample EM with \(n\gtrsim d\vee \log \frac{1}{\delta}\) samples achieves the statistical accuracy of regression parameters and mixing weights:
  \begin{equation}
    \begin{aligned}
      \frac{\left\| M_n(\theta, \nu) - \sgn(\rho) \theta^\ast \right\|}{\|\theta^\ast\|} 
      &= \mathcal{O}\left(\phi^2 \vee \phi^{\frac{3}{2}} \sqrt{\frac{d}{n}} \vee \phi \sqrt{\frac{\log \frac{1}{\delta}}{n}}\right)\\
      \left| N_n(\theta, \nu) - \sgn(\rho) \tanh \nu^\ast\right| &= \mathcal{O}\left( 
        \phi \left\| \pi^\ast -\frac{\mathds{1}}{2} \right\|_1 \vee \left[ 
          \frac{\log\frac{1}{\delta}/n}{\log\left(1+\frac{\log \frac{1}{\delta}/n}{p}\right)}
          \wedge \sqrt{\frac{\log \frac{1}{\delta}}{n}}
        \right]
      \right)
    \end{aligned}
  \end{equation}
  with probability at least \(1-\delta\), where \(M_n(\theta, \nu)\) denotes the EM update rule for regression parameters at the Finite-sample level, 
  and \(\phi := 2 \arccos |\rho|, \rho := \frac{\langle \theta, \theta^{\ast} \rangle}{\| \theta \| \cdot \| \theta^{\ast} \|}\)
  and \(p:=\frac{\phi}{2\pi}\|\pi^\ast - \frac{\mathds{1}}{2}\|_1 + \min(\pi^\ast(1), \pi^\ast(2))\).
\end{propositiontxt}

\begin{remark}
  The above results show that the statistical accuracy of regression parameters satisfies 
  \(\phi^{t+1} \leq \pi \sin \frac{\phi^{t+1}}{2} \leq \pi \|\theta^{t+1} - \sgn(\rho^{t+1}) \theta^\ast\|/\|\theta^\ast\| \lesssim [\phi^t]^2 \vee [\phi^t]^{\frac{3}{2}} \sqrt{d/n} \vee \phi^t \sqrt{\log \sdfrac{1}{\delta}/n}\).
  This implies the quadratic convergence rate for sub-optimality angle \(\phi^t\) and therefore the quadratic convergence rate for the accuracy of normalized regression parameters\(\|\theta^t - \sgn(\rho^0) \theta^\ast\|/\|\theta^\ast\|\),
  when the sub-optimality angle \(\phi^t\) is large enough, i.e.,\(\phi^t \gtrsim d/n \vee \sqrt{\log \sdfrac{1}{\delta}/n}\).
  This aligns with our theoretical finding in Proposition~\ref{prop:quadratic_convergence_angle} at the population level with infinite samples \(n \to \infty\).
  If the sub-optimality angle is very small such that \(\phi^t \lesssim \log \frac{1}{\delta}/n\), 
  then the term \(\phi^t \sqrt{\log \sdfrac{1}{\delta}/n}\) dominates the other terms, 
  and the sub-optimality angle \(\phi^t\) (and therefore the regression parameter accuracy \(\|\theta^t - \sgn(\rho^0) \theta^\ast\|/\|\theta^\ast\|\)) converges to zero linearly at a rate of \(\sqrt{\log \sdfrac{1}{\delta}/n}\).
  By the above analysis, and noting that the normalized regression parameter accuracy \(\|\theta^t - \sgn(\rho^0) \theta^\ast\|/\|\theta^\ast\| \asymp \phi^t\), we establish the quadratic convergence rate for the regression parameter accuracy \(\|\theta^t - \sgn(\rho^0) \theta^\ast\|/\|\theta^\ast\|\) 
  when the normalized regression parameter accuracy is large enough in the noiseless setting without the requirements on the mixing weights.
  While the previous work \cite{ghosh20a} showed the quadratic convergence rate for the accuracy of normalized regression parameters when the accuracy \( \gtrsim \min(\pi^\ast(1), \pi^\ast(2))\) 
  for a variant of EM in the noiseless setting, our results remove restrictions on mixing weights and still establish quadratic convergence rate.
\end{remark}

\begin{propositiontxt}[Initialization with Easy EM, Proposition 5.4 in~\cite{luo24cycloid}]\label{prop:initialization_easy_em}
    In the noiseless setting, suppose we run the sample-splitting finite-sample
    Easy EM with $n' \asymp \frac{n}{\log \frac{1}{\delta}}
    $ fresh samples for each iteration, then after at most $T_0 =\mathcal{O} \left( \log
    \frac{1}{\delta} \right)$ iterations, it satisfies $\varphi^{T_0} \gtrsim
        \sqrt{\frac{\log \frac{1}{\delta}}{n}} 
    $ 
    with probability at least $1 - \delta$.
\end{propositiontxt}

\begin{remark}
  The above result guarantees that, when running the EM algorithm, arbitrary initialization can still satisfy the condition for angle convergence
  \(\varphi^0 \gtrsim \sqrt{\log \sdfrac{1}{\delta}/n}\) in the next stage,
  after only a few iterations of Easy EM with high probability.
  Our theoretical result avoids both using the spectral method for initiallization and requiring the comparison between the sample size \(n\) and \(d^2\log^2\frac{n}{\delta}\) as in previous works (see page 6526 of \cite{kwon2024global} and page 3518 of \cite{dana2019estimate2mix}).
\end{remark}



\begin{propositiontxt}[Convergence of Sub-optimality Angle]\label{prop:convergence_angle}
  In the noiseless setting, suppose $\varphi^0 \gtrsim
  \sqrt{\frac{\log \frac{1}{\delta}}{n}} 
  $, given a positive number \(\varepsilon \lesssim \frac{\log \frac{1}{\delta}}{n}\), we run Easy finite-sample EM for $T_1=\mathcal{O}\left( \log
  \frac{d}{\log \frac{1}{\delta}}\right)$ iterations followed by the standard finite-sample EM for at most \(T' 
  = \mathcal{O}\left( [\log\frac{n}{d}\wedge \log \frac{n}{\log \frac{1}{\delta}}]\vee \log[\log\frac{n}{\ln \frac{1}{\delta}}/ \log\frac{n}{d}] \vee [\log\frac{1}{\varepsilon}/\log\frac{n}{\log \frac{1}{\delta}}]\right) \) iterations
  with $n \gtrsim
      d \vee \log \frac{1}{\delta} 
  $ samples, then it satisfies
  \begin{equation}
    \phi^T \leq \varepsilon,
  \end{equation}
  with probability at least $1 - T \delta$, where $T:=T_1+T',\varphi^0 \assign
  \frac{\pi}{2} - \arccos \left| \frac{\langle \theta^0, \theta^{\ast}
  \rangle}{\| \theta^0 \| \cdot \| \theta^{\ast} \|} \right|$ and $\varphi^T
  \assign \frac{\pi}{2} - \arccos \left| \frac{\langle \theta^T, \theta^{\ast}
  \rangle}{\| \theta^T \| \cdot \| \theta^{\ast} \|} \right|$.
\end{propositiontxt}

\begin{remark}
  Instead of analyzing the convergence of regression parameters and mixing weights at the finite-sample level directly,
  we first establish the convergence of the sub-optimality angle \(\phi^t\) (see detailed proof in Appendix~\ref{sup:finite_sample_analysis}), then later establish the convergence of EM update rules by employing the convergence of the sub-optimality angle \(\phi^t\) together with the bounds on the errors of regression parameters and mixing weights in the finite-sample setting.
  The analysis for the convergence of the sub-optimality angle \(\phi^t\) proceeds in four stages.

  In the first stage, Easy EM is run for \(T_1=\mathcal{O}(\log\frac{\text{statistical error of }\theta}{\text{projected error of }\theta}) =\mathcal{O}(\log (d/\log \frac{1}{\delta}))\) iterations 
  to ensure that the sub-optimality angle \(\varphi^t\) grows from \(\varphi^0 \gtrsim \sqrt{\log \sdfrac{1}{\delta}/n}\) to \(\varphi^{T_1} \gtrsim \sqrt{(d\vee\log \sdfrac{1}{\delta})/n}\), which matches the statistical error of normalized regression parameters, 
  In this stage, the sub-optimality angle \(\varphi^t\) of Easy EM exhibits a linear growth rate, 
  with an error on the order of \(\sqrt{\log\sdfrac{1}{\delta}/n}\) which matches the bound on the projected error of the normalized regression parameters, 
  as indicated by Proposition~\ref{prop:linear_growth_angle} at the population level and Proposition~\ref{prop:projected_error_regression} at the finite-sample level.
  
  In the second stage, standard EM is run for \(T_2=\mathcal{O} \left(\log(1/\text{statistical error of }\theta) \right) =\mathcal{O} \left(\log (n/d) \wedge \log (n/\log \frac{1}{\delta}) \right)\) iterations,
  continuing the linear growth from \(\varphi^{T_1} \gtrsim \sqrt{(d\vee\log \sdfrac{1}{\delta})/n}\) 
  until \(\varphi^{T_1+T_2} \geq 1\) (ensuring \(\phi^{T_1+T_2} \leq \pi - 2\varphi^{T_1+T_2} < 1.4\)).
  In this stage, the linear growth of the sub-optimality angle \(\varphi^t\) of standard EM is guaranteed 
  with an error on the order of \(\sqrt{(d\vee\log \sdfrac{1}{\delta})/n}\), which matches the statistical error of the normalized regression parameters, 
  by leveraging Proposition~\ref{prop:linear_growth_angle} at the population level and Proposition~\ref{prop:statistical_error_regression} at the finite-sample level.
  
  In the third stage, standard EM is run for \(T_3=\mathcal{O}\left(\log[\log(n/d) \wedge \log(n/\log \frac{1}{\delta})]\right)=\mathcal{O} \left(\log\log(1/\text{statistical error of }\theta) \right)\) iterations,
  exhibiting a quadratic convergence rate of the sub-optimality angle \(\phi^t\) from \(\phi^{T_1+T_2} < 1.4\) to \(\phi^{T_1+T_2+T_3} \lesssim \sqrt{(d\vee\log \sdfrac{1}{\delta})/n}\).
  In this stage, the quadratic convergence of the sub-optimality angle \(\phi^t\) of standard EM is shown with an error on the order of \(\sqrt{(d\vee\log \sdfrac{1}{\delta})/n}\),
  which matches the statistical error of the normalized regression parameters, 
  by applying Proposition~\ref{prop:convergence_angle} at the population level and Proposition~\ref{prop:statistical_error_regression} at the finite-sample level.
  
  In the fourth stage, standard EM is run for \(T_4=\mathcal{O} \left(\log[\log(n/\log \frac{1}{\delta})/ \log(n/d)] \vee [\log(1/\varepsilon)/\log(n/\log \frac{1}{\delta})]\right)\) iterations
  to achieve a final sub-optimality angle \(\phi^{T_1+T_2+T_3+T_4} \leq \varepsilon\) (for a positive number \(\varepsilon\lesssim \log \frac{1}{\delta}/n\)).
  This stage begins with \(\phi^{T_1+T_2+T_3+1} \lesssim \sqrt{(d\vee\log \sdfrac{1}{\delta})/n}\,\phi^{T_1+T_2+T_3}\lesssim (d\vee\log \sdfrac{1}{\delta})/n\), after which the convergence rate becomes conditional.
  The sub-optimality angle \(\phi^t\) exhibits super-linear convergence of order \(3/2\) if \(n\lesssim d^2/\log^2\frac{1}{\delta}\) and \(\phi^t \gtrsim \log \frac{1}{\delta}/d\);
  otherwise, it demonstrates a linear convergence with \(\phi^{t+1} \lesssim \sqrt{\log \sdfrac{1}{\delta}/n}\, \phi^t\),
  by applying Proposition~\ref{prop:statistical_accuracy_em_updates} at the finite-sample level.
  
  Therefore, this four-stage analysis (Easy EM for \(T_1\) iterations, standard EM for \(T'=T_2+T_3+T_4\) iterations) establishes the convergence of the sub-optimality angle \(\phi^t\). 
  Our finer analysis reduces \(T_1\) from \(\mathcal{O}\left( \log (n/\log\frac{1}{\delta})\right)\) to \(\mathcal{O}\left( \log (d/\log \frac{1}{\delta})\right)\). 
  Furthermore, by establishing Proposition~\ref{prop:statistical_accuracy_em_updates} and applying it in a new fourth stage (both absent in our conference version~\cite{luo24cycloid}),
  we show that the final accuracy of sub-optimality angle \(\phi^t\) can be arbitrary small (up to \(\varepsilon\)),
  thereby strengthening Proposition 5.5 in the conference version of our paper~\cite{luo24cycloid}.
\end{remark}

\begin{theoremtxt}[Finite-Sample Level Convergence]\label{theorem:finite_sample_convergence}
  In the noiseless setting, suppose any initial mixing weights \(\pi^0\) and any initial regression parameters \(\theta^0 \in \mathbb{R}^d\) ensuring that \(\varphi^0 \gtrsim \sqrt{\frac{\log
    \frac{1}{\delta}}{n}} \). Given a positive number \(\varepsilon \lesssim \frac{\log \frac{1}{\delta}}{n}\), we
    run finite-sample Easy EM for at most \(T_1=\mathcal{O}\left( \log \frac{d}{\log \frac{1}{\delta}}\right)\) 
    iterations followed by the finite-sample standard EM for at most \(T' =\mathcal{O} \left(
    [\log \frac{n}{d} \wedge \log \frac{n}{\log \frac{1}{\delta}}]\vee \log[\log\frac{n}{\ln \frac{1}{\delta}}/ \log\frac{n}{d}] \vee [\log\frac{1}{\varepsilon}/\log\frac{n}{\log \frac{1}{\delta}}]\right)\)
    iterations with \(n \gtrsim d \vee \log \frac{1}{\delta} \) samples, then
    \begin{equation}
      \begin{aligned}
      \frac{\| \theta^{T + 1} - \mathrm{sgn}(\rho^{T+1}) \theta^{\ast} \|}{\| \theta^{\ast} \|} 
      = \mathcal{O}
      \left( \varepsilon \sqrt{\frac{\log \frac{1}{\delta}}{n}} \right),\quad
      \left\| \pi^{T + 1} - \bar{\pi}^{\ast} \right\|_1 
      =  
      \mathcal{O}\left(  \varepsilon \left\| \frac{1}{2} - \pi^{\ast}\right\|_1 
      \vee \left[\frac{\log\frac{1}{\delta}/n}{\log\left(1+\frac{\log \frac{1}{\delta}/n}{p(\varepsilon, \pi^\ast)}\right)}
        \wedge \sqrt{\frac{\log \frac{1}{\delta}}{n}} \right]\right),
      \end{aligned}
    \end{equation}
    with probability at least \(1 - T\delta\), where \(T:=T_1+T',\varphi^0 \assign
    \frac{\pi}{2} - \arccos \left| \frac{\langle \theta^0, \theta^{\ast}
    \rangle}{\| \theta^0 \| \cdot \| \theta^{\ast} \|} \right|, 
    \rho^{T+1}\assign \frac{\langle \theta^{T+1}, \theta^{\ast}
    \rangle}{\| \theta^{T+1} \| \cdot \| \theta^{\ast} \|},
    \bar{\pi}^{\ast} \assign \frac{\mathds{1}}{2} + \tmop{sgn} (\rho^{T+1}) 
  (\pi^{\ast} - \frac{\mathds{1}}{2})\),
  and \(p(\varepsilon, \pi^\ast) := \varepsilon \left\| \pi^\ast - \frac{\mathds{1}}{2} \right\|_1 + \min(\pi^\ast(1), \pi^\ast(2))\).
\end{theoremtxt}

\begin{remark}
  Based on Proposition~\ref{prop:convergence_angle}, which establishes the convergence of sub-optimality angle \(\phi^t\) such that \(\phi^T \leq \varepsilon \lesssim \log \frac{1}{\delta}/n\),
  and Proposition~\ref{prop:statistical_accuracy_em_updates}, which provides the statistical accuracy of the EM updates for regression parameters and mixing weights, 
  we obtain the following bound on the accuracy of normalized regression parameters: \(\|\theta^{T+1} - \sgn(\rho^{T+1}) \theta^\ast\|/\|\theta^\ast\| \lesssim \varepsilon^2 + \varepsilon^{3/2} \sqrt{d/n} + \varepsilon\sqrt{\log \sdfrac{1}{\delta}/n} \asymp \varepsilon\sqrt{\log \sdfrac{1}{\delta}/n}\).
  Consequently, the finite-sample EM update \(M_n(\theta, \nu)\) for regression parameters achieves exact recovery as \(\varepsilon \to 0\),
  and acts as a variant of Ordinary Least Squares (OLS) (see Chapter 3 of \cite{hastie2009elements}) which also guarantees exact recovery in the noiseless setting.
  Similarly, from Proposition~\ref{prop:statistical_accuracy_em_updates} for the statistical accuracy of EM updates, 
  noting that \(\frac{\phi^T}{2\pi} \|\pi^\ast - \frac{\mathds{1}}{2} \|_1 + \min(\pi^\ast(1), \pi^\ast(2)) \lesssim \varepsilon \left\| \pi^\ast - \frac{\mathds{1}}{2} \right\|_1 + \min(\pi^\ast(1), \pi^\ast(2)) \equiv p(\varepsilon, \pi^\ast)\),
  the bound on the accuracy of mixing weights is obtained as shown above.
  
  When the ground truth mixing weights \(\pi^\ast\) are close to \((1, 0)\) or \((0, 1)\),
  we have \(p(\varepsilon, \pi^\ast) \to 0\) as \(\varepsilon \to 0\), implying that the final accuracy of mixing weights \(\left\| \pi^{T+1} - \bar{\pi}^{\ast} \right\|_1 \to 0\) as \(\varepsilon \to 0\).
  This means that, after running EM for a sufficient number of iterations, the estimation error of EM estimates for mixing weights and regression parameters can be arbitrary small (up to \(\varepsilon\)) in the noiseless setting, when the ground truth mixing weights \(\pi^\ast\) are either \((1, 0)\) or \((0, 1)\).
  In this case, 2MLR degenerates to a linear regression model, and EM update rules also guarantee exact recovery of mixing weights in the noiseless setting.
  
  When the ground truth mixing weights \(\pi^\ast=(\frac{1}{2}, \frac{1}{2})\) are balanced, 
  we have \(p(\varepsilon, \pi^\ast) \to \varepsilon \left\| \pi^\ast - \frac{\mathds{1}}{2} \right\|_1 + \min(\pi^\ast(1), \pi^\ast(2)) \asymp \varepsilon\),
  hence the final accuracy of mixing weights is on the order of \(\mathcal{O}(\sqrt{\log \sdfrac{1}{\delta}/n})\),
  which is independent of the dimension \(d\) of the regression parameters.
  This result strengthens Theorem 5.1 in the conference version of our paper~\cite{luo24cycloid}, 
  as it demonstrates the exact recovery of regression parameters and characterizes the dependence of the final accuracy of mixing weights on the ground truth mixing weights \(\pi^\ast\)
  through the newly established bound on the statistical error of mixing weights in Proposition~\ref{prop:statistical_error_mixing_weights}.
\end{remark}


\section{Experiments}\label{sec:experiments}
\begin{figure*}[!t]
  \centering
  \subfloat[]{\includegraphics[width=0.31\textwidth]{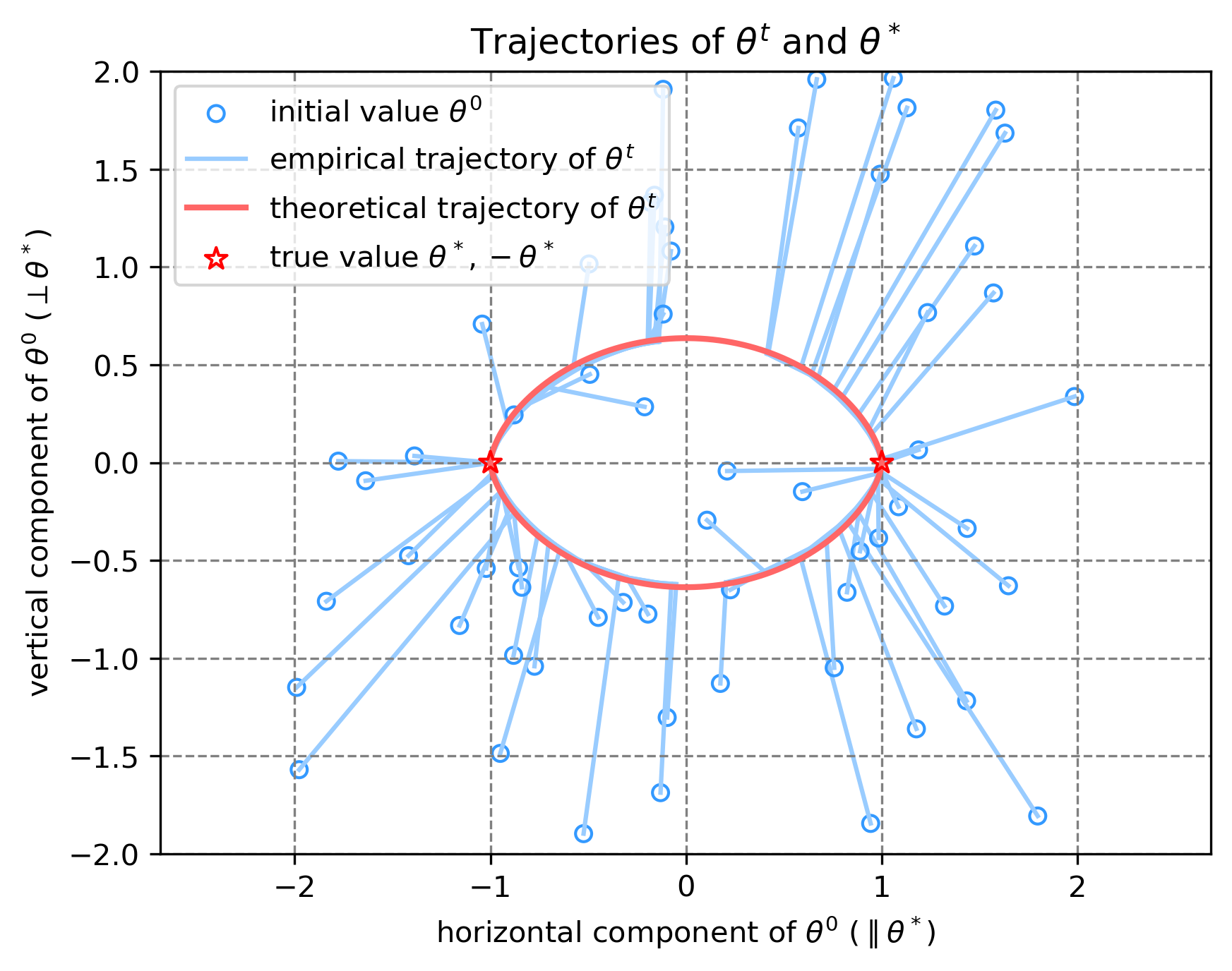}%
  \label{fig:traj_d2}}
  \hfil
  \subfloat[]{\includegraphics[width=0.23\textwidth]{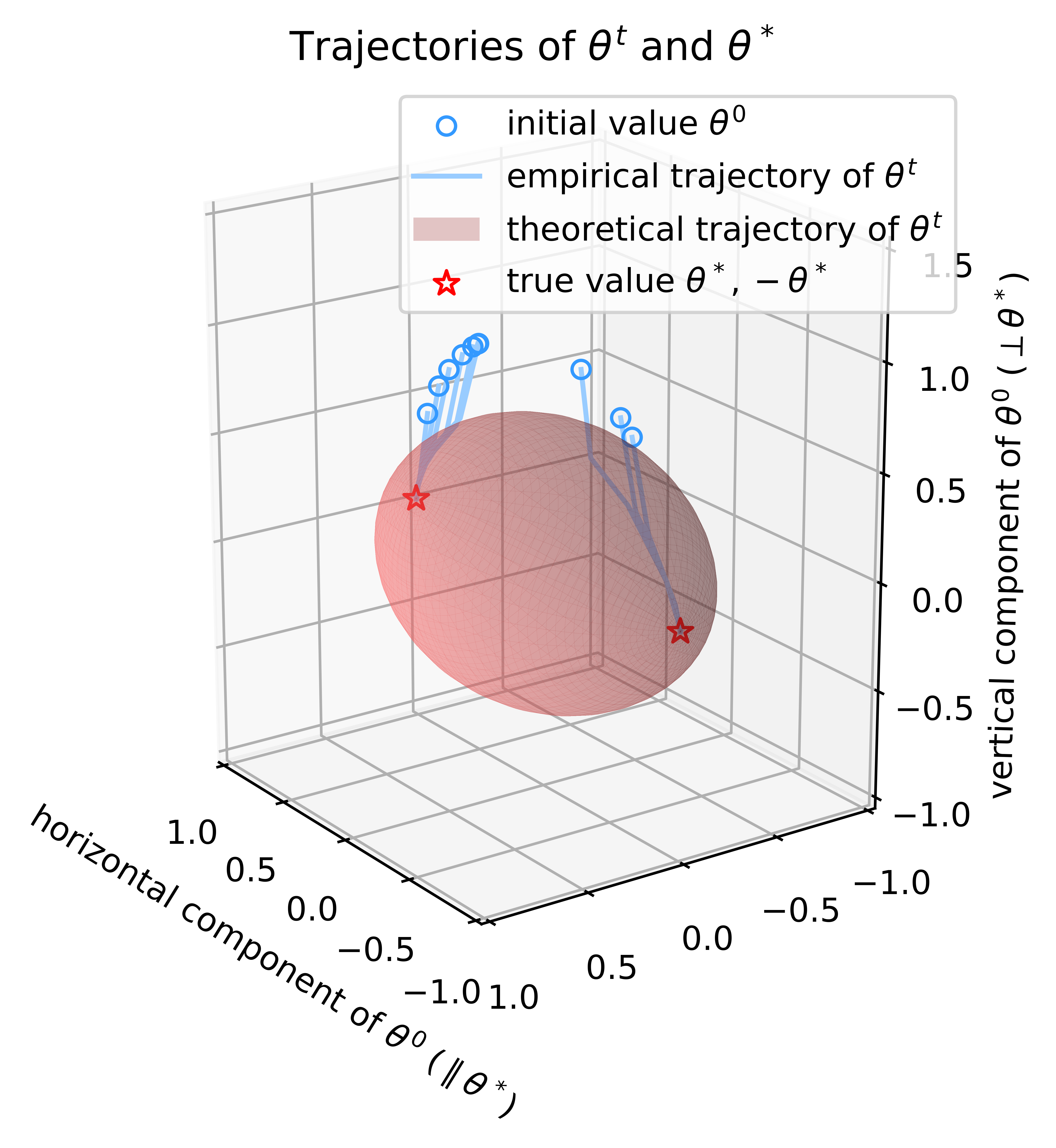}%
  \label{fig:traj_d3}}
  \hfil
  \subfloat[]{\includegraphics[width=0.31\textwidth]{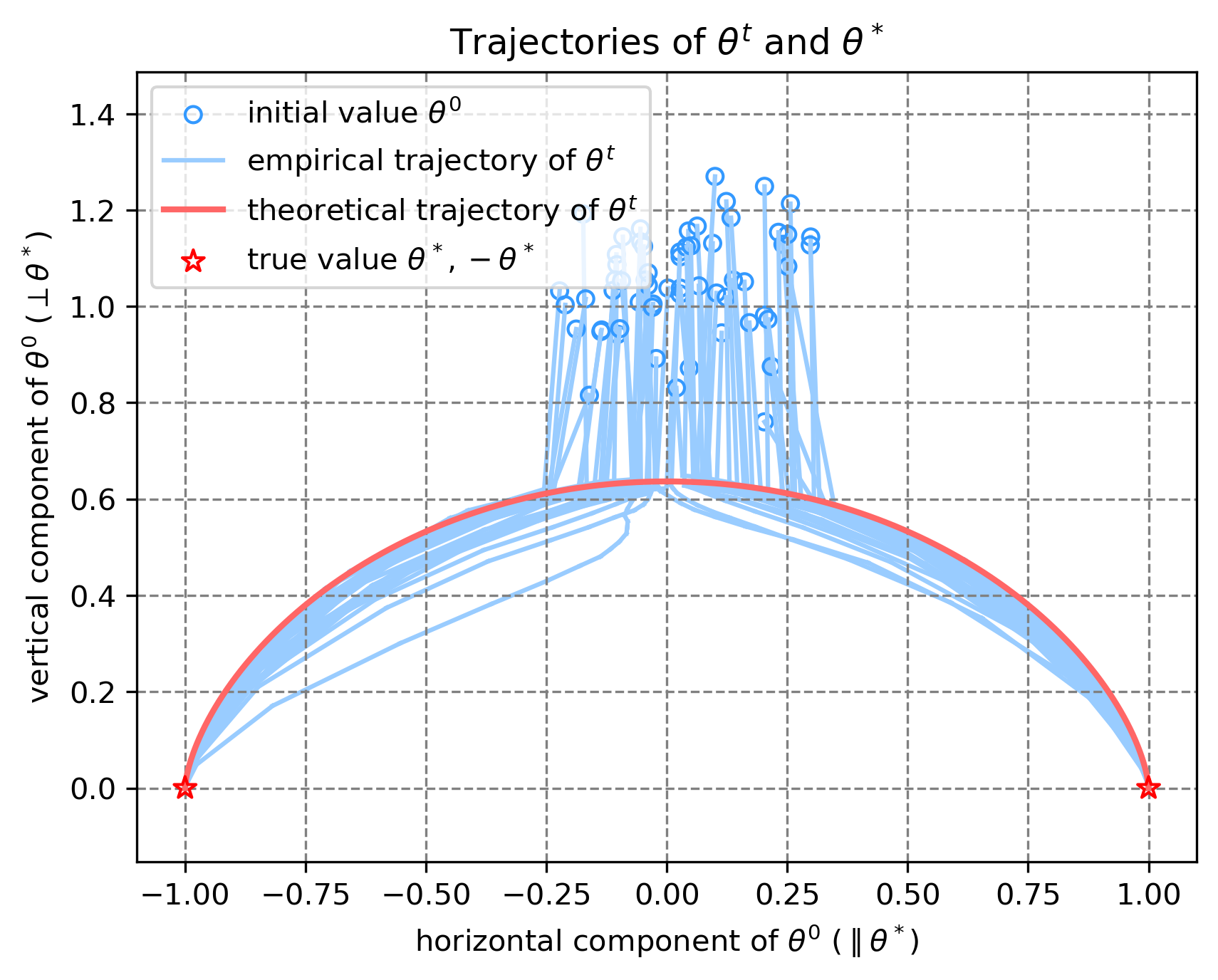}%
  \label{fig:traj_dhigh}}
\caption{Cycloid trajectories of EM iterations for regression parameters $\theta^t$: 
we run 100 iterations of Finite-sample EM at SNR=$10^8$ for varying dimensions ($d=2,3,50$).\\
  (a) $d=2$, trajectories of $\theta^t$ across 60 trials with $\theta^\ast=(1, 0)$, $\pi^\ast=(0.7, 0.3)$; inital values $\theta^0$ and $\pi^0$ are uniformly sampled from $[-2, 2]^2$ and $[0, 1]$, respectively.\\
  (b) $d=3$, trajectories of $\theta^t$ across 10 trials, where $\theta^\ast, \theta^0$ are sampled from three-dimensional unit sphere, and $\pi^\ast, \pi^0$ are drawn uniformly from $[0, 1]$.\\
  (c) $d=50$, trajectories of $\theta^t$ across 60 trials, with $\theta^\ast, \theta^0$ sampled from $\mathcal{N}(0, I_d)$, and $\pi^\ast, \pi^0$ uniformly drawn from $[0, 1]$.}
\label{fig:traj}
\end{figure*}

\begin{figure*}[!htbp]
  \centering
  \subfloat[]{\includegraphics[width=0.31\textwidth]{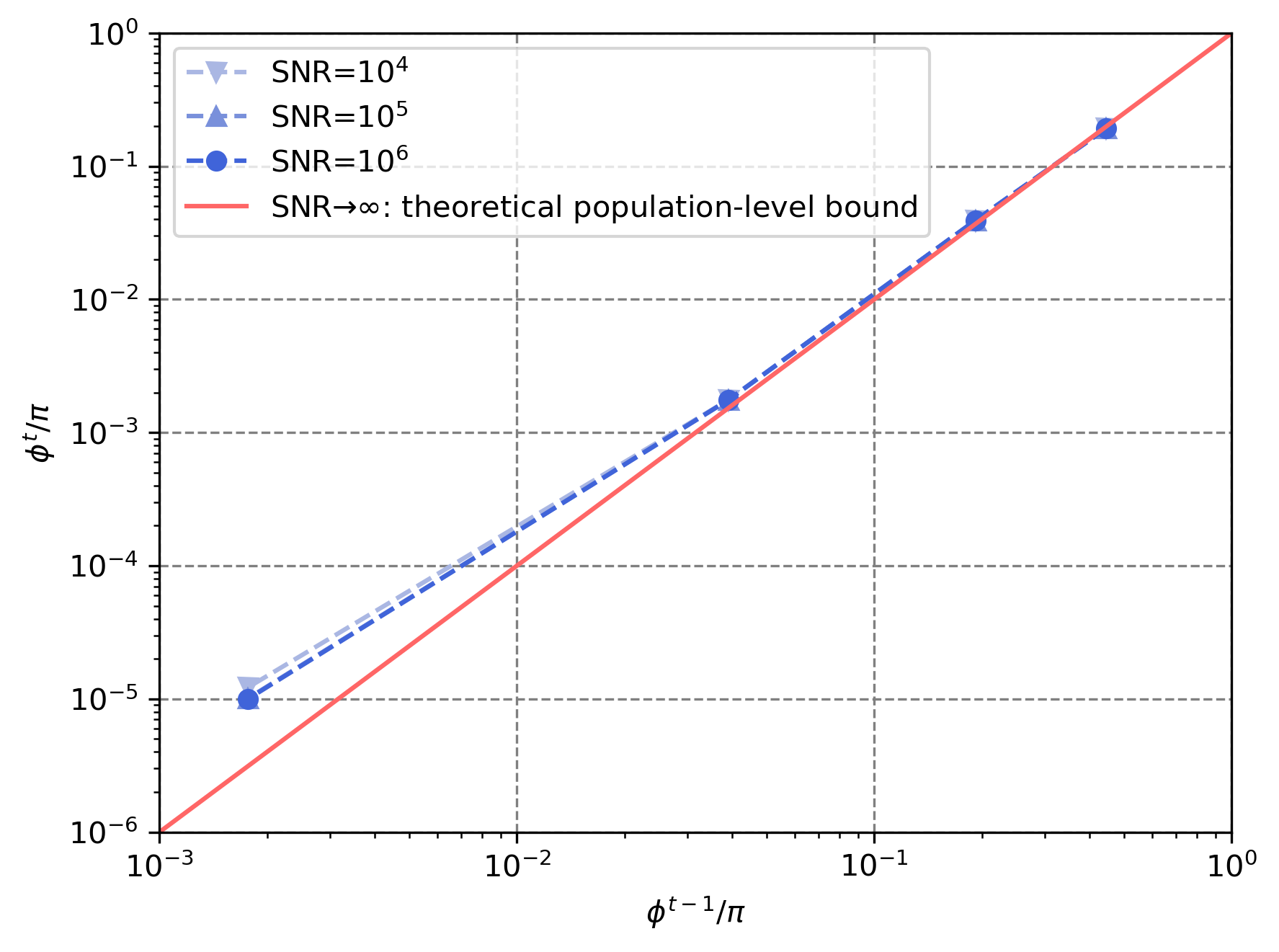}%
  \label{fig:superlinear}}
  \hfil
  \subfloat[]{\includegraphics[width=0.31\textwidth]{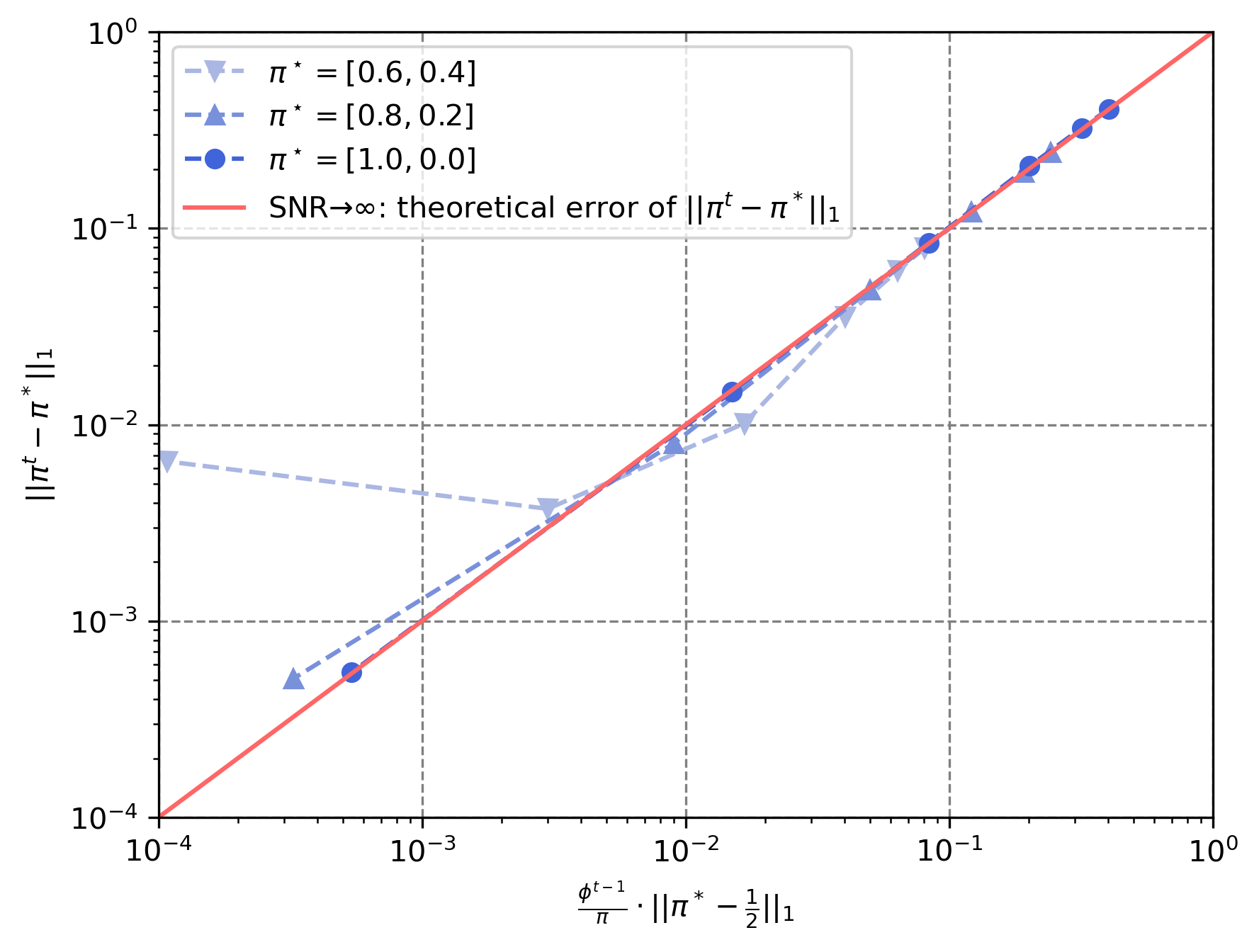}%
  \label{fig:mixing}}
  \hfil
  \subfloat[]{\includegraphics[width=0.31\textwidth]{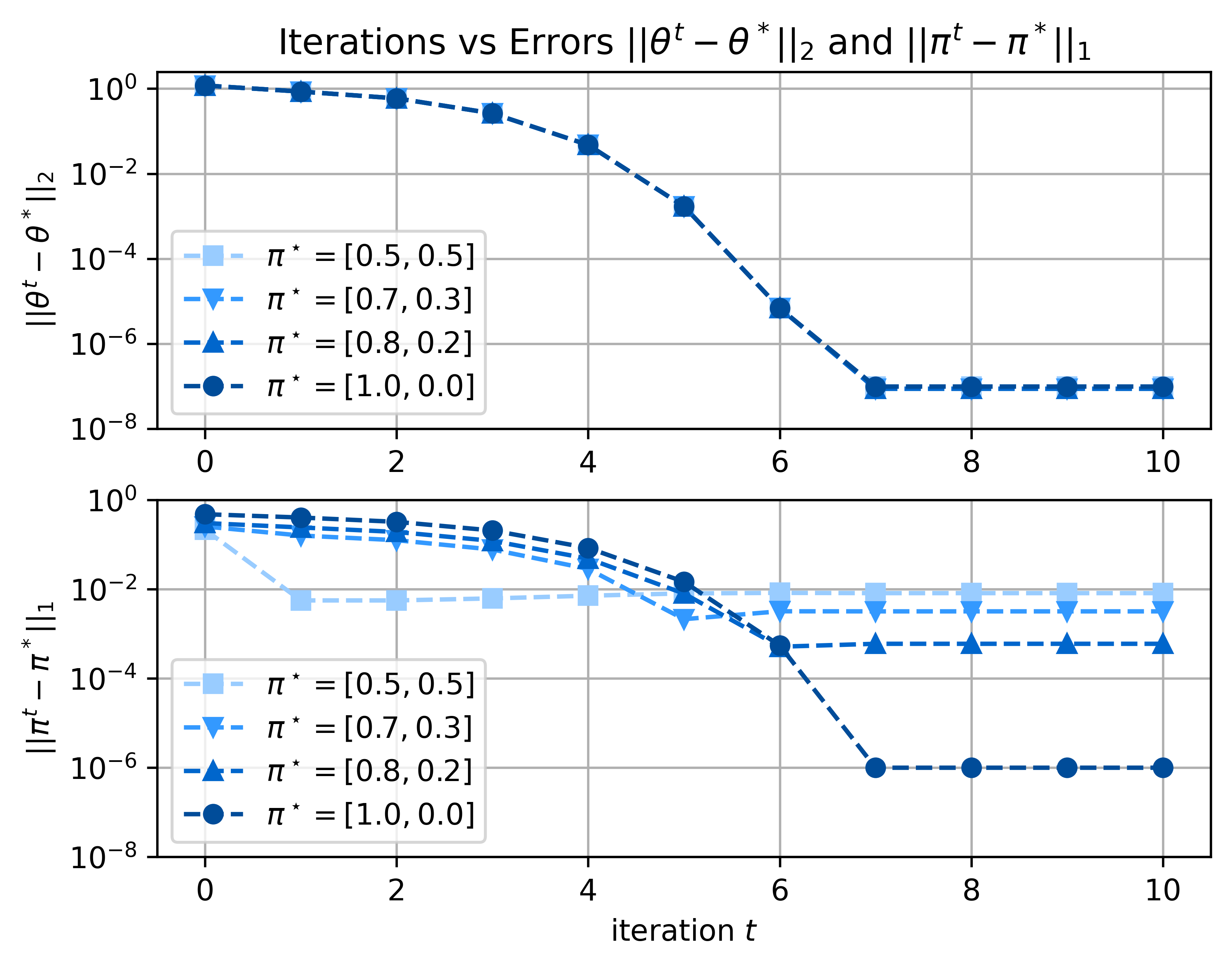}%
  \label{fig:dist}}
\caption{
  Left and middle panels illustrate the quadratic convergence of the sub-optimality angle $\phi^t$ and its correlation with the mixing-weight error. Both use $\theta^\ast$ and $\theta^0$ sampled from the $d=50$ unit sphere, with $\phi^0 = 1.4$ (equivalently, $\varphi^0 = (\pi - 1.4)/2$) in (a) and $\varphi^0 = 0.3$ in (b).
  The right panel (c) shows the accuracy of the EM estimates for the regression parameters and mixing weights over ten EM iterations with $d=50$, $\varphi^0 = 0.3$ at $\mathrm{SNR} = 10^6$, and varying true mixing weights $\pi^\ast =(0.5, 0.5), (0.6, 0.4), (0.8, 0.2)$,and $(1 - 10^{-6}, 10^{-6})$.\\
  (a) Quadratic convergence of the sub-optimality angle $\phi^t$ with all EM iterations starting with $\phi^0 = 1.4$ and $\pi^\ast(1), \pi^0(1)$ drawn uniformly from $[0, 1]$.\\
  (b) Correlation between the mixing-weight error $|\pi^t - \bar{\pi}^\ast|_1$ and the preceding sub-optimality angle $\phi^{t-1} = \arccos\left|\langle \theta^{t-1}, \theta^\ast \rangle/(\|\theta^{t-1}\|\|\theta^\ast\|)\right|$.\\
  (c) Accuracy of the EM estimates for the regression parameters and mixing weights versus iteration for different ground-truth mixing weights.
  }
\label{fig:convg_mix}
\end{figure*}

In this section of numerical experiments, we confirm the theoretical results established in the previous sections.
A total of 5,000 independently and identically distributed (i.i.d.) 
$d$-dimensional covariates, denoted by $\{x_i\}_{i=1}^n$, are sampled from a normal distribution $\mathcal{N}(0, I_d)$. 
The ground truth parameters $\theta^\ast$ are selected at random from a $d$-dimensional unit sphere.
Then, the ground truth mixture weights $\pi^\ast$ for these two components are specified manually or randomly.
We use $\pi^\ast$ to generate latent variable samples $\{z_i\}_{i=1}^n$ from a categorical distribution $\mathcal{CAT}(\pi^\ast)$. 
Gaussian noise \(\varepsilon_i\) is then added to the linear regression \((-1)^{z_i+1} \langle \theta^\ast, x_i \rangle\) indicated by these latent variables,
resulting in output response samples $\{y_i\}_{i=1}^n$.
In all experiments, the full dataset is used for EM updates at each iteration. 
Every point in the plots of Fig.~\ref{fig:convg_mix} represents an average over 50 runs with different initialization values for EM updates.
The code for empirical experiments is available at \url{https://github.com/dassein/cycloid_em_tit}.

\noindent\textbf{Cycloid Trajectory of Regression Parameters. } 
At the population level, we demonstrate that the output of the $t$-th iteration lies on the cycloid in the space $\text{span}\{\theta^{t-1}, \theta^\ast\}$ in the noiseless setting. 
For the corresponding experiments, we set the signal-to-noise ratio (SNR) to $10^8$ and examine different dimensions $d$ (2, 3, and 50). 
As shown in Fig.~\ref{fig:traj}, all iterations remain close to the theoretical cycloid. 
Therefore, the empirical results confirm our theoretical findings in Proposition~\ref{prop:parametric_cycloid}.

\noindent\textbf{Super-linear Convergence of Sub-Optimality Angle. }
We demonstrate the quadratic convergence of $\phi^t/\pi$ in Fig.~\ref{fig:superlinear} in high SNR regime, when $\phi^t$ is large enough.
The experiments are conducted in the dimension ($d$=50), considering different high SNR values ($10^4, 10^5, 10^6$). 
The initial values of the first component of mixing weights and the regression parameters are drawn uniformly from an interval $[0, 1]$ and a unit sphere, respectively. 
Each point corresponding to one of 4 EM iterations in Fig.~\ref{fig:superlinear} represents an average over 50 independent runs with different initializations.
The slope of the plot shows the convergence rate exponent, and the slopes observed at various SNR levels consistently remain close to 2 when $\phi^t$ is large enough.
That observation is consistent with our theoretical result establishing a quadratic convergence rate in Proposition~\ref{prop:quadratic_convergence_angle},
and aligns with our analysis in the remark of Proposition~\ref{prop:statistical_accuracy_em_updates}.

\noindent\textbf{Accuracy of EM Estimate for Mixing Weights and Sub-optimality Angle. }
In the noiseless setting, we show that the accuracy of EM estimate for mixing weights $\|\pi^t-\bar{\pi}^*\|_1$ is proportional to the sub-optimality angle 
$\phi^{t-1}$ at population level in Proposition~\ref{prop:errors_em_updates_angle}.
We illustrate the linear correlation between the accuracy of EM estimate for mixing weights and the sub-optimality angle in Fig.~\ref{fig:mixing}.
For the experimental setup, we set the dimension at $d=50$, and examine different choices of ground truth mixing weights $\pi^\ast=(0.6, 0.4),(0.8, 0.2)$, and $(1, 0)$. 
It is observed that the accuracy of EM estimate for mixing weights at $t$-th iteration is exactly characterized by the sub-optimality angle $\phi^{t-1}$ from the preceding iteration.
Therefore, our empirical results confirm Proposition~\ref{prop:errors_em_updates_angle}, which is consistent with our analysis in the remark of Proposition~\ref{prop:statistical_accuracy_em_updates}.

\noindent\textbf{Finite-Sample Level Convergence with Different Mixing Weights. }
In the noiseless setting, Corollary~\ref{cor:em_updates_noiseless} and Proposition~\ref{prop:recurrence_angle} establish that the EM update for regression parameters $\theta^t$ is unaffected by the ground truth mixing weights $\pi^\ast$.
The first subplot of Fig.~\ref{fig:dist} shows that, at a high SNR ($10^6$), the accuracy of normalized regression parameters(measured in $\ell_2$ norm) remains nearly unchanged across different choices of ground truth mixing weights $\pi^\ast=(0.5, 0.5), (0.7, 0.3), (0.8, 0.2)$, and $(1-10^{-6}, 10^{-6})$, thereby confirming our theoretical analysis 
in Proposition~\ref{prop:errors_em_updates_angle}.

Theorem~\ref{theorem:finite_sample_convergence} further proves that the final accuracy of EM estimate for mixing weights (in $\ell_1$ norm) depends on both the accuracy of normalized regression parametersand the ground truth mixing weights.
In particular, when the accuracy of normalized regression parametersis small, the closer the ground truth mixing weights are to $(1, 0)$ or $(0, 1)$, the smaller the final accuracy of EM estimate for mixing weights.
To verify this result, we investigate statistical accuracy of EM estimate for regression parameters $\theta$ (measured in $\ell_2$ norm) and mixing weights $\pi$ (measured in $\ell_1$ norm) under different ground truth mixing weights $\pi^\ast =(0.5, 0.5), (0.7, 0.3), (0.8, 0.2)$, and $(1-10^{-6}, 10^{-6})$.
The second subplot of Fig.~\ref{fig:dist} depicts the dependence of the statistical accuracy of EM estimate for mixing weights on $\pi^\ast$, providing further support for Theorem~\ref{theorem:finite_sample_convergence}.
In our experiments, $\theta^\ast$ and $\theta^0$ are drawn from the $50$-dimensional unit sphere with initial sub-optimality angle $\varphi^0=0.3$.

\section{Conclusion}\label{sec:conclusion}
We derive explicit expressions for the EM updates in the two-component Mixed Linear Regression (2MLR) model
with unknown mixing weights and regression parameters across all SNR regimes.
We then characterize the properties of EM updates based on the explicit expressions, establishing their structural behavior and boundedness, 
and showing that in the noiseless setting, they follow a cycloid trajectory derived via a recurrence relation for the sub-optimality angle.
In finite high-SNR regimes, we further bound the deviation of the EM updates from this cycloid trajectory.
At the population level, the trajectory-based analysis reveals the order of convergence: linear convergence when the EM estimate is nearly orthogonal to the ground truth regression parameters, 
and quadratic convergence when the angle between the estimate and the ground truth is small.
Furthermore, our work provides a novel trajectory-based framework that establishes non-asymptotic guarantees by tightening bounds of the statistical errors between the finite-sample and population EM updates, 
revealing the connection between EM's statistical accuracy and the sub-optimality angle, 
and establishing convergence guarantees with arbitrary initialization at the finite-sample level.



\newpage
{\setlength{\parindent}{0pt}
\appendices{} 
\tableofcontents

\newpage

\section{Explicit EM Update Expressions and Properties of EM Update Rules}\label{sup:em_update_rules}
\subsection{Connection between EM Update Rules and Gradient Descent of Negative Log-Likelihood}
\begin{theorem}[Proposition~\ref{prop:nll}: Population and Finite-Sample Negative Log-Likelihood]
    Let \(f(\theta, \pi):=-\E_{s\sim p(s\mid \theta^\ast, \pi^\ast)}[\ln p(s\mid \theta, \pi)]\) be the negative log-likelihood function at the population level,
    and \(f_n(\theta, \pi):=-\frac{1}{n}\sum_{i=1}^n \ln p(s_i\mid \theta, \pi)\) be the negative log-likelihood function at the finite-sample level for the dataset \(\mathcal{S}=\{s_i\}_{i=1}^n=\{(x_i, y_i)\}_{i=1}^n\) of \(n\) i.i.d. samples.
    Then, \(f(\theta, \pi)\) and \(f_n(\theta, \pi)\) can be expressed as:
    \begin{eqnarray*}
        f(\theta, \pi) & = & \frac{1}{2\sigma^2}\langle \theta, \E[x x^\top] \cdot \theta \rangle -\E\left[\ln \frac{\cosh\left(y\langle x, \theta \rangle/\sigma^2+\nu\right)}{\cosh \nu}\right]+\mathtt{C}\\
        f_n(\theta, \pi) & = & \frac{1}{2\sigma^2}\left\langle \theta, \frac{1}{n}\sum_{i=1}^n x_i x_i^\top \cdot \theta \right\rangle -\frac{1}{n}\sum_{i=1}^n \ln \frac{\cosh\left(y_i\langle x_i, \theta \rangle/\sigma^2+\nu\right)}{\cosh \nu}+\mathtt{C}_n
    \end{eqnarray*}
    where \(\E[\cdot]=\E_{s\sim p(s\mid \theta^\ast, \pi^\ast)}[\cdot]\) is the expectation over the ground truth distribution \(p(s\mid \theta^\ast, \pi^\ast)\), and \(\mathtt{C}\) and \(\mathtt{C}_n\) are constants that are independent of \(\theta\) and \(\pi\).
\end{theorem}
\begin{proof}
By the detailed derivation of negative log-likelihood in Appendix~\ref{sup:derive_em} (which is adapted from pages 20-25, Appendix B of~\cite{luo24cycloid}), 
we have shown the above expressions for \(f(\theta, \pi)\) and \(f_n(\theta, \pi)\) with these constants \(\mathtt{C}\) and \(\mathtt{C}_n\).
\begin{eqnarray*}
    \mathtt{C} & = & \frac{1}{2\sigma^2} \E[y^2] - \E[\ln p(x)] - \frac{1}{2} \ln \left(2 \pi \sigma^2\right)
    \overset{x\sim \mathcal{N}(0, I_d)}{=} \frac{1}{2\sigma^2} (\|\theta^\ast\|^2 + \sigma^2) - \frac{d}{2} (1+\ln(2\pi)) - \frac{1}{2} \ln \left(2 \pi \sigma^2\right)\\
    \mathtt{C}_n & = & \frac{1}{2\sigma^2}\cdot \frac{1}{n}\sum_{i=1}^n y_i^2 - \frac{1}{n}\sum_{i=1}^n \ln p(x_i) - \frac{1}{2} \ln \left(2 \pi \sigma^2\right)
\end{eqnarray*}
\end{proof}

\begin{theorem}[Lemma~\ref{lemma:em_update_gradients}: EM Update Rules and Gradients]
Let \(U(\theta, \nu):= \E_{s\sim p(s\mid \theta^\ast, \pi^\ast)}\ln \cosh \left(y\langle x, \theta \rangle/\sigma^2+\nu\right)\) at population level and \(U_n(\theta, \nu):= \frac{1}{n}\sum_{i=1}^n \ln \cosh \left(y_i\langle x_i, \theta \rangle/\sigma^2+\nu\right)\) at finite-sample level,
then the population EM update rules \(M(\theta, \nu), N(\theta, \nu)\) and finite-sample EM update rules \(M_n(\theta, \nu), N_n(\theta, \nu)\) for regression parameters \(\theta\) and imbalance of mixing weights \(\tanh \nu\) are:
\begin{equation*}
    \begin{aligned}
        M(\theta, \nu) &= \E[x x^\top]^{-1}\sigma^2\nabla_\theta U(\theta, \nu), \quad 
        &N(\theta, \nu) &= \nabla_\nu U(\theta, \nu),\\
        M_n(\theta, \nu) &= \left(\frac{1}{n}\sum_{i=1}^n x_i x_i^\top\right)^{-1}\sigma^2\nabla_\theta U_n(\theta, \nu), \quad 
        &N_n(\theta, \nu) &= \nabla_\nu U_n(\theta, \nu).
    \end{aligned}
\end{equation*}
\end{theorem}

\begin{proof}
By the derivation of EM update rules in Appendix~\ref{sup:derive_em} (which is adapted from pages 20-25, Appendix B of~\cite{luo24cycloid}), 
and the Leibniz rule such that \(\nabla_\theta \E[\ln\cosh(\cdot)] = \E[\nabla_\theta \ln\cosh(\cdot)], \nabla_\nu \E[\ln\cosh(\cdot)] = \E[\nabla_\nu \ln\cosh(\cdot)]\) and \(\frac{\mathd}{\mathd t} \ln\cosh(t) = \tanh(t)\), we have the population EM update rules as follows, and it is worth noting that \(\E[x x^\top] = I_d\) for \(x\sim \mathcal{N}(0, I_d)\):
\begin{eqnarray*}
    M(\theta, \nu) & = & \E[x x^\top]^{-1} \E_{s\sim p(s\mid \theta^\ast, \pi^\ast)} \tanh \left( \frac{y \langle x, \theta \rangle}{\sigma^2} + \nu \right) y x = \E[x x^\top]^{-1}\nabla_\theta U(\theta, \nu) = \E[x x^\top]^{-1} \sigma^2\nabla_\theta U(\theta, \nu)\\
    N(\theta, \nu) & = & \E_{s\sim p(s\mid \theta^\ast, \pi^\ast)} \tanh \left( \frac{y \langle x, \theta \rangle}{\sigma^2} + \nu \right) = \nabla_\nu U(\theta, \nu)
\end{eqnarray*}
Similarly, we can establish the results for the finite-sample EM update rules by substituting \(f_n, U_n\) for \(f, U\) and  \(\frac{1}{n}\sum_{i=1}^n, s_i:=(x_i, y_i)\) for \(\E_{s\sim p(s\mid \theta^\ast, \pi^\ast)}, s:=(x,y)\) into the above expressions.
\end{proof}

\begin{theorem}[Proposition~\ref{prop:em_update_nll_gradients}: Connection between EM and Gradient Descent]
The population/finite-sample EM update rules give:
\begin{equation*}
    \begin{aligned}
        M(\theta, \nu)   &= \theta - \E[x x^\top]^{-1}\sigma^2 \nabla_\theta f(\theta, \pi), 
        \quad & N(\theta, \nu)   &= \tanh \nu - \nabla_\nu f(\theta, \pi),\\
        M_n(\theta, \nu) &= \theta - \left(\frac{1}{n}\sum_{i=1}^n x_i x_i^\top\right)^{-1}\sigma^2 \nabla_\theta f_n(\theta, \pi), 
        \quad & N_n(\theta, \nu) &= \tanh \nu - \nabla_\nu f_n(\theta, \pi).
    \end{aligned}
\end{equation*}
\end{theorem}
\begin{proof}
Noticing that \(\frac{\mathd}{\mathd t} \ln\cosh(t) = \tanh(t)\) and the direct computation of the gradients of the negative log-likelihood function \(f(\theta, \pi), f_n(\theta, \pi)\) with respect to \(\theta\) and \(\nu\) and comparing with the EM update rules \(M(\theta, \nu), N(\theta, \nu)\) and \(M_n(\theta, \nu), N_n(\theta, \nu)\) gives the results.
It is worth noting that \(\E[x x^\top] = I_d\) for the random variable \(x\sim \mathcal{N}(0, I_d)\).
\end{proof}

\subsection{Explict EM Update Expressions}
\begin{theorem}[Theorem~\ref{theorem:explicit_em_update}: Explict EM Update Expressions]
    Let \(A_\eta = k \eta^2 \sqrt{1+\eta^{-2}}, k=\frac{\| \theta \|}{\| \theta^{\ast} \|}, \eta=\frac{\| \theta^{\ast} \|}{\sigma}\), 
    and \(\varphi = \frac{\pi}{2} - \arccos | \rho |\in (0, \pi/2]\), \(\rho = \frac{\langle \theta, \theta^\ast \rangle}{\|\theta\|\|\theta^\ast\|}\), 
    and \(\varphi_\eta = \arcsin(\sin \varphi/\sqrt{1+\eta^{-2}})\), namely \(\sqrt{1+\eta^{-2}} = \sin \varphi/\sin \varphi_\eta\),
    and Gaussian random variables \(g_\eta, g'\sim \mathcal{N}(0, 1)\) with \(\E[g_\eta g'] = \sin \varphi_\eta\), then population EM update rules for regression parameters \(\theta\) and imbalance of mixing weights \(\tanh \nu\) are:
    \begin{eqnarray*}
        M (\theta, \nu) /\| \theta^{\ast} \| & = & \vec{e}_1 \frac{\sin \varphi}{\sin \varphi_\eta} \mathbb{E}[ \tanh \left(
        A_\eta g_\eta g' + \tmop{sgn} (\rho) (- 1)^{z + 1} \nu
        \right) g_\eta g']\\
        & + & \vec{e}_2  \frac{\cos \varphi}{\cos^2 \varphi_\eta} 
        \tmop{sgn} (\rho) \mathbb{E} [\tanh \left( A_\eta g_\eta g' + \tmop{sgn} (\rho) (- 1)^{z + 1} \nu \right) g_\eta (g_\eta - \sin\varphi_\eta \cdot g')]\\
        N (\theta, \nu) & = & \tmop{sgn} (\rho) \mathbb{E} \left[ \tanh \left( (-
        1)^{z + 1} A_\eta g_\eta g' + \tmop{sgn} (\rho) \nu \right) \right].
    \end{eqnarray*}
\end{theorem}
\begin{proof}

Let's start from the EM update rules,
\(
    M (\theta, \nu) = \mathbb{E} \left[ \tanh \left( \frac{y \langle x,
    \theta \rangle}{\sigma^2} + \nu \right) y x \right],
    N (\theta, \nu) = \mathbb{E} \left[ \tanh \left( \frac{y \langle x,
    \theta \rangle}{\sigma^2} + \nu \right) \right]
\),
also let $k \assign \frac{\| \theta \|}{\| \theta^{\ast} \|}$, $\eta \assign
\frac{\| \theta^{\ast} \|}{\sigma}$, $Z_{\varepsilon} = (- 1)^{z + 1}
\varepsilon / \sigma \sim \mathcal{N} (0, 1)$, and denote cosine of angle
$\rho \assign \frac{\langle \theta, \theta^{\ast} \rangle}{\| \theta \| \cdot
\| \theta^{\ast} \|}$, a pair of orthonormal vectors $\vec{e}_1 \assign
\theta / \| \theta \|, \vec{e}_2 = \frac{\theta^{\ast} - \langle
\theta^{\ast}, \vec{e}_1 \rangle  \vec{e}_1}{\| \theta^{\ast} - \langle
\theta^{\ast}, \vec{e}_1 \rangle \vec{e}_1 \|}$.
By introducing $Z = \langle x, \vec{e}_1 \rangle, Z_{\ast} = \langle
x, \vec{e}_2 \rangle$, and $x^{\perp} = x - Z \vec{e}_1 - Z_{\ast} \vec{e}_2$, then:
\begin{eqnarray*}
y/\| \theta^{\ast} \| &=& (- 1)^{z + 1} [ \rho Z + \sqrt{1 -
   \rho^2} Z_{\ast} + \eta^{-1} \cdot Z_{\varepsilon} ] 
\\
y \langle x, \theta \rangle/\sigma^2 &=& (\| \theta^{\ast} \|^2/\sigma^2) \times
   (y \langle x, \theta \rangle/\| \theta^{\ast} \|^2) = \eta^2 \times
   (- 1)^{z + 1} [ \rho Z + \sqrt{1 - \rho^2} Z_{\ast} + \eta^{-1} \cdot Z_{\varepsilon} ] \cdot k Z 
\\
y x/\| \theta^{\ast} \| &=& (- 1)^{z + 1} [ \rho Z + \sqrt{1 -
   \rho^2} Z_{\ast} + \eta^{-1} \cdot Z_{\varepsilon} ] \cdot [Z
   \vec{e}_1 + Z_{\ast} \vec{e}_2 + x^{\perp}] 
\end{eqnarray*}
Therefore, by using the independence of $Z, Z_{\ast}, Z_{\varepsilon}$ and
$x^{\perp}$, then
\begin{eqnarray*}
  M (\theta, \nu) /\| \theta^{\ast} \| & = & \vec{e}_1 \mathbb{E} \left[ \tanh
  \left( k \eta^2 \left[ \rho Z + \sqrt{1 - \rho^2} Z_{\ast} + \eta^{-1}
  \cdot Z_{\varepsilon} \right] Z + (- 1)^{z + 1} \nu \right) \left[ \rho Z +
  \sqrt{1 - \rho^2} Z_{\ast} + \eta^{-1} \cdot Z_{\varepsilon} \right] Z
  \right]\\
  & + & \vec{e}_2 \mathbb{E} \left[ \tanh \left( k \eta^2 \left[ \rho Z +
  \sqrt{1 - \rho^2} Z_{\ast} + \eta^{-1} \cdot Z_{\varepsilon} \right] Z
  + (- 1)^{z + 1} \nu \right) \left[ \rho Z + \sqrt{1 - \rho^2} Z_{\ast} +
  \eta^{-1} \cdot Z_{\varepsilon} \right] Z_{\ast} \right]\\
  N (\theta, \nu) & = & \mathbb{E} \left[ \tanh \left( (- 1)^{z + 1} k \eta^2
  \left[ \rho Z + \sqrt{1 - \rho^2} Z_{\ast} + \eta^{-1} \cdot
  Z_{\varepsilon} \right] Z + \nu \right) \right]
\end{eqnarray*}
By introducing $Z' \assign (\sqrt{1 - \rho^2} Z_{\ast} + \eta^{-1}
\cdot Z_{\varepsilon})/\sqrt{(1 - \rho^2) + \eta^{-2}}$, then
$\left( Z_{\ast} - (\sqrt{1 - \rho^2}/\sqrt{(1 - \rho^2) + \eta^{-2}}) Z' \right) \ind Z', Z$:
\begin{eqnarray*}
  & & M (\theta, \nu) /\| \theta^{\ast} \| = \vec{e}_1 \mathbb{E} \left[ \tanh
  \left( k \eta^2 \left[ \rho Z + \sqrt{(1 - \rho^2) + \eta^{-2}} Z'
  \right] Z + (- 1)^{z + 1} \nu \right) \left[ \rho Z + \sqrt{(1 - \rho^2) +
  \eta^{-2}} Z' \right] Z \right]\\
  & & + \vec{e}_2 \frac{\sqrt{1 - \rho^2}}{\sqrt{(1 - \rho^2) +
  \eta^{-2}}} \cdot
  \mathbb{E} \left[ \tanh \left( k \eta^2 \left[ \rho Z + \sqrt{(1 -
  \rho^2) + \eta^{-2}} Z' \right] Z + (- 1)^{z + 1} \nu \right) \left[
  \rho Z + \sqrt{(1 - \rho^2) + \eta^{-2}} Z' \right] Z' \right]\\
  & & N (\theta, \nu)  =  \mathbb{E} \left[ \tanh \left( (- 1)^{z + 1} k \eta^2
  \left[ \rho Z + \sqrt{(1 - \rho^2) + \eta^{-2}} Z' \right] Z + \nu
  \right) \right]
\end{eqnarray*}
If \(\rho\neq 0\), then \(\varphi = \frac{\pi}{2} - \arccos | \rho |\in (0, \pi/2]\), therefore \(\rho = \tmop{sgn} (\rho) \sin \varphi, \sqrt{1
- \rho^2} = \cos \varphi\), we define \(A_\eta = k\eta^2 \sqrt{1+\eta^{-2}}\) and \(\varphi_\eta := \arcsin(\sin \varphi/\sqrt{1+\eta^{-2}})\), namely \(\sqrt{1+\eta^{-2}} = \sin \varphi/\sin \varphi_\eta\).
By introducing \(g'\equiv Z'' \assign \tmop{sgn} (\rho) Z \sim \mathcal{N}(0, 1)\) 
and \(g := (\rho Z + \sqrt{1-\rho^2} Z')\sim \mathcal{N}(0, 1), g_\eta \assign (\rho Z + \sqrt{(1-\rho^2) + \eta^{-2}} Z')/\sqrt{1+\eta^{-2}}\) with \(\E[g g'] = \sin\varphi, \E[g_\eta g'] = \sin \varphi_\eta\):
\begin{eqnarray*}
    M (\theta, \nu) /\| \theta^{\ast} \| & = & \vec{e}_1 \frac{\sin \varphi}{\sin \varphi_\eta} \mathbb{E}[ \tanh \left(
    A_\eta g_\eta g' + \tmop{sgn} (\rho) (- 1)^{z + 1} \nu
    \right) g_\eta g']\\
    & + & \vec{e}_2  \frac{\cos \varphi}{\cos^2 \varphi_\eta} 
    \tmop{sgn} (\rho) \mathbb{E} [\tanh \left( A_\eta g_\eta g' + \tmop{sgn} (\rho) (- 1)^{z + 1} \nu \right) g_\eta (g_\eta - \sin\varphi_\eta \cdot g')]\\
    N (\theta, \nu) & = & \tmop{sgn} (\rho) \mathbb{E} \left[ \tanh \left( (-
    1)^{z + 1} A_\eta g_\eta g' + \tmop{sgn} (\rho) \nu \right) \right]
\end{eqnarray*}
If \(\rho=0\), then \(\varphi = 0, \varphi_\eta = 0, \lim_{\varphi\to 0}\sin \varphi_\eta/\sin \varphi = 1/\sqrt{1+\eta^{-2}}, \cos \varphi = \cos \varphi_\eta = 1, A_\eta = k\eta^2 \sqrt{1+\eta^{-2}}\) and \(g = g' = g_\eta\) with \(\E[g g'] = \E[g_\eta g'] =0\).
Therefore, the product \((-1)^{z+1}g_\eta g'\) of two independent standard normal random variables \(g_\eta, g'\) has a density involving Bessel function \(X = (-1)^{z+1}g_\eta g' \sim f_X(x) =K_0(|x|)/\pi\) (see page 50, Section 4.4 Bessel Function Distributions, Chapter 12 Continuous Distributions (General) of~\cite{johnson1970continuous}),
\begin{eqnarray*}
    M(\theta, \nu) /\| \theta^{\ast} \| & = & 
    \vec{e}_1 \sqrt{1+\eta^{-2}} \mathbb{E}[ \tanh \left(A_\eta X +  \nu\right) X]
    +  \vec{e}_2  \mathbb{E} [\tanh \left( A_\eta g_\eta g' + (- 1)^{z + 1} \nu \right) g_\eta^2 ]\\
    & = & \vec{e}_1 \sqrt{1+\eta^{-2}} \mathbb{E}[ \tanh \left(A_\eta X +  \nu\right) X]
    +  \vec{e}_2  \mathbb{E}[(-1)^{z+1}] \E[\tanh \left( A_\eta g_\eta g' + \nu \right) g_\eta^2 ]\\
    N (\theta, \nu) & = &  \mathbb{E} \left[ \tanh \left( A_\eta X + \nu \right) \right]
\end{eqnarray*}
The above expressions when \(\rho =0\) also align with the expressions when \(\rho\neq 0\) with the convention of \(\sgn(0) = 1\) and \((\sin \varphi_\eta/\sin \varphi)_{\rho=0} =\lim_{\varphi\to 0}\sin \varphi_\eta/\sin \varphi = 1/\sqrt{1+\eta^{-2}}\).
\end{proof}

\newpage
\subsection{Bound of the EM Update Rule in General SNR Regime}
\begin{theorem}[Proposition~\ref{prop:boundedness}: Boundedness of EM Update Rule]
    The EM update rule for \(\theta\) is bounded by the following bound, which depends on SNR \(\eta = \| \theta^{\ast} \|/\sigma\):
    \[
    \left\|M(\theta, \nu)\right\| \leq  \frac{\arctan \eta}{\pi/2} \|\theta^\ast\|+ \frac{2}{\pi} \sigma.
    \]
\end{theorem}

\begin{proof}
Since we can express \(g'= \sin \varphi_\eta g_\eta + \cos \varphi_\eta h_\eta\) for some \(h_\eta\sim \mathcal{N}(0, 1)\) with \(h_\eta \ind g_\eta\), 
noting that \(\E[|g_\eta h_\eta|] = \E[|g_\eta|]\cdot \E[|h_\eta|] = \frac{2}{\pi}\) and \(\E[|g_\eta g'|] = \frac{2}{\pi}[\varphi_\eta \sin \varphi_\eta + \cos \varphi_\eta]\), then for \(M(\theta, \nu)\in\text{span}\{\vec{e}_1, \vec{e}_2\}\):
\begin{eqnarray*}
    \left| \langle M(\theta, \nu)/\|\theta^\ast\|, \vec{e}_1 \rangle \right|
    & \leq & \frac{\sin \varphi}{\sin \varphi_\eta} \E[|g_\eta g'|] = \frac{2}{\pi} \sin \varphi\left[\varphi_\eta + \frac{\cos \varphi_\eta}{\sin \varphi_\eta} \right]\\
    \left| \langle M(\theta, \nu)/\|\theta^\ast\|, \vec{e}_2 \rangle \right|
    & \leq & \frac{\cos \varphi}{\cos \varphi_\eta} \E[|g_\eta h_\eta|] = \frac{2}{\pi} \frac{\cos \varphi}{\cos \varphi_\eta} \\
    \left\|M(\theta, \nu)\right\|/\|\theta^\ast\| & = & 
    \sqrt{\left| \langle M(\theta, \nu)/\|\theta^\ast\|, \vec{e}_1 \rangle \right|^2 + \left| \langle M(\theta, \nu)/\|\theta^\ast\|, \vec{e}_2 \rangle \right|^2}\\
    &=& \frac{2}{\pi} \sqrt{H(\sin \varphi, \eta^{-2})} \leq \frac{2}{\pi} \sqrt{H(1, \eta^{-2})} = \frac{2}{\pi} \left(\arctan \eta + \eta^{-1}\right)
\end{eqnarray*}
where the last inequality is due to the fact (Lemma~\ref{suplem:H_nondecreasing}) that \(H(s, \delta)\) is non-decreasing with respect to \(s=\sin \varphi \in[0,1]\) given \(\delta:=\eta^{-2}\geq0\), and \(H(s, \delta)\) is defined as:
\[
H(s, \delta) = \left(s \arcsin\frac{s}{\sqrt{1+\delta}} + \sqrt{1-s^2+\delta}\right)^2 + \frac{(1+\delta)(1-s^2)}{1-s^2 + \delta}
\]
Therefore, we establish the bound for \(\left\|M(\theta, \nu)\right\|\), which depends on SNR \(\eta = \| \theta^{\ast} \|/\sigma\).
\[
\left\|M(\theta, \nu)\right\| \leq  \frac{\arctan \eta}{\pi/2} \|\theta^\ast\|+ \frac{2}{\pi} \sigma
\]
\end{proof}

\subsection{Analysis for the Fixed Points of EM Updates in General SNR Regime}
\begin{theorem}[Proposition~\ref{prop:distinct_fixed_points}: Distinct Fixed Points of EM Update Rules]
    The EM update rules \(M(\theta, \nu), N(\theta, \nu)\) for regression parameters \(\theta\) and the imbalance of mixing weights \(\tanh \nu\) have the following three distinct fixed points \((\theta^\ast, \nu^\ast), (-\theta^\ast, -\nu^\ast), (\vec{0}, 0)\):
    \[
    \begin{pmatrix}
        M(\theta^\ast, \nu^\ast) \\ N(\theta^\ast, \nu^\ast)
    \end{pmatrix} = \begin{pmatrix}
        \theta^\ast \\ \tanh \nu^\ast
    \end{pmatrix},\quad
    \begin{pmatrix}
        M(-\theta^\ast, -\nu^\ast) \\ N(-\theta^\ast, -\nu^\ast)
    \end{pmatrix} = \begin{pmatrix}
        -\theta^\ast \\ \tanh(-\nu^\ast)
    \end{pmatrix},\quad
    \begin{pmatrix}M(\vec{0}, 0) \\ N(\vec{0}, 0) \end{pmatrix} = \begin{pmatrix} \vec{0} \\ \tanh(0) \end{pmatrix}.
    \]
\end{theorem}
\begin{proof}
    For the point \((\theta, \nu) = (\vec{0}, 0)\), by the equations~\eqref{eq:theta},~\eqref{eq:nu} of the EM update rules, we have:
    \begin{eqnarray*}
        M(\vec{0}, 0) = \E\left[\tanh\left(\frac{y\langle x, \vec{0}\rangle}{\sigma^2}+0\right) yx\right] = \tanh(0)\cdot \E[yx] = \vec{0}, \quad
        N(\vec{0}, 0) = \E\left[\tanh\left(\frac{y\langle x, \vec{0}\rangle}{\sigma^2}+0\right)\right] = \tanh(0)
    \end{eqnarray*}
    For the point \((\theta, \nu) = (\theta^\ast, \nu^\ast)\), we have \(k=\|\theta\|/\|\theta^\ast\|=1, \varphi =\pi/2, \sgn(\rho) =1,\vec{e}_1 = \theta/\|\theta\|=\theta^\ast/\|\theta^\ast\|\), 
    then \(\cos \varphi =0,\sin \varphi =1, \sin \varphi_\eta = \sin \varphi/\sqrt{1+\eta^{-2}}=1/\sqrt{1+\eta^{-2}}, \cos \varphi_\eta = \eta^{-1}/\sqrt{1+\eta^{-2}},
    A_\eta = k\eta^2 \sqrt{1+\eta^{-2}}=\eta^2 \sqrt{1+\eta^{-2}}=\sin\varphi_\eta/ \cos^2\varphi_\eta\), 
    by the expressions of the EM update rules with \(g_\eta, g'\sim \mathcal{N}(0, 1)\) and \(\E[g_\eta g']=\sin \varphi_\eta\):
    \begin{eqnarray*}
    M(\theta^\ast, \nu^\ast)&=& \theta^\ast \frac{\cos^2 \varphi_\eta}{\sin \varphi_\eta} 
    \E\left[\tanh\left(\sin \varphi_\eta \frac{g_\eta g'}{\cos^2 \varphi_\eta}+(-1)^{z+1}\nu^\ast \right)\frac{g_\eta g'}{\cos^2 \varphi_\eta}\right]\\
    N(\theta^\ast, \nu^\ast)&=& \mathbb{E} \left[ \tanh \left( (-1)^{z + 1} \sin \varphi_\eta \frac{g_\eta g'}{\cos^2 \varphi_\eta} + \nu^\ast \right) \right]
    \end{eqnarray*}
    Note that the distribution of \(g_\eta g'/\cos^2 \varphi_\eta\) has the density involving Bessel function \(\cos \varphi_\eta \exp(\sin \varphi_\eta\cdot x)K_0(|x|)/\pi\) (see page 50, Section 4.4 Bessel Function Distributions, Chapter 12 Continuous Distributions (General) of~\cite{johnson1970continuous}),
    therefore by introducing the random variable \(X\) with a simple probability density function \(X\sim f_X(x) = K_0(|x|)/\pi\), we can write the expectations as:
    \begin{eqnarray*}
        & &
        \E\left[\tanh\left(\sin \varphi_\eta \frac{g_\eta g'}{\cos^2 \varphi_\eta}+(-1)^{z+1}\nu^\ast \right)\frac{g_\eta g'}{\cos^2 \varphi_\eta}\right]
        = \cos \varphi_\eta 
        \E_X\E_z\left[\exp(\sin\varphi_\eta X)\tanh(\sin\varphi_\eta X + (-1)^{z+1}\nu^\ast)X \right] \\
        & &\E\left[ \tanh \left( (-1)^{z + 1} \sin \varphi_\eta \frac{g_\eta g'}{\cos^2 \varphi_\eta} + \nu^\ast \right) \right]
        = \cos \varphi_\eta \E_X\E_z\left[\exp(\sin\varphi_\eta X) \tanh \left( (-1)^{z+1} \sin \varphi_\eta X + \nu^\ast \right) \right]
    \end{eqnarray*}
    Note that \(\P[z=1]=\frac{1}{2}+\frac{1}{2}\tanh \nu^\ast=\frac{\exp(\nu^\ast)}{2\cosh\nu^\ast}, \P[z=-1]=\frac{1}{2}-\frac{1}{2}\tanh \nu^\ast=\frac{\exp(-\nu^\ast)}{2\cosh\nu^\ast}\), \(X\) is a symmetric random variable, 
    and \((\exp(t)+\exp(-t))\tanh(t)=2\sinh(t), 2\sinh(a+b)=\exp(b)\exp(a)-\exp(-b)\exp(-a), 2\cosh(b) = \exp(b)+\exp(-b)\):
    \begin{eqnarray*}
        & & \E_X\E_z\left[\exp(\sin\varphi_\eta X)\tanh(\sin\varphi_\eta X + (-1)^{z+1}\nu^\ast)X \right] 
        = \frac{1}{\cosh\nu^\ast} \E_X[\sinh(\sin \varphi_\eta X + \nu^\ast)X]
        = \E_X\left[\exp(\sin \varphi_\eta X)X\right]\\
        & & \E_X\E_z\left[\exp(\sin\varphi_\eta X) \tanh \left( (-1)^{z+1} \sin \varphi_\eta X + \nu^\ast \right) \right]
        = \frac{1}{\cosh\nu^\ast} \E_X\left[\sinh(\sin \varphi_\eta X + \nu^\ast)\right]
        = \tanh \nu^\ast \E_X \left[\exp(\sin \varphi_\eta X) \right]
    \end{eqnarray*}
    By letting \(\nu\to 0\) in the third formula in table 6.611
    with modified Bessel function \(K_\nu\), Section 6.61 Combinations of Bessel functions and exponentials, Page 703
    of~\cite{gradshteyn2014table}, applying Leibniz integral rule or the dominated convergence theorem to exchange the order of taking limit
    and taking expectations (see Theorem 1.5.8, page 24 of~\cite{durrett2019probability}) with \(|\alpha|< 1\):
    \begin{eqnarray*}
        \E_X\left[\exp(\alpha X)\right] = \frac{1}{\sqrt{1-\alpha^2}}\implies
        \E_X\left[\exp(\alpha X)X\right]=\E_X\left[\partial_\alpha \exp(\alpha X)\right]
        = \frac{\mathd }{\mathd \alpha} \E_X\left[\exp(\alpha X)\right] = \frac{\alpha}{(1-\alpha^2)^{3/2}}
    \end{eqnarray*}
    By substituting \(\alpha \leftarrow \sin \varphi_\eta\) in the above expressions, we have:
    \[
    \E_X\left[\exp(\sin \varphi_\eta X)X\right] = \frac{\sin \varphi_\eta}{\cos^3 \varphi_\eta}, \quad
    \E_X\left[\exp(\sin \varphi_\eta X)\right] = \frac{1}{\cos \varphi_\eta}
    \]
    Putting the above expressions together, we have shown that:
    \begin{eqnarray*}
        M(\theta^\ast, \nu^\ast) &=& \theta^\ast \frac{\cos^2 \varphi_\eta}{\sin \varphi_\eta} \cos \varphi_\eta 
        \E_X\left[\exp(\sin \varphi_\eta X)X\right] = \theta^\ast \frac{\cos^2 \varphi_\eta}{\sin \varphi_\eta}\cos \varphi_\eta \frac{\sin \varphi_\eta}{\cos^3 \varphi_\eta} = \theta^\ast\\
        N(\theta^\ast, \nu^\ast) &=& 
        \cos \varphi_\eta \tanh \nu^\ast \E_X\left[\exp(\sin \varphi_\eta X) \right] = \cos \varphi_\eta \tanh \nu^\ast \frac{1}{\cos \varphi_\eta} = \tanh \nu^\ast
    \end{eqnarray*}
    For the point \((\theta, \nu) = (-\theta^\ast, -\nu^\ast)\), we have \(k=\|\theta\|/\|\theta^\ast\|=1, \varphi =\pi/2, \sgn(\rho)=-1,\vec{e}_1 = \theta/\|\theta\|=-\theta^\ast/\|\theta^\ast\|\), 
    then \(\cos \varphi =0,\sin \varphi =1, \sin \varphi_\eta = \sin \varphi/\sqrt{1+\eta^{-2}}=-1/\sqrt{1+\eta^{-2}}, \cos \varphi_\eta = \eta^{-1}/\sqrt{1+\eta^{-2}},
    A_\eta = k\eta^2 \sqrt{1+\eta^{-2}}=\eta^2 \sqrt{1+\eta^{-2}}=\sin\varphi_\eta/ \cos^2\varphi_\eta\), 
    by the expressions of the EM update rules with \(g_\eta, g'\sim \mathcal{N}(0, 1)\) and \(\E[g_\eta g']=\sin \varphi_\eta\):
    \begin{eqnarray*}
        M(-\theta^\ast, -\nu^\ast)&=& -\theta^\ast \frac{\cos^2 \varphi_\eta}{\sin \varphi_\eta} 
        \E\left[\tanh\left(\sin \varphi_\eta \frac{g_\eta g'}{\cos^2 \varphi_\eta}+(-1)^{z+1}\nu^\ast \right)\frac{g_\eta g'}{\cos^2 \varphi_\eta}\right]
        = -M(\theta^\ast, \nu^\ast) = -\theta^\ast\\
        N(-\theta^\ast, -\nu^\ast)&=& -\mathbb{E} \left[ \tanh \left( (-1)^{z + 1} \sin \varphi_\eta \frac{g_\eta g'}{\cos^2 \varphi_\eta} + \nu^\ast \right) \right]
        = -N(\theta^\ast, \nu^\ast) = \tanh (-\nu^\ast)
    \end{eqnarray*}
    Therefore, we have verified these three distinct fixed points \((\theta^\ast, \nu^\ast), (-\theta^\ast, -\nu^\ast), (\vec{0}, 0)\) of the EM update rules.
\end{proof}

\begin{theorem}[Proposition~\ref{prop:contraction_property}: Contraction Property of EM Update Rule around the Fixed Point]
    Suppose that the unit direction vector \(\vec{e}^\perp\) 
    is orthogonal to the ground truth \(\theta^\ast\) of regression parameters, namely \(\langle \theta^\ast, \vec{e}^\perp \rangle = 0, \|\vec{e}^\perp\|=1\), 
    there exists \(k^\ast(\eta)>0\) which is determined by \(\eta = \| \theta^{\ast} \|/\sigma\) such that the \((k^\ast(\eta)\|\theta^\ast\|\vec{e}^\perp, 0)\) is a fixed point of the EM update rule \(M(\theta, \nu), N(\theta, \nu)\) for regression parameters \(\theta\) and the imbalance of mixing weights \(\tanh \nu\) given any SNR \(\eta = \| \theta^{\ast} \|/\sigma > 0\):
    \[
    \begin{pmatrix}
        M(k^\ast(\eta)\|\theta^\ast\|\vec{e}^\perp, 0) \\ N(k^\ast(\eta)\|\theta^\ast\|\vec{e}^\perp, 0)
    \end{pmatrix} = \begin{pmatrix}
        k^\ast(\eta)\|\theta^\ast\|\vec{e}^\perp \\ \tanh(0)
    \end{pmatrix}.
    \]
    For any \(k>0\),\(N(k\|\theta^\ast\|\vec{e}^\perp, 0) = \tanh(0)\) and  \(M(k\|\theta^\ast\|\vec{e}^\perp, 0), \vec{e}^\perp\) have the same direction, 
    \[
    \begin{cases}
        0< \|M(k\|\theta^\ast\|\vec{e}^\perp, 0)\|/\|\theta^\ast\|-k < k^\ast(\eta)-k\quad \text{if } k<k^\ast(\eta),\\
        0= \|M(k\|\theta^\ast\|\vec{e}^\perp, 0)\|/\|\theta^\ast\|-k = k^\ast(\eta)-k\quad \text{if } k=k^\ast(\eta),\\
        0>\|M(k\|\theta^\ast\|\vec{e}^\perp, 0)\|/\|\theta^\ast\|-k > k^\ast(\eta)-k\quad \text{if } k>k^\ast(\eta).
    \end{cases}
    \]
    Moreover, the bounds of \(k^\ast(\eta)\) are:
    \[
        \frac{1}{\sqrt{3}} < k^\ast(\eta) < \min\left(\frac{2}{\pi}\sqrt{1+\eta^{-2}}, 1\right), \forall\, \eta>0.
    \]
    with \(\lim_{\eta\to 0_+} k^\ast(\eta) = \frac{1}{\sqrt{3}}\) and \(\lim_{\eta\to \infty} k^\ast(\eta) = \frac{2}{\pi}\).
\end{theorem}

\begin{proof}
    Suppose the regression parameters \(\theta = k\|\theta^\ast\|\vec{e}^\perp\) for some \(k>0\) and the imbalance of mixing weights \(\tanh \nu = 0\), 
    then by proof of for the expressions of the EM update rules, we have the following expressions with \(A_\eta =k\eta^2 \sqrt{1+\eta^{-2}}\):
    \begin{eqnarray*}
        M(k\|\theta^\ast\|\vec{e}^\perp, 0)/\| \theta^\ast \| & = & 
        \vec{e}_1 \sqrt{1+\eta^{-2}} \mathbb{E}[ \tanh \left(A_\eta X +  0\right) X]
        +  \vec{e}_2  \mathbb{E}[(-1)^{z+1}] \E[\tanh \left( A_\eta g_\eta g' + 0 \right) g_\eta^2 ]\\
        & = & \vec{e}^\perp \sqrt{1+\eta^{-2}} \mathbb{E}[ \tanh \left(A_\eta X\right) X]\\
        N(k\|\theta^\ast\|\vec{e}^\perp, 0) & = &  \E \left[ \tanh \left( A_\eta X + 0 \right) \right] 
        = \E \left[ \tanh \left( A_\eta X \right) \right] = 0
    \end{eqnarray*}
    due to the fact that \(\vec{e}_1 = \theta/\|\theta\| = \vec{e}^\perp\) and \(X=(-1)^{z+1}g_\eta g'\sim f_X(x) =K_0(|x|)/\pi\) is a symmetric random variable, and \(g_\eta, g'\sim \mathcal{N}(0, 1)\) with \(\E[g_\eta g']=0\), 
    \[
    \E[\tanh(A_\eta g_\eta g')g_\eta^2] = \E_{g_\eta}\E_{g'}[\tanh(A_\eta g_\eta g')g_\eta^2] = \E_{g_\eta}[0] =0
    \]
    Moreover, the EM update \(M(k\|\theta^\ast\|\vec{e}^\perp, 0)\) and \(\vec{e}^\perp\) have the same direction, since 
    \[
    M(k\|\theta^\ast\|\vec{e}^\perp, 0)=\vec{e}^\perp \sqrt{1+\eta^{-2}} \mathbb{E}[ \tanh \left(A_\eta X\right) X] \in \text{span}\{\vec{e}^\perp\}, \left\langle \frac{M(k\|\theta^\ast\|\vec{e}^\perp, 0)}{\|\theta^\ast\|}, \vec{e}^\perp \right\rangle=\sqrt{1+\eta^{-2}} \mathbb{E}[ \tanh \left(A_\eta X\right) X] > 0
    \]
    For the ease of analysis, we introduce the following function \(G_\eta(k)\) with \(k>0\):
    \[
    G_\eta(k) = \frac{\|M(k\|\theta^\ast\|\vec{e}^\perp, 0)\|}{\|\theta^\ast\|} - k = \sqrt{1+\eta^{-2}} \mathbb{E}[ \tanh \left(A_\eta X\right) X] - k
    = \sqrt{1+\eta^{-2}} \E[\tanh(k\eta^2 \sqrt{1+\eta^{-2}} X)X] - k
    \]
    We obtain the derivative and second derivative of \(G_\eta(k)\) as follows by applying Leibniz integral rule or the dominated convergence theorem to exchange the order of taking derivative and taking expectations (see Theorem 1.5.8, page 24 of~\cite{durrett2019probability}):
    \begin{eqnarray*}
    \frac{\partial G_\eta(k)}{\partial k} &=& \sqrt{1+\eta^{-2}} \E[\partial_k \tanh(k\eta^2 \sqrt{1+\eta^{-2}} X)X] -1
    = \eta^2(1+\eta^{-2}) \E\left[[1-\tanh^2(k\eta^2 \sqrt{1+\eta^{-2}} X)]X^2\right] -1\\
    \frac{\partial^2 G_\eta(k)}{\partial k^2} &=& \eta^2(1+\eta^{-2}) \E\partial_k\left[\frac{X^2}{\cosh^2(k\eta^2 \sqrt{1+\eta^{-2}} X)}\right] 
    = -2\eta^4(1+\eta^{-2})^{3/2} \E\left[\frac{\tanh(k\eta^2\sqrt{1+\eta^{-2}}X)X^3}{\cosh^2(k\eta^2\sqrt{1+\eta^{-2}}X)}\right] < 0
    \end{eqnarray*}
    The initial/final values of \(G_\eta(k)\) and its derivative are obtained by applying the dominated convergence theorem to exchange the order of taking derivative and taking expectations (see Theorem 1.5.8, page 24 of~\cite{durrett2019probability}) and using the fact \(\E[X^2]=1\):
    \begin{eqnarray*}
        \lim_{k\to 0_+} G_\eta(k) &=& \sqrt{1+\eta^{-2}} \E[\lim_{k\to 0_+} \tanh(k\eta^2 \sqrt{1+\eta^{-2}} X)X] - 0 = \E[0]-0 =0\\
        \lim_{k\to 0_+} \frac{\partial G_\eta(k)}{\partial k} &=& \eta^2(1+\eta^{-2}) \E\left[[1-\lim_{k\to 0_+}\tanh^2(k\eta^2 \sqrt{1+\eta^{-2}} X)]X^2\right] -1 = \eta^2(1+\eta^{-2}) -1=\eta^{2} > 0\\
        \lim_{k\to \infty} G_\eta(k) &=& \sqrt{1+\eta^{-2}} \E[\lim_{k\to \infty} \tanh(k\eta^2 \sqrt{1+\eta^{-2}} X)X] -\lim_{k\to \infty} k = \E[|X|]-\infty = -\infty\\
        \lim_{k\to \infty} \frac{\partial G_\eta(k)}{\partial k} &=& \eta^2(1+\eta^{-2}) \E\left[[1-\lim_{k\to \infty}\tanh^2(k\eta^2 \sqrt{1+\eta^{-2}} X)]X^2\right] -1 = \eta^2(1+\eta^{-2})\E[0] -1 = -1
    \end{eqnarray*}
    The signs of some special values of \(G_\eta(k)\) are determined by Lemma~\ref{suplem:bounds_expectation_A} 
    with \(A:= \eta^2 \sqrt{1+\eta^{-2}}, \frac{1}{\sqrt{1+\eta^{-2}}}=\frac{\sqrt{4A^2+1}-1}{2A}\) 
    and \(A' = \frac{1}{\sqrt{3}}\eta^2\sqrt{1+\eta^{-2}}, \frac{1}{\sqrt{1+\eta^{-2}}}=\frac{\sqrt{12A'^2+1}-1}{2\sqrt{3}A'}\)
    and noting the fact that \(\tanh(x) < 1, \forall x> 0\) and \(\E[|X|] = \frac{2}{\pi}\):
    \begin{eqnarray*}
        & & G_\eta\left(\frac{2}{\pi}\sqrt{1+\eta^{-2}}\right) < \sqrt{1+\eta^{-2}}\E[|X|] - \frac{2}{\pi}\sqrt{1+\eta^{-2}} = 0\\
        & & G_\eta(1) =  \sqrt{1+\eta^{-2}}\left(\E[\tanh(AX)X] - \frac{\sqrt{4A^2+1}-1}{2A}\right) < 0\\
        & & G_\eta\left(\frac{1}{\sqrt{3}}\right) = 
        \sqrt{1+\eta^{-2}}\left(\E[\tanh(A'X)X] - \frac{\sqrt{12A'^2+1}-1}{6A'}\right) > 0
    \end{eqnarray*}
    \textbf{Existence and uniqueness of \(k^\ast(\eta)\)}: Therefore, there exists \(k^\ast(\eta)\in(\frac{1}{\sqrt{3}}, \min(\frac{2}{\pi}\sqrt{1+\eta^{-2}}, 1))\) such that \(G_\eta(k^\ast(\eta)) = 0\) by the continuity of \(G_\eta(k)\) and the intermediate value theorem
    (see 4.23 Theorem on page 93 of~\cite{rudin1976}). Moreover, the root \(k^\ast(\eta)\) of \(G_\eta(k) = 0\) is unique, otherwise it would contradict to the fact of \(G_\eta(\frac{1}{\sqrt{3}}) > 0\) and \(\frac{\partial^2 G_\eta(k)}{\partial k^2} < 0\).
    If there exists another root \(\tilde{k}^\ast(\eta)\) such that \(G_\eta(\tilde{k}^\ast(\eta)) = 0\) (we may assume \(k^\ast(\eta)<\tilde{k}^\ast(\eta)\)), 
    then there exists \(\bar{k}\in(k^\ast(\eta), \tilde{k}^\ast(\eta))\) such that \(\frac{\partial G_\eta(k)}{\partial k}\mid_{k=\bar{k}} = 0 =\frac{G_\eta(\tilde{k}^\ast(\eta))-G_\eta(k^\ast(\eta))}{\tilde{k}^\ast(\eta)-k^\ast(\eta)}\) 
    by using the mean value theorem (see 5.10 Theorem on page 108 of~\cite{rudin1976}).
    Then by \(\frac{\partial^2 G_\eta(\bar{k})}{\partial k^2} < 0\), we have \(\frac{\partial G_\eta(k)}{\partial k} > \frac{\partial G_\eta(k)}{\partial k}\mid_{k=\bar{k}} = 0\) for \(k\in(\frac{1}{\sqrt{3}}, \bar{k})\), since \(k^\ast(\eta) \in (\frac{1}{\sqrt{3}}, \bar{k})\), 
    we have \(G_\eta(k^\ast(\eta)) > G_\eta(\frac{1}{\sqrt{3}}) > 0\), which contradicts to the fact of \(G_\eta(k^\ast(\eta)) = 0\).

    \textbf{Contraction mapping with fixed point \(k^\ast(\eta)\)}: Again, by applying the mean value theorem to \(G_\eta(k)\) on \((0, k^\ast(\eta))\) and noting the fact that \(\lim_{k\to 0_+} G_\eta(k) = G_\eta(k^\ast(\eta)) = 0\), 
    there exists \(\bar{k}\in(0, k^\ast(\eta))\) such that \(G_\eta(\bar{k}) = 0\) and \(\frac{\partial G_\eta(k)}{\partial k}\mid_{k=\bar{k}} = 0\).
    Noting \(\frac{\partial^2 G_\eta(k)}{\partial k^2} < 0\), we have \(\frac{\partial G_\eta(k)}{\partial k} > \frac{\partial G_\eta(k)}{\partial k}\mid_{k=\bar{k}} = 0\) for \(k\in(0, \bar{k})\), 
    and \(\frac{\partial G_\eta(k)}{\partial k} < \frac{\partial G_\eta(k)}{\partial k}\mid_{k=\bar{k}} = 0\) for \(k\in(\bar{k}, \infty)\), and therfore 
    \begin{eqnarray*}
        G_\eta(k) > \lim_{k\to 0_+} G_\eta(k) = 0\quad \forall k\in(0, \bar{k}];\quad
        G_\eta(k) > G_\eta(k^\ast(\eta)) = 0\quad \forall k\in(\bar{k}, k^\ast(\eta))\\
        \implies G_\eta(k) > 0\quad \forall k\in(0, k^\ast(\eta));\quad
        G_\eta(k) < G_\eta(k^\ast(\eta)) = 0\quad \forall k\in(k^\ast(\eta), \infty)
    \end{eqnarray*}
    Recalling the definition of \(G_\eta(k):=\frac{\|M(k\|\theta^\ast\|\vec{e}^\perp, 0)\|}{\|\theta^\ast\|} - k\) for \(k>0\), 
    noting \(G_\eta(k)+k=\|M(k\|\theta^\ast\|\vec{e}^\perp, 0)\|/\|\theta^\ast\| = \sqrt{1+\eta^{-2}} \mathbb{E}[ \tanh \left( k\eta^2 \sqrt{1+\eta^{-2}} X\right) X]\) is strictly increasing with respect to \(k\), we have:
    \begin{eqnarray*}
        0< \frac{\|M(k\|\theta^\ast\|\vec{e}^\perp, 0)\|}{\|\theta^\ast\|} - k = G_\eta(k) = (G_\eta(k)+k) -k < (G_\eta(k^\ast(\eta))+k^\ast(\eta)) -k = k^\ast(\eta) -k\quad \text{if } k<k^\ast(\eta)\\
        0= \frac{\|M(k\|\theta^\ast\|\vec{e}^\perp, 0)\|}{\|\theta^\ast\|} - k = G_\eta(k) = (G_\eta(k)+k) -k = (G_\eta(k^\ast(\eta))+k^\ast(\eta)) -k = k^\ast(\eta) -k\quad \text{if } k=k^\ast(\eta)\\
        0> \frac{\|M(k\|\theta^\ast\|\vec{e}^\perp, 0)\|}{\|\theta^\ast\|} - k = G_\eta(k) = (G_\eta(k)+k) -k > (G_\eta(k^\ast(\eta))+k^\ast(\eta)) -k = k^\ast(\eta) -k\quad \text{if } k>k^\ast(\eta)
    \end{eqnarray*}
    \textbf{Limit behavior of \(k^\ast(\eta)\)}: To analyze the limit bahavior of \(k^\ast(\eta)\) as \(\eta\to 0_+\) and \(\eta\to \infty\), 
    we notice that \(A:=\eta^2\sqrt{1+\eta^{-2}}=\eta\sqrt{1+\eta^2}\to 0_+\) as \(\eta\to 0_+\) and \(A\to \infty\) as \(\eta\to \infty\),
    and there is a one-to-one increasing mapping \(\eta(A)\) from \(A\) to \(\eta\).
    Consequently, we write \(k^\ast(\eta(A))\) as \(k^\ast(A)\) for simplicity. 
    From the definition of \(k^\ast(\eta)\) such that \(G_\eta(k^\ast(\eta)) = 0\), 
    we show that \(k=k^\ast(A)\) is the unique root of \(F_1(A, k) = 0\) for any \(A>0\), and is also the unique root of \(F_2(A, k) = 0\) for any \(A>0\),
    where \(F_1(A, k), F_2(A, k)\) are defined as follows:
    \begin{eqnarray*}
        F_1(A, k) &=& \left(\E[\tanh(kA X)X] - k \frac{\sqrt{4A^2+1}-1}{2A}\right)/A^3\\
        F_2(A, k) &=& \E[\tanh(kA X)X] - k \frac{\sqrt{4A^2+1}-1}{2A}
    \end{eqnarray*}
    These \(F_1(A, k), F_2(A, k)\) are continuous in \(A\) and \(k\) for any \(A>0, k>0\), their limits as \(A\to 0_+\) and \(A\to \infty\) are as follows:
    \begin{eqnarray*}
        F_{1, 0_+}(k) &= & \lim_{A\to 0_+} F_1(A, k) = \lim_{A\to 0_+} 3k^3 -k +\mathcal{O}(A^2) = k(3k^2-1) \\
        F_{2, \infty}(k) &= & \lim_{A\to \infty} F_2(A, k) = \E[|X|]- k = \frac{2}{\pi} - k \\
    \end{eqnarray*}
    due to that \(\tanh(t) = t - \frac{t^3}{3} + \mathcal{O}(t^5), \frac{\sqrt{4A^2+1}-1}{2A} = A - A^3 + \mathcal{O}(A^{5})\) and \(\E[X^{2n}]=[(2n-1)!!]^2, \E[|X|] = \frac{2}{\pi}\) by Lemma~\ref{suplem:moments}.
    We denote the unique positive root of \(F_{1, 0_+}(k) = 0\) as \(k^\ast_{0_+}\) and the unique root of \(F_{2, \infty}(k) = 0\) as \(k^\ast_{\infty}\).
    \[
    k^\ast_{0_+} = \frac{1}{\sqrt{3}},\quad 
    k^\ast_{\infty} = \frac{2}{\pi}
    \]
    We show that the limits of \(k^\ast(A)\) as \(A\to 0_+\) and \(A\to \infty\) are \(k^\ast_{0_+}\) and \(k^\ast_{\infty}\), respectively, 
    \[
    \lim_{\eta\to 0_+}k^\ast(\eta) = \lim_{A\to 0_+} k^\ast(A) = k^\ast_{0_+} = \frac{1}{\sqrt{3}},\quad 
    \lim_{\eta\to \infty}k^\ast(\eta) = \lim_{A\to \infty} k^\ast(A) = k^\ast_{\infty} = \frac{2}{\pi}
    \]
    by checking the following conditions of \(F_1(A, k), F_2(A, k)\) and \(F_{1, 0_+}(k), F_{2, \infty}(k)\) and applying Lemma~\ref{suplem:limit_delta_epsilon} for the rigourous analysis
    and susbstiting \(A_0\leftarrow 0_+, \infty\), \(F\leftarrow F_1, F_2\), \(k_{A_0}^\ast\leftarrow k^\ast_{0_+}, k^\ast_{\infty}\) in this lemma, respectively.

    \textbf{Conditions of \(F_1(A, k), F_2(A, k)\) and \(F_{1, 0_+}(k), F_{2, \infty}(k)\)}: (1) \(F_1(A, k), F_2(A, k)\) are continuous in \(k\) for any \(A>0\), since 
    \(\frac{\partial F_1(A, k)}{\partial k}\) and \(\frac{\partial F_2(A, k)}{\partial k}\) exists for any \(k,A>0\).
    The root \(k=k^\ast(A)\) of \(F_1(A,k) =0, F_2(A,k) =0\) is unique for any \(A>0\).

    (2) \(F_1(A,k), F_2(A, k)\) converge to \(F_{1, 0_+}(k), F_{2, \infty}(k)\) on compact interval \(I:=[\frac{1}{2}, 1]\) as \(A\to 0_+\) and \(A\to \infty\), respectively.
    \[
    \sup_{k\in I}|F_1(A, k) - F_{1, 0_+}(k)| \leq \sup_{k\in I} \max(|30k^5 A^2|, |2 k A^2|) =  30 A^2 < \varepsilon
    \]
    for \(0<A < \sqrt{\varepsilon/30}, \forall\varepsilon > 0\), by using \(t-\frac{1}{3} t^3 <\tanh(t) < t - \frac{t^3}{3} + \frac{2}{15}t^5, \forall t>0\), \(A - A^3 < \frac{\sqrt{4A^2+1}-1}{2A} < A-A^3 +2A^5\) and \(\E[X^{2n}]=[(2n-1)!!]^2\) in Lemma~\ref{suplem:moments}.
    Similarly, we have
    \begin{eqnarray*}
    \sup_{k\in I}|F_2(A, k) - F_{2, \infty}(k)| &\leq& \sup_{k\in I} |2\E[|X|\exp(-2kA|X|)]|+ \left|\frac{k}{2} A^{-1}\right|
    < \frac{4}{\sqrt{2\pi}}\int^\infty_0 \sqrt{x} \exp(-Ax)\mathd x + \frac{1}{2}A^{-1}\\
    &=& \sqrt{2} A^{-\frac{3}{2}} + \frac{1}{2}A^{-1} \leq \frac{5}{2}A^{-1} < \varepsilon
    \end{eqnarray*}
    for \(A>\max(\frac{1}{2}, \frac{5}{2\varepsilon}), \forall\varepsilon > 0\), by using \(1 -2\exp(-2t) <\tanh(t) < 1-\exp(-2t), \forall t>0\), \(1- \frac{1}{2}A^{-1} < \frac{\sqrt{4A^2+1}-1}{2A} < 1- \frac{1}{2}A^{-1} + \frac{1}{8}A^{-2}\leq 1- \frac{1}{4}A^{-1}, \forall A\geq \frac{1}{2}\)
    and \(K_0(x) < K_{\frac{1}{2}}(x) = \sqrt{\frac{\pi}{2}} \frac{\exp(-x)}{\sqrt{x}}, \forall x>0\) by the monotonicity property of \(K_\nu\) with respect to \(\nu\geq0\) (see Section 10.37 Inequalities; Monotonicity; Section 10.39 Relation to Other
    Functions, Chapter 10 Bessel Function of~\cite{olver2010nist}).

    (3) The roots \(k^\ast_{0_+}=\frac{1}{\sqrt{3}}, k^\ast_{\infty}=\frac{2}{\pi}\) of \(F_{1, 0_+}(k)=0, F_{2, \infty}(k)=0\) are the interior points of the compact interval \(I=[\frac{1}{2}, 1]\), and 
    \(\frac{\mathd F_{1, 0_+}(k)}{\mathd k}\mid_{k=k^\ast_{0_+}} = 2 \neq 0, \frac{\mathd F_{2, \infty}(k)}{\mathd k}\mid_{k=k^\ast_{\infty}} = -1 \neq 0\),
    and \(\frac{\mathd F_{1, 0_+}(k)}{\mathd k}=9k^2-1, \frac{\mathd F_{2, \infty}(k)}{\mathd k}=-1\) are continuous in \(k\) on \(I\).
\end{proof}

\subsection{EM Update Rule in the Noiseless Setting}
\begin{theorem}[Corollary~\ref{cor:em_updates_noiseless}: EM Updates in Noiseless Setting, Corollary 3.3 in~\cite{luo24cycloid}]
    In the noiseless setting, namely SNR \(\eta := \| \theta^{\ast} \|/\sigma \to \infty\), the population EM update rules \(M(\theta, \nu), N(\theta, \nu)\) for regression parameters \(\theta\) and imbalance of mixing weights \(\tanh \nu\) are:
    \begin{eqnarray*}
        & & \frac{M(\theta, \nu)}{\| \theta^{\ast} \|} = \frac{2}{\pi} \left[ \sgn(\rho)\varphi \frac{\theta^{\ast}}{\| \theta^{\ast} \|}  + \cos \varphi \frac{\theta}{\| \theta \|}  \right]\\
        & & N(\theta, \nu) = \sgn(\rho) \frac{2}{\pi} \varphi \cdot \tanh \nu^{\ast}
    \end{eqnarray*}
    where \(\varphi := \frac{\pi}{2}-\arccos |\rho|\), \(\rho := \frac{\langle \theta, \theta^\ast \rangle}{\|\theta\|\|\theta^\ast\|}\).
\end{theorem}
\begin{proof}
    We provide a completely new proof based on our previous newly established results rather than the old approach in~\cite{luo24cycloid}.
    Taking the limit $\eta \rightarrow +\infty$, then \(\sqrt{1+\eta^{-2}} \rightarrow 1, A_\eta \to +\infty, g_\eta \to g\) and \(\varphi_\eta \to \varphi, \cos \varphi_\eta \rightarrow \cos \varphi, \sin \varphi_\eta \rightarrow \sin \varphi\), we have
    \begin{eqnarray*}
        \lim_{\eta \rightarrow \infty} M (\theta, \nu) /\| \theta^{\ast} \| & = &
        \vec{e}_1 \mathbb{E} [|g g'|] + \vec{e}_2 \frac{1}{\cos \varphi} \tmop{sgn} (\rho) \mathbb{E} [g(g - \sin \varphi \cdot g') \tmop{sgn}(g g')]\\
        & = & \vec{e}_1 \E[|g g'|] + \vec{e}_2 \frac{\tmop{sgn}(\rho)}{\cos \varphi} \left( \E[g^2 \tmop{sgn}(g g')] - \sin \varphi \E[| g g' |] \right)\\
        \lim_{\eta \rightarrow \infty} N (\theta, \nu) & = & \tmop{sgn} (\rho)
        \mathbb{E} [(- 1)^{z + 1}] \cdot \mathbb{E}[\tmop{sgn} (g g')]
        = \tmop{sgn} (\rho) \tanh \nu^{\ast} \mathbb{E} [\tmop{sgn}(g g')]
    \end{eqnarray*}
    By applying Lemma~\ref{suplem:expectation}, we have \(\E[\sgn(g g')] = \frac{2}{\pi} \varphi, \E[g^2 \sgn(g g')] = \frac{2}{\pi} [\varphi + \sin \varphi \cos \varphi]\), and \(\E[|g g'|] = \frac{2}{\pi} [\varphi \sin \varphi + \cos \varphi]\):
    \begin{eqnarray*}
    \lim_{\eta \rightarrow \infty} M (\theta, \nu) /\| \theta^{\ast} \| 
    & = & \frac{2}{\pi} [\varphi \sin \varphi + \cos \varphi] \vec{e}_1
    + \tmop{sgn} (\rho) \frac{2}{\pi} \varphi \cos \varphi \vec{e}_2\\
    \lim_{\eta \rightarrow \infty} N (\theta, \nu) & = & \tmop{sgn} (\rho)
    \frac{2}{\pi} \varphi \cdot \tanh \nu^{\ast}
    \end{eqnarray*}
    Noting that \(\vec{e}_1 = \theta / \| \theta \|, \sgn(\rho)\sin \varphi \vec{e}_1 + \cos \varphi \vec{e}_2 = \theta^{\ast} / \| \theta^{\ast} \|\), we have:
    \[
    \lim_{\eta \rightarrow \infty} \frac{M (\theta, \nu)}{\| \theta^{\ast} \|} = \frac{2}{\pi}\left[ \sgn(\rho)\varphi \frac{\theta^{\ast}}{\| \theta^{\ast} \|}  + \cos \varphi \frac{\theta}{\| \theta \|}  \right]
    \]
\end{proof}

\subsection{Deviation from Cycloid Limit of EM Updates in High SNR Regime}
\begin{theorem}[Proposition~\ref{prop:deviation_cycloid_limit}: Deviation from Cycloid Limit of EM Updates in High SNR Regime]
    In the finite high SNR regime given by \(\eta \gtrsim \frac{1}{\min(1, \sqrt{k})\cos \varphi} \vee 1\), if the mixing weights are known balanced \(\pi = \pi^\ast = (1/2, 1/2)\), 
    then the difference between the EM update rule and its limit is bounded by:
    \[
        \left\|\frac{M(\theta)}{\| \theta^{\ast} \|} - \lim_{\eta \rightarrow \infty} \frac{M(\theta)}{\| \theta^{\ast} \|}\right\|
        = \mathcal{O}\left(\eta^{-2} \vee \cos^2 \varphi \frac{\log \Lambda}{\Lambda^4}\right)
    \]
    where \(\Lambda \assign \eta \sqrt{k} \cos \varphi\), \(k:= \frac{\|\theta\|}{\|\theta^\ast\|}\) and \(\varphi := \frac{\pi}{2}-\arccos |\rho|\), \(\rho := \frac{\langle \theta, \theta^\ast \rangle}{\|\theta\|\|\theta^\ast\|}\).
\end{theorem}
\begin{proof}
When the mixing weights are known balanced \(\pi = \pi^\ast = (1/2, 1/2)\), we must have \(\nu = \nu^\ast =0\). 
Consequently, the EM update rules for the imbalance of mixing weights directly gives 
\(
    N(\theta, \nu)_{\nu=0, \nu^\ast =0} = 0
\)
by using the symmetry of the data distributions with two balanced components.
Furthermore, the EM update rule for regresssion parameteres becomes
\begin{eqnarray*}
    \frac{M(\theta)}{\| \theta^{\ast} \|} =
    \frac{M (\theta, \nu)_{\nu=0, \nu^\ast =0}}{\| \theta^{\ast} \|} = \vec{e}_1 \frac{\sin \varphi}{\sin \varphi_\eta} \mathbb{E}[ \tanh \left(
    A_\eta g_\eta g' \right) g_\eta g']
    + \vec{e}_2  \frac{\cos \varphi}{\cos^2 \varphi_\eta} \mathbb{E} [\tanh \left( A_\eta g_\eta g' \right) g_\eta (g_\eta - \sin\varphi_\eta \cdot g')]
\end{eqnarray*}

We further bound the difference between the EM update rule and its limit (the cycloid trajectory) when SNR goes to infinity.
\begin{eqnarray*}
    & &\frac{M(\theta)}{\| \theta^{\ast} \|} - \lim_{\eta \rightarrow \infty} \frac{M(\theta)}{\| \theta^{\ast} \|}
     =  \vec{e}_1 \left[ 
    \frac{\sin \varphi}{\sin \varphi_\eta} \underbrace{\E[\tanh(A_\eta g_\eta g')g_\eta g' - |g_\eta g'|]}_{T_{1,1}}
    + \underbrace{\E\left[\frac{\sin \varphi}{\sin \varphi_\eta}|g_\eta g'|  - |g g'| \right]}_{T_{1,2}} \right]\\
    & + & \vec{e}_2 \sgn(\rho) \left[
        \frac{\cos \varphi}{\cos \varphi_\eta} \underbrace{\E[\tanh(A_\eta g_\eta g')g_\eta (g_\eta - \sin\varphi_\eta \cdot g')
         - \sgn(g_\eta g') g_\eta (g_\eta - \sin\varphi_\eta \cdot g')]/\cos \varphi_\eta}_{T_{2,1}}
    \right]\\
    & + & \vec{e}_2 \sgn(\rho) \underbrace{\E\left[\frac{\cos \varphi}{\cos \varphi_\eta} \sgn(g_\eta g') g_\eta (g_\eta - \sin\varphi_\eta \cdot g')/\cos \varphi_\eta - \sgn(g g') g(g-\sin \varphi \cdot g')/\cos \varphi\right]}_{T_{2,2}}
\end{eqnarray*}

\textbf{Evaluation of the terms \(T_{1, 2}, T_{2,2}\):}
By applying the Lemma~\ref{suplem:expectation} for \(g_\eta, g, g'\sim \mathcal{N}(0, 1)\) with \(\E[g_\eta g'] = \sin \varphi_\eta, \E[g g'] = \sin \varphi\), 
we have \(E[|g_\eta g'|]=\frac{2}{\pi}[\varphi_\eta \sin \varphi_\eta + \cos \varphi_\eta], \E[|g g'|]=\frac{2}{\pi}[\varphi \sin \varphi + \cos \varphi]\):
\[
T_{1, 2} = \E\left[\frac{\sin \varphi}{\sin \varphi_\eta}|g_\eta g'|  - |g g'| \right] 
= \frac{\sin \varphi}{\sin \varphi_\eta} \frac{2}{\pi}[\varphi_\eta \sin \varphi_\eta + \cos \varphi_\eta] - \frac{2}{\pi}[\varphi \sin \varphi + \cos \varphi]
= \frac{2}{\pi} \left[ (\varphi_\eta -\varphi)\sin \varphi - \frac{\sin (\varphi_\eta -\varphi)}{\sin \varphi_\eta}\right]
\]
Again, by applying the Lemma~\ref{suplem:expectation} for \(g_\eta, g, g'\), we have \(\E[\sgn(g_\eta g') g_\eta (g_\eta - \sin\varphi_\eta \cdot g')/\cos \varphi_\eta] = \frac{2}{\pi} \varphi_\eta \cos \varphi_\eta\) and \(\E[\sgn(g g') g(g-\sin \varphi \cdot g')/\cos \varphi] = \frac{2}{\pi} \varphi \cos \varphi\):
\[
T_{2, 2} = \E\left[\frac{\cos \varphi}{\cos \varphi_\eta} \sgn(g_\eta g') g_\eta (g_\eta - \sin\varphi_\eta \cdot g')/\cos \varphi_\eta - \sgn(g g') g(g-\sin \varphi \cdot g')/\cos \varphi\right]
=  \frac{2}{\pi} (\varphi_\eta - \varphi) \cos \varphi
\]
We bound the following quantities to establish the asymptotic notations of \(T_{1, 1}, T_{2, 1}\) in terms of SNR \(\eta\) and \(\varphi\), :
\[
    0\leq 1 - \frac{\sin \varphi_\eta}{\sin \varphi} = 1- (1+\eta^{-2})^{-\frac{1}{2}} 
    \leq \frac{\eta^{-2}}{2}, \quad 0 \leq \frac{\sin \varphi}{\sin \varphi_\eta} -1 = \sqrt{1+\eta^{-2}} - 1 \leq \frac{\eta^{-2}}{2}
\]
\[
    0 \leq \varphi -\varphi_\eta \leq \pi \sin \frac{\varphi - \varphi_\eta}{2} = \pi\frac{(1 - \frac{\sin \varphi_\eta}{\sin \varphi})\sin \varphi}{2 \cos \frac{\varphi + \varphi_\eta}{2}} 
    \leq \frac{\pi(\frac{\eta^{-2}}{2})\sin \varphi}{2 \cos \varphi} = \frac{\pi \sin \varphi}{4 \cos \varphi} \eta^{-2}
\]
\begin{eqnarray*}
& &\left|\frac{\sin (\varphi - \varphi_\eta)}{(\varphi - \varphi_\eta)}\cdot \frac{1}{\sin \varphi_\eta \sin \varphi}-1\right|\\
&\leq & \frac{1}{\sin^2 \varphi}\left(\frac{\sin \varphi}{\sin \varphi_\eta} -1\right) \frac{\sin (\varphi_\eta -\varphi)}{(\varphi_\eta -\varphi)} 
+ \frac{1}{\sin^2 \varphi}\left|\frac{\sin (\varphi_\eta -\varphi)}{(\varphi_\eta -\varphi)} - 1\right|
+ \left(\frac{1}{\sin^2 \varphi}-1 \right)\\
& \leq & \frac{1}{\sin^2 \varphi}\left(\frac{\eta^{-2}}{2}+ \frac{1}{6} \times \left[\frac{\pi \sin \varphi}{4 \cos \varphi} \eta^{-2}\right]^2 + \cos^2 \varphi \right)
=\frac{\cos^2 \varphi}{\sin^2 \varphi} + \frac{\eta^{-2}}{\sin^2 \varphi}\mathcal{O}\left(1\vee\frac{\sin^2 \varphi}{\cos^2 \varphi}\eta^{-2}\right)\\
\end{eqnarray*}
Therefore, when \(\eta \gtrsim \frac{1}{\cos \varphi}\), we have \(\mathcal{O}\left(1\vee\eta^{-2}\frac{\sin^2 \varphi}{\cos^2 \varphi}\right) =\mathcal{O}(1)\):
\begin{eqnarray*}
   |T_{1, 2}| = \frac{2}{\pi} \left|\frac{\sin (\varphi - \varphi_\eta)}{(\varphi - \varphi_\eta)}\frac{1}{\sin \varphi_\eta \sin \varphi}-1\right| (\varphi - \varphi_\eta)\sin \varphi
   \leq \frac{\cos \varphi\eta^{-2}}{2} + \frac{\eta^{-4}}{2\cos \varphi}\mathcal{O}\left(1\vee\frac{\eta^{-2}\sin^2 \varphi}{\cos^2 \varphi}\right)
   = \frac{\cos \varphi\eta^{-2}}{2} + \mathcal{O}\left(\frac{\eta^{-4}}{\cos \varphi}\right)
\end{eqnarray*}
\[
|T_{2, 2}| = \frac{2}{\pi}(\varphi - \varphi_\eta) \cos \varphi \leq \frac{\sin \varphi \eta^{-2}}{2} 
\]

\textbf{Asymptotic notations of the terms \(T_{1, 1}, T_{2, 1}\):}
Since we can express \(g_\eta, g'\) in terms of independent random varaibles \(R, U\) 
with \(R \sim r\exp(-\frac{r^2}{2})\mathbb{I}_{r\geq 0}\) (standard Rayleigh distribution with \(\E[R^2]=2\)) 
and \(U \sim \mathrm{Unif} [0, 4 \pi)\),
\[
g_\eta = R \cos((U-\varphi_\eta)/2), \quad g' = R \sin((U+\varphi_\eta)/2)
\]
Noting that \(\varrho := R^2/2 \sim \mathrm{Exp}(1)\) follows the standard exponential distribution with \(\E[\varrho] = 1\),
applying \(g g' = \varrho [\sin U + \sin \varphi_\eta]\) and  
\(\sgn(g_\eta g') g_\eta (g_\eta - \sin\varphi_\eta \cdot g')/\cos \varphi_\eta = \varrho \sgn(\sin(U+\varphi_\eta))|\cos U + \cos \varphi_\eta|\), then:
\begin{eqnarray*}
    T_{1, 1} &=& \E[(\tanh(A_\eta |g_\eta g'|)-1)|g_\eta g'|] 
    = -2 \E_U \E_{\varrho}\left[\frac{|\sin U + \sin \varphi_\eta| \varrho}{\exp(2A_\eta |\sin U + \sin \varphi_\eta \varrho|) + 1}\right] \\
    T_{2, 1} &=& \E[(\tanh(A_\eta |g_\eta g'|)-1) \sgn(g_\eta g') g_\eta (g_\eta - \sin\varphi_\eta \cdot g')]/\cos \varphi_\eta
    = -2\E_U \E_\varrho\left[\frac{\sgn(\sin(U+\varphi_\eta))|\cos U + \cos \varphi_\eta| \varrho}{\exp(2A_\eta |\sin U + \sin \varphi_\eta|\varrho) + 1}\right]
\end{eqnarray*}
Since \(\exp(t) \leq \exp(t) + 1\leq 2 \exp(t)\) for all \(t\geq 0\), and
\[
\E_\varrho\left[\rho \exp(-2A_\eta |\sin U + \sin \varphi_\eta| \varrho) \right] = \frac{1}{(1+2A_\eta |\sin U + \sin \varphi_\eta|)^2}
\]
we have shown that \(T_1, T_2\) can be approximated by the following expectations upto some constant factors:
\[
T_{1, 1} \asymp \E_U \left[ \frac{|\sin U + \sin \varphi_\eta|}{(1+2A_\eta |\sin U + \sin \varphi_\eta|)^2} \right],\quad
T_{2, 1} \asymp \E_U \left[ \frac{\sgn(\sin(U+\varphi_\eta))|\cos U + \cos \varphi_\eta|}{(1+2A_\eta |\sin U + \sin \varphi_\eta|)^2} \right]
\]
When \(\Lambda' \assign \sqrt{A_\eta} \cos \varphi_\eta  = \eta (1+\eta^{-2})^{\frac{1}{4}} \sqrt{k}\cos \varphi_\eta \gtrsim 1\), applying Lemma~\ref{suplem:integral_bigO} for asymptotic notations of above expectations:
\begin{eqnarray*}
    T_{1, 1} &=& \mathcal{O}\left(\frac{1}{A_\eta^2 \cos \varphi_\eta} \log\left(A_\eta \cos^2 \varphi_\eta\right) \right)
    = \cos^3 \varphi_\eta \mathcal{O}\left(\frac{\log \Lambda'}{[\Lambda']^4} \right) \\
    T_{2, 1} &=& \mathcal{O}\left(\frac{\tan \varphi_\eta}{A_\eta^2 \cos \varphi_\eta} \log\left(A_\eta \cos^2 \varphi_\eta\right) \right)
    = \sin \varphi_\eta \cos^2 \varphi_\eta \mathcal{O}\left(\frac{\log \Lambda'}{[\Lambda']^4} \right)
\end{eqnarray*}
where the condition \(\Lambda' \gtrsim 1\) is satisfied and therefore \(\Lambda' \asymp \sqrt{k} (\eta \cos \varphi \vee 1)\) when SNR \(\eta \gtrsim \frac{1}{\sqrt{k}\cos \varphi} \vee 1\) is large enough.

\textbf{Bounds for the difference between the EM update rule and its limit:}
When SNR \(\eta \gtrsim \frac{1}{\min(1, \sqrt{k})\cos \varphi} \vee 1\) is large enough, and we set \(\Lambda \assign \sqrt{k} \eta \cos \varphi \asymp \sqrt{k} (\eta \cos \varphi \vee 1) \asymp \Lambda'\), 
then \(\sin \varphi/ \sin \varphi_\eta = \sqrt{1+\eta^{-2}}=\Theta(1), \cos \varphi/\cos \varphi_\eta = \Theta(1)\):
\begin{eqnarray*}
& & \left\|\frac{M(\theta)}{\| \theta^{\ast} \|} - \lim_{\eta \rightarrow \infty} \frac{M(\theta)}{\| \theta^{\ast} \|}\right\|
    = \left\| \left(\vec{e}_1 \frac{\sin \varphi}{\sin \varphi_\eta} T_{1, 1} + \vec{e}_2 \frac{\cos \varphi}{\cos \varphi_\eta} T_{2, 1}\right) 
    + (\vec{e}_1 T_{1, 2} + \vec{e}_2 \sgn(\rho) T_{2, 2})\right\|\\
    & \leq & \left\|\vec{e}_1 \cos \varphi_\eta + \vec{e}_2 \sin \varphi_\eta \right\| \cos^2 \varphi \mathcal{O}\left(\frac{\log \Lambda}{\Lambda^4} \right)
    + \left\| \vec{e}_1 \cos \varphi + \vec{e}_2 \sin \varphi \right\| \frac{\eta^{-2}}{2}
    + \left\| \vec{e}_1 \right\| \mathcal{O}\left(\frac{\eta^{-4}}{\cos \varphi}\right)\\
    &= & 
    \frac{\eta^{-2}}{2} + \mathcal{O}\left(\frac{\eta^{-4}}{\cos \varphi}\vee \cos^2 \varphi \frac{\log \Lambda}{\Lambda^4}\right)\\
    &=& \mathcal{O}\left(\eta^{-2} \vee \cos^2 \varphi \frac{\log \Lambda}{\Lambda^4}\right)
\end{eqnarray*}
\end{proof}
\section{Population Level Analysis in Noiseless Regime}\label{sup:population_analysis}

\subsection{Recurrence Relation and Cycloid Trajectory of EM Iterations in Noiseless Setting}
\begin{theorem}[Proposition~\ref{prop:recurrence_angle}: Recurrence Relation of Sub-optimality Angle, Proposition 4.3 in~\cite{luo24cycloid}]
    In the noiseless setting, namely SNR \(\eta := \| \theta^{\ast} \|/\sigma \to \infty\), 
    if the sub-optimality angle \(\varphi^t \neq \frac{\pi}{2}\), 
    then the recurrence relation of the sub-optimality angle \(\varphi^t\) for population EM updates is:
    \[
    \tan \varphi^{t+1} = \tanh \varphi^t + \varphi^t (\tan^2 \varphi^t + 1)
    \]
    where \(\varphi^t := \frac{\pi}{2} - \arccos |\rho^t|\), \(\rho^t := \frac{\langle \theta^t, \theta^\ast \rangle}{\|\theta^t\|\|\theta^\ast\|}\).
\end{theorem}
\begin{proof}
    The proof based on Corollary 3.3 of EM Updates in Noiseless Setting of~\cite{luo24cycloid} can be found on page 31, Appendix D of~\cite{luo24cycloid}.
    Let's start from the EM update rule for \(\theta^t\) in previous Corollary, since \(\rho^0, \rho^{t - 1}\) have the same sign
    (validated by checking the sign of \(\langle \theta^t, \theta^{\ast} \rangle\))
    \[ \frac{\theta^t}{\| \theta^{\ast} \|} = \left( \frac{\pi}{2} \right)^{- 1}
       \left[ \varphi^{t - 1}  \frac{\tmop{sgn} (\rho^0) \theta^{\ast}}{\|
       \theta^{\ast} \|} + \cos \varphi^{t - 1}  \frac{\theta^{t - 1}}{\|
       \theta^{t - 1} \|} \right] \]
    With $\left\langle \frac{\tmop{sgn} (\rho^0) \theta^{\ast}}{\| \theta^{\ast}
    \|}, \frac{\theta^{t - 1}}{\| \theta^{t - 1} \|} \right\rangle = | \rho^{t -
    1} | = \sin \varphi^{t - 1}$
    \[ \sin \varphi^t  \frac{\| \theta^t \|}{\| \theta^{\ast} \|} = | \rho^t |
       \frac{\| \theta^t \|}{\| \theta^{\ast} \|} = \left\langle
       \frac{\theta^t}{\| \theta^{\ast} \|}, \frac{\tmop{sgn} (\rho^0)
       \theta^{\ast}}{\| \theta^{\ast} \|} \right\rangle = \left( \frac{\pi}{2}
       \right)^{- 1}  [\varphi^{t - 1} + \cos \varphi^{t - 1} \sin \varphi^{t -
       1}] \]
    \[ \frac{\| \theta^t \|}{\| \theta^{\ast} \|} = \left( \frac{\pi}{2}
       \right)^{- 1} \sqrt{[\varphi^{t - 1}]^2 + \cos^2 \varphi^{t - 1} + 2
       \varphi^{t - 1} \cos \varphi^{t - 1} \sin \varphi^{t - 1}} \]
    Therefore
    \[ \sin \varphi^t = \frac{\varphi^{t - 1} + \cos \varphi^{t - 1} \sin
       \varphi^{t - 1}}{\sqrt{[\varphi^{t - 1}]^2 + \cos^2 \varphi^{t - 1} + 2
       \varphi^{t - 1} \cos \varphi^{t - 1} \sin \varphi^{t - 1}}} \]
    Hence
    \[ \cos \varphi^t = \sqrt{1 - \sin^2 \varphi^t} = \frac{\cos^2 \varphi^{t -
       1}}{\sqrt{[\varphi^{t - 1}]^2 + \cos^2 \varphi^{t - 1} + 2 \varphi^{t -
       1} \cos \varphi^{t - 1} \sin \varphi^{t - 1}}} \]
    Thus, we obtain the recurrence relation for $\varphi^t$
    \[ \tan \varphi^t = \frac{\sin \varphi^t}{\cos \varphi^t} = \frac{\varphi^{t
       - 1}}{\cos^2 \varphi^{t - 1}} + \frac{\sin \varphi^{t - 1}}{\cos
       \varphi^{t - 1}} = \tan \varphi^{t - 1} + \varphi^{t - 1}  [\tan^2
       \varphi^{t - 1} + 1] \]
\end{proof}

\begin{theorem}[Proposition~\ref{prop:parametric_cycloid}: Parametric Equation for Cycloid Trajectory of EM Updates, Proposition 4.4 in~\cite{luo24cycloid}]
    In the noiseless setting, namely SNR \(\eta := \| \theta^{\ast} \|/\sigma \to \infty\),
    the coordinates \(\mathtt{x}^t, \mathtt{y}^t\) of normalized vector \(\frac{\theta^t}{\|\theta^\ast\|}=\mathtt{x}^t \hat{e}_1 + \mathtt{y}^t \hat{e}_2^t = \mathtt{x}^t \hat{e}_1 + \mathtt{y}^t \hat{e}_2^0, \forall t\in\mathbb{Z}_+\)
    for population EM updates can be parameterized by the sub-optimality angle \(\phi^{t-1}\) as follows:
    \begin{eqnarray*}
        1-\sgn(\rho^0)\mathtt{x}^t & = &\frac{1}{\pi}[\phi^{t-1} - \sin \phi^{t-1}]\\
        \mathtt{y}^t & = &\frac{1}{\pi}[1- \cos \phi^{t-1}]
    \end{eqnarray*}
    where \(\varphi^{t-1} := \frac{\pi}{2} - \arccos |\rho^{t-1}|\), \(\rho^{t-1} := \frac{\langle \theta^{t-1}, \theta^\ast \rangle}{\|\theta^{t-1}\|\|\theta^\ast\|}\).
    Hence, the trajecotry of EM iterations \(\theta^t\) is on the cycloid with a parameter \(\frac{\|\theta^\ast\|}{\pi}\), on the plane \(\text{span}\{\theta^0, \theta^\ast\}\).
\end{theorem}
\begin{proof}
    The proof based on Corollary 3.3 of EM Updates in Noiseless Setting and Proposition 4.3 of Recurrence Relation of~\cite{luo24cycloid} can be found on page 32, Appendix D of~\cite{luo24cycloid}.
    Since \(\tan \varphi^t = \tan \varphi^{t - 1} + \varphi^{t - 1}  (\tan^2\varphi^{t - 1} + 1)\) in the proven recurrence relation, it shows that \(\tan \varphi^t \geq \tan \varphi^{t - 1} \geq 0\), therefore \(0 \leq \varphi^0 \leq \varphi^1 \leq \cdots \leq \varphi^{t - 1} \leq \varphi^t < \frac{\pi}{2}\).
  
  Let \(\hat{e}_1 \assign \frac{\theta^{\ast}}{\| \theta^{\ast} \|}\), and
  \(\hat{e}^t_2 \assign \hat{e}_2 \mid_{\theta = \theta^t} = \frac{\theta -
  \hat{e}_1  \hat{e}_1^{\top} \theta}{\| \theta - \hat{e}_1  \hat{e}_1^{\top}
  \theta \|} \mid_{\theta = \theta^t} = \frac{\frac{\theta^t}{\| \theta^t \|}
  - [\rho^t] \frac{\theta^{\ast}}{\| \theta^{\ast} \|}}{\sqrt{1 -
  [\rho^t]^2}}\) and \(\langle \hat{e}_1, \hat{e}^t_2 \rangle = 0, \| \hat{e}_1
  \| = \| \hat{e}^t_2 \| = 1\)
  \begin{eqnarray*}
    \frac{\left( \frac{\pi}{2} \right)}{\| \theta^{\ast} \| \| \theta^{t - 1}
    \|}  \langle \theta^{t - 1} - \hat{e}_1  \hat{e}_1^{\top} \theta^{t - 1},
    \theta^t - \hat{e}_1  \hat{e}_1^{\top} \theta^t \rangle & = & \frac{\left(
    \frac{\pi}{2} \right)}{\| \theta^{\ast} \| \| \theta^{t - 1} \|} 
    \{\langle \theta^{t - 1}, \theta^t \rangle - \langle \theta^{t - 1},
    \hat{e}_1 \rangle \langle \hat{e}_1, \theta^t \rangle\}\\
    & = & [\varphi^{t - 1} \sin \varphi^{t - 1} + \cos \varphi^{t - 1}] -
    \sin \varphi^{t - 1}  [\varphi^{t - 1} + \cos \varphi^{t - 1} \sin
    \varphi^{t - 1}]\\
    & = & \cos^3 \varphi^{t - 1} > 0
  \end{eqnarray*}
  Hence, we conclude that \(\langle \hat{e}^{t - 1}_2, \hat{e}^t_2 \rangle > 0\), With \(\hat{e}^{t - 1}_2, \hat{e}^t_2 \perp \hat{e}_1\) and \(\hat{e}^{t - 1}_2, \hat{e}^t_2 \in \tmop{span} \{\theta^t, \theta^{t - 1}, \theta^{\ast}\} \subset \tmop{span} \{\theta^0, \theta^{\ast} \}\), \(\| \hat{e}^{t - 1}_2\| = \| \hat{e}^t_2 \| = 1\), we validate \(\hat{e}^0_2 = \cdots = \hat{e}^{t - 1}_2 = \hat{e}^t_2\).

  By the definition of \(\hat{e}^t_2\), we obtain \(\theta^t = \| \theta^t \| \{\tmop{sgn} (\rho^0) \sin \varphi^t  \hat{e}_1 + \cos \varphi^t \hat{e}^t_2 \} = \| \theta^t \| \{\tmop{sgn} (\rho^0) \sin \varphi^t \hat{e}_1 + \cos \varphi^t \hat{e}^0_2 \}\)
  
  Since \(\theta^t \in \tmop{span} \{\theta^{t - 1}, \theta^{\ast} \}\), then
  \(\theta^t \in \tmop{span} \{\theta^0, \theta^{\ast} \}\), we can express
  \(\frac{\theta^t}{\| \theta^{\ast} \|} = \mathtt{x}^t \hat{e}_1 +
  \mathtt{y}^t \hat{e}^t_2 = \mathtt{x}^t \hat{e}_1 + \mathtt{y}^t
  \hat{e}^0_2\).
  
  Comparing the expressions for \(\theta^t\), we derive the following result.
  \[ \frac{\theta^t}{\| \theta^{\ast} \|} = \mathtt{x}^t \hat{e}_1 +
     \mathtt{y}^t \hat{e}^0_2 = \left\{ \tmop{sgn} (\rho^0) \sin \varphi^t
     \cdot \frac{\| \theta^t \|}{\| \theta^{\ast} \|} \right\} \hat{e}_1 +
     \left\{ \cos \varphi^t \cdot \frac{\| \theta^t \|}{\| \theta^{\ast} \|}
     \right\} \hat{e}^0_2 \]
  With the recurrence relation \(\sin \varphi^t \cdot \frac{\| \theta^t \|}{\|
  \theta^{\ast} \|} = \left( \frac{\pi}{2} \right)^{- 1}  [\varphi^{t - 1} +
  \cos \varphi^{t - 1} \sin \varphi^{t - 1}]\), \(\cos \varphi^t \cdot \frac{\|
  \theta^t \|}{\| \theta^{\ast} \|} = \left( \frac{\pi}{2} \right)^{- 1}
  \cos^2 \varphi^{t - 1}\), which we showed in the proof of Recurrence Relation, we derive the implicit equation of \(\mathtt{x}^t, \mathtt{y}^t (t \geq 1)\) as follows:
  \begin{eqnarray*}
    \mathtt{x}^t & = & \left\langle \frac{\theta^t}{\| \theta^{\ast} \|},
    \hat{e}_1 \right\rangle = \left( \frac{\pi}{2} \right)^{- 1} \tmop{sgn}
    (\rho^0)  [\varphi^{t - 1} + \cos \varphi^{t - 1} \sin \varphi^{t - 1}]\\
    \mathtt{y}^t & = & \left\langle \frac{\theta^t}{\| \theta^{\ast} \|},
    \hat{e}^0_2 \right\rangle = \left\langle \frac{\theta^t}{\| \theta^{\ast}
    \|}, \hat{e}^t_2 \right\rangle = \left( \frac{\pi}{2} \right)^{- 1} \cos^2
    \varphi^{t - 1}
  \end{eqnarray*}
  Let's cancel out the parameter \(\varphi^{t - 1}\) in the parameterized curve
  \(\varphi^{t - 1} \mapsto (\mathtt{x}^t, \mathtt{y}^t)\)
  \[ \tmop{sgn} (\rho^0) \frac{\pi}{2} \mathtt{x}^t = \sqrt{\left(
     \frac{\pi}{2} \mathtt{y}^t \right)  \left( 1 - \frac{\pi}{2} \mathtt{y}^t
     \right)} + \arccos \sqrt{\frac{\pi}{2} \mathtt{y}^t} \]
  Let \(\phi \assign 2 \left( \frac{\pi}{2} - \varphi \right) \in (0,\pi]\), then we rewrite the implicit equations of \(\mathtt{x}^t, \mathtt{y}^t (t \geq 1)\) as follows, which is a cycloid curve~\cite{harris1998handbook}:
  \begin{eqnarray*}
    1 - \tmop{sgn} (\rho^0) \mathtt{x}^t & = & \pi^{- 1}  [\phi - \sin
    \phi]_{\phi = \phi^{t - 1}}\\
    \mathtt{y}^t & = & \pi^{- 1}  [1 - \cos \phi]_{\phi = \phi^{t - 1}}
  \end{eqnarray*}
\end{proof}

\newpage
\subsection{Linear Growth and Quadratic Convergence of Sub-optimality Angles in Noiseless Setting}
\begin{theorem}[Proposition~\ref{prop:linear_growth_angle}: Linear Growth of Sub-optimality Angle]
    In the noiseless setting, namely SNR \(\eta := \| \theta^{\ast} \|/\sigma \to \infty\),
    \(\tan \varphi^t\) of the sub-optimality angle \(\varphi^t\) grows at least linearly:
    \[
    \tan \varphi^{t+1} \geq 2\cdot\tan \varphi^t
    \]
    where \(\varphi^t := \frac{\pi}{2} - \arccos |\rho^t|\), \(\rho^t := \frac{\langle \theta^t, \theta^\ast \rangle}{\|\theta^t\|\|\theta^\ast\|}\).
\end{theorem}
\begin{proof}
    We start from the recurrence relation of \(\varphi^t\) for the population EM update rule:
    \[
    \tan \varphi^{t+1} = \tan \varphi^t + \varphi^t (\tan^2 \varphi^t + 1)
    \]
    The inequality \(\tan \varphi^{t+1} \geq 2\cdot \tan \varphi^t\) holds when \(\varphi^t = 0\). 
    Let's consider the case when \(\varphi^t > 0\).
    By the trigonometry identity \(\sin(2\varphi^t) = 2 \tan \varphi^t / (1+\tan^2 \varphi^t)\), 
    and the inequality \(2\varphi^t \geq \sin(2\varphi^t)\), we have:
    \[
    \tan \varphi^{t+1} = (1+ \frac{2 \varphi^t}{\sin(2\varphi^t)}) \tan \varphi^t \geq 2\cdot \tan \varphi^t
    \]

\end{proof}

\begin{theorem}[Proposition~\ref{prop:quadratic_convergence_angle}: Quadratic Convergence of Sub-optimality Angle]
    In the noiseless setting, namely SNR \(\eta := \| \theta^{\ast} \|/\sigma \to \infty\),
    the sub-optimality angle \(\phi^t\) converges quadratically to zero when \(\phi^t \leq 1.4\) is small enough:
    \[
    \frac{\phi^{t+1}}{\pi} \leq \left[ \frac{\phi^t}{\pi} \right]^2
    \]
    where \(\phi^t : = 2\arccos |\rho^t|\), \(\rho^t := \frac{\langle \theta^t, \theta^\ast \rangle}{\|\theta^t\|\|\theta^\ast\|}\).
\end{theorem}
\begin{proof}
If \(\phi^t=0\), then \(\phi^{t+1}=0\) by the EM updates in noiseless setting.
Otherwise, we have \(\phi^{t+1} > 0\) when \(\phi^t > 0\) from the EM updates in noiseless setting.
We start from the recursive relation of \(\phi^t\) for the population EM update rule:
\[
    \tan \varphi^{t+1} = \tanh \varphi^t + \varphi^t (\tan^2 \varphi^t + 1)
\]
By the definition \(\phi^t = 2(\frac{\pi}{2} - \varphi^t)\) and apply the inequality \(x \leq \tan x, \forall x \in [0, \frac{\pi}{2})\), we have:
\[
\frac{1}{\tan \frac{\phi^{t+1}}{2}} = \frac{1}{\tan \frac{\phi^t}{2}} + \frac{\pi -\phi^t}{2}\cdot \frac{1}{\sin^2 \frac{\phi^t}{2}}
\implies
    \frac{\phi^{t+1}}{\pi} \leq \frac{2}{\pi}\tan\frac{\phi^{t+1}}{2}
    = \left[ \frac{\phi^t}{\pi} \right]^2\cdot \frac{2\pi(1-\cos \phi^t)/[\phi^t]^2}{\sin \phi^t + (\pi - \phi^t)}
\]
By the inequality based on Taylor's theorem \(\cos x \geq 1 - \frac{x^2}{2} + \frac{x^4}{24} - \frac{x^6}{720}, \sin x \geq x - \frac{x^3}{6}, \forall x\geq 0\), we have:
\[
    \frac{2\pi(1-\cos \phi^t)}{[\phi^t]^2} \leq \pi - \frac{\pi}{12}[\phi^t]^2 + \frac{\pi}{360}[\phi^t]^4
    \leq \pi - \frac{1}{6}[\phi^t]^3 \leq \sin \phi^t + (\pi - \phi^t)
\]
where the second inequality is due to \(30 -\frac{60}{\pi} \phi^t - [\phi^t]^2 \geq 0\) when \(0\leq \phi^t \leq 1.4\).
Therefore, when \(0\leq \phi^0 \leq 1.4\), we have \(0\leq \phi^{t+1} \leq \phi^t \leq \cdots \leq \phi^0 \leq 1.4\) and
\[
\frac{\phi^{t+1}}{\pi} \leq \left[ \frac{\phi^t}{\pi} \right]^2
\]
\end{proof}

\subsection{Accuracy of Regression Parameters and Mixing Weights in Noiseless Setting}
\begin{theorem}[Proposition~\ref{prop:errors_em_updates_angle}: Accuracy of EM Updates and Sub-optimality Angle]
    In the noiseless setting, namely SNR \(\eta := \| \theta^{\ast} \|/\sigma \to \infty\),
    \begin{eqnarray*}
        \frac{\| \theta^t - \sgn(\rho^0) \theta^\ast\|}{\| \theta^\ast\|} &\leq & \frac{\left[\phi^{t-1}\right]^2}{2\pi} \\
        \|\pi^t - \bar{\pi}^\ast \|_1 & = & \frac{\phi^{t-1}}{\pi}\cdot \left\| \pi^\ast - \frac{\mathds{1}}{2} \right\|_1
    \end{eqnarray*}
    where \(\phi^{t-1} := 2(\frac{\pi}{2} - \arccos |\rho^{t-1}|)\), \(\rho^{t-1} := \frac{\langle \theta^{t-1}, \theta^\ast \rangle}{\|\theta^{t-1}\|\|\theta^\ast\|}\),
    and \(\bar{\pi}^\ast := \frac{\mathds{1}}{2}+\sgn(\rho^0)(\pi^\ast-\frac{\mathds{1}}{2}), \mathds{1} := (1, 1)\).
\end{theorem}
\begin{proof}
    Let's prove the first inequality by starting from the equations of Cycloid Trajectory parameterized by the angle \(\phi^{t-1}\).
    \begin{eqnarray*}
        \frac{\|\theta^t - \sgn(\rho^0)\theta^\ast\|}{\|\theta^\ast\|} & = & 
        \frac{1}{\pi} \left\| [\phi^{t-1}-\sin \phi^{t-1}]\hat{e}_1 + [1-\cos \phi^{t-1}]\hat{e}_2^0 \right\|
        = \frac{1}{\pi} \left\| \int_0^{\phi^{t-1}}([\phi - \cos \phi]\hat{e}_1 + [\sin \phi]\hat{e}_2^0) \mathd\phi \right\| \\
        & \leq & \frac{1}{\pi} \int_0^{\phi^{t-1}} \left\| [\phi - \cos \phi]\hat{e}_1 + [\sin \phi]\hat{e}_2^0 \right\| \mathd\phi 
        = \frac{1}{\pi} \int_0^{\phi^{t-1}} 2 \sin \frac{\phi}{2} \mathd\phi
        = \frac{4}{\pi} (1- \cos \frac{\phi^{t-1}}{2})\\
        & \leq & \frac{[\phi^{t-1}]^2}{2\pi}
    \end{eqnarray*}
    In the above derivation, the first inequality is due to the fact that \(\|\int_0^{\phi^t} \vec{v}(\phi) \mathd \phi\| \leq \int_0^{\phi^t} \|\vec{v}(\phi)\| \mathd \phi\) for vector \(\vec{v}(\phi)\) and \(\phi^{t-1}\geq 0\),
    and the second inequality follows from the inequality \(1- \cos x \leq \frac{x^2}{2}\).

    The proof for the accuracy in mixing weights (the indentity in the second line) is adapted from page 34, Appendix D of~\cite{luo24cycloid}.
    Using EM Updates in Noiseless Setting, and note that \(\mathrm{sgn} (\rho^{t - 1}) = \mathrm{sgn} (\rho^{0})\), we obtain that equation.
\begin{equation*}
   \tanh (\nu^t) = \mathrm{sgn} (\rho^{0}) \left( \frac{2}{\pi}
   \varphi^{t - 1} \right) \cdot \tanh (\nu^{\ast})
\end{equation*}
Since $\pi^t(1)=\frac{1+ \tanh(\nu^t)}{2}, \pi^t(2)=\frac{1- \tanh(\nu^t)}{2}$ and 
$\bar{\pi}^{\ast}(1)=\frac{1+ \mathrm{sgn} (\rho^{0})\tanh(\nu^\ast)}{2}, \bar{\pi}^{\ast}(2)=\frac{1- \mathrm{sgn} (\rho^{0})\tanh(\nu^\ast)}{2}$.
\begin{equation*}
   \| \pi^t - \bar{\pi}^{\ast} \|_1 = |\pi^t(1)-\bar{\pi}^{\ast}(1)| + |\pi^t(2)-\bar{\pi}^{\ast}(2)|
   = | \tanh(\nu^t) - \mathrm{sgn} (\rho^{0})\tanh(\nu^\ast)| = \left| 1 - \frac{2}{\pi} \varphi^{t - 1}
   \right| \cdot \left\| \pi^{\ast} - \frac{\mathds{1}}{2} \right\|_1
   = \frac{\phi^{t-1}}{\pi}\cdot \left\| \pi^{\ast} - \frac{\mathds{1}}{2} \right\|_1
\end{equation*}
In the above equation, we use such an identity \(\tanh (\nu^{\ast}) = \left\| \pi^{\ast} - \frac{\mathds{1}}{2} \right\|_1\)
and the fact that \(\phi^{t-1} = \pi - 2\varphi^{t-1}\).
\end{proof}

\subsection{Convergence Guarantee of EM Updates at the Population Level in Noiseless Setting}
\begin{theorem}[Theorem~\ref{theorem:population_level_convergence}: Population Level Convergence, Theorem 4.1 in~\cite{luo24cycloid}]
    In the noiseless setting, namely SNR \(\eta := \| \theta^{\ast} \|/\sigma \to \infty\), if the initial sup-optimality angle cosine \(\rho^0:=\frac{\langle \theta^0, \theta^\ast \rangle}{\|\theta^0\|\|\theta^\ast\|}\neq 0\),
    then with the number of total iterations at most \(T=\mathcal{O}(\log \frac{1}{|\rho^0|}\vee \log\log\frac{1}{\varepsilon})\), the error of EM update at the population level is bounded by:
    \(\frac{\| \theta^{T+1} - \sgn(\rho^0) \theta^\ast\|}{\| \theta^\ast\|} < \varepsilon\) and \(\|\pi^{T+1} - \bar{\pi}^\ast \|_1 = \left\| \pi^\ast - \frac{\mathds{1}}{2} \right\|_1 \mathcal{O}(\sqrt{\varepsilon})\).
    where \(\bar{\pi}^\ast := \frac{\mathds{1}}{2}+\sgn(\rho^0)(\pi^\ast-\frac{\mathds{1}}{2}), \mathds{1} := (1, 1)\).
\end{theorem}
\begin{proof}
    We provide a new straightforward proof for the convergence gaurantee of EM updates at the population level in the noiseless setting, 
    by invoking the results for accuracy in regeressioon parameters and mixing weights, 
    and applying the linear growth and quadratic convergence of the sub-optimality angles.
    The other old lengthy proof of this Theorem 4.1 of Population Level Convergence in~\cite{luo24cycloid} can be found on pages 34-35, Appendix D of~\cite{luo24cycloid}.

    Since when the sub-optimality angle \(\phi^t > 1.4\) is large, the other sub-optimality angle \(\varphi^t = \frac{\pi}{2} - \frac{\phi^t}{2} \leq \frac{\pi}{2}-0.7 \approx 0.8708 < 1 \) is also small enough.
    Suppose the initial sub-optimality angle \(\varphi^0 < 1\),
    then after at most \(T_1:=\lceil (\ln\tan 1+ \ln\frac{1}{|\rho^0|})/\ln 2 \rceil_+ = \Theta(\log \frac{1}{|\rho^0|})\) EM iterations, the sub-optimality angle \(\phi^t\leq 1.4, \forall t\geq T_1\) will be sufficiently small enough.
    \begin{eqnarray*}
        & &\tan \varphi^t \geq \tan \varphi^{T_1} \geq 2^{T_1} \tan \varphi^0
        \geq \tan 1\cdot \frac{1}{|\rho^0|}\cdot \tan \varphi^0 = \tan 1 \cdot \frac{1}{\sin \varphi^0}\cdot \frac{\sin \varphi^0}{\cos \varphi^0} \geq \tan 1\quad \forall t\geq T_1\\
        & &\implies \varphi^t \geq 1 
        \implies \phi^t = 2\left(\frac{\pi}{2} - \varphi^t\right)< 1.4\quad \forall t\geq T_1
    \end{eqnarray*}
    Otherwise, if the initial sub-optimality angle \(\varphi^0 \geq 1\), then \(T_1=0\) and we still have \(\phi^t<  1.4, \forall t\geq T_1\). 
    Then by the quadratic convergence of the sub-optimality angle \(\phi^t\) when \(\phi^t < 1.4\) is small enough, 
    then after at most \(T_2:=\lceil (\ln\ln \frac{\pi}{2 \varepsilon}-\ln 2 -\ln\ln \frac{\pi}{1.4}) / \ln 2 \rceil_+=\Theta(\log \log \frac{1}{\varepsilon})\), we have the following inequality:
    \begin{eqnarray*}
        \phi^{T_1+T_2} \leq \pi \left(\frac{\phi^{T_1}}{\pi}\right)^{2^{T_2}} < \pi \left(\frac{1.4}{\pi}\right)^{2^{T_2}}
        \leq \pi \left(\frac{1.4}{\pi}\right)^{\frac{\frac{1}{2}\ln \frac{\pi}{2 \varepsilon}}{\ln \frac{\pi}{1.4}}}
        = \pi \left(\frac{\pi}{2\varepsilon} \right)^{-\frac{1}{2}}
        =\sqrt{2\pi\varepsilon}
    \end{eqnarray*}
    Therefore, after at most \(T:=T_1+T_2=\Theta(\log \frac{1}{|\rho^0|}\vee \log\log\frac{1}{\varepsilon})\) EM iterations, we have \(\phi^T < \sqrt{2\pi\varepsilon}\) for all \(t\geq T\).
    By using the results for accuracy in regeressioon parameters and mixing weights, we have the following bounds:
    \[
    \frac{\| \theta^{T+1} - \sgn(\rho^0) \theta^\ast\|}{\| \theta^\ast\|} \leq \frac{[\phi^{T}]^2}{2\pi}  < \varepsilon, \quad
    \|\pi^{T+1} - \bar{\pi}^\ast \|_1 = \frac{\phi^{T}}{\pi}\cdot \left\| \pi^\ast - \frac{\mathds{1}}{2} \right\|_1 
    < \sqrt{\frac{2\varepsilon}{\pi}}\cdot \left\| \pi^\ast - \frac{\mathds{1}}{2} \right\|_1
    = \mathcal{O}(\sqrt{\varepsilon})\cdot \left\| \pi^\ast - \frac{\mathds{1}}{2} \right\|_1
    \]
    where the last inequality of both bounds follows from \(\phi^{T} < \sqrt{2\pi\varepsilon}\).
\end{proof}

\newpage
\section{Finite-Sample Level Analysis in Noiseless Setting}\label{sup:finite_sample_analysis}

\subsection{Statistical Error, Statistical Accuracy of EM Updates in Noiseless Setting}
\begin{theorem}[Proposition~\ref{prop:projected_error_regression}: Projected Error of Easy EM Update for Regression Parameters, Proposition 5.2 in~\cite{luo24cycloid}]
    In the noiseless setting, the projection on $\text{span}\{\theta,\theta^\ast\}$ for the statistical error of $\theta$ satisfies
    \begin{equation}\nonumber
      \frac{\|P_{\theta,\theta^\ast}
      [M_n^{\tmop{easy}} (\theta, \nu) - M (\theta, \nu)]\|}{\|\theta^\ast\|}
      = \mathcal{O}\left(\sqrt{\frac{\log \frac{1}{\delta}}{n}} \vee \frac{\log \frac{1}{\delta}}{n}\right),
    \end{equation}
    with probability at least $1 - \delta$, where $M_n (\theta, \nu), M (\theta,
    \nu)$ are the EM update rules for $\theta$ at the Finite-sample level and the
    population level respectively, and the orthogonal projection matrix $P_{\theta,\theta^\ast}$ satisfies
    $\text{span}(P_{\theta,\theta^\ast})=\text{span}\{\theta,\theta^\ast\}$ 
    .
\end{theorem}

\begin{proof}
    We provide a new straightfowrd proof here.
    In the noiseless setting, as SNR \(\eta = \frac{\| \theta^\ast\|}{\sigma} \to \infty\), we have \(y_i \to (-1)^{z_i+1} \langle x_i, \theta^\ast \rangle, \tanh(y_i\langle x_i, \theta\rangle/\sigma^2+\nu)\to(-1)^{z_i+1}\sgn(\langle x_i, \theta^\ast \rangle \langle x_i, \theta \rangle)\)
    for the \(i\)-th data point \(z_i, s_i=(x_i, y_i)\).
    \[
    M^{\tmop{easy}}_n(\theta, \nu) = \frac{1}{n} \sum_{i=1}^n 
    \sgn\left(\langle x_i, \theta^\ast \rangle \langle x_i, \theta \rangle \right) \langle x_i, \theta^\ast \rangle x_i
    \]
    Without loss of generality, we assume \(\dim \text{span}\{\theta, \theta^\ast\} = 2\) and therefore \(\dim \text{span}\{\theta, \theta^\ast\}^\perp = d-2\).
    
    We introduce \(\rho := \frac{\langle \theta, \theta^{\ast} \rangle}{\| \theta \| \cdot \| \theta^{\ast} \|}\), 
    and define \(\hat{e}_1 := \frac{\theta^{\ast}}{\| \theta^{\ast} \|}\), \(\vec{e}_1 := \frac{\theta}{\| \theta \|}\),
    and \(\hat{e}_2 := \frac{\theta-\langle \theta, \hat{e}_1 \rangle \hat{e}_1}{\| \theta-\langle \theta, \hat{e}_1 \rangle \hat{e}_1 \|}\),
    \(\vec{e}_2 := \frac{\theta^\ast-\langle \theta^\ast, \vec{e}_1 \rangle \vec{e}_1}{\| \theta^\ast-\langle \theta^\ast, \vec{e}_1 \rangle \vec{e}_1 \|}\).
    Also, let \(P_{\theta,\theta^\ast} = \vec{e}_1 \vec{e}_1^\top + \vec{e}_2 \vec{e}_2^\top\), 
    which can be expressed with orthonormal basis matrices \(U_{\theta,\theta^\ast} = (\sgn(\rho)\vec{e}_1, \vec{e}_2)\) 
    such that \(P_{\theta,\theta^\ast} = U_{\theta,\theta^\ast} U_{\theta,\theta^\ast}^\top\).
    Also, we have \(\varphi = \frac{\pi}{2} - \arccos |\rho|\), therefore \(\rho = \sgn(\rho) \sin \varphi\).
    By defining \(g_i := \langle x_i, \hat{e}_1\rangle, g'_i := \sgn(\rho) \langle x_i, \vec{e}_1\rangle\), 
    then \(g_i, g'_i \sim \mathcal{N}(0, 1)\) with \(\mathbb{E}[g_i g'_i] = \sin \varphi\), and \(h'_i=\langle x_i, \vec{e}_2 \rangle \sim \mathcal{N}(0, 1)\) satisfies \(g_i = \sin \varphi\cdot g'_i + \cos \varphi \cdot h'_i, h'_i\ind g'_i\).
    \begin{eqnarray*}
        U_{\theta,\theta^\ast}^\top \frac{M_n^{\tmop{easy}}(\theta, \nu)}{\|\theta^\ast\|} 
        & = & \left(\frac{1}{n} \sum_{i=1}^n |g_i g'_i|,\space \sgn(\rho)  \frac{1}{n} \sum_{i=1}^n \sgn(g_i g'_i)g_i h'_i \right)^\top
    \end{eqnarray*}
    \begin{eqnarray*}
        & & \frac{\|P_{\theta,\theta^\ast}[M_n^{\tmop{easy}}(\theta, \nu) - M (\theta, \nu)]\|}{\|\theta^\ast\|}
        = \left\| U_{\theta,\theta^\ast}^\top \frac{[M_n^{\tmop{easy}}(\theta, \nu) - M (\theta, \nu)]}{\|\theta^\ast\|}\right\|
        = \left\| U_{\theta,\theta^\ast}^\top \frac{M_n^{\tmop{easy}}(\theta, \nu)}{\|\theta^\ast\|}- \E\left[U_{\theta,\theta^\ast}^\top \frac{M^{\tmop{easy}}(\theta, \nu)}{\|\theta^\ast\|}\right]\right\|\\
        & = & \sqrt{\left[\left(\frac{1}{n}\sum_{i=1}^n -\E\right)|g_i g'_i|\right]^2 + \left[\left(\frac{1}{n}\sum_{i=1}^n -\E\right)\sgn(g_i g'_i)g_i h'_i\right]^2}
    \end{eqnarray*}
    where \(|g_i g'_i|, \sgn(g_i g'_i)g_i h'_i\) are i.i.d. Sub-Exponential random variables by Proposition 2.7.1 part (b) on page 33 of~\cite{vershynin2018prob} and noting that for \(\forall q \in \mathbb{Z}_+\), 
    \[
    \E\left[(|g_i g'_i| - \E[|g_i g'_i|])^q\right]^{1/q} \leq \E\left[|g_i g'_i|^q\right]^{1/q} + \E[|g_i g'_i|]
    \leq \E\left[\frac{g_i^{2q}+(g'_i)^{2q}}{2}\right]^{1/q} + \E\left[\frac{g_i^2+(g'_i)^2}{2}\right]
    \leq 2\mathe q + 1\leq 3 \mathe q,
    \]
    \[
    \E\left[(\sgn(g_i g'_i)g_i h'_i - \E[\sgn(g_i g'_i)g_i h'_i])^q\right]^{1/q} 
    \leq \E\left[|g_i h'_i|^q\right]^{1/q} + |\E[\sgn(g_i g'_i)g_i h'_i]|
    \leq \E\left[|g_i h'_i|^q\right]^{1/q} + \E[|g_i h'_i|]
    \leq 2\mathe q + 1\leq 3 \mathe q,
    \]
    where the first inequality follows from Minkowski's inequality, and the second last inequality is due to the facts for moment of Gaussian random variables \(g_i, g'_i, h'_i \sim \mathcal{N}(0, 1)\).
    Consequently, by Bernstein's inequality (see Corollary 2.8.3 on page 38 of~\cite{vershynin2018prob}) for Sub-Exponential random variables \(|g_i g'_i|, \sgn(g_i g'_i)g_i h'_i\), 
    with probability at least \(1-\delta\), we have
    \[
    \frac{\left\|P_{\theta,\theta^\ast}[M_n^\text{easy}(\theta, \nu) - M (\theta, \nu)] \right\|}{\|\theta^\ast\|} 
    \leq \left|\left(\frac{1}{n}\sum_{i=1}^n -\E\right)|g_i g'_i| \right|
    + \left|\left(\frac{1}{n}\sum_{i=1}^n -\E\right)\sgn(g_i g'_i)g_i h'_i \right|
    \lesssim \sqrt{\frac{\log \frac{1}{\delta}}{n}} \vee \frac{\log \frac{1}{\delta}}{n}
    \]
\end{proof}

\begin{theorem}[Proposition~\ref{prop:statistical_error_regression}: Statistical Error of EM Update for Regression Parameters, Proposition 5.3 in~\cite{luo24cycloid}]
    In the noiseless setting, the statistical error of $\theta$ for finite-sample EM
    updates with \(n\gtrsim d\vee \log \frac{1}{\delta}\) samples satisfies
    \small
    \begin{equation}\nonumber
      \frac{\| M_n^{\tmop{easy}} (\theta, \nu) - M (\theta, \nu) \|}{\| \theta^{\ast} \|}
      = \mathcal{O} \left( \sqrt{\frac{d\vee \log \frac{1}{\delta}}{n}}\right)
      ,\quad
      \frac{\| M_n (\theta, \nu) - M (\theta, \nu) \|}{\| \theta^{\ast} \|}
      = \mathcal{O} \left( \sqrt{\frac{d\vee \log \frac{1}{\delta}}{n}} 
      \right),
    \end{equation}
    with probability at least $1 - \delta$, $M_n (\theta, \nu), M (\theta,
    \nu)$ denote the EM update rules for $\theta$ at the Finite-sample level and the
    Population level.
\end{theorem}

\begin{proof}
    Let the ensemble matrix of data samples be \(\mathsf{\Sigma} = \frac{1}{n} \sum_{i=1}^n x_i x_i^\top\), and
    note that \(M(\theta, \nu) = \E \left[ M_n^{\tmop{easy}}(\theta, \nu) \right]\), 
    by letting \(\vec{v} = \left(M_n^{\tmop{easy}}(\theta, \nu) - M(\theta, \nu)\right)/ \|\theta^\ast\|
    = \left( M_n^{\tmop{easy}}(\theta, \nu) - \E[M^{\tmop{easy}}(\theta, \nu)] \right) / \|\theta^\ast\|\), we have
    \[
     \frac{M_n(\theta, \nu) - M(\theta, \nu)}{\|\theta^\ast\|} = \mathsf{\Sigma}^{-1} \vec{v} - \mathsf{\Sigma}^{-1}(\mathsf{\Sigma} - I_d) \frac{M(\theta, \nu)}{\|\theta^\ast\|}
    \]
    Without loss of generality, we assume \(\dim \text{span}\{\theta, \theta^\ast\} = 2\) and therefore \(\dim \text{span}\{\theta, \theta^\ast\}^\perp = d-2\).
    Let \(P_{\theta,\theta^\ast}^\perp  = \sum_{j=1}^{d-2} \vec{e}_j^\perp (\vec{e}_j^\perp)^\top\)
    and \(U_{\theta,\theta^\ast}^\perp = \left(\vec{e}_1^\perp, \ldots, \vec{e}_{d-2}^\perp\right)\) with \(\text{span}(U_{\theta,\theta^\ast}^\perp) = \text{span}\{\theta, \theta^\ast\}^\perp\).
    For the length of vector \(\vec{v}\), noting that orthogonal projection matrice \(P_{\theta,\theta^\ast}^\perp\) 
    can be expressed with orthonormal basis matrix \(U_{\theta,\theta^\ast}^\perp\)
    such that \(P_{\theta,\theta^\ast}^\perp = U_{\theta,\theta^\ast}^\perp (U_{\theta,\theta^\ast}^\perp)^\top\), we have
    \[
    \|\vec{v}\|^2  
    = \langle (P_{\theta,\theta^\ast}+P_{\theta,\theta^\ast}^\perp) \vec{v}, (P_{\theta,\theta^\ast}+P_{\theta,\theta^\ast}^\perp) \vec{v} \rangle 
    = \langle \vec{v}, (P_{\theta,\theta^\ast}+P_{\theta,\theta^\ast}^\perp) \vec{v} \rangle 
    = \langle \vec{v}, P_{\theta,\theta^\ast} \vec{v} \rangle + \langle \vec{v}, P_{\theta,\theta^\ast}^\perp \vec{v} \rangle 
    = \|P_{\theta,\theta^\ast} \vec{v}\|^2 + \|(U_{\theta,\theta^\ast}^\perp)^\top \vec{v}\|^2
    \]
    where the first term is bounded by \(\|P_{\theta,\theta^\ast} \vec{v}\| \lesssim \sqrt{\frac{\log \frac{1}{\delta}}{n}}\vee \frac{\log \frac{1}{\delta}}{n}\), 
    and we rewrite the second term by introducing independent Gaussian variables
    \(g_i := \langle x_i, \theta^\ast/\|\theta^\ast\|\rangle \sim \mathcal{N}(0, 1), \vec{g}_i := \sgn(\langle x_i, \theta^\ast\rangle \langle x_i, \theta\rangle) (U_{\theta,\theta^\ast}^\perp)^\top x_i \sim \mathcal{N}(0, I_{d-2})\) with \(g_i \ind \vec{g}_i\).
    \[
        \|(U_{\theta,\theta^\ast}^\perp)^\top \vec{v}\|^2
    = \frac{1}{n^2}\left\|\sum_{i=1}^n g_i \vec{g}_i\right\|^2
    = \frac{1}{n^2} \left(\sum_{i=1}^n g_i^2\right) \cdot \left\| \sum_{i=1}^n \frac{g_i}{\sqrt{\sum_{i'=1}^n (g_{i'})^2}} \vec{g}_i\right\|^2 
    = \frac{1}{n^2} Z_1 \cdot Z_2
    \]
    where \(Z_1 := \sum_{i=1}^n g_i^2\sim\chi^2(n)\) and 
    the weighted sum \(\sum_{i=1}^n (g_i/\sqrt{\sum_{i'=1}^n (g_{i'})^2}) \vec{g}_i\sim\mathcal{N}(0, I_{d-2})\) by the rotational invariance of Gaussian distribution,
    therefore \(Z_2 = \| \sum_{i=1}^n (g_i/\sqrt{\sum_{i'=1}^n (g_{i'})^2}) \vec{g}_i\|^2 \sim \chi^2(d-2)\) with \(Z_1 \ind Z_2\).
    By the concentration inequality for Chi-square distribution (see Lemma 1, page 1325 in~\cite{laurent2000adaptive}), given \(n\geq d\), with probability at least \(1-\delta/2\), 
    \[
    \|\vec{v}\| \leq \|P_{\theta,\theta^\ast} \vec{v}\| + \frac{1}{n} \sqrt{Z_1 \cdot Z_2}
    \lesssim \sqrt{\frac{\log \frac{1}{\delta}}{n}}\vee \frac{\log \frac{1}{\delta}}{n}
    + \frac{1}{n}\sqrt{\left(n\vee \log \frac{1}{\delta}\right)\left(d \vee \log \frac{1}{\delta}\right)}
    \asymp \sqrt{\frac{d\vee \log \frac{1}{\delta}}{n}} \vee \frac{\log \frac{1}{\delta}}{n}
    \]
    From the assumption \(n\gtrsim d\vee \log \frac{1}{\delta}\), by Example 6.2 (Operator norm bounds for the standard Gaussian ensemble) and Example 6.3 (Gaussian covariance estimation) on page 162 of~\cite{wainwright2019high}, with probability at least \(1-\delta/2\), we have
    \[
    \gamma_{\min} (\mathsf{\Sigma}) \asymp 1, \quad \left\| \mathsf{\Sigma} - I_d \right\|_2 \lesssim \sqrt{\frac{d\vee \log \frac{1}{\delta}}{n}}
    \]
    where \(\gamma_{\min}\) represents the minimum eigenvalue, and \(\left\| \cdot\right\|_2\) is the \(\ell_2\)-operator norm for matrix.
    Noting that \(\sqrt{\frac{\log \frac{1}{\delta}}{n}} \gtrsim \frac{\log \frac{1}{\delta}}{n}\) when \(n\gtrsim d\vee \log \frac{1}{\delta}\) and \(\|M(\theta, \nu)\|/\|\theta^\ast\|\leq 1\) in the noiseless setting, 
    and putting the above results together, we have
    \[
    \frac{\| M_n(\theta, \nu) - M(\theta, \nu) \|}{\|\theta^\ast\|} 
    \leq \frac{\|\vec{v}\|}{\gamma_{\min}(\mathsf{\Sigma})} 
    + \frac{\left\| \mathsf{\Sigma}-I_d\right\|_2}{\gamma_{\min}(\mathsf{\Sigma})} \frac{\|M(\theta, \nu)\|}{\|\theta^\ast\|}
    \lesssim \sqrt{\frac{d\vee \log \frac{1}{\delta}}{n}}
    \]
\end{proof}

\begin{theorem}[Proposition~\ref{prop:statistical_error_mixing_weights}: Statistical Error of EM Update for Mixing Weights]
    In the noiseless setting, the statistical error of mixing weights for finite-sample EM updates
    satisfies
    \[
    \left| N_n(\theta, \nu) - N(\theta, \nu) \right| 
    = \mathcal{O}\left(\frac{\log \frac{1}{\delta}/n}{\log\left(1+ \frac{\log\frac{1}{\delta}/n}{p}\right)}\wedge \sqrt{\frac{\log\frac{1}{\delta}}{n}} \right) 
    \]
    with probability at least \(1-\delta\), where \(N_n(\theta, \nu), N(\theta, \nu)\) denote the EM update rules for imbalance \(\tanh \nu\) of the mixing weights at the Finite-sample level and the Population level, 
    and \(p :=\| \pi^\ast - \frac{\mathds{1}}{2} \|_1 \frac{\phi}{2\pi} + \min(\pi^\ast(1), \pi^\ast(2))\), \(\phi = 2 \arccos |\rho|, \rho = \frac{\langle \theta, \theta^{\ast} \rangle}{\| \theta \| \cdot \| \theta^{\ast} \|}\).
\end{theorem}

\begin{proof}
    In the noiseless setting, as SNR \(\eta \rightarrow \infty\), we have \(y_i \rightarrow (-1)^{z_i+1}\langle x_i, \theta^\ast\rangle\) and therefore \(\tanh(y_i\langle x_i, \theta\rangle/\sigma^2 + \nu)\to (-1)^{z_i+1}\sgn(\langle x_i, \theta^\ast\rangle \langle x_i, \theta\rangle)\)
    for the \(i\)-th data point \(z_i, s_i =(x_i, y_i)\).
    \[
    N_n(\theta, \nu) = \frac{1}{n} \sum_{i=1}^n (-1)^{z_i+1}\sgn(\langle x_i, \theta^\ast\rangle \langle x_i, \theta\rangle)
    \]
    Let \(\rho = \frac{\langle \theta, \theta^{\ast} \rangle}{\| \theta \| \cdot \| \theta^{\ast} \|}\), 
    and \(\varphi = \frac{\pi}{2} - \arccos |\rho|, \phi = 2 \arccos |\rho|\), therefore \(\rho = \sgn(\rho) \sin \varphi\).
    By defining \(g_i = \langle x_i, \theta^\ast/\|\theta^\ast\|\rangle, g'_i = \langle x_i, \sgn(\rho) \theta/\|\theta\| \rangle\),
    and introducing \(V_i = (1+\sgn(\nu^\ast) (-1)^{z_i+1} \sgn(g_i g'_i))/2\), noting \(N(\theta, \nu) = \E[N_n(\theta, \nu)]\), we have
    \[
    \left| N_n(\theta, \nu) - N(\theta, \nu) \right| = 2 \left| \frac{1}{n} \sum_{i=1}^n (V_i-\E[V_i]) \right|
    \]
    with the i.i.d. random variables \(\{V_i\}_{i=1}^n \stackrel{\text{i.i.d.}}{\sim} \text{Bern}(q)\) following Bernoulli distribution with parameter \(q\geq \frac{1}{2}\) such that
    \[
    q = \E[V_i] = \frac{1}{2} \left(1 + \sgn(\nu^\ast) \E[(-1)^{z_i+1}] \E[\sgn(g_i g'_i)] \right)
     = \frac{1}{2} + \tanh |\nu^\ast| \cdot \frac{\varphi}{\pi}
    \]
    by using the facts that \(\E[(-1)^{z_i+1}] = \tanh \nu^\ast\) and \(\E[\sgn(g_i g'_i)] = \frac{2}{\pi}\varphi\) (see Lemma~\ref{suplem:expectation}).
    Then, we can express the \(1-q\) with \(\phi = \pi - 2 \varphi\) and \(\|\pi - \frac{\mathds{1}}{2} \|_1 = \tanh|\nu^\ast|, \min(\pi^\ast(1), \pi^\ast(2)) = \frac{1 - \tanh|\nu^\ast|}{2}\)
    \[
    1-q = \frac{\pi - \tanh|\nu^\ast| \cdot 2\varphi}{2 \pi} = \left\|\pi - \frac{\mathds{1}}{2} \right\|_1\cdot \frac{\phi}{2\pi} + \min(\pi^\ast(1), \pi^\ast(2)) 
    \]
    By Lemma~\ref{suplem:prob_inequality} of concentration inequality for Bernoulli distribution, with probability at least \(1-\delta\),
    \[
    \left| N_n(\theta, \nu) - N(\theta, \nu) \right| = 
    2\left| \frac{1}{n} \sum_{i=1}^n (V_i-\E[V_i]) \right| 
    \lesssim \min\left(\frac{\log \frac{1}{\delta}/n}{\log\left(1+ \frac{\log\frac{1}{\delta}/n}{1-q}\right)}, \sqrt{\frac{\log\frac{1}{\delta}}{n}} \right) 
    \]
    We complete the proof by introducing \(p :=1-q= \| \pi^\ast - \frac{\mathds{1}}{2} \|_1 \frac{\phi}{2\pi} + \min(\pi^\ast(1), \pi^\ast(2)) \asymp \tanh|\nu^\ast| \phi + (1-\tanh|\nu^\ast|)\). 
\end{proof}

\begin{lemma}[Projected Statistical Accuracy of Easy EM Update for Regression Parameters]\label{lem:projected_statistical_accuracy_regression}
  In the noiseless setting, the projection on \(\text{span}\{\theta, \theta^\ast\}\) for the statistical accuracy of regression parameters \(\theta\) satisfies
  \[
  \frac{\left\| P_{\theta,\theta^\ast}\left[\sgn(\rho)M^{\tmop{easy}}_n(\theta, \nu) - M^{\tmop{easy}}_n(\theta^\ast, \nu^\ast)\right] \right\|}{\|\theta^\ast\|} = \mathcal{O}\left(\phi^2 \vee \phi \frac{\log \frac{1}{\delta}}{n}\right)
  \]
  with probability at least \(1-\delta\), where \(M^{\tmop{easy}}_n(\theta, \nu)\) denotes the Easy EM update rule for regression parameters at the Finite-sample level, 
  orthogonal projection matrix \(P_{\theta,\theta^\ast}\) satisfies \(\text{span}(P_{\theta,\theta^\ast}) = \text{span}\{\theta, \theta^\ast\}\), 
  and \(\phi = 2 \arccos |\rho|, \rho = \frac{\langle \theta, \theta^{\ast} \rangle}{\| \theta \| \cdot \| \theta^{\ast} \|}\).
\end{lemma}

\begin{proof}
  Without loss of generality, we assume \(\dim \text{span}\{\theta, \theta^\ast\} = 2\) and therefore \(\dim \text{span}\{\theta, \theta^\ast\}^\perp = d-2\).
  We introduce \(\rho := \frac{\langle \theta, \theta^{\ast} \rangle}{\| \theta \| \cdot \| \theta^{\ast} \|}\), 
  and define \(\hat{e}_1 := \frac{\theta^{\ast}}{\| \theta^{\ast} \|}\), \(\vec{e}_1 := \frac{\theta}{\| \theta \|}\),
  and \(\hat{e}_2 := \frac{\theta-\langle \theta, \hat{e}_1 \rangle \hat{e}_1}{\| \theta-\langle \theta, \hat{e}_1 \rangle \hat{e}_1 \|}\),
  \(\vec{e}_2 := \frac{\theta^\ast-\langle \theta^\ast, \vec{e}_1 \rangle \vec{e}_1}{\| \theta^\ast-\langle \theta^\ast, \vec{e}_1 \rangle \vec{e}_1 \|}\).
  Also, let \(P_{\theta,\theta^\ast} = \hat{e}_1 \hat{e}_1^\top + \hat{e}_2 \hat{e}_2^\top\), 
  which can be expressed with orthonormal basis matrices \(U_{\theta,\theta^\ast} = (\hat{e}_1, \sgn(\rho)\hat{e}_2)\) 
  such that \(P_{\theta,\theta^\ast} = U_{\theta,\theta^\ast} U_{\theta,\theta^\ast}^\top\).
  Also, we have \(\phi = 2 \arccos |\rho|\), therefore \(\rho = \sgn(\rho) \cos \frac{\phi}{2}\).
  By defining \(g_i := \langle x_i, \hat{e}_1\rangle, g'_i := \sgn(\rho) \langle x_i, \vec{e}_1\rangle\), 
  then \(g_i, g'_i \sim \mathcal{N}(0, 1)\) with \(\mathbb{E}[g_i g'_i] = \cos \frac{\phi}{2}\), and \(h_i=\langle x_i, \sgn(\rho)\hat{e}_2 \rangle \sim \mathcal{N}(0, 1)\) satisfies \(g'_i = \cos \frac{\phi}{2}\cdot g_i + \sin \frac{\phi}{2} \cdot h_i, h_i\ind g_i\).
  \begin{eqnarray*}
    -U_{\theta,\theta^\ast}^\top \frac{\left[\sgn(\rho)M^{\tmop{easy}}_n(\theta, \nu) - M^{\tmop{easy}}_n(\theta^\ast, \nu^\ast)\right]}{\|\theta^\ast\|}
    =\left( \frac{1}{n}\sum_{i=1}^n [1-\sgn(g_i g'_i)] g_i^2 
    ,\space -\frac{1}{n}\sum_{i=1}^n [\sgn(g_i g'_i)-1] g_i h_i \right)^\top
  \end{eqnarray*}
  Applying Proposition~\ref{prop:errors_em_updates_angle} of Errors of EM Updates and Sub-optimality Angle, and \(\E[M^{\tmop{easy}}_n(\theta, \nu)] = M(\theta, \nu),\E[M^{\tmop{easy}}_n(\theta^\ast, \nu^\ast)] = \theta^\ast \in \text{span}\{\theta, \theta^\ast\}
  = \text{span}(U_{\theta,\theta^\ast})\), then expectation \(\E[U_{\theta,\theta^\ast}^\top\vec{v}]\) of \(\vec{v} := [\sgn(\rho)M^{\tmop{easy}}_n(\theta, \nu) - M^{\tmop{easy}}_n(\theta^\ast, \nu^\ast)]/\|\theta^\ast\|\) is bounded:
  \[
  \|\E[U_{\theta,\theta^\ast}^\top\vec{v}]\| =
  \left\| U_{\theta,\theta^\ast}^\top \frac{\E\left[\sgn(\rho)M^{\tmop{easy}}_n(\theta, \nu) - M^{\tmop{easy}}_n(\theta^\ast, \nu^\ast)\right]}{\|\theta^\ast\|} \right\|
  =\|\E[\vec{v}]\|
  = \frac{\| M(\theta, \nu) - \sgn(\rho) \theta^\ast \|}{\|\theta^\ast \|} \leq \frac{\phi^2}{2 \pi}
  \]
  Applying \(\sqrt{(s+s')^2 + (t+t')^2}\leq \sqrt{s^2 + t^2} + \sqrt{2}\max(s', t')_+\) for \(s+s', t+t'\geq0\), 
  noting \(P_{\theta,\theta^\ast}= U_{\theta,\theta^\ast} U_{\theta,\theta^\ast}^\top\): 
  \[
    \|P_{\theta,\theta^\ast} \vec{v}\| = \|U_{\theta,\theta^\ast}^\top \vec{v}\| 
    \lesssim \phi^2 + \max\left(\left(\frac{1}{n}\sum_{i=1}^n -\E\right)[1-\sgn(g_i g'_i)] g_i^2, \left(\frac{1}{n}\sum_{i=1}^n -\E\right)[\sgn(g_i g'_i)-1] g_i h_i\right)_+
  \]
  By introducing \(\{\varrho_i\}_{i=1}^n = \{R_i^2/2\}_{i=1}^n \stackrel{\text{i.i.d.}}{\sim} \tmop{Exp}(1)\) with \(\{R_i\}_{i=1}^n \stackrel{\text{i.i.d.}}{\sim} r\exp(r^2/2)\mathds{1}_{r\geq 0}\) (the standard Rayleigh distribution), 
  \(\{V_i\}_{i=1}^n \stackrel{\text{i.i.d.}}{\sim} \tmop{Unif}[0, 2\pi)\), and \(U_i := 2 V_i - \pi \mod (2\pi) \sim \tmop{Unif}[0, 2\pi)\), 
  we express \(g_i = R_i \cos V_i, h_i = R_i \sin V_i\) and follows:
  \[
  X_i := [1-\sgn(g_i g'_i)] g_i^2 = 2\varrho_i (1-\cos U_i) \mathds{1}_{U_i\in [0, \phi)} \geq 0,\quad
  Y_i := [\sgn(g_i g'_i)-1] g_i h_i = 2\varrho_i \sin U_i \mathds{1}_{U_i\in [0, \phi)} \geq 0
  \]
  Noting \(\E[\varrho_i^q] = q!\) for \(q\in \mathbb{Z}_+\), and \(2 (1-\cos U_i)\mathds{1}_{U_i\in [0, \phi)} \leq \phi^2 \mathds{1}_{U_i\in [0, \phi)}, 2\sin U_i\leq 2\phi \mathds{1}_{U_i\in [0, \phi)}\), then for \(q\geq 3\):
  \[
  \E[X_i^2] \leq 2! [\phi^2]^2 \E[\mathds{1}_{U_i\in [0, \phi)}] = \frac{\phi^5}{\pi},\quad
  \E[X_i^q] \leq q! [\phi^2]^q \E[\mathds{1}_{U_i\in [0, \phi)}] = \frac{q! \phi^{2q+1}}{2\pi} = \frac{q!}{2} \frac{\phi^5}{\pi} [\phi^2]^{q-2}
  \]
  \[
  \E[Y_i^2] \leq 2! [2\phi]^2 \E[\mathds{1}_{U_i\in [0, \phi)}] = \frac{4\phi^3}{\pi},\quad
  \E[Y_i^q] \leq q! [2\phi]^q \E[\mathds{1}_{U_i\in [0, \phi)}] = \frac{q! 2^{q} \phi^{q+1}}{2\pi} = \frac{q!}{2} \frac{4\phi^3}{\pi} [2\phi]^{q-2}
  \]
  By Bernstein's inequality (see Theorem 2.10 on page 37 of~\cite{boucheron2013concentration}) with probability at least \(1-\delta\), noting \(\phi^2 \lesssim \phi\), we have:
  \[
  \max\left(\left(\frac{1}{n}\sum_{i=1}^n-\E\right) X_i, \left(\frac{1}{n}\sum_{i=1}^n-\E\right) Y_i\right)_+
  \lesssim \max\left(\phi^{\frac{5}{2}}\sqrt{\frac{\log \frac{1}{\delta}}{n}}+\phi^2 \frac{\log \frac{1}{\delta}}{n} , \phi^{\frac{3}{2}}\sqrt{\frac{\log \frac{1}{\delta}}{n}}+\phi \frac{\log \frac{1}{\delta}}{n} \right) \asymp \phi^{\frac{3}{2}}\sqrt{\frac{\log \frac{1}{\delta}}{n}}+\phi \frac{\log \frac{1}{\delta}}{n}
  \]
  Therefore, we complete the proof by using \(a + \sqrt{ab} + b \asymp a + b \asymp a\vee b\) for \(a, b \geq 0\).
\end{proof}

\begin{theorem}[Proposition~\ref{prop:statistical_accuracy_em_updates}: Statistical Accuracy of EM Updates for Regression Parameters and Mixing Weights]
  In the noiseless setting, the finite-sample EM with \(n\gtrsim d\vee \log \frac{1}{\delta}\) samples achieves the statistical accuracy of regression parameters and mixing weights:
  \begin{equation*}
    \begin{aligned}
      \frac{\left\| M_n(\theta, \nu) - \sgn(\rho) \theta^\ast \right\|}{\|\theta^\ast\|} 
      &= \mathcal{O}\left(\phi^2 \vee \phi^{\frac{3}{2}} \sqrt{\frac{d}{n}} \vee \phi \sqrt{\frac{\log \frac{1}{\delta}}{n}}\right)\\
      \left| N_n(\theta, \nu) - \sgn(\rho) \tanh \nu^\ast\right| &= \mathcal{O}\left( 
        \phi \left\| \pi^\ast -\frac{\mathds{1}}{2} \right\|_1 \vee \left[ 
          \frac{\log\frac{1}{\delta}/n}{\log\left(1+\frac{\log \frac{1}{\delta}/n}{p}\right)}
          \wedge \sqrt{\frac{\log \frac{1}{\delta}}{n}}
        \right]
      \right)
    \end{aligned}
  \end{equation*}
  with probability at least \(1-\delta\), where \(M_n(\theta, \nu)\) denotes the EM update rule for regression parameters at the Finite-sample level, 
  and \(\phi := 2 \arccos |\rho|, \rho := \frac{\langle \theta, \theta^{\ast} \rangle}{\| \theta \| \cdot \| \theta^{\ast} \|}\)
  and \(p:=\frac{\phi}{2\pi}\|\pi^\ast - \frac{\mathds{1}}{2}\|_1 + \min(\pi^\ast(1), \pi^\ast(2))\).
\end{theorem}

\begin{proof}
  Noting \(\mathsf{\Sigma} M_n(\theta, \nu) = M_n^{\text{easy}}(\theta, \nu)=\frac{1}{n}\sum_{i=1}^n \sgn(\langle x_i, \theta^\ast \rangle \langle x_i, \theta \rangle) \langle x_i, \theta^\ast \rangle x_i\) in the noiseless setting,
  with the ensemble matrix of data samples \(\mathsf{\Sigma} = \frac{1}{n} \sum_{i=1}^n x_i x_i^\top\), and letting \(\vec{v} := [\sgn(\rho)M^{\tmop{easy}}_n(\theta, \nu) - M^{\tmop{easy}}_n(\theta^\ast, \nu^\ast)]/\|\theta^\ast\|\), then:
  \[
    \frac{M_n(\theta, \nu) - \sgn(\rho) \theta^\ast}{\|\theta^\ast\|} 
    = \sgn(\rho) \mathsf{\Sigma}^{-1} \cdot \vec{v} 
  \]
  
  Let \(P_{\theta,\theta^\ast}^\perp  = \sum_{j=1}^{d-2} \hat{e}_j^\perp (\hat{e}_j^\perp)^\top\)
  and \(U_{\theta,\theta^\ast}^\perp = \left(\hat{e}_1^\perp, \ldots, \hat{e}_{d-2}^\perp\right)\) with \(\text{span}(U_{\theta,\theta^\ast}^\perp) = \text{span}\{\theta, \theta^\ast\}^\perp\).
  For the length of vector \(\vec{v}\), noting that orthogonal projection matrice \(P_{\theta,\theta^\ast}^\perp\) 
  can be expressed with orthonormal basis matrix \(U_{\theta,\theta^\ast}^\perp\)
  such that \(P_{\theta,\theta^\ast}^\perp = U_{\theta,\theta^\ast}^\perp (U_{\theta,\theta^\ast}^\perp)^\top\).
  \[
  \|\vec{v}\|^2  
  = \langle (P_{\theta,\theta^\ast}+P_{\theta,\theta^\ast}^\perp) \vec{v}, (P_{\theta,\theta^\ast}+P_{\theta,\theta^\ast}^\perp) \vec{v} \rangle 
  = \langle \vec{v}, (P_{\theta,\theta^\ast}+P_{\theta,\theta^\ast}^\perp) \vec{v} \rangle 
  = \langle \vec{v}, P_{\theta,\theta^\ast} \vec{v} \rangle + \langle \vec{v}, P_{\theta,\theta^\ast}^\perp \vec{v} \rangle 
  = \|P_{\theta,\theta^\ast} \vec{v}\|^2 + \|(U_{\theta,\theta^\ast}^\perp)^\top \vec{v}\|^2
  \]
  where the first term is bounded by \(\|P_{\theta,\theta^\ast} \vec{v}\| \lesssim \phi^2 \vee \phi\sqrt{\frac{\log \frac{1}{\delta}}{n}}\), 
  and we rewrite the second term by introducing Gaussian r.v.
  \(g_i := \langle x_i, \theta^\ast/\|\theta^\ast\|\rangle, g'_i := \langle x_i, \sgn(\rho)\theta/\|\theta\|\rangle \sim \mathcal{N}(0, 1), \vec{g}_i := \sgn(\langle x_i, \theta^\ast\rangle \langle x_i, \theta\rangle) (U_{\theta,\theta^\ast}^\perp)^\top x_i \sim \mathcal{N}(0, I_{d-2})\) with \(g_i, g'_i \ind \vec{g}_i\).
  \[
      \|(U_{\theta,\theta^\ast}^\perp)^\top \vec{v}\|^2
  = \frac{1}{n^2}\left\|\sum_{i=1}^n [1-\sgn(g'_i g_i)]g_i \vec{g}_i\right\|^2
  = \frac{2}{n^2} \left(\sum_{i=1}^n [1-\sgn(g'_i g_i)]g_i^2\right) \cdot \left\| \sum_{i=1}^n \frac{[1-\sgn(g'_i g_i)]g_i}{\sqrt{\sum_{i'=1}^n 2[1-\sgn(g'_{i'} g_{i'})](g_{i'})^2}} \vec{g}_i\right\|^2 
  \]
  where the weighted sum \(\sum_{i=1}^n ([1-\sgn(g'_i g_i)]g_i/\sqrt{\sum_{i'=1}^n 2[1-\sgn(g'_{i'} g_{i'})](g_{i'})^2}) \vec{g}_i\sim\mathcal{N}(0, I_{d-2})\) by the rotational invariance of Gaussian distribution,
  therefore \(\| \sum_{i=1}^n ([1-\sgn(g'_i g_i)]g_i/\sqrt{\sum_{i'=1}^n 2[1-\sgn(g'_{i'} g_{i'})](g_{i'})^2}) \vec{g}_i\|^2 \sim \chi^2(d-2)\).
  By the concentration inequality for Chi-square distribution (see Lemma 1, page 1325 in~\cite{laurent2000adaptive}), then:
  \[
  \frac{1}{n}\left\| \sum_{i=1}^n \frac{[1-\sgn(g'_i g_i)]g_i}{\sqrt{\sum_{i'=1}^n 2[1-\sgn(g'_{i'} g_{i'})](g_{i'})^2}} \vec{g}_i\right\|^2 \lesssim \frac{d\vee \log \frac{1}{\delta}}{n} 
  \]
  Letting \(\{\varrho_i\}_{i=1}^n = \{R_i^2/2\}_{i=1}^n \stackrel{\text{i.i.d.}}{\sim} \tmop{Exp}(1)\) with \(\{R_i\}_{i=1}^n \stackrel{\text{i.i.d.}}{\sim} r\exp(r^2/2)\mathds{1}_{r\geq 0}\) (the standard Rayleigh distribution), 
  \(\{V_i\}_{i=1}^n \stackrel{\text{i.i.d.}}{\sim} \tmop{Unif}[0, 2\pi)\), and \(U_i := 2 V_i - \pi \mod (2\pi) \sim \tmop{Unif}[0, 2\pi)\), 
  we express \(g_i = R_i \cos V_i, h_i = R_i \sin V_i, g'_i = \cos \frac{\phi}{2}\cdot g_i+\sin\frac{\phi}{2}\cdot h_i\).
  \[
  X_i := [1-\sgn(g_i g'_i)] g_i^2 = 2\varrho_i (1-\cos U_i) \mathds{1}_{U_i\in [0, \phi)} \geq 0,\quad
  \E[X_i] \leq \E[\varrho_i] \E[\phi^2 \mathds{1}_{U_i\in[0, \phi)}] = \frac{\phi^3}{2\pi}
  \]
  By the concentration inequality for \(X_i\) obtained in proof of previous Lemma, \((\frac{1}{n}\sum_{i=1}^n -\E)X_i \lesssim \phi^{\frac{5}{2}}\sqrt{\frac{\log \frac{1}{\delta}}{n}} + \phi^2 \frac{\log \frac{1}{\delta}}{n}\), then:
  \[
  \frac{1}{n}\sum_{i=1}^n [1-\sgn(g'_i g_i)]g_i^2 = \frac{1}{n}\sum_{i=1}^n X_i 
  \lesssim \phi^3 + \phi^{\frac{5}{2}}\sqrt{\frac{\log \frac{1}{\delta}}{n}} + \phi^2 \frac{\log \frac{1}{\delta}}{n}
  \asymp \phi^2 \left(\phi \vee \frac{\log \frac{1}{\delta}}{n}\right)
  \]
  By combining the above inequalities together, we bound \(\vec{v}\) as follows with \(n\gtrsim d\vee \log \frac{1}{\delta}\) and probability at least \(1-\delta/4\).
  \[
  \|\vec{v}\| \leq \|P_{\theta, \theta^\ast} \vec{v}\| + \|(U_{\theta,\theta^\ast}^\perp)^\top \vec{v}\|
  \lesssim \phi^2 \vee \phi \sqrt{\frac{\log \frac{1}{\delta}}{n}}
  + \sqrt{ \phi^2 \left(\phi \vee \frac{\log \frac{1}{\delta}}{n}\right) \frac{d\vee \log \frac{1}{\delta}}{n} }
  \asymp \phi^2 \vee \phi^{\frac{3}{2}} \sqrt{\frac{d}{n}} \vee \phi \sqrt{\frac{\log \frac{1}{\delta}}{n}}
  \] 
  From the assumption \(n\gtrsim d\vee \log \frac{1}{\delta}\), by Example 6.2 (Operator norm bounds for the standard Gaussian ensemble) and Example 6.3 (Gaussian covariance estimation) on page 162 of~\cite{wainwright2019high}, with probability at least \(1-\delta/4\), we have
  \(
  \gamma_{\min} (\mathsf{\Sigma}) \asymp 1
  \),
  where \(\gamma_{\min}\) represents the minimum eigenvalue.
  \[
    \frac{\left\| M_n(\theta, \nu) - \sgn(\rho) \theta^\ast \right\|}{\|\theta^\ast\|} 
    \leq \frac{\|\vec{v}\|}{\gamma_{\min}(\mathsf{\Sigma})} \asymp
    \phi^2 \vee \phi^{\frac{3}{2}} \sqrt{\frac{d}{n}} \vee \phi \sqrt{\frac{\log \frac{1}{\delta}}{n}}
  \]
  For the mixing weights, by using the Triangle inequality, Proposition for error of mixing weights in population level,
  and the Proposition for statistical error of mixing weights in finite-sample level, we have the following with probability at least \(1-\delta/2\).
  \begin{eqnarray*}
    \left| N_n(\theta, \nu) - \sgn(\rho) \tanh \nu^\ast\right|
    & \leq & \left| N(\theta, \nu) - \sgn(\rho) \tanh\nu^\ast \right| + \left| N_n(\theta, \nu) - N(\theta, \nu) \right|\\
    &\lesssim& \phi \left\| \pi^\ast - \frac{\mathds{1}}{2} \right\|_1 
    + \min\left(\frac{\log \frac{1}{\delta}/n}{\log\left(1+ \frac{\log\frac{1}{\delta}/n}{p}\right)}, \sqrt{\frac{\log\frac{1}{\delta}}{n}} \right) 
  \end{eqnarray*}
  where \(p := \| \pi^\ast - \frac{\mathds{1}}{2} \|_1 \frac{\phi}{2\pi} + \min(\pi^\ast(1), \pi^\ast(2)) \asymp \tanh|\nu^\ast| \phi + (1-\tanh|\nu^\ast|)\).
\end{proof}

\subsection{Initialization of Easy EM Updates and Perturbation in Sub-Optimality Angle}

\begin{theorem}[Proposition~\ref{prop:initialization_easy_em}: Initialization with Easy EM, Proposition 5.4 in~\cite{luo24cycloid}]
    In the noiseless setting, suppose we run the sample-splitting finite-sample
    Easy EM with \ $n' \asymp \frac{n}{\log \frac{1}{\delta}}
    $ fresh samples for each iteration, then after at most $T_0 =\mathcal{O} \left( \log
    \frac{1}{\delta} \right)$ iterations, it satisfies $\varphi^{T_0} \gtrsim
        \sqrt{\frac{\log \frac{1}{\delta}}{n}} 
    $ 
    with probability at least $1 - \delta$.
\end{theorem}
\begin{proof}
  See the detailed proof on pages 41-44, Appendix E of~\cite{luo24cycloid} based on Berry-Esseen bound for the central limit theorem (see Theorem 1.1 of~\cite{ross2011fundamentals}).
  The proposition is restated here for completeness and to keep the paper self-contained.
\end{proof}

\begin{lemma}[Perturbation in Sub-Optimality Angle]\label{lem:single_iter}
    For $\vartheta \assign \sin \varphi \hat{e}_1 + \cos \varphi \hat{e}_2$,
    where $\{ \hat{e}_1, \hat{e}_2 \}$ is an orthonormal basis for the subspace
    $\tmop{span} \{ \hat{e}_1, \hat{e}_2 \} \subset \mathbb{R}^d$, and $\varphi
    \in \left( 0, \frac{\pi}{2} \right)$; with a pertubation vector $\varrho \in
    \mathbb{R}^d$ with lengh $\| \varrho \| = r \in (0, \sin \varphi)$; then the
    angle $\varphi' \assign
    \arcsin \frac{| \langle \vartheta + \varrho, \hat{e}_1 \rangle |}{\|
    \vartheta + \varrho \|}$, satisfies $\varphi' \geq \varphi - \arcsin r$
\end{lemma}

\begin{proof}
    Note that with $\| \varrho \| = r \in (0, \sin \varphi)$, $\langle \vartheta
    + \varrho, \hat{e}_1 \rangle = \langle \vartheta, \hat{e}_1 \rangle +
    \langle \varrho, \hat{e}_1 \rangle \geq \sin \varphi - \| \varrho \| = \sin
    \varphi - r > 0$, thus
    \[ \sin \varphi' = \frac{\langle \vartheta + \varrho, \hat{e}_1 \rangle}{\|
       \vartheta + \varrho \|} > 0 \]
    Express the pertubation vector by $\varrho = - r' \cos (\varphi - \Delta)
    \hat{e}_1 + r' \sin (\varphi - \Delta) \hat{e}_2 + \sqrt{r^2 - \left[ {r'} 
    \right]^2} \hat{e}$, where $r' \in [0, r], \Delta \in (- \pi, \pi]$ and
    $\hat{e} \in \tmop{span} \{ \hat{e}_1, \hat{e}_2 \}^{\perp}, \| \hat{e} \| =
    1$
    \begin{eqnarray*}
      \langle \vartheta + \varrho, \hat{e}_1 \rangle & = & \langle \vartheta,
      \hat{e}_1 \rangle + \langle \varrho, \hat{e}_1 \rangle = \sin \varphi - r'
      \cos (\varphi - \Delta)\\
      \| \vartheta + \varrho \| & = & \left\| [\sin \varphi - r' \cos (\varphi -
      \Delta)] \hat{e}_1 + [\cos \varphi + r' \sin (\varphi - \Delta)] \hat{e}_2
      + \sqrt{r^2 - \left[ {r'}  \right]^2} \hat{e} \right\|
      = \sqrt{[1 + r^2] - 2 r' \sin \Delta}
    \end{eqnarray*}
    Hence, let $p : = \frac{r' | \sin \Delta |}{r}\in [0, 1], - r' \cos \Delta \geq -
    \sqrt{[r']^2 - [r p]^2} \geq - \sqrt{r^2 - [r p]^2} = - r \sqrt{1 - p^2}$
    and $p = \frac{r'}{r} | \sin \Delta | \leq 1$
    \begin{eqnarray*}
      \sin \varphi' 
      & = & \frac{[1 - r' \sin \Delta] \sin \varphi - r' \cos \Delta \cos
      \varphi}{\sqrt{[1 + r^2] - 2 r' \sin \Delta}}
      \geq \frac{[1 - r p] \sin \varphi - r \sqrt{1 - p^2} \cos
      \varphi}{\sqrt{[1 + r^2] - 2 r p}} \assign \psi (p)
    \end{eqnarray*}
    For $\psi (p) \assign \frac{[1 - r p] \sin \varphi - r \sqrt{1 - p^2} \cos
    \varphi}{\sqrt{[1 + r^2] - 2 r p}}$, by $\cos
    \left( \varphi + \left[ \frac{\pi}{2} - \arcsin p \right] \right) \leq \cos
    \varphi$,
    then $\cos \varphi - r \cos \left( \varphi + \left[ \frac{\pi}{2} - \arcsin
    p \right] \right) > 0$
    \begin{eqnarray*}
      \frac{\mathd}{\mathd p} \log \psi 
      & = & \frac{r \cdot \left\{ \cos \varphi - r \cos \left( \varphi + \left[
      \frac{\pi}{2} - \arcsin p \right] \right) \right\} (p - r)}{\sqrt{1 - p^2}
      \left\{ [1 - r p] \sin \varphi - r \sqrt{1 - p^2} \cos \varphi \right\}
      \cdot \{ [1 + r^2] - 2 r p \}}
    \end{eqnarray*}
    therefore, $\frac{\mathd}{\mathd p} \log \psi < 0, \forall p \in (0, r) ;
    \frac{\mathd}{\mathd p} \log \psi > 0, \forall p \in (r, 1)$, hence $\psi
    (p) \geq \psi (p) \mid_{p = r}$ for $p \in [0, 1]$
    \begin{eqnarray*}
      \sin \varphi' & \geq & \psi (p)
      \geq \psi (p) \mid_{p = r}
      \sqrt{1 - r^2} \sin \varphi - r \cos \varphi
      = \sin (\varphi - \arcsin r)
    \end{eqnarray*}
    Note that $r \in (0, \sin \varphi)$, that is $\frac{\pi}{2} > \varphi >
    \varphi - \arcsin r > 0$, and we show that
    \[ \varphi' \geq \varphi - \arcsin r \]
\end{proof}

\subsection{Convergence Guarantees for Finite-Sample EM Updates in Noiseless Setting}
\begin{theorem}[Proposition~\ref{prop:convergence_angle}: Convergence of Sub-Optimality Angle]
    In the noiseless setting, suppose $\varphi^0 \gtrsim
    \sqrt{\frac{\log \frac{1}{\delta}}{n}} 
    $, given a positive number \(\varepsilon \lesssim \frac{\log \frac{1}{\delta}}{n}\), we run Easy finite-sample EM for $T_1=\mathcal{O}\left( \log
    \frac{d}{\log \frac{1}{\delta}}\right)$ iterations followed by the standard finite-sample EM for at most \(T' 
    = \mathcal{O}\left( [\log\frac{n}{d}\wedge \log \frac{n}{\log \frac{1}{\delta}}]\vee \log[\log\frac{n}{\ln \frac{1}{\delta}}/ \log\frac{n}{d}] \vee [\log\frac{1}{\varepsilon}/\log\frac{n}{\log \frac{1}{\delta}}]\right) \) iterations
    with $n \gtrsim
        d \vee \log \frac{1}{\delta} 
    $ samples, then it satisfies
    \begin{equation*}
      \phi^T \leq \varepsilon,
    \end{equation*}
    with probability at least $1 - T \delta$, where $T:=T_1+T',\varphi^0 \assign
    \frac{\pi}{2} - \arccos \left| \frac{\langle \theta^0, \theta^{\ast}
    \rangle}{\| \theta^0 \| \cdot \| \theta^{\ast} \|} \right|$ and $\varphi^T
    \assign \frac{\pi}{2} - \arccos \left| \frac{\langle \theta^T, \theta^{\ast}
    \rangle}{\| \theta^T \| \cdot \| \theta^{\ast} \|} \right|$.
\end{theorem}

\begin{proof}
    We provide a new concise and rigorous proof for the Proposition for convergence of angle, which is reproduced here for completeness.
    The other old lengthy proof of this Proposition for convergence of angle can be found on pages 45-50, Appendix E of \cite{luo24cycloid}.
    We divide the proof into four stages: 
    
    (1) In the first stage, we show that after running the Easy EM for at most \(T_1=\mathcal{O}\left( \log \frac{d}{\log \frac{1}{\delta}}\right)\) iterations at the finite-sample level, given the initial sub-optimality angle \(\varphi^0 \gtrsim \sqrt{\frac{\log \frac{1}{\delta}}{n}}\),
    the sub-optimality angle satisfies \(\varphi^{T_1} \gtrsim \sqrt{\frac{d\vee\log \frac{1}{\delta}}{n}}\).

    (2) In the second stage, we show that after running the standard EM for at most \(T_2 =\mathcal{O} \left(
    \log \frac{n}{d} 
      \wedge \log \frac{n}{\log \frac{1}{\delta}} \right)\) iterations at the finite-sample level, given that \(\varphi^{T_1} \gtrsim \sqrt{\frac{d\vee\log \frac{1}{\delta}}{n}}\),
    the sub-optimality angle satisfies \(\phi^{T_1+T_2} \lesssim 1\).

    (3) In the third stage, we show that after running the standard EM for at most \(T_3 =\mathcal{O} \left(
    \log[\log \frac{n}{d} 
      \wedge \log \frac{n}{\log \frac{1}{\delta}}] \right)\) iterations at the finite-sample level, given that \(\phi^{T_1+T_2} \lesssim 1\),
    the sub-optimality angle satisfies \(\phi^{T_1+T_2+T_3} \lesssim \sqrt{\frac{d\vee\log \frac{1}{\delta}}{n}}\).

    (4) In the fourth stage, we show that after running the standard EM for at most \(T_4 = \mathcal{O} \left( 
      \log [\log\frac{n}{\ln \frac{1}{\delta}}/ \log\frac{n}{d}]
      \vee [\log\frac{1}{\varepsilon}/\log\frac{n}{\log \frac{1}{\delta}}]\right)\) iterations at the finite-sample level, given that \(\phi^{T_1+T_2+T_3} \lesssim \sqrt{\frac{d\vee\log \frac{1}{\delta}}{n}}\),
    the sub-optimality angle satisfies \(\phi^{T_1+T_2+T_3+T_4} \leq \varepsilon\).

    We denote \(\vec{e}_1 = \theta^\ast/\|\theta^\ast\|\) as the unit vector in the direction of \(\theta^\ast\), 
    also define \(\varphi^t = \frac{\pi}{2} - \arccos \left| \frac{\langle \theta^t, \theta^\ast \rangle}{\| \theta^t \| \cdot \| \theta^\ast \|} \right|, \phi^t = 2 \arccos \left| \frac{\langle \theta^t, \theta^\ast \rangle}{\| \theta^t \| \cdot \| \theta^\ast \|} \right|\) as the sub-optimality angles. 
    Also, we denote \(\bar{\theta}^t = M(\theta^{t-1}, \nu^{t-1})\) and \(\bar{\varphi}^t := \frac{\pi}{2} - \arccos \left| \frac{\langle \bar{\theta}^t, \theta^\ast \rangle}{\| \bar{\theta}^t \| \cdot \| \theta^\ast \|} \right|, \bar{\phi}^t = 2 \arccos \left| \frac{\langle \bar{\theta}^t, \theta^\ast \rangle}{\| \bar{\theta}^t \| \cdot \| \theta^\ast \|} \right|\) as the sub-optimality angles of \(\bar{\theta}^t\) for the analysis of population level.
    Further, we denote \(c_{\text{proj}}, c_{\text{stat}}\) as the the universal constants for the projected error and the statistical error
    such that \(\|P_{\theta, \theta^\ast}[M_n^{\text{easy}}(\theta, \nu)-M(\theta, \nu)]\|/\|\theta^\ast\| \leq c_{\text{proj}}\times \sqrt{\frac{\log \frac{1}{\delta}}{n}}\) and \(\|M_n^{\text{easy}}(\theta, \nu)-M(\theta, \nu)\|/\|\theta^\ast\|, \|M_n(\theta, \nu)-M(\theta, \nu)\|/\|\theta^\ast\| \leq c_{\text{stat}}\times \sqrt{\frac{d\vee\log \frac{1}{\delta}}{n}}\)
    with \(n \gtrsim d \vee \log \frac{1}{\delta}\) large enough sample size.
    
    \textbf{Stage 1: run Easy EM from \(\varphi^0 \gtrsim \sqrt{\frac{\log \frac{1}{\delta}}{n}}\) to \(\varphi^{T_1} \gtrsim \sqrt{\frac{d\vee\log \frac{1}{\delta}}{n}}\)}

    Let's select sample number \(n \gtrsim d \vee \log \frac{1}{\delta}\) large enough
    such that assumption \(\varphi^t < C_1 \sqrt{\frac{d \vee\log \frac{1}{\delta}}{n}} < \frac{1}{20}\) holds for all \(t \in [0, T_1)\),
    and the statistical error of \(\|\theta^{t} - \bar{\theta}^t\|/\|\theta^\ast\| \leq c_{\text{stat}}\times \sqrt{\frac{d\vee\log \frac{1}{\delta}}{n}} < \frac{2}{\pi}\times\frac{1}{20}\) is also small enough.
    then we can establish the lower bound for the projection of \(\bar{\theta}^{t+1} =M(\theta^t, \nu^t)\) onto \(\vec{e}_1\), which is denoted by 
    \begin{eqnarray*}
      \frac{\langle \bar{\theta}^{t+1}, \hat{e}_1 \rangle}{\langle \theta^t, \hat{e}_1 \rangle}
      = \frac{\|M(\theta^t, \nu^t)\|}{\| \bar{\theta}^{t} +  (\theta^t - \bar{\theta}^t) \|}\times \frac{\tan \bar{\varphi}^{t+1}}{\tan \varphi^t} \times \frac{\cos \bar{\varphi}^{t+1}}{\cos \varphi^{t}}
      \geq \frac{2/\pi}{(2/\pi)(\frac{1}{20}+\cos \frac{1}{20})+(2/\pi)\times\frac{1}{20}}\times \frac{1+ \sqrt{5}}{2} \times \frac{1}{\tan \frac{1}{20} + 1} > \frac{7}{5}
    \end{eqnarray*}
    where the first inequality is by the fact the distance from the point on the cycloid trajectory to the origin is increasing as the sub-optimality angle \(\varphi^t\) is increasing, 
    and its minimum value is \(2/\pi\), and the second inequality is due to the Proposition for the linear growth of \(\tan \varphi^t\),
    and the last inequality is from the proof of the Proposition of the recurrence relation.
    Without loss of generality, we assume \(\langle \theta^0, \hat{e}_1 \rangle >0\), then by the assumption \(\varphi^0 \gtrsim \sqrt{\frac{\log \frac{1}{\delta}}{n}}\), we have \(\langle \theta^t, \hat{e}_1 \rangle >0\) for all \(t\), then
    \begin{eqnarray*}
      \sin \varphi^{t+1} = \frac{\langle \bar{\theta}^{t+1}, \hat{e}_1 \rangle - \langle \bar{\theta}^{t+1}-\theta^{t+1}, \hat{e}_1 \rangle}{\| \bar{\theta}^{t+1} + (\theta^{t+1}-\bar{\theta}^{t+1}) \|}
      \geq \frac{\frac{7}{5}\langle \theta^t, \hat{e}_1 \rangle/\|\theta^\ast\| - c_{\text{proj}}\times \sqrt{\frac{\log \frac{1}{\delta}}{n}}}{\|\bar{\theta}^{t+1} \|/\|\theta^\ast\| + \frac{2}{\pi}\times\frac{1}{20}}
      \geq \frac{\frac{7}{5}\times \frac{19}{20}}{\frac{1}{10}+\cos \frac{1}{20}} \sin \varphi^t - \frac{c_{\text{proj}}}{\frac{1}{10}+\cos \frac{1}{20}}\times \sqrt{\frac{\log \frac{1}{\delta}}{n}}
    \end{eqnarray*}
    where the second inequality is due the following fact and noting that \(\sin \varphi^t = \frac{\langle \theta^t, \hat{e}_1 \rangle}{\|\theta^t\|}\):
    \begin{eqnarray*}
    \frac{2}{\pi} \leq
    \|\bar{\theta}^{t+1} \|/\|\theta^\ast\| \leq \frac{2}{\pi} \times \left(\frac{1}{20} + \cos \frac{1}{20}\right), \quad
    \|\theta^t\|/\|\theta^\ast\| \geq \|\bar{\theta}^t\|/\|\theta^\ast\| - \|\theta^t - \bar{\theta}^t\|/\|\theta^\ast\|
    \geq \frac{2}{\pi} \times \left(1- \frac{1}{20}\right)
    \end{eqnarray*}
    Suppose \(\varphi^0 \geq 45 c_{\text{proj}} \sqrt{\frac{\log \frac{1}{\delta}}{n}}\) by \(\varphi^0 \gtrsim \sqrt{\frac{\log \frac{1}{\delta}}{n}}\), 
    then \(\sin \varphi^0 - 30 c_{\text{proj}} \sqrt{\frac{\log \frac{1}{\delta}}{n}} \geq 10 c_{\text{proj}} \sqrt{\frac{\log \frac{1}{\delta}}{n}}\), the linear growth is shown.
    \begin{eqnarray*}
    \sin \varphi^{t+1} \geq (1+\frac{1}{30}) \sin \varphi^t - c_{\text{proj}}\times \sqrt{\frac{\log \frac{1}{\delta}}{n}}
    \implies \sin \varphi^{t+1} - 30c_{\text{proj}}\times \sqrt{\frac{\log \frac{1}{\delta}}{n}} \geq (1+\frac{1}{30}) \left(\sin \varphi^t - 30c_{\text{proj}}\times \sqrt{\frac{\log \frac{1}{\delta}}{n}}\right)
    \end{eqnarray*}
    Since the population EM update only depends on the sub-optimality angle \(\varphi^t\), the above inequality for \(t>0\) also holds for \(t=0\) without loss of generality.
    Therefore, by the induction, we have shown that after running the Easy EM for at most \(T_1\leq \lceil \frac{\ln(C_1/(10c_{\text{proj}}))+\frac{1}{2}\max(0, \ln(d/\log \frac{1}{\delta}))}{\ln(1+1/30)} \rceil_+ \asymp \log \frac{d}{\log \frac{1}{\delta}}\) iterations,
    we have 
    \(\varphi^{T_1} \geq \sin \varphi^{T_1} \geq C_1 \sqrt{\frac{d \vee\log \frac{1}{\delta}}{n}}\) with some universal constant \(C_1 > 0\).

    \textbf{Stage 2: run Standard EM from \(\varphi^{T_1} \gtrsim \sqrt{\frac{d\vee\log \frac{1}{\delta}}{n}}\) to \(\phi^{T_1+T_2} \lesssim 1\)}

    Given that \(\varphi^{T_1} \geq C_1 \sqrt{\frac{d \vee\log \frac{1}{\delta}}{n}}\), for some positive constant \(C_1 > 0\), here we let \(C_1 = 20 c_{\text{stat}}\),
    we continue to run the standard EM with the initial sub-optimality angle \(\varphi^{T_1} \geq 20 c_{\text{stat}} \sqrt{\frac{d \vee\log \frac{1}{\delta}}{n}} > c_{\text{stat}} \sqrt{\frac{d \vee\log \frac{1}{\delta}}{n}}\).

    For the ease of analysis, we introduce the following notations: \(r:= \frac{\pi}{2}c_{\text{stat}} \sqrt{\frac{d \vee\log \frac{1}{\delta}}{n}} <\frac{1}{20}\) and \(s:= \frac{r}{\sqrt{1-r^2}} < (1+\frac{1}{100})r\), namely \(r = \frac{s}{\sqrt{1+s^2}}\).
    Then by the previous establised Lemma for single iteration of pertubation and using the fact that \(\|\theta^{t+1} - \bar{\theta}^{t+1}\| \leq c_{\text{stat}} \sqrt{\frac{d \vee\log \frac{1}{\delta}}{n}}\) and \(\|\bar{\theta}^{t+1}\| \geq \frac{2}{\pi}\), we have shown that
    \[
    \varphi^{t+1} \geq \bar{\varphi}^{t+1} - \arcsin\left(\frac{\|\theta^{t+1} - \bar{\theta}^{t+1}\|}{\|\bar{\theta}^{t+1}\|} \right)
    \geq \bar{\varphi}^{t+1} - \arcsin\left(\frac{c_{\text{stat}} \sqrt{\frac{d \vee\log \frac{1}{\delta}}{n}}}{2/\pi}\right)
    = \bar{\varphi}^{t+1} - \arcsin r
    \]
    Applying the linear growth of \(\tan\varphi\) in population level in the proven Proposition, we have \(\tan \bar{\varphi}^{t+1} \geq 2 \tan \varphi^t\). 
    and noting that \(\tan(x-y) = \frac{\tan x - \tan y}{1+\tan x \tan y}, \tan(\arcsin r) =s\), 
    \begin{eqnarray*}
      \tan \varphi^{t+1} \geq \tan(\bar{\varphi}^{t+1} - \arcsin r) \geq \frac{2 \tan \varphi^t - s}{1+2 \tan \varphi^t \cdot s}
      = (1+\frac{1}{10})\tan \varphi^t - \frac{s}{1+2s\cdot \tan \varphi^t}
      + \frac{\tan \varphi^t(2 - 1.1 -1.1 s \tan \varphi^t)}{1+2s\cdot \tan \varphi^t}      
    \end{eqnarray*}
    In particular, when \(\varphi^t \leq 1\), we have \(2 - 1.1 -2.2 s \tan \varphi^t >0\) for \(s < (1+\frac{1}{100})r < (1+\frac{1}{100}) \frac{1}{20}\), therefore,
    \begin{eqnarray*}
      \tan \varphi^{t+1} 
      \geq (1+\frac{1}{10})\tan \varphi^t - s
      \geq (1+\frac{1}{10})\tan \varphi^t - (1+\frac{1}{100})r
      \geq (1+\frac{1}{10})\tan \varphi^t - \frac{16}{10} c_{\text{stat}} \sqrt{\frac{d \vee\log \frac{1}{\delta}}{n}}
    \end{eqnarray*}
    By the initial assumption \(\varphi^{T_1} \geq 20 c_{\text{stat}} \sqrt{\frac{d \vee\log \frac{1}{\delta}}{n}} > c_{\text{stat}} \sqrt{\frac{d \vee\log \frac{1}{\delta}}{n}}\), 
    as long as \(\varphi^t \leq 1\), we have
    \[
    \tan \varphi^{t+1} - 16 c_{\text{stat}} \sqrt{\frac{d \vee\log \frac{1}{\delta}}{n}} \geq 
    (1+\frac{1}{10})\left(\tan \varphi^{t} - 16 c_{\text{stat}} \sqrt{\frac{d \vee\log \frac{1}{\delta}}{n}}\right) > 0
    \]
    Therefore, by the induction, we have shown that after running the standard EM for at most \(T_2\leq \left\lceil \frac{\ln(\frac{\tan 1}{4c_{\text{stat}}})+\frac{1}{2}\min(\ln\frac{n}{d},\ln\frac{n}{\log \frac{1}{\delta}})}{\ln(1+1/10)} \right\rceil_+ \asymp \log\frac{n}{d}\wedge \log \frac{n}{\log \frac{1}{\delta}}\) iterations,
    we have 
    \(\varphi^{T_1+T_2} \geq 1\), and therefore \(\phi^{T_1+T_2} = \pi - 2\varphi^{T_1+T_2} < 1.2\).

    \textbf{Stage 3: run Standard EM from \(\phi^{T_1+T_2} \lesssim 1\) to \(\phi^{T_1+T_2+T_3} \lesssim \sqrt{\frac{d\vee\log \frac{1}{\delta}}{n}}\)}

    Given that \(\phi^{T_1+T_2} < 1.2 < 1.4\), we continue to run the standard EM with quadratic convergence rate.
    Recall the established inequality \(\varphi^{t+1} \geq \bar{\varphi}^{t+1} - \arcsin r\), then by the relation \(\phi^{t+1} = \pi - 2 \varphi^{t+1}, \bar{\phi}^{t+1} = \pi - 2 \bar{\varphi}^{t+1}\)
    and noting that \(\arcsin r \leq \frac{\pi}{2}r\).
    \[
    \phi^{t+1} \leq \bar{\phi}^{t+1} + 2 \arcsin r \leq \bar{\phi}^{t+1} + \pi r
    \]
    Applying Proposition for the quadratic convergence rate of \(\phi^t\) in population level, namely 
    \(\bar{\phi}^{t+1}/\pi \leq [\phi^t/\pi]^2\) when \(\phi^t \leq 1.4\).
    \[
      \frac{\phi^{t+1}}{\pi} \leq \left[ \frac{\phi^t}{\pi} \right]^2 + r
    \]
    In one iteration, \(\phi^{T_1+T_2+1}/\pi \leq (1.2/\pi)^2 + r < (1.2/\pi)^2 + \frac{1}{20} < \frac{1}{5}\), 
    then we have \(\phi^t < \frac{\pi}{5} < \frac{\pi}{4}\) in the following iterations.
    \[
    \left[ \frac{\phi^t}{\pi} - 2r \right] \leq \left[ \frac{\phi^t}{\pi} \right]^2 - r
    = \left[ \frac{\phi^t}{\pi} - 2r \right]^2 - 4r^2 - \frac{4r}{\pi}\left(\frac{\pi}{4}-\phi^t\right)
    < \left[ \frac{\phi^t}{\pi} - 2r \right]^2
    \]
    Therefore, by the above inequality, we have shown that after running the standard EM for at most \(T_3\leq 1+\left\lceil \frac{\ln\left(\frac{\ln(1/r)}{\ln 5} \right)}{\ln 2} \right\rceil_+ 
    = 1+\left\lceil \frac{\ln[\frac{1}{2}\min(\ln(n/d),\ln(n/\log \frac{1}{\delta}))]-\ln(\frac{\pi}{2} c_{\text{stat}})-\ln\ln 5}{\ln 2} \right\rceil_+ 
    \asymp \log[\log\frac{n}{d}\wedge \log \frac{n}{\log \frac{1}{\delta}}]\) iterations,
    we have \(\phi^{T_1+T_2+T_3} \leq 3r = \frac{3\pi}{2}c_{\text{stat}} \sqrt{\frac{d \vee\log \frac{1}{\delta}}{n}}\).
    
    \textbf{Stage 4: run Standard EM from \(\phi^{T_1+T_2+T_3} \lesssim \sqrt{\frac{d\vee\log \frac{1}{\delta}}{n}}\) to \(\phi^{T_1+T_2+T_3+T_4} \lesssim\frac{\log \frac{1}{\delta}}{n}\)}
    
    By the previous establised Proposition of statistical accuracy, we have the following inequality:
    \[
    \phi^{t+1} \leq \pi \sin \frac{\phi^{t+1}}{2} \leq \pi\frac{\|\theta^{t+1} - \sgn(\rho^{t+1}) \theta^{\ast} \|}{\|\theta^{\ast} \|} \lesssim [\phi^t]^2 \vee [\phi^t]^{\frac{3}{2}} \sqrt{\frac{d}{n}} \vee \phi^t \sqrt{\frac{\log \frac{1}{\delta}}{n}}
    \]
    With the assumption that \(n\gtrsim d \vee \log \frac{1}{\delta}\) and the fact that \(\phi^{T_1+T_2+T_3} \lesssim \sqrt{\frac{d\vee \log \frac{1}{\delta}}{n}} \lesssim 1\), then for \(t \geq T_1+T_2+T_3+1\), we have
    \[
    \phi^{t} \leq \phi^{T_1+T_2+T_3+1} \lesssim \sqrt{\frac{d\vee \log \frac{1}{\delta}}{n}}\phi^{T_1+T_2+T_3} \lesssim \frac{d\vee \log \frac{1}{\delta}}{n}, \quad
    \phi^t \lesssim [\phi^t]^{\frac{3}{2}} \sqrt{\frac{d}{n}} \vee \phi^t \sqrt{\frac{\log \frac{1}{\delta}}{n}} 
    \]
    Under the assumption that \(n\gtrsim d\vee \log \frac{1}{\delta}\), if (i) \(n\lesssim d^2/\log \frac{1}{\delta}\)(therefore, \(d\vee \log \frac{1}{\delta} \asymp d\) must hold), then only when \(\phi^t \gtrsim \log \frac{1}{\delta}/d\):
    \[ 
    \phi^{t+1} \lesssim [\phi^t]^{\frac{3}{2}} \sqrt{\frac{d}{n}} \vee \phi^t \sqrt{\frac{\log \frac{1}{\delta}}{n}} 
    \lesssim [\phi^t]^{\frac{3}{2}} \sqrt{\frac{d}{n}} \implies
    \phi^{[T_1+T_2+T_3+1] + \Theta\left(\left\lceil \ln\left(\ln \frac{n}{\log \frac{1}{\delta}} /\left(2 \ln \frac{n}{d}\right)\right)/\ln\frac{3}{2}\right\rceil\right)} \lesssim \frac{\log \frac{1}{\delta}}{d}
    \]
    then, for \(t \geq T_1+T_2+T_3+1 + \Theta\left(\left\lceil \ln\left(\ln \frac{n}{\ln \frac{1}{\delta}} /\left(2 \ln \frac{n}{d}\right)\right)/\ln\frac{3}{2}\right\rceil\right)\), we have
    \[
    \phi^{t+1} \lesssim \phi^t \sqrt{\frac{\log \frac{1}{\delta}}{n}} \implies
    \phi^{[T_1+T_2+T_3+1] + \Theta\left(\left\lceil \ln\left(\ln \frac{n}{\ln \frac{1}{\delta}} /\left(2 \ln \frac{n}{d}\right)\right)/\ln\frac{3}{2}\right\rceil\right) + \Theta\left(\left\lceil 2[\ln\frac{1}{\varepsilon}-\ln\frac{d}{\ln\frac{1}{\delta}}]/\ln\frac{n}{\log \frac{1}{\delta}}\right\rceil\right)} \leq \varepsilon
    \]
    Otherwise, if (ii) \(n\gtrsim d^2/\log \frac{1}{\delta}\), then for \(t \geq T_1+T_2+T_3+1\), we have
    \[
    \phi^{t+1} \lesssim \phi^t \sqrt{\frac{\log \frac{1}{\delta}}{n}} \implies
    \phi^{[T_1+T_2+T_3+1] + \Theta\left(\left\lceil 2[\ln\frac{1}{\varepsilon}-\ln\frac{n}{\ln\frac{1}{\delta}}\wedge \ln\frac{n}{d}]/\ln\frac{n}{\ln \frac{1}{\delta}}\right\rceil\right)} \leq \varepsilon
    \]
    Combining the above two case (i) and (ii), we have shown that after running the standard EM for at most \(T_4 \lesssim \log [\log\frac{n}{\ln \frac{1}{\delta}}/ \log\frac{n}{d}] \vee [\log\frac{1}{\varepsilon}/\log\frac{n}{\log \frac{1}{\delta}}]\) iterations,
    we have \(\phi^{T_1+T_2+T_3+T_4} \leq \varepsilon\).

    \textbf{Summary: number of iterations} Therefore, after running Easy EM for at most \(T_1 =\mathcal{O}\left( \log \frac{d}{\log \frac{1}{\delta}}\right)\) iterations,
    and then running Standard EM for at most \(T':= T_2+T_3+T_4=\mathcal{O}\left( \log\frac{n}{d}\wedge \log \frac{n}{\log \frac{1}{\delta}}
    +\log[\log\frac{n}{d}\wedge \log \frac{n}{\log \frac{1}{\delta}}] +\log [\log\frac{n}{\ln \frac{1}{\delta}}/ \log\frac{n}{d}] \vee [\log\frac{1}{\varepsilon}/\log\frac{n}{\log \frac{1}{\delta}}]\right)
    = \mathcal{O}\left( [\log\frac{n}{d}\wedge \log \frac{n}{\log \frac{1}{\delta}}]\vee \log[\log\frac{n}{\ln \frac{1}{\delta}}/ \log\frac{n}{d}] \vee [\log\frac{1}{\varepsilon}/\log\frac{n}{\log \frac{1}{\delta}}]\right) \) iterations,
    we have \(\phi^T = \phi^{T_1+T_2+T_3+T_4} \leq \varepsilon\) with probability at least \(1 - T\delta\), where \(T := T_1+T' = T_1 + T_2 +T_3 + T_4\).
\end{proof}

\begin{theorem}[Theorem~\ref{theorem:finite_sample_convergence}: Finite-Sample Level Convergence]
    In the noiseless setting, suppose any initial mixing weights \(\pi^0\) and any initial regression parameters \(\theta^0 \in \mathbb{R}^d\) ensuring that \(\varphi^0 \gtrsim \sqrt{\frac{\log
      \frac{1}{\delta}}{n}} \). Given a positive number \(\varepsilon \lesssim \frac{\log \frac{1}{\delta}}{n}\), we
      run finite-sample Easy EM for at most \(T_1=\mathcal{O}\left( \log \frac{d}{\log \frac{1}{\delta}}\right)\) 
      iterations followed by the finite-sample standard EM for at most \(T' =\mathcal{O} \left(
      [\log \frac{n}{d} \wedge \log \frac{n}{\log \frac{1}{\delta}}]\vee \log[\log\frac{n}{\ln \frac{1}{\delta}}/ \log\frac{n}{d}] \vee [\log\frac{1}{\varepsilon}/\log\frac{n}{\log \frac{1}{\delta}}]\right)\)
      iterations with \(n \gtrsim d \vee \log \frac{1}{\delta} \) samples, then
      \begin{equation*}
        \begin{aligned}
        \frac{\| \theta^{T + 1} - \mathrm{sgn}(\rho^{T+1}) \theta^{\ast} \|}{\| \theta^{\ast} \|} 
        = \mathcal{O}
        \left( \varepsilon \sqrt{\frac{\log \frac{1}{\delta}}{n}} \right),\quad
        \left\| \pi^{T + 1} - \bar{\pi}^{\ast} \right\|_1 
        =  
        \mathcal{O}\left(  \varepsilon \left\| \frac{1}{2} - \pi^{\ast}\right\|_1 
        \vee \left[\frac{\log\frac{1}{\delta}/n}{\log\left(1+\frac{\log \frac{1}{\delta}/n}{p(\varepsilon, \pi^\ast)}\right)}
         \wedge \sqrt{\frac{\log \frac{1}{\delta}}{n}} \right]\right),
        \end{aligned}
      \end{equation*}
      with probability at least \(1 - T\delta\), where \(T:=T_1+T',\varphi^0 \assign
      \frac{\pi}{2} - \arccos \left| \frac{\langle \theta^0, \theta^{\ast}
      \rangle}{\| \theta^0 \| \cdot \| \theta^{\ast} \|} \right|, 
      \rho^{T+1}\assign \frac{\langle \theta^{T+1}, \theta^{\ast}
      \rangle}{\| \theta^{T+1} \| \cdot \| \theta^{\ast} \|},
      \bar{\pi}^{\ast} \assign \frac{\mathds{1}}{2} + \tmop{sgn} (\rho^{T+1}) 
    (\pi^{\ast} - \frac{\mathds{1}}{2})\),
    and \(p(\varepsilon, \pi^\ast) := \varepsilon \left\| \pi^\ast - \frac{\mathds{1}}{2} \right\|_1 + \min(\pi^\ast(1), \pi^\ast(2))\).
\end{theorem}

\begin{proof}
  By Proposition for convergence of angle, we have shown that \(\phi^T \leq \varepsilon \lesssim \frac{\log \frac{1}{\delta}}{n}\) after \(T\) iterations with \(n \gtrsim d \vee \log \frac{1}{\delta}\) samples.
  Then, by using Proposition for statistical accuracy of EM updates, and 
  noting that \(p^T := \frac{\phi^T}{2\pi} \|\pi^\ast - \frac{\mathds{1}}{2} \|_1 + \min(\pi^\ast(1), \pi^\ast(2))\leq \varepsilon \|\pi^\ast - \frac{\mathds{1}}{2} \|_1 + \min(\pi^\ast(1), \pi^\ast(2)) \equiv p(\varepsilon, \pi^\ast)\):
  \begin{eqnarray*}
    \frac{\left\| \theta^{T+1} - \mathrm{sgn}(\rho^{T+1}) \theta^{\ast} \right\|}{\| \theta^{\ast} \|}
    &=& \frac{\left\| M_n(\theta^T, \nu^T) - \sgn(\rho^{T}) \theta^{\ast} \right\|}{\| \theta^{\ast} \|}
    \lesssim [\phi^T]^2 \vee [\phi^T]^{\frac{3}{2}} \sqrt{\frac{d}{n}} \vee \phi^T \sqrt{\frac{\log\frac{1}{\delta}}{n}}
    \lesssim \varepsilon \sqrt{\frac{\log\frac{1}{\delta}}{n}}\\
    \left\| \pi^{T+1} - \bar{\pi}^{\ast} \right\|_1
    & = & 
    \left| N_n(\theta^T, \nu^t) - \sgn(\rho^T) \tanh \nu^{\ast} \right|
    \lesssim \phi^T \left\| \pi^\ast - \frac{\mathds{1}}{2} \right\|_1 
    \vee \left[ \frac{\log\frac{1}{\delta}/n}{\log\left(1+\frac{\log \frac{1}{\delta}/n}{p^T}\right)}
    \wedge \sqrt{\frac{\log \frac{1}{\delta}}{n}}\right]\\
    & \leq & \varepsilon \left\| \pi^\ast - \frac{\mathds{1}}{2} \right\|_1 
    \vee \left[ \frac{\log\frac{1}{\delta}/n}{\log\left(1+\frac{\log \frac{1}{\delta}/n}{p(\varepsilon, \pi^\ast)}\right)}
    \wedge \sqrt{\frac{\log \frac{1}{\delta}}{n}}\right]
  \end{eqnarray*}
  
\end{proof}

\newpage
\section{Derivations for EM Update Rules}\label{sup:derive_em}
The detailed derivations for the EM update rules are adapted from pages 20-25, Appendix B of~\cite{luo24cycloid},
and included here in this Appendix Chapter for completeness.
\subsection{Negative Log-Likelihood, Surrogate Function, and Q Function}
\begin{lemma}
The negative expected log-likelihood $f (\theta, \pi) \assign -\mathbb{E}_{s \sim p (s \mid
\theta^{\ast}, \pi^{\ast})} [\log p (s \mid \theta, \pi)]$ for the mixture model of $s:=(x, y), z\in[M]$ 
with the mixing weights $\pi^\ast\in\mathbb{R}^M$ and regression parameters $\theta$ is as follows.
\begin{eqnarray*}
    -f(\theta, \pi)   & = & -\tmop{KL}_s [p (s \mid \theta^{\ast}, \pi^{\ast}) | | p (s \mid
    \theta, \pi)] -\mathcal{H}_s [p (s \mid \theta^{\ast}, \pi^{\ast})]\\
    & = & \mathbb{E}_{s \sim p (s \mid \theta^{\ast}, \pi^{\ast})} [\ln p (s
    \mid \theta, \pi)]\\
    & = & \mathbb{E}_{s \sim p (s \mid \theta^{\ast}, \pi^{\ast})}
    \mathbb{E}_{z \sim q_s (z)} \ln p (s, z \mid \theta, \pi) +\mathbb{E}_{s
    \sim p (s \mid \theta^{\ast}, \pi^{\ast})} \mathcal{H}_z [q_s (z)]
    +\mathbb{E}_{s \sim p (s \mid \theta^{\ast}, \pi^{\ast})} \tmop{KL}_z [q_s
    (z) | | p (z \mid s ; \theta, \pi)]
\end{eqnarray*}
where $\tmop{KL}_s, \mathcal{H}_s$ are KL divengence and Shannon's
entropy wrt. $s = (x, y)$;

$\tmop{KL}_z, \mathcal{H}_z, \tmop{softmax}_z$ are KL divengence, Shannon's
entropy and softmax wrt. $z \in \mathcal{Z}= [M]$;

$\{ q_s (z) \mid s \in \mathcal{X} \times \mathcal{Y}=\mathbb{R}^d \times
\mathbb{R} \}$ is a family of distributions wrt. $z \in \mathcal{Z}= [M]$,
namely $\sum_{z \in \mathcal{Z}} q_s (z) = 1$.

\end{lemma}
\begin{proof}
Note that $p(s \mid \theta, \pi)=\frac{p(s, z \mid \theta, \pi)}{p(z \mid s ; \theta, \pi)}$, we obtain the following expression.
\begin{eqnarray*}
    -f(\theta, \pi)   & = & -\tmop{KL}_s [p (s \mid \theta^{\ast}, \pi^{\ast}) | | p (s \mid
    \theta, \pi)] -\mathcal{H}_s [p (s \mid \theta^{\ast}, \pi^{\ast})]\\
    & = & \mathbb{E}_{s \sim p (s \mid \theta^{\ast}, \pi^{\ast})} [\log p (s
    \mid \theta, \pi)]\\
    & = & \mathbb{E}_{s \sim p (s \mid \theta^{\ast}, \pi^{\ast})} \left[
    \sum_{z \in \mathcal{Z}} q_s (z) \log p (s \mid \theta, \pi) \right]\\
    & = & \mathbb{E}_{s \sim p (s \mid \theta^{\ast}, \pi^{\ast})} \left[ 
    \sum_{z \in \mathcal{Z}} q_s (z) \ln \left( \frac{p (s, z \mid \theta,
    \pi)}{q_s (z)} \cdot \frac{q_s (z)}{p (z \mid s ; \theta, \pi)} \right) \right]\\
    & = & \mathbb{E}_{s \sim p (s \mid \theta^{\ast}, \pi^{\ast})}
    \mathbb{E}_{z \sim q_s (z)} \ln p (s, z \mid \theta, \pi) +\mathbb{E}_{s
    \sim p (s \mid \theta^{\ast}, \pi^{\ast})} \mathcal{H}_z [q_s (z)]
    +\mathbb{E}_{s \sim p (s \mid \theta^{\ast}, \pi^{\ast})} \tmop{KL}_z [q_s
    (z) | | p (z \mid s ; \theta, \pi)]
\end{eqnarray*}
\end{proof}

\begin{lemma}
    The surrogate function $g^t$ of $f (\theta, \pi) \assign -\mathbb{E}_{s \sim p (s \mid
    \theta^{\ast}, \pi^{\ast})} [\ln p (s \mid \theta, \pi)]$ at $(t-1)$-th iteration $(\theta^{t-1}, \pi^{t-1})$ be expressed as follows.
    \begin{eqnarray*}
        -g^t(\theta, \pi)   
        & = & \Bigg\{ \mathbb{E}_{s \sim p (s \mid \theta^{\ast}, \pi^{\ast})}
        \mathbb{E}_{z \sim q_s (z)} \ln p (s, z \mid \theta, \pi) +\mathbb{E}_{s
        \sim p (s \mid \theta^{\ast}, \pi^{\ast})} \mathcal{H}_z [q_s (z)]
        \Bigg\}_{q_s
        (z) =p (z \mid s ; \theta^{t-1}, \pi^{t-1})}
    \end{eqnarray*}
    that is $g^t(\theta, \pi)\geq f(\theta, \pi)$, and 
    $g^t(\theta, \pi)\mid_{(\theta, \pi)=(\theta^{t-1}, \pi^{t-1})}= f(\theta, \pi)\mid_{(\theta, \pi)=(\theta^{t-1}, \pi^{t-1})}$,

    $\nabla_\theta g^t(\theta, \pi)\mid_{(\theta, \pi)=(\theta^{t-1}, \pi^{t-1})}=  \nabla_\theta f(\theta, \pi)\mid_{(\theta, \pi)=(\theta^{t-1}, \pi^{t-1})}$;
    $\nabla_\pi g^t(\theta, \pi)\mid_{(\theta, \pi)=(\theta^{t-1}, \pi^{t-1})}=  \nabla_\pi f(\theta, \pi)\mid_{(\theta, \pi)=(\theta^{t-1}, \pi^{t-1})}$;

    where $\tmop{KL}_s, \mathcal{H}_s$ are KL divengence and Shannon's
    entropy wrt. $s = (x, y)$;
    
    $\tmop{KL}_z, \mathcal{H}_z, \tmop{softmax}_z$ are KL divengence, Shannon's
    entropy and softmax wrt. $z \in \mathcal{Z}= [M]$
\end{lemma}
\begin{proof}
    Let $r^t \assign g^t -f$, note that $r^t = \mathbb{E}_{s \sim p (s \mid \theta^{\ast}, \pi^{\ast})} \tmop{KL}_z [q_s
    (z) | | p (z \mid s ; \theta, \pi)]_{q_s
    (z) =p (z \mid s ; \theta^{t-1}, \pi^{t-1})}\geq 0$, and 
    \begin{eqnarray*}
        r^t(\theta^{t-1}, \pi^{t-1}) &=& \mathbb{E}_{s \sim p (s \mid \theta^{\ast}, \pi^{\ast})} \tmop{KL}_z [q_s
        (z) | | p (z \mid s ; \theta^{t-1}, \pi^{t-1})]_{q_s
        (z) =p (z \mid s ; \theta^{t-1}, \pi^{t-1})}= 0\\
        \left[\frac{\mathd q_s(z) \ln \frac{q_s(z)}{p (z \mid s ; \theta, \pi)}}{\mathd p (z \mid s ; \theta, \pi)}\right]_{(\theta, \pi)=(\theta^{t-1}, \pi^{t-1})} &=& 
        - \frac{q_s(z) }{ p (z \mid s ; \theta, \pi)}_{q_s
        (z) =p (z \mid s ; \theta^{t-1}, \pi^{t-1}), (\theta, \pi)=(\theta^{t-1}, \pi^{t-1})} = -1
    \end{eqnarray*}
    Hence, the gradients of $r^t$ wrt. $(\theta, \pi)$ at $(t-1)$-th iteration $(\theta^{t-1}, \pi^{t-1})$ are all zeros by the chain rule. 
    \begin{eqnarray*}
        & & \nabla_\theta \tmop{KL}_z[q_s(z) | | p (z \mid s ; \theta, \pi)]_{q_s(z) =p (z \mid s ; \theta^{t-1}, \pi^{t-1}), (\theta, \pi)=(\theta^{t-1}, \pi^{t-1})} =
        - \left[\nabla_\theta \sum_{z\in \mathcal{Z}} p (z \mid s ; \theta, \pi)\right]_{(\theta, \pi)=(\theta^{t-1}, \pi^{t-1})} = \vec{0}\\
        & & \nabla_\theta r^t(\theta, \pi)_{(\theta, \pi)=(\theta^{t-1}, \pi^{t-1})} 
        = \E_{s\sim p (s \mid \theta^{\ast}, \pi^{\ast})} \nabla_\theta \tmop{KL}_z[q_s(z) | | p (z \mid s ; \theta, \pi)]_{q_s(z) =p (z \mid s ; \theta^{t-1}, \pi^{t-1}), (\theta, \pi)=(\theta^{t-1}, \pi^{t-1})}
        = \vec{0}
    \end{eqnarray*}
    The same proof can be applied to $\nabla_\pi r^t(\theta, \pi)_{(\theta, \pi)=(\theta^{t-1}, \pi^{t-1})} =\vec{0}$.
\end{proof}

\newpage
\begin{lemma}
Assuming $(z; \pi)\ind \theta$, and $\pi \ind s \mid z$, and $x\ind (z; \theta, \pi)$; then 
\begin{eqnarray*}
    Q(\theta, \pi\mid \theta^{t-1}, \pi^{t-1}) & \assign& 
    \left[ \mathbb{E}_{s \sim p (s \mid \theta^{\ast}, \pi^{\ast})}
    \mathbb{E}_{z \sim q_s (z)} \ln p (s, z \mid \theta, \pi)\right]_{q_s
    (z) \leftarrow p (z \mid s ; \theta^{t-1}, \pi^{t-1})}\\
    &= & \mathbb{E}_{s \sim p\left(s \mid \theta^*, \pi^*\right)} \mathbb{E}_{z \sim q_s(z)} \ln p(y \mid x, z ; \theta) \\
    & + &\mathbb{E}_{s \sim p\left(s \mid \theta^*, \pi^*\right)} \ln p(x) \\
    & - &\mathrm{KL}_z\left[\pi^t(z)|| \pi(z)\right] \\
    & - &\mathcal{H}_z\left[\pi^t(z)\right]\\
    p (z \mid s ; \theta, \pi) &=& \tmop{softmax}_z
    (\ln \pi (z) + \ln p (y \mid x, z ; \theta))
\end{eqnarray*}
where $\tmop{KL}_s, \mathcal{H}_s$ are KL divengence and Shannon's
entropy wrt. $s = (x, y)$;

$\tmop{KL}_z, \mathcal{H}_z, \tmop{softmax}_z$ are KL divengence, Shannon's
entropy and softmax wrt. $z \in \mathcal{Z}= [M]$;

$q_s(z) \leftarrow p\left(z \mid s ; \theta^{t-1}, \pi^{t-1}\right)$ and $\pi^t=\{\pi(z)\}_{z\in \mathcal{Z}},\pi^t(z) \assign \mathbb{E}_{s \sim p\left(s \mid \theta^*, \pi^*\right)} q_s(z)$.
\end{lemma}

\begin{proof}
    \
\begin{itemizedot}

    \item $(z ; \pi) \ind \theta$ are independent:
    
    therefore $p (z \mid \theta, \pi) = p (z \mid \pi) = \pi (z)$
    
    \item $\pi \ind (s ; \theta) \mid z$ are conditional independent given $z$:
    
    then $p (\pi \mid z) = p (\pi \mid z; \theta) = p (\pi \mid z, s ; \theta)$,
    it implies $\frac{p (s , z ; \theta, \pi)}{p (z ; \theta, \pi)} =\frac{p (s , z ; \theta)}{p (z ; \theta)}$;
    
    hence $p (s \mid z ; \theta, \pi) = p (s \mid z ; \theta)$, $p (s, z \mid
    \theta, \pi) = p (s \mid z ; \theta, \pi) p (z \mid \theta, \pi) = p (s \mid
    z ; \theta) \pi (z)$
    
    \item $x \ind (z ; \theta, \pi)$ are independent:
    
    then $p (x \mid z ; \theta) = p (x)$
    
    hence $p (s \mid z ; \theta) = p (y \mid x, z ; \theta) \cdot p (x \mid z ;
    \theta) = p (y \mid x, z ; \theta) \cdot p (x)$
    
    therefore $p (s \mid \theta, \pi) = \sum_{z \in \mathcal{Z}} p (s, z \mid
    \theta, \pi) = p (x) \cdot \sum_{z \in \mathcal{Z}} \pi (z) p (y \mid x, z ;
    \theta)$
    
    and $p (z \mid s ; \theta, \pi) = \frac{p (s, z \mid \theta, \pi)}{p (s \mid
    \theta, \pi)} = \frac{\pi (z) \cdot p (y \mid x, z ; \theta)}{\sum_{z' \in
    \mathcal{Z}} \pi (z') \cdot p (y \mid x, z' ; \theta)} = \tmop{softmax}_z
    (\log \pi (z) + \log p (y \mid x, z ; \theta))$
\end{itemizedot}
With the above assumptions, we obtain that $p(s, z\mid \theta, \pi)=  p (y \mid x, z ; \theta) \cdot p (x) \cdot \pi(z)$, further prove this Lemma.
\end{proof}

\subsection{Negative Log-Likelihood, Surrogate Function for General MLR and 2MLR}
\begin{lemma}
For MLR $y= \langle x, \theta_z^\ast\rangle + \varepsilon, z\in\mathcal{Z}=[M]$, $\theta \assign \{\theta_z \}_{z\in\mathcal{Z}}, \pi \assign \{\pi(z)\}_{z\in\mathcal{Z}}$, 
with assumptions: $(z; \pi)\ind \theta$, and $\varepsilon \ind (x, z; \theta, \pi)$, and $x\ind (z; \theta, \pi)$ and $\varepsilon \sim \mathcal{N}(0, \sigma^2)$
\begin{eqnarray*}
    f(\theta, \pi) & = & -\mathbb{E}_{\mathrm{s} \sim p\left(s \mid \theta^*, \pi^*\right)} \log \sum_{z \in \mathcal{Z}} \exp \left[-\frac{\left\|y-\left\langle\theta_z, x\right\rangle\right\|^2}{2 \sigma^2}+\log \pi(z)\right]-\mathbb{E}_{s \sim p\left(s \mid \theta^*, \pi^*\right)} \log p(x)-c\\
    g^t(\theta, \pi) & =&\left(2 \sigma^2\right)^{-1} \mathbb{E}_{s \sim p\left(s \mid \theta^*, \pi^*\right)} \mathbb{E}_{z \sim q_x(z)}\left\|y-\left\langle\theta_z, x\right\rangle\right\|^2 \\
    & +&\mathrm{KL}_z\left[\pi^t(z) \| \pi(z)\right] \\
    & +&\mathcal{H}_z\left[\pi^t(z)\right]-\mathbb{E}_{s \sim p\left(s \mid \theta^*, \pi^*\right)} \log p(x)-c \\
    & -&\mathbb{E}_{s \sim p\left(s \mid \theta^*, \pi^*\right)} \mathcal{H}_z\left[q_s(z)\right]
\end{eqnarray*}
where $c=-\frac{1}{2} \log \left(2 \pi \sigma^2\right)$ and $q_s(z) \leftarrow p\left(z \mid s ; \theta^{t-1}, \pi^{t-1}\right)=\operatorname{softmax}_z\left(-\frac{\left\|y-\left\langle\theta_z^{t-1}, x\right\rangle\right\|^2}{2 \sigma^2}+\log \pi^{t-1}(z)\right)$,
and $\pi^t=\{\pi(z)\}_{z\in \mathcal{Z}},\pi^t(z) \assign \mathbb{E}_{s \sim p\left(s \mid \theta^*, \pi^*\right)} q_s(z)$.
\end{lemma}
\begin{proof}
    Since $(z; \pi)\ind \theta$, and $\varepsilon \ind (x, z; \theta, \pi)$, and $x\ind (z; \theta, \pi)$, then implies $\pi \ind s \mid z$
    because of $p(\pi \mid z, s)=p(\pi \mid z, x, y)=p(\pi \mid z, x, \varepsilon)=p(\pi \mid z)=\pi(z)$. Hence, we can apply the previous Lemma.
    
    Furthermore, $p (y, x, z ;
    \theta) = p (\varepsilon, x, z ; \theta) \left| \frac{\partial
    \varepsilon}{\partial y} \right| = p (\varepsilon) \cdot p (x, z ; \theta)$
    \[ p (y \mid x, z ; \theta) = p (\tmmathbf{\varepsilon}) = (2 \pi \sigma^2)^{-
        \frac{1}{2}} \exp \left( - \frac{\| \varepsilon \|^2}{2 \sigma^2} \right) =
        (2 \pi \sigma^2)^{- \frac{1}{2}} \exp \left( - \frac{\| y - \langle x,
        \theta_z \rangle \|^2}{2 \sigma^2} \right) =\mathcal{N} (\langle x,
        \theta_z \rangle, \sigma^2) \]
    Hence, we obtain $\log p (y \mid x, z ; \theta) = - \frac{\| y - \langle x,
    \theta_z \rangle \|^2}{2 \sigma^2} + c$ , where $c=-\frac{1}{2} \log \left(2 \pi \sigma^2\right)$.

    Subsequently, note that $g^t(\theta, \pi) = Q(\theta, \pi\mid \theta^{t-1}, \pi^{t-1}) - \mathbb{E}_{s \sim p\left(s \mid \theta^*, \pi^*\right)} \mathcal{H}_z\left[q_s(z)\right]$,
    we prove the expression for $g^t$ by substituting $\log p (y \mid x, z ; \theta)$ in the previous Lemma.

    As shown in the proof of the previous Lemma, $p(s, z\mid \theta, \pi)=  p (y \mid x, z ; \theta) \cdot p (x) \cdot \pi(z)$.

    Hence $p(s\mid \theta, \pi) = \sum_{z\in\mathcal{Z}} p(s, z\mid \theta, \pi)=  p (x) \cdot \sum_{z\in\mathcal{Z}} p (y \mid x, z ; \theta) \cdot  \pi(z)$,
    we prove the expression for $f$ by substituting $\log p (y \mid x, z ; \theta)$ in the previous Lemma.
\end{proof}

\begin{theorem}
    For 2MLR, $y= (-1)^{z+1}\langle x, \theta^\ast\rangle + \varepsilon, z\in\mathcal{Z}=\{1, 2\}$, $\theta_1 = \theta, \theta_2 = -\theta, \pi \assign \{\pi(z)\}_{z\in\mathcal{Z}}$, 
    with assumptions: $(z; \pi)\ind \theta$, and $\varepsilon \ind (x, z; \theta, \pi)$, and $x\ind (z; \theta, \pi)$ and $\varepsilon \sim \mathcal{N}(0, \sigma^2)$.

    Then the negative expected log-likelihood $f (\theta, \pi) \assign
    -\mathbb{E}_{s \sim p (s \mid \theta^{\ast}, \pi^{\ast})} [\ln p (s \mid
    \theta, \pi)]$ and the surrogate function $g^t (\theta, \pi)$ can be expressed as follows.
    \begin{eqnarray*}
    f  & = & (2 \sigma^2)^{- 1}  \langle \theta, \mathbb{E}_{s \sim p (s \mid
        \theta^{\ast}, \pi^{\ast})} x x^{\top} \cdot \theta \rangle + \ln \cosh
        \nu -\mathbb{E}_{s \sim p (s \mid \theta^{\ast}, \pi^{\ast})} \ln \cosh
        \left( \frac{y \langle x, \theta \rangle}{\sigma^2} + \nu \right)\\
        & + &  (2 \sigma^2)^{- 1} \mathbb{E}_{s \sim p (s \mid \theta^{\ast},
        \pi^{\ast})} y^2 -\mathbb{E}_{s \sim p (s \mid \theta^{\ast},
        \pi^{\ast})} \ln p (x) - c\\
    g^t & = & (2 \sigma^2)^{- 1}  \langle \theta, \mathbb{E}_{s \sim p (s \mid
    \theta^{\ast}, \pi^{\ast})} x x^{\top} \cdot \theta \rangle + \log \cosh
    \nu -\mathbb{E}_{s \sim p (s \mid \theta^{\ast}, \pi^{\ast})} \log \cosh
    \left( \frac{y \langle x, \theta \rangle}{\sigma^2} + \nu \right)
    \mid_{(\theta, \nu) \leftarrow (\theta^{t - 1}, \nu^{t - 1})}\\
    & - & \left\langle \nabla_{\theta} \mathbb{E}_{s \sim p (s \mid
    \theta^{\ast}, \pi^{\ast})} \log \cosh \left( \frac{y \langle x, \theta
    \rangle}{\sigma^2} + \nu \right) \mid_{(\theta, \nu) \leftarrow (\theta^{t
    - 1}, \nu^{t - 1})}, \theta - \theta^{t - 1} \right\rangle\\
    & - & \left\langle \nabla_{\nu} \mathbb{E}_{s \sim p (s \mid
    \theta^{\ast}, \pi^{\ast})} \log \cosh \left( \frac{y \langle x, \theta
    \rangle}{\sigma^2} + \nu \right) \mid_{(\theta, \nu) \leftarrow (\theta^{t
    - 1}, \nu^{t - 1})}, \nu - \nu^{t - 1} \right\rangle\\
    & + & (2 \sigma^2)^{- 1} \mathbb{E}_{s \sim p (s \mid \theta^{\ast},
    \pi^{\ast})} y^2 -\mathbb{E}_{s \sim p (s \mid \theta^{\ast},
    \pi^{\ast})} \log p (x) - c
    \end{eqnarray*}
    where $\nu \assign \frac{\log \pi(1) -\log \pi(2)}{2}$,  $c=-\frac{1}{2} \log \left(2 \pi \sigma^2\right)$.
\end{theorem}

\begin{proof}
    \
    Note that $\tmop{sigmoid} (2 t) + \tmop{sigmoid} (- 2 t) = 1$ and
    $\tmop{sigmoid} (2 t) - \tmop{sigmoid} (- 2 t) = \tanh (t)$ for $\forall t$,

    $q_s(z) \leftarrow p\left(z \mid s ; \theta^{t-1}, \pi^{t-1}\right)
    =\tmop{softmax}_z(-\frac{\left\|y-\left\langle\theta_z^{t-1}, x\right\rangle\right\|^2}{2 \sigma^2}+\log \pi^{t-1}(z))
    =\tmop{sigmoid}(2(-1)^{z+1}\left[\frac{y \langle x, \theta^{t-1} \rangle}{\sigma^2} +\nu^{t-1}\right])$,

    $\nu^{t-1}\assign \frac{\log \pi^{t-1}(1) -\log \pi^{t-1}(2)}{2}$.
    \begin{eqnarray*}
    &   & \left(2 \sigma^2\right)^{-1} \mathbb{E}_{s \sim p\left(s \mid \theta^*, \pi^*\right)} 
    \mathbb{E}_{z \sim q_x(z)}\left\|y-\left\langle\theta_z, x\right\rangle\right\|^2 \\
    & = & \mathbb{E}_{s \sim p (s \mid \theta^{\ast}, \pi^{\ast})}
    \mathbb{E}_{z \sim q_s (z)} \frac{\| (- 1)^{z + 1} y - \langle x, \theta
    \rangle \|^2}{2 \sigma^2}\\
    & = & \mathbb{E}_{s \sim p (s \mid \theta^{\ast}, \pi^{\ast})} \frac{y^2 +
    \langle x, \theta \rangle^2}{2 \sigma^2} -\mathbb{E}_{s \sim p (s \mid
    \theta^{\ast}, \pi^{\ast})} \frac{y \langle x, \theta \rangle}{\sigma^2}
    \tanh \left( \frac{y \langle x, \theta^{t-1} \rangle}{\sigma^2} +\nu^{t-1} \right)\\
    & = & (2 \sigma^2)^{- 1}  \langle \theta, \mathbb{E}_{s \sim p (s \mid
    \theta^{\ast}, \pi^{\ast})} x x^{\top} \cdot \theta \rangle
    -  \left\langle \mathbb{E}_{s \sim p (s \mid \theta^{\ast},
    \pi^{\ast})} \tanh \left( \frac{y \langle x, \theta^{t - 1}
    \rangle}{\sigma^2} + \nu^{t - 1} \right) \frac{y x}{\sigma^2}, \theta \right\rangle
    + (2 \sigma^2)^{- 1} \mathbb{E}_{s \sim p (s \mid \theta^{\ast},
    \pi^{\ast})} y^2
    \end{eqnarray*}

    Consider the other terms for $g^t$, note $\pi^t(z) \assign \mathbb{E}_{s \sim p\left(s \mid \theta^*, \pi^*\right)} q_s(z)$, and 
    $\pi(z) = \tmop{sigmoid}(2(-1)^{z+1}\nu)$.

    Note that $\log 2 + \log \cosh \nu = - [\frac{\log \pi(1)+ \log \pi(2)}{2}]$ and $\tmop{sigmoid} (2 t) - \tmop{sigmoid} (- 2 t) = \tanh (t)$ for $\forall t$.
    \begin{eqnarray*}
    & &\mathrm{KL}_z\left[\pi^t(z) \| \pi(z)\right]
    + \mathcal{H}_z\left[\pi^t(z)\right]\\
    & = & - \mathbb{E}_{s \sim p\left(s \mid \theta^*, \pi^*\right)} \sum_{z\in\mathcal{Z}} q_s(z) \log \pi(z)\\
    & = & \log 2 + \log \cosh \nu -\mathbb{E}_{s \sim p (s \mid
    \theta^{\ast}, \pi^{\ast})} \tanh \left( \frac{y \langle x, \theta^{t - 1}
    \rangle}{\sigma^2} + \nu^{t - 1} \right) \cdot \nu\\
    \end{eqnarray*}
    To sum up, we obtain the following.
    \begin{eqnarray*}
        g^t &  = & (2 \sigma^2)^{- 1}  \langle \theta, \mathbb{E}_{s \sim p (s \mid
        \theta^{\ast}, \pi^{\ast})} x x^{\top} \cdot \theta \rangle -
        \left\langle \mathbb{E}_{s \sim p (s \mid \theta^{\ast},
        \pi^{\ast})} \tanh \left( \frac{y \langle x, \theta^{t - 1}
        \rangle}{\sigma^2} + \nu^{t - 1} \right) \frac{y x}{\sigma^2}, \theta \right\rangle\\
        &  & + \log 2 + \log \cosh \nu -\mathbb{E}_{s \sim p (s \mid
        \theta^{\ast}, \pi^{\ast})} \tanh \left( \frac{y \langle x, \theta^{t - 1}
        \rangle}{\sigma^2} + \nu^{t - 1} \right) \cdot \nu\\
        &  & -\mathbb{E}_{s \sim p (s \mid \theta^{\ast}, \pi^{\ast})}
        \mathcal{H}_z [q_s (z)]\\
        &  & + (2 \sigma^2)^{- 1} \mathbb{E}_{s \sim p (s \mid \theta^{\ast},
        \pi^{\ast})} y^2 -\mathbb{E}_{s \sim p (s \mid \theta^{\ast},
        \pi^{\ast})} \log p (x) - c
    \end{eqnarray*}
    For the negative expectation of log-likelihood $f$, we show the following.
    \begin{eqnarray*}
    f & = & -\mathbb{E}_{s \sim p (s \mid \theta^{\ast}, \pi^{\ast})} \log
    \sum_{z \in \mathcal{Z}} \exp \left[ - \frac{\| y - \langle \theta_z, x
    \rangle \|^2}{2 \sigma^2} + \log \pi (z) \right] -\mathbb{E}_{s \sim p (s
    \mid \theta^{\ast}, \pi^{\ast})} \log p (x) - c\\
    & = & (2 \sigma^2)^{- 1} \mathbb{E}_{s \sim p (s \mid \theta^{\ast},
    \pi^{\ast})} [y^2 + \langle x, \theta \rangle^2] - \frac{\log \pi
    (1) + \log \pi (2)}{2}\\
    &  & - \log 2 -\mathbb{E}_{s \sim p (s \mid \theta^{\ast}, \pi^{\ast})}
    \log \cosh \left( \frac{y \langle x, \theta \rangle}{\sigma^2} + \nu
    \right) -\mathbb{E}_{s \sim p (s \mid \theta^{\ast}, \pi^{\ast})} \log p
    (x) - c\\
    & = & (2 \sigma^2)^{- 1}  \langle \theta, \mathbb{E}_{s \sim p (s \mid
    \theta^{\ast}, \pi^{\ast})} x x^{\top} \cdot \theta \rangle + \log \cosh
    \nu -\mathbb{E}_{s \sim p (s \mid \theta^{\ast}, \pi^{\ast})} \log \cosh
    \left( \frac{y \langle x, \theta \rangle}{\sigma^2} + \nu \right)\\
    &  & + (2 \sigma^2)^{- 1} \mathbb{E}_{s \sim p (s \mid \theta^{\ast},
    \pi^{\ast})} y^2 -\mathbb{E}_{s \sim p (s \mid \theta^{\ast},
    \pi^{\ast})} \log p (x) - c
    \end{eqnarray*}

    Note $g^t =f$ at $(\theta, \nu)=(\theta^{t-1}, \nu^{t-1})$, by comparing the expressions  $f, g^t$,
    and use $\frac{\mathd \log \cosh(t)}{\mathd t} = \tanh (t)$.
    \begin{eqnarray*}
    g^t & = & (2 \sigma^2)^{- 1}  \langle \theta, \mathbb{E}_{s \sim p (s \mid
    \theta^{\ast}, \pi^{\ast})} x x^{\top} \cdot \theta \rangle + \log \cosh
    \nu -\mathbb{E}_{s \sim p (s \mid \theta^{\ast}, \pi^{\ast})} \log \cosh
    \left( \frac{y \langle x, \theta \rangle}{\sigma^2} + \nu \right)
    \mid_{(\theta, \nu) \leftarrow (\theta^{t - 1}, \nu^{t - 1})}\\
    & - & \left\langle \nabla_{\theta} \mathbb{E}_{s \sim p (s \mid
    \theta^{\ast}, \pi^{\ast})} \log \cosh \left( \frac{y \langle x, \theta
    \rangle}{\sigma^2} + \nu \right) \mid_{(\theta, \nu) \leftarrow (\theta^{t
    - 1}, \nu^{t - 1})}, \theta - \theta^{t - 1} \right\rangle\\
    & - & \left\langle \nabla_{\nu} \mathbb{E}_{s \sim p (s \mid
    \theta^{\ast}, \pi^{\ast})} \log \cosh \left( \frac{y \langle x, \theta
    \rangle}{\sigma^2} + \nu \right) \mid_{(\theta, \nu) \leftarrow (\theta^{t
    - 1}, \nu^{t - 1})}, \nu - \nu^{t - 1} \right\rangle\\
    & + & (2 \sigma^2)^{- 1} \mathbb{E}_{s \sim p (s \mid \theta^{\ast},
    \pi^{\ast})} y^2 -\mathbb{E}_{s \sim p (s \mid \theta^{\ast},
    \pi^{\ast})} \log p (x) - c
    \end{eqnarray*}
\end{proof}

\begin{theorem}
    For 2MLR, $y= (-1)^{z+1}\langle x, \theta^\ast\rangle + \varepsilon, z\in\mathcal{Z}=\{1, 2\}$, $\theta_1 = \theta, \theta_2 = -\theta, \pi \assign \{\pi(z)\}_{z\in\mathcal{Z}}$, 
    with assumptions: $(z; \pi)\ind \theta$, and $\varepsilon \ind (x, z; \theta, \pi)$, and $x\ind (z; \theta, \pi)$ and $\varepsilon \sim \mathcal{N}(0, \sigma^2)$.

    Then the negative Maximum Likelihood Estimat (MLE) $f_n (\theta, \pi) \assign
    -\frac{1}{n}\sum_{i=1}^n [\log p (s_i \mid
    \theta, \pi)]$ and the surrogate function $g_n^t (\theta, \pi)$ for the dataset $\mathcal{S}\assign \{s_i\}_{i=1}^n=\{(x_i, y_i)\}_{i=1}^n$ of $n$ i.i.d. samples can be expressed as follows.
    \begin{eqnarray*}
    f_n  & = & (2 \sigma^2)^{- 1}  \langle \theta, \frac{1}{n}\sum_{i=1}^n x_i x_i^{\top} \cdot \theta \rangle + \log \cosh
        \nu -\frac{1}{n}\sum_{i=1}^n \log \cosh
        \left( \frac{y_i \langle x_i, \theta \rangle}{\sigma^2} + \nu \right)\\
        & + &  (2 \sigma^2)^{- 1} \frac{1}{n}\sum_{i=1}^n y_i^2 -\frac{1}{n}\sum_{i=1}^n \log p (x_i) - c\\
    g_n^t & = & (2 \sigma^2)^{- 1}  \langle \theta, \frac{1}{n}\sum_{i=1}^n x_i x_i^{\top} \cdot \theta \rangle + \log \cosh
    \nu -\frac{1}{n}\sum_{i=1}^n \log \cosh
    \left( \frac{y_i \langle x_i, \theta \rangle}{\sigma^2} + \nu \right)
    \mid_{(\theta, \nu) \leftarrow (\theta^{t - 1}, \nu^{t - 1})}\\
    & - & \left\langle \nabla_{\theta} \frac{1}{n}\sum_{i=1}^n \log \cosh \left( \frac{y_i \langle x_i, \theta
    \rangle}{\sigma^2} + \nu \right) \mid_{(\theta, \nu) \leftarrow (\theta^{t
    - 1}, \nu^{t - 1})}, \theta - \theta^{t - 1} \right\rangle\\
    & - & \left\langle \nabla_{\nu} \frac{1}{n}\sum_{i=1}^n \log \cosh \left( \frac{y_i \langle x_i, \theta
    \rangle}{\sigma^2} + \nu \right) \mid_{(\theta, \nu) \leftarrow (\theta^{t
    - 1}, \nu^{t - 1})}, \nu - \nu^{t - 1} \right\rangle\\
    & + & (2 \sigma^2)^{- 1} \frac{1}{n}\sum_{i=1}^n y_i^2 -\frac{1}{n}\sum_{i=1}^n \log p (x_i) - c
    \end{eqnarray*}
    where $\nu \assign \frac{\log \pi(1) -\log \pi(2)}{2}$,  $c=-\frac{1}{2} \log \left(2 \pi \sigma^2\right)$.
\end{theorem}
\begin{proof}
This is proved by susbstituting $\frac{1}{n}\sum_{i=1}^n, s_i \assign (x_i, y_i)$ for $\mathbb{E}_{s \sim p (s \mid \theta^{\ast},
\pi^{\ast})}, s\assign (x, y)$ in the previous Theorem.
\end{proof}

\newpage
\subsection{Derivations of EM Update Rules for 2MLR}
\begin{theorem}{(Derivation of Population EM Update Rules: Eq.~\eqref{eq:theta},~\eqref{eq:nu})}
For 2MLR, $y= (-1)^{z+1}\langle x, \theta^\ast\rangle + \varepsilon, z\in\mathcal{Z}=\{1, 2\}$, $\theta_1 = \theta, \theta_2 = -\theta, \pi \assign \{\pi(z)\}_{z\in\mathcal{Z}}$, 
with assumptions: $(z; \pi)\ind \theta$, and $\varepsilon \ind (x, z; \theta, \pi)$, and $x\ind (z; \theta, \pi)$ and $\varepsilon \sim \mathcal{N}(0, \sigma^2), x\sim \mathcal{N}(0, I_d)$.

The EM update rules $M(\theta^{t-1}, \nu^{t-1}), N(\theta^{t-1}, \nu^{t-1})$ for $\theta, \tanh(\nu)$ at the population level, namely the minimizer of the surrogate $g^t$ / the maximizer of $Q$, are the following.
\begin{eqnarray*}
    M(\theta^{t-1}, \nu^{t-1}) & = & \mathbb{E}_{s \sim p (s \mid \theta^{\ast},
    \pi^{\ast})} \tanh\left(\frac{y\langle x, \theta^{t-1}\rangle}{\sigma^2}+\nu^{t-1}\right) y x\\
    N(\theta^{t-1}, \nu^{t-1}) & = & \mathbb{E}_{s \sim p (s \mid \theta^{\ast},
    \pi^{\ast})} \tanh\left(\frac{y\langle x, \theta^{t-1}\rangle}{\sigma^2}+\nu^{t-1}\right)
\end{eqnarray*}
\end{theorem}
\begin{proof}
Take the gradients of $g^t$ wrt. $\theta, \nu$, we obtain the following.
\begin{eqnarray*}
    \nabla_\theta g^t & = & (\sigma^2)^{- 1} \mathbb{E}_{s \sim p (s \mid
    \theta^{\ast}, \pi^{\ast})} x x^{\top} \cdot \theta
    - (\sigma^2)^{- 1} \mathbb{E}_{s \sim p (s \mid \theta^{\ast},
    \pi^{\ast})} \tanh\left(\frac{y\langle x, \theta^{t-1}\rangle}{\sigma^2}+\nu^{t-1}\right) y x\\
    \nabla_\nu g^t & = & \tanh \nu - \mathbb{E}_{s \sim p (s \mid \theta^{\ast},
    \pi^{\ast})} \tanh\left(\frac{y\langle x, \theta^{t-1}\rangle}{\sigma^2}+\nu^{t-1}\right)
\end{eqnarray*}
Furthermore, the Hessian of $g^t$ wrt. $\theta, \nu$ are positive-definite, we show that the solution to $\nabla_\theta g^t=0, \nabla_\nu g^t=0$ must be the minimizer of $g^t$.
Since the Hessian of $g^t$ at the solution is positive-definite, \(\nabla_\theta (\nabla_\nu g^t) =\vec{0}\) and
\begin{eqnarray*}
    \nabla^2_\theta g^t & = & (\sigma^2)^{- 1} \mathbb{E}_{s \sim p (s \mid
    \theta^{\ast}, \pi^{\ast})} x x^{\top}\\
    \nabla^2_\nu g^t & = & \cosh^{-2} \nu 
\end{eqnarray*}
Note that $\mathbb{E}_{s \sim p (s \mid
\theta^{\ast}, \pi^{\ast})} x x^{\top} = I_d$ for $x\sim \mathcal{N}(0, I_d)$, we derive the expressions for EM update rules.
\end{proof}

\begin{theorem}{(Derivation of Finite-Sample EM Update Rules: Eq.~\eqref{eq:finite})}
For 2MLR, $y= (-1)^{z+1}\langle x, \theta^\ast\rangle + \varepsilon, z\in\mathcal{Z}=\{1, 2\}$, $\theta_1 = \theta, \theta_2 = -\theta, \pi \assign \{\pi(z)\}_{z\in\mathcal{Z}}$, 
with assumptions: $(z; \pi)\ind \theta$, and $\varepsilon \ind (x, z; \theta, \pi)$, and $x\ind (z; \theta, \pi)$ and $\varepsilon \sim \mathcal{N}(0, \sigma^2), x\sim \mathcal{N}(0, I_d)$.
    
The EM update rules $M_n(\theta^{t-1}, \nu^{t-1}), N_n(\theta^{t-1}, \nu^{t-1})$ for $\theta, \tanh(\nu)$ at the finite-sample level, namely the minimizer of the surrogate $g_n^t$, are the following.
\begin{eqnarray*}
    M_n(\theta^{t-1}, \nu^{t-1}) & = & \left(\frac{1}{n}\sum_{i=1}^n x_i x_i^\top\right)^{-1}\cdot
    \frac{1}{n}\sum_{i=1}^n \tanh\left(\frac{y\langle x, \theta^{t-1}\rangle}{\sigma^2}+\nu^{t-1}\right) y x\\
    N_n(\theta^{t-1}, \nu^{t-1}) & = & \frac{1}{n}\sum_{i=1}^n \tanh\left(\frac{y\langle x, \theta^{t-1}\rangle}{\sigma^2}+\nu^{t-1}\right)
\end{eqnarray*}
\end{theorem}
\begin{proof}
This is proved by susbstituting $\frac{1}{n}\sum_{i=1}^n, s_i \assign (x_i, y_i)$ for $\mathbb{E}_{s \sim p (s \mid \theta^{\ast},
\pi^{\ast})}, s\assign (x, y)$ in the previous Theorem, but note that $\frac{1}{n}\sum_{i=1}^n x_i x_i^\top \not\equiv I_d$.
\end{proof}

\newpage
\section{Helper Lemmas Used in the Proofs of Results}\label{sup:lemma}

\subsection{Asymptotic Analysis of Integrals}
\begin{lemma}\label{suplem:integral_bigO}
    Suppose $0 \leq a < \frac{\pi}{2}$ and $0 < b \lesssim \cos^2 a$, then
    \begin{eqnarray*}
      \int_0^{2 \pi} \frac{b^2 | \sin x + \sin a | \mathd x}{(b + | \sin x +
      \sin a |)^2} & = & \mathcal{O} \left( \frac{b^2}{\cos a} \log \left(
      \frac{\cos^2 a}{b} \right) \right)\\
      \int_0^{2 \pi} \frac{b^2 \tmop{sgn} (\sin (x + a)) | \cos x + \cos a |
      \mathd x}{(b + | \sin x + \sin a |)^2} & = & \mathcal{O} \left( \tan a
      \cdot \frac{b^2}{\cos a} \log \left( \frac{\cos^2 a}{b} \right) \right)
    \end{eqnarray*}
\end{lemma}

\begin{proof}
    Consider the two roots of $\sin x + \sin a = 0$:
    \( r_1 = \pi + a, r_2 = 2 \pi - a \),
    split integrals into two regions, select $\delta = \cos a$:
    \begin{eqnarray*}
      \text{Near roots} &  & E = E_1 \cup E_2, E_i \assign \{
      (r_i + u) \tmop{mod} 2 \pi \mid | u | < \delta \}, \forall i \in \{ 1, 2
      \}\\
      \text{Away from roots} &  & \bar{E} 
      = \bar{E}_1 \cup \bar{E}_2 
      , \bar{E}_1 = \{ x \mid r_1 + \delta \leq x \leq r_2 - \delta \}
      , \bar{E}_2 = \left\{ x \mid 0 \leq x \leq r_1 - \delta \infixor r_2 +
      \delta \leq x \leq 2 \pi \right\}
    \end{eqnarray*}
    (i) Consider the region $E$ near the roots:
    
    Note $| \cos r_i | = \cos a$ and by Taylor's theorem, apply $| u | < \delta$ and $\delta = \cos a$:
    \[ \frac{1}{2} \cos a \cdot | u | \leq\cos a \cdot | u | - \frac{u^2}{2} \leq | \sin (r_i + u) + \sin a | \leq
       \cos a \cdot | u | + \frac{u^2}{2}\leq
       \frac{3}{2} \cos a \cdot | u |, \forall i \in \{ 1, 2 \} \]
    For the sign $\tmop{sgn} (\sin (x + a))$,
    \begin{eqnarray*}
      \tmop{sgn} (\sin (r_1 + u + a)) & = & \tmop{sgn} (\sin (\pi + a + u + a))
      = - \tmop{sgn} (\sin (2 a + u))
    \end{eqnarray*}
    For the sign $- \tmop{sgn} (\sin x + \sin a)$, note $- \frac{\pi}{2} < - 1
    \leq a - \cos a = a - \delta \leq a + u \leq a + \delta = a + \cos a \leq
    \frac{\pi}{2} $for $| u | < \delta$,
    \[ - \tmop{sgn} (\sin (r_1 + u) + \sin a) = - \tmop{sgn} (\sin (\pi + a + u)
       + \sin a) = \tmop{sgn} (\sin (a + u) - \sin a) = \tmop{sgn} (u) \]
    \[ - \tmop{sgn} (\sin (r_2 + u) + \sin a) = - \tmop{sgn} (\sin (2 \pi - a +
       u) + \sin a) = \tmop{sgn} (\sin (a - u) - \sin a) = - \tmop{sgn} (u) \]
    (ii) Consider the region $\bar{E}$ away from the roots:
    
    By convexity of $\sin (t)$ on $[\pi, 2 \pi]$, note $r_1 = \pi + a, \delta =
    \cos a$, $\left| \cos \left( r_1 + \frac{\delta}{2} \right) \right| = \cos
    \left( a + \frac{\cos a}{2} \right) \geq \frac{1}{2} \cos a$, for $\forall x
    \in \bar{E}$,
    \[ | \sin x + \sin a | = | \sin x - \sin r_1 | \geq \sin r_1 - \sin (r_1 +
       \delta) \geq \delta \times \left| \cos \left( r_1 + \frac{\delta}{2}
       \right) \right| \geq \cos a \times \cos \left( a + \frac{\cos a}{2}
       \right) \geq \frac{1}{2} \cos^2 a \]
    For the measure $| \bar{E}_1 |$ of the set $\bar{E}_1 = \{ x \mid r_1 +
    \delta \leq x \leq r_2 - \delta \}$, note that $\cos a \leq \frac{\pi}{2} -
    a \leq \frac{\pi}{2} \cos a$,
    \[ 0 \leq | \bar{E}_1 | = (r_2 - r_1) - 2 \delta = (\pi - 2 a) - 2 \cos a
       \leq (\pi - 2) \cos a \]
    
    Regarding the differences of values of $\sin$, letting $r (a) \assign
    \tfrac{\pi}{2} - a$ for brevity, then by Taylor's theorem 
    \begin{eqnarray*}
      \sin (a + r (a) t) - \sin a & = & \cos a \cdot r (a) t - \frac{\sin
      \theta}{2} r^2 (a) t^2 \quad t \in [0, 1], \theta \in [a, a + r (a)] =
      \left[ a, \frac{\pi}{2} \right]\\
      \sin a - \sin (a - r (a) t) & = & \cos a \cdot r (a) t + \frac{\sin
      \theta'}{2} r^2 (a) t^2 \quad t \in [0, 1], \theta' \in [a - r (a), a] =
      \left[ 2 a - \frac{\pi}{2}, a \right]
    \end{eqnarray*}
    By $0 \leq t^2 \leq t \leq 1$, $0 \leq \sin a \leq \sin \theta \leq 1, \cos
    a \leq r (a) \leq \frac{\pi}{2} \cos a$ and $1 - 2 \cos^2 a = \sin \left( 2
    a - \frac{\pi}{2} \right) \leq \sin \theta' \leq \sin a$, note $\sin$ is
    concave on $\left[ a, \frac{\pi}{2} \right]$, and $1 + \frac{1 - 2 \cos^2
    a}{2} \cdot \frac{r (a)}{\cos a} t \geq 1 - \frac{\pi}{4}$, then
    \begin{eqnarray*}
      \frac{\cos^2 a}{2} t = \left( \frac{1 - \sin^2 a}{2} \right) t \leq (1 -
      \sin a) t \leq & \sin (a + r (a) t) - \sin a & \leq \cos a \cdot r (a) t\\
      \left( 1 - \frac{\pi}{4} \right) \cos a \cdot r (a) t \leq \left( 1 +
      \frac{1 - 2 \cos^2 a}{2} \cdot \frac{r (a)}{\cos a} t \right) \cos a \cdot
      r (a) t \leq & \sin a - \sin (a - r (a) t) & \leq \left( \cos a \cdot r
      (a) + \frac{r^2 (a)}{2} \right) t
    \end{eqnarray*}
    By applying $\cos a \leq r (a) \leq \frac{\pi}{2} \cos a$ again, thus
    \begin{eqnarray*}
      \frac{1}{\pi} \cos a \cdot r (a) t \leq & \sin (a + r (a) t) - \sin a &
      \leq \cos a \cdot r (a) t\\
      \left( 1 - \frac{\pi}{4} \right) \cos a \cdot r (a) t \leq & \sin a - \sin
      (a - r (a) t) & \leq \left( 1 + \frac{\pi}{4} \right) \cos a \cdot r (a) t
    \end{eqnarray*}
    Namely, for $\forall u \in [0, r (a)] = \left[ 0, \tfrac{\pi}{2} - a
    \right]$,
    \begin{eqnarray*}
      \left( 1 - \frac{\pi}{4} \right) \cos a \cdot u \leq \frac{1}{\pi} \cos a
      \cdot u \leq & \sin (a + u) - \sin a & \leq \cos a \cdot u \leq \left( 1 +
      \frac{\pi}{4} \right) \cos a \cdot u\\
      \left( 1 - \frac{\pi}{4} \right) \cos a \cdot u \leq & \sin a - \sin (a -
      u) & \leq \left( 1 + \frac{\pi}{4} \right) \cos a \cdot u
    \end{eqnarray*}
    \tmtextbf{(1) }For the integral $\int_0^{2 \pi} \frac{b^2 | \sin x + \sin a
    | \mathd x}{(b + | \sin x + \sin a |)^2}$
    
    In the region $E = E_1 \cup E_2$ near the roots, by using $\frac{1}{2} \cos
    a \cdot | u | \leq | \sin (r_i + u) + \sin a | \leq \frac{3}{2} \cos a \cdot
    | u |$ for $| u | < \delta$,
    \[ b^2 \int_{| u | < \delta} \frac{\frac{1}{2} \cos a | u | \mathd u}{\left(
       b + \frac{3}{2} \cos a | u | \right)^2} \leq \int_{E_i} \frac{b^2 | \sin
       x + \sin a | \mathd x}{(b + | \sin x + \sin a |)^2} \leq b^2 \int_{| u |
       < \delta} \frac{\frac{3}{2} \cos a | u | \mathd u}{\left( b + \frac{1}{2}
       \cos a | u | \right)^2} \quad \forall i \in \{ 1, 2 \} \]
    Let $t = \frac{u}{\delta} = \frac{u}{\cos a}$, and note that $\int_E =
    \int_{E_1} + \int_{E_2}$, then
    \[ 2 \times \frac{\frac{1}{2}}{\left( \frac{3}{2} \right)^2} \times
       \frac{b^2}{\cos a} \int^1_0 \frac{t \mathd t}{\left( \frac{\frac{2}{3}
       b}{\cos^2 a} + t \right)^2} \leq \int_E \frac{b^2 | \sin x + \sin a |
       \mathd x}{(b + | \sin x + \sin a |)^2} \leq 2 \times
       \frac{\frac{3}{2}}{\left( \frac{1}{2} \right)^2} \times \frac{b^2}{\cos
       a} \int^1_0 \frac{t \mathd t}{\left( \frac{2 b}{\cos^2 a} + t \right)^2}
    \]
    Note that $\int_0^1 \frac{t \mathd t}{(C + t)^2} = \ln (1 + C^{- 1}) -
    \frac{1}{1 + C} = \Theta (\log C^{- 1})$ for $0 < C \lesssim 1$, and under the
    assumption $0 < b \lesssim \cos^2 a$,
    \[ \int_E \frac{b^2 | \sin x + \sin a | \mathd x}{(b + | \sin x + \sin a
       |)^2} = \Theta \left( \frac{b^2}{\cos a} \log \left( \frac{\cos^2 a}{b}
       \right) \right) \]
    In the region $\bar{E} = \bar{E}_1 \cup \bar{E}_2$ away from the roots, note
    that $0 \leq | \bar{E}_1 | \leq (\pi - 2) \cos a$, $| \sin x + \sin a | \geq
    \frac{\cos^2 a}{2}, \forall x \in \bar{E}$,
    \[ 0 \leq \int_{\bar{E}_1} \frac{b^2 | \sin x + \sin a | \mathd x}{(b + |
       \sin x + \sin a |)^2} \leq | \bar{E}_1 | \times \sup_{x \in \bar{E}_1}
       \frac{b^2}{| \sin x + \sin a |} \leq (\pi - 2) \cos a \times \frac{2
       b^2}{\cos^2 a} = \Theta \left( \frac{b^2}{\cos a} \right) \]
    Under the assumption $0 < b \lesssim \cos^2 a$, then $2 | \sin x + \sin a
    | \geq \cos^2 a \gtrsim b, \forall x \in \bar{E}$, and note $\sin x + \sin a >
    0, \forall x \in \bar{E}_2$,
    \[ \frac{1}{9} \cdot b^2 \times \int_{- a + \delta}^{\pi + a - \delta}
       \frac{\mathd x}{\sin x + \sin a} \leq \int_{\bar{E}_2} \frac{b^2 | \sin x
       + \sin a | \mathd x}{(b + | \sin x + \sin a |)^2} \leq b^2 \times \int_{-
       a + \delta}^{\pi + a - \delta} \frac{\mathd x}{\sin x + \sin a} \]
    Note that $\int_{- a + \delta}^{\pi + a - \delta} \frac{\mathd x}{\sin x +
    \sin a} = \frac{2}{\cos a} \ln \left( \frac{\cos \left( a - \frac{\delta}{2}
    \right)}{\sin \left( \frac{\delta}{2} \right)} \right) = \Theta \left(
    \frac{1}{\cos a} \right)$, since $\frac{\cos \left( a - \frac{\delta}{2}
    \right)}{\sin \left( \frac{\delta}{2} \right)} \in [1.8, 3]$ for $\delta =
    \cos a, 0 \leq a < \frac{\pi}{2}$,
    \[ \int_{\bar{E}_2} \frac{b^2 | \sin x + \sin a | \mathd x}{(b + | \sin x +
       \sin a |)^2} = \Theta \left( \frac{b^2}{\cos a} \right) \]
    Note that $\int_{\bar{E}} = \int_{\bar{E}_1} + \int_{\bar{E}_2}$, then
    \[ \int_{\bar{E}} \frac{b^2 | \sin x + \sin a | \mathd x}{(b + | \sin x +
       \sin a |)^2} = \Theta \left( \frac{b^2}{\cos a} \right) \]
    Combine these two regions $E$ and $\bar{E}$, $[0, 2 \pi] = E \cup \bar{E}$
    \[ \int_0^{2 \pi} \frac{b^2 | \sin x + \sin a | \mathd x}{(b + | \sin x +
       \sin a |)^2} = \Theta \left( \frac{b^2}{\cos a} \log \left( \frac{\cos^2
       a}{b} \right) \right) + \Theta \left( \frac{b^2}{\cos a} \right) = \Theta
       \left( \frac{b^2}{\cos a} \log \left( \frac{\cos^2 a}{b} \right) \right)
    \]
    \tmtextbf{(2)} For the integral $\int_0^{2 \pi} \frac{b^2 \tmop{sgn} (\sin
    (x + a)) | \cos x + \cos a | \mathd x}{(b + | \sin x + \sin a |)^2}$
    \begin{eqnarray*}
      &  & \int_0^{2 \pi} \frac{b^2 \tmop{sgn} (\sin (x + a)) | \cos x + \cos a
      | \mathd x}{(b + | \sin x + \sin a |)^2} = \int_0^{2 \pi} \frac{b^2 (-
      1)^{\tmmathbf{1}_{0 \leq u \leq 2 a}} (- \cos (a + u) + \cos a) \mathd
      x}{(b + | - \sin (a + u) + \sin a |)^2}\\
      & = & \left\{ - \int_0^{\frac{\pi}{2} - a} \frac{b^2 (- \cos (a + u) +
      \cos a) \mathd u}{(b + \sin (a + u) - \sin a)^2} + \int_{- \left(
      \frac{\pi}{2} - a \right)}^0 \frac{b^2 (- \cos (a + u) + \cos a) \mathd
      u}{(b - \sin (a + u) + \sin a)^2} \right\}\\
      & + & \left\{ - \int_{\frac{\pi}{2} - a}^{2 \left( \frac{\pi}{2} - a
      \right)} \frac{b^2 (- \cos (a + u) + \cos a) \mathd u}{(b + \sin (a + u) -
      \sin a)^2} + \int_{2 \left( \frac{\pi}{2} - a \right)}^{3 \left(
      \frac{\pi}{2} - a \right)} \frac{b^2 (- \cos (a + u) + \cos a) \mathd
      u}{(b - \sin (a + u) + \sin a)^2} \right\}
      + \int_{3 \left( \frac{\pi}{2} - a \right)}^{\frac{3 \pi}{2} + a}
      \frac{b^2 (- \cos (a + u) + \cos a) \mathd u}{(b - \sin (a + u) + \sin
      a)^2}\\
      & = & 2 b^2 \cos a \int_0^{\frac{\pi}{2} - a} \left[ - \frac{1}{(b + \sin
      (a + u) - \sin a)^2} + \frac{1}{(b - \sin (a - u) + \sin a)^2} \right]
      \mathd u + 2 b^2 \cos a \int_0^{2 a} \frac{\mathd u}{(b + \cos u + \sin
      a)^2}\\
      & = & - 2 b^2 \cos a \times 2 \sin a \int_0^{\frac{\pi}{2} - a} \frac{4
      \sin^2 \frac{u}{2} (b + \cos a \sin u)}{(b + \sin (a + u) - \sin a)^2 (b -
      \sin (a - u) + \sin a)^2} \mathd u
      + 2 b^2 \cos a \int_0^{2 a} \frac{\mathd u}{(b + \cos u + \sin a)^2}
    \end{eqnarray*}
    Note that $\max \{ \sin (a + u) - \sin a, \sin a - \sin (a - u) \} \leq
    \left( 1 + \frac{\pi}{4} \right) \cos a \cdot u$, and $\left( 1 -
    \frac{\pi}{4} \right) \cos a \cdot u \leq \max \{ \sin (a + u) - \sin a,
    \sin a - \sin (a - u) \}$, for $r (a) = \frac{\pi}{2} - a$, we have $\cos a
    \leq r (a) \leq \frac{\pi}{2} \cos a$, $\frac{2}{\pi} u \leq \sin u \leq u,
    \frac{8}{\pi^2} u^2 \leq 4 \sin^2 \frac{u}{2} \leq u^2$, then we can
    establish such upper/lower bounds for the integral:
    \begin{eqnarray*}
      &  & \int_0^{\frac{\pi}{2} - a} \frac{4 \sin^2 \frac{u}{2} (b + \cos a
      \sin u)}{(b + \sin (a + u) - \sin a)^2 (b - \sin (a - u) + \sin a)^2}
      \mathd u \leq \int_0^{\frac{\pi}{2} - a} \frac{(b + \cos a \cdot u) u^2
      \mathd u}{\left( b + \left( 1 - \frac{\pi}{4} \right) \cos a \cdot u
      \right)^4}\\
      & \leq & \frac{1}{1 - \tfrac{\pi}{4}} \int^{r (a)}_0 \frac{u^2 \mathd
      u}{\left( b + \left( 1 - \frac{\pi}{4} \right) \cos a \cdot u \right)^3} =
      \frac{1}{1 - \frac{\pi}{4}} \int_0^1 \frac{t^2 \mathd t}{\left( \frac{b}{r
      (a)} + \left( 1 - \frac{\pi}{4} \right) \cos a \cdot t \right)^3}\\
      & \leq & \frac{1}{\left( 1 - \frac{\pi}{4} \right)^4} \times
      \frac{1}{\cos^3 a} \int_0^1 \frac{t^2 \mathd t}{\left( \frac{b / \cos^2
      a}{\left( 1 - \frac{\pi}{4} \right) \cdot \frac{\pi}{2}} + t \right)^3}\\
      &  & \int_0^{\frac{\pi}{2} - a} \frac{4 \sin^2 \frac{u}{2} (b + \cos a
      \sin u)}{(b + \sin (a + u) - \sin a)^2 (b - \sin (a - u) + \sin a)^2}
      \mathd u \geq \frac{8}{\pi^2} \int_0^{\frac{\pi}{2} - a} \frac{\left( b +
      \frac{2}{\pi} \cos a \cdot u \right) u^2 \mathd u}{\left( b + \left( 1 +
      \frac{\pi}{4} \right) \cos a \cdot u \right)^4}\\
      & \geq & \frac{16}{\pi^3 \left( 1 + \frac{\pi}{4} \right)} \int^{r (a)}_0
      \frac{u^2 \mathd u}{\left( b + \left( 1 + \frac{\pi}{4} \right) \cos a
      \cdot u \right)^3} = \frac{16}{\pi^3 \left( 1 + \frac{\pi}{4} \right)}
      \int_0^1 \frac{t^2 \mathd t}{\left( \frac{b}{r (a)} + \left( 1 +
      \frac{\pi}{4} \right) \cos a \cdot t \right)^3}\\
      & \geq & \frac{16}{\pi^3 \left( 1 + \frac{\pi}{4} \right)^4} \times
      \frac{1}{\cos^3 a} \int_0^1 \frac{t^2 \mathd t}{\left( \frac{b / \cos^2
      a}{1 + \frac{\pi}{4}} + t \right)^3}
    \end{eqnarray*}
    Note $\int_0^1 \frac{t^2 \mathd t}{(C + t)^3} = \ln (1 + C^{- 1}) - \frac{3
    + 2 C}{2 (1 + C)^2} = \Theta (\log C^{- 1})$ for $0 < C \lesssim 1$, and under the
    assumption $0 < b \lesssim \cos^2 a$,
    \[ 2 b^2 \cos a \times 2 \sin a \int_0^{\frac{\pi}{2} - a} \frac{4 \sin^2
       \frac{u}{2} (b + \cos a \sin u)}{(b + \sin (a + u) - \sin a)^2 (b - \sin
       (a - u) + \sin a)^2} \mathd u = \Theta \left( \tan a \cdot
       \frac{b^2}{\cos a} \log \left( \frac{\cos^2 a}{b} \right) \right) \]
    For the other integral, note that
    \begin{eqnarray*}
      \cos u - \cos 2 a & \geq & \min \left\{ \frac{1 - \cos 2 a}{2 a}, \sin 2 a
      \right\} (2 a - u) = \min \left\{ \frac{\sin a}{a}, 2 \cos a \right\} \sin
      a \cdot (2 a - u) \geq \cos a \cdot \sin a \cdot (2 a - u)\\
      \cos u + \sin a & \geq & \frac{1 + 2 \sin a}{1 + \sin a} \cos^2 a + 2 a
      \sin a \times \cos a \cdot \left( 1 - \frac{u}{2 a} \right) \geq \cos^2 a
      + 2 \sin^2 a \times \cos a \cdot \left( 1 - \frac{u}{2 a} \right)
    \end{eqnarray*}
    By susbstitution $t = 1 - \frac{u}{2 a} \in [0, 1]$, and $a \leq
    \frac{\pi}{2} \sin a$ for $a \in \left[ 0, \tfrac{\pi}{2} \right]$
    \begin{eqnarray*}
      \int_0^{2 a} \frac{\mathd u}{(b + \cos u + \sin a)^2} & \leq & 2 a
      \int_0^1 \frac{\mathd t}{(\cos^2 a + 2 \sin^2 a \times \cos a \cdot t)^2}
      = \frac{2 a}{\cos^4 a} \int_0^1 \frac{\mathd t}{\left( 1 + \frac{2 \sin^2
      a}{\cos a} t \right)^2} = \frac{2 a}{\cos^4 a} \cdot \frac{1}{1 + \frac{2
      \sin^2 a}{\cos a}}\\
      & = & \frac{2 a}{\cos^3 a} \cdot \frac{1}{1 + (1 - \cos a) (1 + 2 \cos
      a)} \leq \frac{2 a}{\cos^3 a} \leq \pi \times \tan a \cdot \frac{1}{\cos^2
      a}
    \end{eqnarray*}
    Thus, we have
    \[ 2 b^2 \cos a \int_0^{2 a} \frac{\mathd u}{(b + \cos u + \sin a)^2}
       =\mathcal{O} \left( \tan a \cdot \frac{b^2}{\cos a} \right) \]
    Combine these two integrals
    \begin{eqnarray*}
      \int_0^{2 \pi} \frac{b^2 \tmop{sgn} (\sin (x + a)) | \cos x + \cos a |
      \mathd x}{(b + | \sin x + \sin a |)^2} & = & \Theta \left( \tan a \cdot
      \frac{b^2}{\cos a} \log \left( \frac{\cos^2 a}{b} \right) \right)
      -\mathcal{O} \left( \tan a \cdot \frac{b^2}{\cos a} \right)\\
      & = & \mathcal{O} \left( \tan a \cdot \frac{b^2}{\cos a} \log \left(
      \frac{\cos^2 a}{b} \right) \right)
    \end{eqnarray*}
\end{proof}

\newpage
\subsection{Identities and Inequalities for Useful Expectations}
\begin{lemma}
    \label{suplem:expectation}
    Suppose standard Gaussian variables \(g, g' \sim \mathcal{N} (0, 1)\) have the correlation coefficient \(\E[g g'] = \sin \varphi\) with \(\varphi \in [-\frac{\pi}{2}, \frac{\pi}{2}]\), then
    \begin{eqnarray*}
        \E[\sgn(g g')]=\frac{2}{\pi} \varphi, \quad
        \E[|g g'|]=\frac{2}{\pi}[\varphi \sin \varphi + \cos \varphi], \quad
        \E[g^2\sgn(g g')]=\frac{2}{\pi} [\varphi + \sin \varphi \cos \varphi].
    \end{eqnarray*}
\end{lemma}
\begin{proof}
    Let's prove the first identity (Grothendieck's identity) at the beginning.
    Since we can express \(g, g'\) in terms of two independent variables \(R, U\) 
    with \(R \sim r\exp(-\frac{r^2}{2})\mathbb{I}_{r\geq 0}\) (standard Rayleigh distribution with \(\E[R^2]=2\)) 
    and \(U \sim \mathrm{Unif} [0, 4 \pi)\), 
    \[
        g = R \cos((U-\varphi)/2), \quad g' = R \sin((U+\varphi)/2)
    \]
    Noting that \(\sgn(g g') = \sgn(\sin(U)+\sin \varphi)\) for \(R\neq 0\), we have
    \[
     \E[\sgn(g g')] = \E[\sgn(\sin(U)+\sin \varphi)] = \frac{(+1)\cdot(2\pi+4\varphi)+(-1)\cdot(2\pi-4\varphi)}{4\pi} 
     = \frac{2}{\pi} \varphi
    \]
    We start the proof of the second identity by applying Price's theorem (see~\cite{price1958useful}), 
    \[
        \frac{1}{\cos \varphi}\cdot \frac{\mathd \E[|g g'|]}{\mathd \varphi} 
        = \frac{\mathd \E[|g g'|]}{\mathd \sin \varphi}
        = \E\left[ \frac{\partial^2 |g g'|}{\partial g \partial g'} \right] = \E[\sgn(g g')] = \frac{2}{\pi} \varphi
    \]
    Note that when \(\varphi = -\frac{\pi}{2}\), \(\E[|g g'|]_{\varphi = -\frac{\pi}{2}} = \E[g^2] = 1\), we have
    \[
        \E[|g g'|] = \E[|g g'|]_{\varphi = -\frac{\pi}{2}} + \frac{2}{\pi}\int^{\varphi}_{-\frac{\pi}{2}} \varphi' \cos \varphi' \mathd \varphi' 
        = \frac{2}{\pi}[\varphi \sin \varphi + \cos \varphi]
    \]
    To prove the third identity, we start by expressing \(g' = \sin \varphi \cdot g + \cos \varphi \cdot h\) for some \(h\sim \mathcal{N}(0, 1)\) with \(h\ind g\),
    and applying Stein's lemma (see Lemma 2.1 of~\cite{ross2011fundamentals}), 
    and noting that \(\frac{\partial}{\partial g} \left\{ |g| \sgn(\sin \varphi \cdot g + \cos \varphi \cdot h) \right\} 
    =  \sgn(g g') + 2 \sin \varphi \cdot |g| \delta(g')\),
    \[
        \E[g^2 \sgn(g g')]=\E_h\E_g[g \cdot |g| \sgn(\sin \varphi \cdot g + \cos \varphi \cdot h)]
        = \E_h \E_g\left[\frac{\partial |g| \sgn(\sin \varphi \cdot g + \cos \varphi \cdot h)}{\partial g}  \right]
        = \E[\sgn(g g')] + 2 \sin \varphi \E[|g| \delta(g')]
    \]
    where the former term is the first identity \(\E[\sgn(g g')]=\frac{2}{\pi} \varphi\), and the latter term can be evaluated by the symmetricity of \(g, g'\),
    \[
        \E[|g| \delta(g')] = \E[|g'|\delta(g)] =  \E_h\E_g[|\sin \varphi \cdot g + \cos \varphi \cdot h| \delta(g)] 
        = \E_h\E_g[|\cos \varphi \cdot h| \delta(g)]
        = \cos \varphi \E_h[|h|] \E_g[\delta(g)]  
        = \frac{1}{\pi} \cos \varphi
    \]
    Thus, by combining these two terms, we obtain the third identity, and the proof is complete.
\end{proof}

\begin{lemma}\label{suplem:moments}
  Suppose a random variable \(X\) has a density \(X\sim f_X(x) = K_0(|x|)/\pi\) involving modified Bessel function \(K_0\) of the second kind of order \(0\), then:
  \[
      \E[|X|] = \frac{2}{\pi}, \quad \E[X^{2n}] = [(2n-1)!!]^2 = [(2n-1)\times \cdots \times 3\times 1]^2\quad \forall n\in\mathbb{Z}_+
  \]
\end{lemma}
\begin{proof}
  Since the product \(X= g g'\) of two independent standard Gaussian variables \(g, g'\sim \mathcal{N}(0, 1), \E[g g'] = 0\) has a density \(X\sim f_X(x) = K_0(|x|)/\pi\) involving modified Bessel function \(K_0\) of the second kind of order \(0\) 
  (see page 50, Section 4.4 Bessel Function Distributions, Chapter 12 Continuous Distributions (General) of~\cite{johnson1970continuous}), we have
  \begin{eqnarray*}
  \E[|X|] &=& \E[|gg'|] = \E[|g|]\E[|g'|] = \E[|g|]^2 = \frac{2}{\pi}\\
  \E[X^{2n}] &=& \E[(gg')^{2n}]= \E[g^{2n}]\E[g'^{2n}] =\E[g^{2n}]^2 = [(2n-1)!!]^2
  \end{eqnarray*}
  due to the well-known results that \(\E[g^{2n}] = (2n-1)!!,\E[|g|] = \sqrt{\frac{2}{\pi}}\) for \(g\sim \mathcal{N}(0, 1)\) 
  (see equations (13.11), (13.14) on page 89 ang page 91, Section 3 moments and other properties, Chapter 13 Normal Distributions of~\cite{johnson1970continuous}).
\end{proof}

\begin{lemma}\label{suplem:bounds_expectation_A}
    Suppose \(A>0\) and a random variable \(X\) has a density \(X\sim f_X(x) = K_0(|x|)/\pi\) involving modified Bessel function \(K_0\) of the second kind of order \(0\), then:
    \[
        \frac{\sqrt{12A^2+1}-1}{6A} < \E[\tanh(AX)X] < \frac{\sqrt{4A^2+1}-1}{2A}
    \]
\end{lemma}

\begin{proof}
    Let \(f(A) : = \E[\tanh(AX)X]\), and \(\underline{B}(A) := \frac{\sqrt{12A^2+1}-1}{6A}, \overline{B}(A) := \frac{\sqrt{4A^2+1}-1}{2A}\).
    For \(A>0\), we have:
    \[
      0< f(A) < \E[|X|] = \frac{2}{\pi},\quad 0<\underline{B}(A)<\lim_{A\to \infty}\underline{B}(A)=\frac{1}{\sqrt{3}}, \quad 0<\overline{B}(A) < \lim_{A\to \infty}\overline{B}(A)=1
    \]
    And \(f(A), \underline{B}(A), \overline{B}(A)\) are all increasing and concave with respect to \(A>0\), due to the fact that:
    \[
      f'(A) = \E[(1-\tanh^2(AX))X^2] > 0, f''(A) = -2\E\left[\frac{\tanh(AX)X^3}{\cosh^2(AX)}\right] < 0;\quad
      \underline{B}'(A) > 0, \underline{B}''(A) < 0;
      \overline{B}'(A) > 0, \overline{B}''(A) < 0
    \]
    By Taylor expansions of \(\tanh(t) = t - \frac{t^3}{3} + \frac{2t^5}{15} + \cdots\) and \(\underline{B}(A), \overline{B}(A)\) and using \(\E[X^{2n}] = [(2n-1)!!]^2\) in Lemma~\ref{suplem:moments}:
    \begin{eqnarray*}
        A - 3A^3  <f(A)< A - 3A^3 + 30 A^5,\quad
        A - 3A^3  < \underline{B}(A) < A - 3A^3 + 18 A^5,\quad
        A - A^3  < \overline{B}(A) < A - A^3 + 2A^5
    \end{eqnarray*}
    When \(0< A\leq 1/4\) is small enough, we have shown:
    \[
      f(A) < A - 3A^3 + 30 A^5 = A-A^3 -2A^3(1-15A^2) < A-A^3 < \overline{B}(A)
    \]
    Similarly, we show \(\underline{B}(A) <f(A)\) when \(A\) is sufficiently small by comparing the series approximations of \(f(A), \underline{B}(A)\) at \(A\to 0_+\), 
    and we extent it to the range of \(0<A\leq 1/4\) by seprating \(\underline{B}(A)\) and \(f(A)\) with piecewise linear functions
    based on the numerical evaluations of \(f(A),f'(A), \underline{B}(A),\underline{B}'(A)\) at finite points and noting \(f(A),\underline{B}(A)\) are concave.

    Let's show the inequalities \(\underline{B}(A) < f(A)\) and \(f(A) < \overline{B}(A)\) for \(1/4\leq A \leq 5/4\) by numerical evaluations at some points \(A_1 = \frac{1}{4}, A_2=\frac{1}{3}, A_3=\frac{1}{2}, A_4=\frac{3}{4}, A_5=\frac{5}{4}\).
    \begin{eqnarray*}
      & & \underline{B}(A) \leq \underline{B}(A_i)+\underline{B}'(A_i)(A-A_i) < f(A_i)+\frac{f(A_{i+1})-f(A_i)}{A_{i+1}-A_i}(A-A_i) \leq f(A)\\
      & & f(A) \leq f(A_i) + f'(A_i)(A-A_i) < \overline{B}(A_i)+\frac{\overline{B}(A_{i+1})-\overline{B}(A_i)}{A_{i+1}-A_i}(A-A_i) \leq \overline{B}(A) \quad A\in[A_i, A_{i+1}], i\in\{1,2,3,4\}
    \end{eqnarray*}
    The above inequalities hold due to that \(f(A), \underline{B}(A), \overline{B}(A)\) are concave and the numerical evaluations show 
    \(\underline{B}(A_i)<f(A_i)<\overline{B}(A_i), \underline{B}(A_i)+\underline{B}'(A_i)(A_{i+1}-A_i)<f(A_{i+1}), f(A_i) + f'(A_i)(A_{i+1}-A_i)<\overline{B}(A_{i+1})\) for \(i\in\{1,2,3,4\}\).
    Therefore, we have shown \(\underline{B}(A) < f(A) < \overline{B}(A)\) for \(0< A \leq 5/4\) by combining these two cases.
    For the case of \(A>5/4\), we have shown:
    \begin{eqnarray*}
      & & \underline{B}(A)< \underline{B}(2.33) \approx 0.5102 < 0.5111 \approx f(5/4) < f(A) \quad \forall A\in(5/4, 2.33)\\
      & & \underline{B}(A) < \lim_{A\to \infty}\underline{B}(A) = \frac{1}{\sqrt{3}} \approx 0.5774 < 0.5776 \approx f(2.33) \leq f(A) \quad \forall A\in[2.33, \infty)\\
      & & f(A) \leq \lim_{A\to \infty}f(A) = \E[|X|] = \frac{2}{\pi} \approx 0.6366 < 0.6770 \approx \overline{B}(5/4) < \overline{B}(A) \quad \forall A\in(5/4, \infty)
    \end{eqnarray*}
    by the fact that \(f(A), \underline{B}(A), \overline{B}(A)\) are increasing and \(\E[|X|] = \frac{2}{\pi}\) in Lemma~\ref{suplem:moments}, then the proof is complete.
\end{proof}

\subsection{Limit Behavior and Non-Increasing Property of Some Helper Functions}
\begin{lemma}\label{suplem:limit_delta_epsilon}
  Suppose these three conditions for function \(F(A, k): \mathbb{R}_+ \times \mathbb{R}_+\to \mathbb{R}\) hold:
  \begin{enumerate}
    \item \(F(A, k)\) is continuous in \(k\) for any \(A>0\), its unique root \(k=k^\ast(A)\) of \(F(A,k) =0\) is a interior point of the compact interval \(I:=[a, b], 0<a<b<\infty\).
    \item \(F(A, k)\) converges uniformly to \(F_{A_0}(k):=\lim_{A\to A_0}F(A, k)\) on compact interval \(I=[a, b]\) as \(A\to A_0\), where \(A_0\) is a positive number or \(0_{+}, \infty\).
    \item The unique root \(k=k_{A_0}^\ast\) of \(F_{A_0}(k) =0\) is a interior point of the compact interval \(I=[a, b]\); 
    the derivative \(\frac{\mathd F_{A_0}(k)}{\mathd k}\mid_{k=k_{A_0}^\ast} \neq 0\), and \(\frac{\mathd F_{A_0}(k)}{\mathd k}\) is continuous in \(k\) on \(I=[a, b]\).
  \end{enumerate}
  Then, we have:
  \[
  \lim_{A\to A_0} k^\ast(A) = k_{A_0}^\ast.
  \]
\end{lemma}

\begin{proof}
  To prove that $\lim_{A\to A_0} k^\ast(A) = k_{A_0}^\ast$, we must show that for any $\varepsilon > 0$, 
  if \(A_0\) is finite (\(A_0\) is a positive number or \(0_+\)), there exists a \(\delta > 0\) such that \(|k^\ast(A) - k_{A_0}^\ast| < \varepsilon\) for all \(A > 0\) with \(|A - A_0| < \delta\); 
  and if \(A_0 = \infty\), there exists an \(N > 0\) such that \(|k^\ast(A) - k_{A_0}^\ast| < \varepsilon\) for all \(A > N\).

  By the condition 3, for the root \(k=k_{A_0}^\ast\in (a, b)\) of \(F_{A_0}(k) =0\), there exists a \(0<\delta_{A_0}< \min(\varepsilon, k_{A_0}^\ast -a, b-k_{A_0}^\ast)\) such that 
  \[
    F_{A_0}(k^\ast_{A_0}+\Delta k)\cdot F_{A_0}(k^\ast_{A_0}-\Delta k) =(F_{A_0}(k^\ast_{A_0}+\Delta k)-F_{A_0}(k^\ast_{A_0}))(F_{A_0}(k^\ast_{A_0}-\Delta k)-F_{A_0}(k^\ast_{A_0})) < 0
  \]
  for any \(\Delta k\in(0, \delta_{A_0})\). Let's select \(\Delta k = \delta_{A_0}/2\), then we have:
  \[
    F_{A_0}\left(k^\ast_{A_0}+\frac{\delta_{A_0}}{2}\right)\cdot F_{A_0}\left(k^\ast_{A_0}-\frac{\delta_{A_0}}{2}\right) < 0.
  \]
  Without loss of generality, we may assume \(F_{A_0}\left(k^\ast_{A_0}+\frac{\delta_{A_0}}{2}\right)> 0, F_{A_0}\left(k^\ast_{A_0}-\frac{\delta_{A_0}}{2}\right)< 0\), and select:
  \[
    m_{A_0} := \min\left(\left|F_{A_0}\left(k^\ast_{A_0}+\frac{\delta_{A_0}}{2}\right) \right|, \left| F_{A_0}\left(k^\ast_{A_0}-\frac{\delta_{A_0}}{2}\right) \right| \right) > 0
  \]

  By the condition 2, since \(k^\ast_{A_0}+\frac{\delta_{A_0}}{2}, k^\ast_{A_0}-\frac{\delta_{A_0}}{2}\in [a, b]=I\), 
  there exists a \(\delta > 0\) for all \(A > 0\) with \(|A - A_0| < \delta\) if \(A_0\) is finite; or there exists an \(N > 0\) for all \(A > N\) if \(A_0 = \infty\) such that:
  \[
    \max\left(\left| F(A, k^\ast_{A_0}+\frac{\delta_{A_0}}{2})-F_{A_0}(k^\ast_{A_0}+\frac{\delta_{A_0}}{2})\right|, \left| F(A, k^\ast_{A_0}-\frac{\delta_{A_0}}{2})-F_{A_0}(k^\ast_{A_0}-\frac{\delta_{A_0}}{2})\right| \right) \leq \sup_{k\in I}|F_{A_0}(k) - F_{A_0}(k^\ast_{A_0})| < \frac{m_{A_0}}{2}.
  \]
  Therefore, the signs of function values will not change as long as \(A\) is close to \(A_0\).
  \begin{eqnarray*}
    F(A, k^\ast_{A_0}+\frac{\delta_{A_0}}{2}) & \geq & F_{A_0}(k^\ast_{A_0}+\frac{\delta_{A_0}}{2}) - \left| F(A, k^\ast_{A_0}+\frac{\delta_{A_0}}{2})-F_{A_0}(k^\ast_{A_0}+\frac{\delta_{A_0}}{2})\right| > m_{A_0} - \frac{m_{A_0}}{2} = \frac{m_{A_0}}{2} > 0\\
    F(A, k^\ast_{A_0}-\frac{\delta_{A_0}}{2}) & \leq & F_{A_0}(k^\ast_{A_0}-\frac{\delta_{A_0}}{2}) + \left| F(A, k^\ast_{A_0}-\frac{\delta_{A_0}}{2})-F_{A_0}(k^\ast_{A_0}-\frac{\delta_{A_0}}{2})\right| < -m_{A_0} + \frac{m_{A_0}}{2} = -\frac{m_{A_0}}{2} < 0
  \end{eqnarray*}
  By condition 1, the application of the intermediate value theorem, under the condition that \(F(A, k)\) is continuous in \(k\) for any \(A>0\) and the uniqueness of the root \(k=k^\ast(A)\) of \(F(A, k) =0\), ensures that:
  \[
  k^\ast(A) \in \left(k^\ast_{A_0}-\frac{\delta_{A_0}}{2}, k^\ast_{A_0}+\frac{\delta_{A_0}}{2}\right) \subset (a, b) \subset I
  \]
  Then, we have:
  \[
  |k^\ast(A) - k_{A_0}^\ast| < \max\left(\left|(k^\ast_{A_0}-\frac{\delta_{A_0}}{2})-k^\ast_{A_0}\right|, \left|(k^\ast_{A_0}+\frac{\delta_{A_0}}{2})-k^\ast_{A_0}\right|\right) =\frac{\delta_{A_0}}{2} < \min(\varepsilon, k_{A_0}^\ast -a, b-k_{A_0}^\ast)/2 < \varepsilon
  \]
  for all \(A > 0\) with \(|A - A_0| < \delta\) if \(A_0\) is finite (\(A_0\) is a positive number or \(0_+\)); or for all \(A > N\) if \(A_0 = \infty\).
  Therefore, we have shown that \(\lim_{A\to A_0} k^\ast(A) = k_{A_0}^\ast\), and the proof is complete.
\end{proof}

\begin{lemma}\label{suplem:H_nondecreasing}
  Suppose \(s\in[0, 1]\) and \(\delta\geq 0\), then the function \(H(s, \delta)\) is non-decreasing with respect to \(s\in[0,1]\) given fixed \(\delta\geq 0\), and \(H(s, \delta)\) is defined as:
  \[
      H(s, \delta) = \left(s \arcsin\frac{s}{\sqrt{1+\delta}} + \sqrt{1-s^2+\delta}\right)^2 + \frac{(1+\delta)(1-s^2)}{1-s^2 + \delta}
  \]
\end{lemma}
\begin{proof}
  When \(\delta = 0\), we have \(H(s, 0) = (s \arcsin s + \sqrt{1-s^2})^2\), which is non-decreasing with respect to \(s\in[0,1]\), 
  due to the fact that \(\partial_s H(s, 0) = 2(s \arcsin s + \sqrt{1-s^2}) \cdot\arcsin s > 0, \forall s\in(0, 1)\). 
  When \(\delta > 0\), we introduce following notations to show \(\partial_s H(s, \delta) > 0, \forall s\in(0, 1)\):
  \[
      r := 1-s^2, y := \arcsin\frac{s}{\sqrt{1+\delta}}, q := \sqrt{1-s^2+\delta}, K(s) := sy+q
  \]
  Then we have \(H(s, \delta) = K(s)^2 + \frac{(1+\delta)r}{r + \delta}\) and 
  \[
  y'(s)= 1/q, q' = -s/q, K'(s) = y \implies \partial_s H(s, \delta) = 2K(s)K'(s) + \partial_s\left(\frac{(1+\delta)r}{r + \delta}\right) 
  = 2\left[K(s) y - \frac{s(1+\delta)\delta}{q^4}\right]
  \]
  By susbstitution \(s = \sqrt{1+\delta}\sin y, q=\sqrt{1+\delta}\cos y, K(s) = \sqrt{1+\delta}(y\sin y + \cos y)\) for some \(y\in(0, 1)\), we have
  \[
      K(s) y - \frac{s(1+\delta)\delta}{q^4} 
      =  \frac{\sin y}{\sqrt{1+\delta}\cos^4 y} \Phi(y) \geq \frac{\sin y}{\sqrt{1+\delta}\cos^4 y} \Phi(0_+)
      > 0
  \]
  where the last inequality is due to the fact that \(\Phi(y)\) is non-decreasing with respect to \(y\in(0, \frac{\pi}{2})\) given fixed \(\delta\), and \(\Phi(0_+) = 1\), and \(\Phi(s)\) is defined as:
  \[
      \Phi(y) = (1+\delta)y (y+\cot y)\cos^4 y -\delta
      \implies
      \partial_y \Phi(y) = \frac{(1+\delta)\cos^3y (1+4 y \sin y)}{2 \sin^2 y}\cos(2y)(\tanh(2y) - 2y)\geq0
  \]
  Therefore, \(\partial_s H(s, \delta) > 0, \forall s\in(0, 1)\) and \(H(s, \delta)\) is non-decreasing with respect to \(s\in[0,1]\) given fixed \(\delta\geq 0\).
\end{proof}

\subsection{Probability Concentration Inequality for Bernoulli Random Variables}
\begin{lemma}\label{suplem:prob_inequality}
    Let $q \assign \max (p, 1 - p)$ for $V_i \overset{\tmop{iid}}{\sim}
      \tmop{Bern} (p), \forall i \in [n] := \{1, 2, \cdots, n\}$, then
    
    for $t\in \mathbb{R}_{\geq 0}$
        \[ \mathbb{P} \left( \left|\frac{1}{n} \sum_{i \in [n]} (V_i -\mathbb{E} [V_i]) \right|
              \geq t \right) \leq 2 \exp(-2n t^2) \]
    for $t \in [\mathe (1 - q), q]\neq \varnothing$
        \[ \mathbb{P} \left( \left|\frac{1}{n} \sum_{i \in [n]} (V_i -\mathbb{E} [V_i]) \right|
           \geq t \right) \leq 2\exp \left( - n \left\{ \frac{t}{q}  \left[ \ln
           \frac{t}{(1 - q)} - 1 \right] + \frac{t^2}{2 q^2} \right\} \right) \]
    for $t \in(q, \infty)$, we have
        \( \mathbb{P} \left( \left|\frac{1}{n} \sum_{i \in [n]} (V_i -\mathbb{E} [V_i]) \right|
              \geq t \right) = 0 \).
    Moreover, with probability at least \(1-\delta\),
    \[
    \left|\frac{1}{n} \sum_{i \in [n]} (V_i -\mathbb{E} [V_i]) \right| 
    \lesssim \min\left(\frac{\log \frac{1}{\delta}/n}{\log\left(1+ \frac{\log\frac{1}{\delta}/n}{1-q}\right)}, \sqrt{\frac{\log\frac{1}{\delta}}{n}} \right) 
    \]
\end{lemma}

\begin{proof}
    Let's denote $q \assign \max (p, 1 - p), V' \assign V_i -\mathbb{E} [V_i], i \in
    [n]$ , thus $\mathbb{E} \left[ {V'}^2 \right] = \tmop{Var} [V_i] = p (1 - p) =
    q (1 - q), |V'|\leq q$.
    With Chenorff bound and let $\psi (\lambda)\assign \mathbb{E} [\exp
    (\lambda (V_i -\mathbb{E} [V_i]))] = - \lambda p + \ln (1 + p (\exp
    (\lambda) - 1)) \leq \frac{\lambda^2}{8}$
    \begin{eqnarray*}
      \log \mathbb{P} [\sum_{i \in [n]} (V_i -\mathbb{E} [V_i]) \geq n t] & \leq
      & \inf_{\lambda > 0}  \left\{ \log \mathbb{E} \left[ \exp \left(
      \lambda \sum_{i \in [n]} (V_i -\mathbb{E} [V_i]) \right) \right] - \lambda
      n t \right\}
      \leq
      \inf_{\lambda > 0}  \left\{ n \left[ \frac{\lambda^2}{8} -
      \lambda t \right] \right\} = - 2 n t^2
    \end{eqnarray*}
    Hence, one side of the first probability inequality is proved, the other side is proved by replacing \(V_i\) with \(1-V_i\).
  Let's focus on next concentration inequality. We begin with bounding $\psi (\lambda)$, note that $2 \sinh (x) > \exp (x) - x - 1, \forall x > 0$.
    \begin{eqnarray*}
      \psi (\lambda) 
      = 1 + \sum_{k \geq 2} \frac{\mathbb{E} [ {V'}^k ]}{k!}
      \lambda^k \leq 1 + \sum_{k \geq 2} \frac{\mathbb{E} [ {V'}^2 ]
      \cdot q^{k - 2}}{k!} \lambda^k
      \leq 1 + \frac{(1 - q)}{q} \{ \exp (q \lambda) - q \lambda - 1 \}
      \leq 1 + 2 \frac{(1 - q)}{q} \sinh(q \lambda)
    \end{eqnarray*}
  Let $\mu \assign 2 \frac{(1 - q)}{q} \sinh (q \lambda)$, then $\lambda =
  \frac{1}{q} \tmop{arcsinh} \left( \frac{q}{2 (1 - q)} \mu \right)$,
  and $\mu' \assign \frac{q}{(1 - q)} \mu$, $\gamma \assign \frac{(1 - q)}{q}
  \in (0, 1], \tau \assign \frac{t}{(1 - q)}$.
  \begin{eqnarray*}
    & & \log \mathbb{P} \left( \frac{1}{n} \sum_{i \in [n]} (V_i -\mathbb{E}
    [V_i]) \geq t \right) 
    \leq 
    n \inf_{\lambda > 0}  [\ln \psi (\lambda) - \lambda t]
     \leq  n \inf_{\lambda > 0}  \left\{ \ln \left[ 1 + 2 \frac{(1 -
    q)}{q} \sinh (q \lambda) \right] - \lambda t \right\} \\
    &=& n \gamma \inf_{\mu' > 0}  \left\{ \frac{1}{\gamma} \ln (1 + \gamma
    \mu') - \tau \tmop{arcsinh} \left( \frac{\mu'}{2} \right) \right\} 
    \leq n \gamma \inf_{\mu' > 0}  \left\{ \frac{1}{\gamma} \ln (1 +
    \gamma \mu') - \tau \ln (\mu') \right\}
  \end{eqnarray*}
  If $t > q$, then $\tau \assign \frac{t}{(1 - q)} \in [
  \frac{1}{\gamma}, \infty )$,
  \( n \gamma \inf_{\mu' > 0}  \left\{ \frac{1}{\gamma} \ln (1 + \gamma \mu')
     - \tau \ln (\mu') \right\} = - \infty \).
  If $t \in [\mathe (1 - q), q]$, $\tau \assign \frac{t}{(1 - q)} \in [ \mathe,
  \frac{1}{\gamma} ]$.
  \begin{eqnarray*}
    &  & n \gamma \inf_{\mu' > 0}  \left\{ \frac{1}{\gamma} \ln (1 + \gamma
    \mu') - \tau \log (\mu') \right\}\\
    & = & n \gamma \left\{ \frac{1}{\gamma} \ln (1 + \gamma \mu') - \tau
    \ln (\mu') \right\}_{\mu' = \frac{\tau}{1 - \gamma \tau}}
     =  n \gamma \left\{ - \tau \ln \tau + \frac{1}{\gamma} [- (1 - \gamma
    \tau) \ln (1 - \gamma \tau)] \right\}\\
    & \leq & n \gamma \left\{ - \tau \ln \tau + \frac{\gamma \tau}{\gamma}
    \left( 1 - \frac{\gamma \tau}{2} \right) \right\}
     \leq  n \gamma \left\{ - \tau [\ln \tau - 1] - \frac{\gamma}{2}
    \tau^2 \right\}
     =  n \left\{ - \frac{t}{q}  \left[ \ln \frac{t}{(1 - q)} - 1 \right]
    - \frac{t^2}{2 q^2} \right\}
  \end{eqnarray*}
  Therefore, one side of the other probability inequality is proved, similarly, the other side is proved by replacing \(V_i\) with \(1-V_i\).
  Also, we show the probability is 0 when $t\geq q$. 
  
  By the first concentration inequality, with probability at least \(1-\delta\),
  \[
   \left|\frac{1}{n} \sum_{i \in [n]} (V_i -\mathbb{E} [V_i]) \right| \lesssim \sqrt{\frac{\log(1/\delta)}{n}}
  \]
  By the second concentration inequality and letting \(t = 2 \frac{\ln(2/\delta)}{n}/\ln(1+ \frac{\ln(2/\delta)}{n\times \mathe(1-q)})\geq 2\mathe(1-q)\), with probability at least \(1-\delta\),
  \[
   \left|\frac{1}{n} \sum_{i \in [n]} (V_i -\mathbb{E} [V_i]) \right| \leq 2 \frac{\ln(2/\delta)}{n}/\ln\left(1+ \frac{\ln(2/\delta)}{n\times \mathe(1-q)}\right)
   \asymp \frac{\log \frac{1}{\delta}/n}{\log\left(1+ \frac{\log\frac{1}{\delta}/n}{1-q}\right)}
  \]
  Since the failure probability is less than \(\delta\), by using \(2 x/\sqrt{1+x} > \ln(1+x)\) for \(x = \frac{\ln(2/\delta)}{n\times \mathe(1-q)}> 0\),
  \[
    2 \exp\left(-n t \log \frac{t}{\mathe(1-q)}\right)_{t=2 \frac{\ln(2/\delta)}{n}/\ln(1+ \frac{\ln(2/\delta)}{n\times \mathe(1-q)})}
    < 2 \exp\left(-n \frac{\ln(2/\delta)}{n}\right)
    = \delta
  \]
  The proof is completed by combining the bounds for \(t \in \mathbb{R}_{\geq 0}\) and \(t \in [\mathe (1 - q), q]\).
\end{proof}
}

\bibliographystyle{IEEEtran}
\bibliography{ref_EM2MLR}

\begin{IEEEbiography}[{\includegraphics[width=1in,height=1.25in,clip,keepaspectratio]{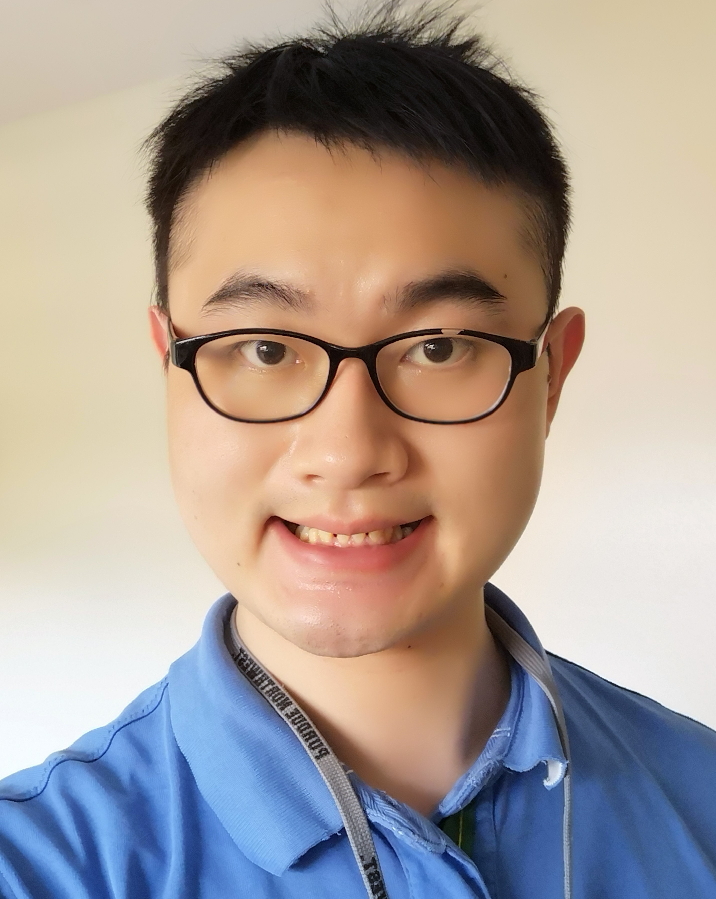}}]{Zhankun Luo}
(Graduate Student Member, IEEE) received the B.S. degree in telecommunication engineering from Beijing Institute of Technology, China, in 2019, and the M.S.E. degree in electrical and computer engineering in 2021 from Purdue University Northwest, Hammond, IN, USA. 
He is currently pursuing the Ph.D. degree with the Elmore Family School of Electrical and Computer Engineering, Purdue University, USA. 
His research interests include machine learning, statistical learning, generative diffusion models, 
and optimization for distributed and decentralized algorithms.
\end{IEEEbiography}

\begin{IEEEbiography}[{\includegraphics[width=1in,height=1.25in,clip,keepaspectratio]{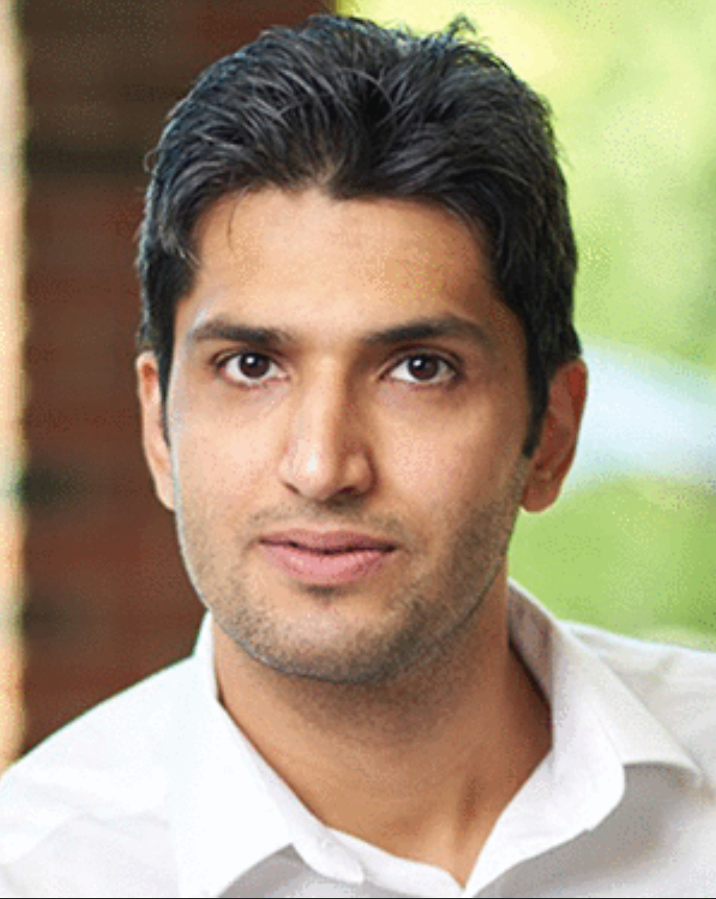}}]{Abolfazl Hashemi}
(Member, IEEE) received the B.Sc. degree in electrical engineering from Sharif University of Technology, Iran, in 2014, 
and the M.S.E. and Ph.D. degrees in electrical and computer engineering from the University of Texas at Austin, USA, in 2016 and 2020, respectively. 
From 2020 to 2021, he was a Postdoctoral Scholar with the Oden Institute for Computational Engineering and Sciences, University of Texas at Austin. 
Since 2021, he has been an Assistant Professor with the Elmore Family School of Electrical and Computer Engineering, Purdue University. 
His research interests include large-scale optimization. He was the 2019 Schmidt Science Fellows Award nominee from UT Austin, the recipient of the Iranian National Elite Foundation Fellowship, and the Best Student Paper Award finalist at the 2018 American Control Conference.
\end{IEEEbiography}

\end{document}